\documentclass[ejs]{imsart}

\RequirePackage[OT1]{fontenc}
\RequirePackage{amsthm,amsmath}
\RequirePackage[numbers]{natbib}
\RequirePackage[colorlinks,citecolor=blue,urlcolor=blue]{hyperref}

\pubyear{2017}
\volume{0}
\issue{0}
\firstpage{1}
\lastpage{99}

\startlocaldefs
\numberwithin{equation}{section}
\theoremstyle{plain}

\endlocaldefs

\usepackage{amsfonts}
\usepackage{amstext}
\usepackage{parskip}

\usepackage{times}
\usepackage{graphicx}
\usepackage{epsfig}
\usepackage{amsthm,amsmath,amssymb}
\usepackage{graphicx}
\usepackage{epsfig}

\newenvironment{subeqnarray}
{\begin{subequations}\begin{eqnarray}}%
{\end{eqnarray}\end{subequations}\hskip-4.0pt}

\def\fatnorm#1{|\kern-.2ex|\kern-.2ex| #1 |\kern-.2ex|\kern-.2ex|}
\newcommand{\twonorm}[1]{\left\lVert#1\right\rVert_2}
\newcommand{\shnorm}[1]{\lVert#1\rVert_2}
\newcommand{\stnorm}[1]{\lVert#1\rVert}

\newcommand{\shtwonorm}[1]{\lVert#1\rVert_2}
\newcommand{\fnorm}[1]{\left\lVert#1\right\rVert_{F}}
\newcommand{\norm}[1]{\left\lVert#1\right\rVert}
\newcommand{\abs}[1]{\left\lvert#1\right\rvert}

\newcommand{\Sp}{\mathbb{S}}
\newcommand{\T}{\mathcal{T}}

\newcommand{\V}{\mathcal{V}}

\newcommand{\N}{{\mathcal N}}

\newcommand{\cov}{\textsf{Cov}}
\newcommand{\SNR}{\ensuremath\textsf{S/N}}
\newcommand{\signal}{\ensuremath\textsf{S}}
\newcommand{\noise}{\ensuremath\textsf{N}}

\newcommand{\spin}{\textsf{span}}

\newcommand{\rn}{\rho_n}

\newcommand{\half}{\ensuremath{\frac{1}{2}}}
\newcommand{\inv}[1]{\frac{1}{#1}}
\newcommand{\ip}[1]{\;\langle{\,#1\,}\rangle\;}
\newcommand{\size}[1]{\ensuremath{\left|#1\right|}}
\newcommand{\onenorm}[1]{\ensuremath{\left\|#1\right\|_1}}

\newcommand{\maxnorm}[1]{\ensuremath{\left\|#1\right\|_{\max}}}
\newcommand{\expct}[1]{\ensuremath{\mathbb E}\left(#1\right)}
\newcommand{\silent}[1]{}
\newcommand{\mvec}[1]{\rm{vec}\left\{\,#1\,\right\}}

\newcommand{\ve}{\varepsilon}

\def\qed{\hskip1pt $\;\;\scriptstyle\Box$}
\def\Ber{\mathop{\text{Bernoulli}\kern.2ex}}
\def\corr{\mathop{\text{corr}\kern.2ex}}
\def\prec{\mathop{\text{precision}\kern.2ex}}
\def\recall{\mathop{\text{recall}\kern.2ex}}
\def\cov{\mathop{\text{Cov}\kern.2ex}}
\def\mnorm{\mathcal{N}_{f,m}\kern.2ex}
\def\var{\mathop{\text{Var}\kern.2ex}}
\def\ess{\mathop{\text{ess}\kern.2ex}}
\def\dom{\mathop{\text{dom}\kern.2ex}}
\def\lin{\mathop{\text{lin}\kern.2ex}}

\newcommand{\func}[1]{\ensuremath{\mathrm{#1}}}
\newcommand{\diag}{\func{diag}}

\newcommand{\RE}{\textnormal{\textsf{RE}}}

\let\hat\widehat
\let\tilde\widetilde

\newcommand{\tr}{{\rm tr}}

\def\E{{\mathbb E}}

\newcommand{\prob}[1]{\ensuremath{\mathbb P}\left(#1\right)}

\newcommand{\beq}{\begin{equation}}
\newcommand{\eeq}{\end{equation}}
\newcommand{\ben}{\begin{eqnarray}}
\newcommand{\een}{\end{eqnarray}}
\newcommand{\bnum}{\begin{enumerate}}
\newcommand{\enum}{\end{enumerate}}
\newcommand{\bit}{\begin{itemize}}
\newcommand{\eit}{\end{itemize}}
\newcommand{\bens}{\begin{eqnarray*}}
\newcommand{\eens}{\end{eqnarray*}}

\newcommand{\R}{\mathbb{R}}

\newcommand{\Sc}{\ensuremath{S^c}}

\newcommand{\ora}{{\rm oracle}}
\newcommand{\SMR}{\ensuremath\textsf{S/M}}
\newcommand{\bignoise}{\ensuremath\textsf{M}}

\newcommand{\nl}{\nu_{\ell}}
\newcommand{\bal}{\bar\alpha_{\ell}}
\newcommand{\al}{\alpha_{\ell}}
\newcommand{\au}{\alpha_{u}}

\newcommand{\static}{\text{\rm stat}}
\newcommand{\e}{\epsilon}
\newcommand{\bare}{\bar{\epsilon}_{\static}}
\newcommand{\vp}{\varpi}

\newcommand{\grad}{\nabla}
\newcommand{\pen}{\ensuremath{\rho_{\lambda}}}
\newcommand{\loss}{\ensuremath\mathcal{L}_n}

\newcommand{\PaulBhalf}{\ensuremath{\hat\tau_B^{1/2}}}

\newcommand{\z}{\ensuremath{\varrho}}

\newcommand{\onem}{\textstyle \frac{1}{m}}
\newcommand{\onen}{\textstyle \frac{1}{n}}

\newcommand{\Ball}{{B}}
\newcommand{\B}{\mathcal{B}}
\newcommand{\A}{\mathcal{A}}
\newcommand{\AR}{{\operatorname{AR(1)}}}
\newcommand{\U}{\Upsilon}
\newcommand{\up}{\upsilon}

\newcommand{\ind}{\mathbb{I}}

\newtheorem{theorem}{Theorem}
\newtheorem{assumption}{Assumption}[section]

\newcommand{\Cone}{{\rm Cone}}
\newcommand{\cone}{{\rm Cone}}

\def\supp{\mathop{\text{\rm supp}\kern.2ex}}
\def\conv{\mathop{\text{\rm conv}\kern.2ex}}
\def\span{\mathop{\text{\rm span}\kern.2ex}}
\def\argmin{\mathop{\text{arg\,min}\kern.2ex}}
\def\absconv{\mathop{\text{\rm absconv}\kern.2ex}}
\def\half{\frac{1}{2}}
\let\hat\widehat
\def\W{\Cone}

\newtheorem{lemma}[theorem]{Lemma}

\newtheorem{definition}[assumption]{Definition}

\newtheorem{remark}[assumption]{Remark}
\newtheorem{corollary}[theorem]{Corollary}

\def\qed{\hskip1pt $\;\;\scriptstyle\Box$}

\newenvironment{proofof}[1]{\hspace*{20pt}{\it Proof}{ of #1}.\hskip10pt}{\qed\vskip5pt}
\newenvironment{proofof2}{\hskip10pt}{\qed\vskip5pt}

\begin{document}

\begin{frontmatter}
\title{Errors-in-variables models with dependent measurements}
\runtitle{Errors-in-variables}

\begin{aug}
\author{\fnms{Mark} \snm{Rudelson}\thanksref{t2}\ead[label=e1]{rudelson@umich.edu}}
\and
\author{\fnms{Shuheng} \snm{Zhou}\thanksref{t3}\ead[label=e2]{shuhengz@umich.edu}}

\address{Department of Mathematics, Department of Statistics\\
University of Michigan, Ann Arbor, MI 48109 \\
\printead{e1,e2}}

\thankstext{t2}{
Mark Rudelson is partially supported by NSF grant DMS 1161372 and
USAF Grant FA9550-14-1-0009. 
}
\thankstext{t3}{Shuheng Zhou is supported in part by NSF under Grant DMS-1316731 and 
Elizabeth Caroline Crosby Research Award from the Advance Program at the University of Michigan.}
\runauthor{Rudelson and Zhou}

\affiliation{University of Michigan}

\end{aug}

\begin{abstract}

Suppose that we observe $y \in \R^n$ and $X \in \R^{n \times m}$ in the
following errors-in-variables model: \begin{eqnarray*} y & =  & X_0 \beta^* +
\e \\ X & = & X_0 + W \end{eqnarray*} where $X_0$ is an $n \times m$ design
matrix with independent subgaussian row vectors, $\e \in \R^n$ is a noise
vector and $W$ is a mean zero $n \times m$ random noise matrix with independent
subgaussian column vectors, independent of $X_0$ and $\e$.  This model is
significantly different from those analyzed in the literature in the sense that
we allow the measurement error for each covariate to be a dependent vector
across its $n$ observations. Such error structures appear in the science
literature when modeling the trial-to-trial fluctuations in response strength
shared across a set of neurons.

Under sparsity and restrictive eigenvalue type of conditions, we show that one
is able to recover a sparse vector $\beta^* \in \R^m$ from the model given a
single observation matrix $X$ and the response vector $y$.  We establish
consistency in estimating $\beta^*$ and obtain the rates of convergence in the
$\ell_q$ norm, where $q = 1, 2$ for the Lasso-type estimator, and for $q \in
[1, 2]$ for a Dantzig-type Conic programming estimator.  We show error bounds
which approach that of the regular Lasso and the Dantzig selector in case the
errors in $W$ are tending to 0.  We analyze the convergence rates of the
gradient descent methods for solving the nonconvex programs and show that the
composite gradient descent algorithm is guaranteed to converge at a geometric
rate to a neighborhood of the global minimizers: the size of the neighborhood
is bounded by the statistical error in the $\ell_2$ norm. Our analysis reveals
interesting connections between computational and statistical efficiency and the
concentration of measure phenomenon in random matrix theory. We provide
simulation evidence illuminating the theoretical predictions.

\end{abstract}


\tableofcontents
\end{frontmatter}

\section{Introduction}
The matrix variate normal model has a long history in psychology and social sciences.
In recent years,  it is becoming increasingly popular in biology and genomics,
neuroscience, econometric theory, image and signal processing, wireless
communication, and machine learning; 
see for example~\cite{Dawid81,GV92,Dut99,WJS08,BCW08,Yu09,Efr09,AT10,KLLZ13}
and references therein. 
We call the random matrix $X$, which contains $n$ rows and $m$
columns a single data matrix, or one instance from the matrix variate
normal distribution. We say that an $n \times m$ random matrix $X$ follows a matrix normal distribution with a separable covariance matrix
$\Sigma_X = A \otimes B$ and mean $M \in \R^{n \times m}$, which we write 
$X_{n \times m} \sim \N_{n,m}(M, A_{m \times m} \otimes B_{n \times n}).$
This is equivalent to say $\mvec{X}$ follows a multivariate normal distribution with mean
$\mvec{M}$ and covariance $\Sigma_X = A \otimes B$. 
Here, $\mvec{X}$ is formed by stacking the columns of $X$ into a vector in $\R^{mn}$.
Intuitively, $A$ describes the covariance between columns of $X$, while
$B$  describes the covariance between rows of $X$. See~\cite{Dawid81,GV92} for more characterization and examples.

In this paper, we introduce the related sum of Kronecker product models to encode
the covariance structure of a matrix variate distribution. The
proposed models and methods incorporate ideas from recent advances in
graphical models, high-dimensional regression model with observation
errors, and matrix decomposition. 
Let $A_{m \times m}, B_{n \times n}$ be symmetric positive definite 
covariance matrices. Denote the Kronecker sum of $A =(a_{ij})$ 
and $B = (b_{ij})$ by
\bens
\Sigma & = & A \oplus B := A \otimes I_n + I_m \otimes B \\
& = &
\left[
\begin{array}{cccc} 
a_{11}I_n + B & a_{12} I_n & \ldots & a_{1m} I_n \\
a_{21} I_n & a_{22} I_n +  B & \ldots & a_{2m} I_n \\
\ldots &  & & \\
a_{m1} I_n & a_{m2} I_n & \ldots & a_{mm} I_n + B
\end{array} \right]_{(m n) \times (m n)}
\eens
where $I_n$ is an $n \times n$ identity matrix.
This covariance model arises naturally from the context
of errors-in-variables regression model defined as follows.

Suppose that we observe $y \in \R^n$ and $X \in \R^{n \times m}$ in the following model:
\begin{subeqnarray}
\label{eq::oby}
y & =  & X_0 \beta^* + \e \\
\label{eq::obX}
X & = & X_0 + W
\end{subeqnarray}
where $X_0$ is an $n \times m$ design matrix with independent row vectors, $\e \in \R^n$ is a noise vector and 
$W$ is a mean zero $n \times m$ random noise matrix, independent of $X_0$ and $\e$, with independent column
vectors $\omega^1, \ldots, \omega^m$.

In particular, we are interested in the additive model of $X = X_0 + W$ such that
\ben
\label{eq::dataplus}
\mvec{X} \sim \N(0, \Sigma) \; \; \text{ where } \; \; 
\Sigma = A \oplus B := A \otimes I_n + I_m \otimes B
\een
where  we use one covariance component $A \otimes
I_n$ to describe the covariance of matrix $X_0 \in \R^{n \times m}$,
which is considered as the {\it signal} matrix, and the other
component $I_m \otimes B$ to describe that of the {\it noise matrix }
$W \in \R^{n \times m}$, where $\E \omega^j \otimes \omega^j = B$ for all $j$,
where $\omega^j$ denotes the $j^{th}$ column vector of $W$.
Our focus is on deriving the statistical properties of two estimators for estimating $\beta^*$ in
\eqref{eq::oby} and~\eqref{eq::obX} despite the presence of the
additive error $W$ in the observation matrix $X$.
We will show that  our theory and analysis works with a model much more general than
that in \eqref{eq::dataplus}, which we will define in
Section~\ref{sec::method}.

Before we go on to define our estimators, we now use an example to motiviate~\eqref{eq::dataplus} and its
subgaussian generalization in~\eqref{eq::subgdata}.
Suppose that there are $n$ patients in a particular study, for which
 we use $X_0$ to model the ``systolic blood pressure'' and $W$ to model
the seasonal effects. In this case, $X$ models the fact that 
among the $n$ patients we measure, each patient has its own row vector
of observed set of blood pressures across time, and each column vector
in $W$ models the seasonal variation on top of the true signal at a
particular day/time. Thus we consider $X$ as measurement of $X_0$ with
$W$ being the observation error. 
That is, we model the seasonal effects on blood pressures across a
set of patients in a particular study with a vector of dependent
entries. Thus $W$ is a matrix which consists of repeated independent 
sampling of spatially dependent vectors, if we regard the individuals 
as having spatial coordinates, for example, through their geographic locations.
We will come back to discuss this example in Section~\ref{sec::discuss}.

\subsection{The model and the method} 
\label{sec::method}
We first need to define an independent isotropic vector with {\em
  subgaussian} 
marginals as in Definition~\ref{def:psi2-vector}.
For a vector $y = (y_1, \ldots, y_p)$ in $\R^p$, denote by 
$\twonorm{y} = \sqrt{\sum_{j} y_j^2}$ the length of $y$.
\begin{definition}
\label{def:psi2-vector} 
Let $Y$ be a random vector in $\R^p$
\bnum
\item
$Y$ is called isotropic
if for every $y \in \R^p$, $\expct{\abs{\ip{Y, y}}^2} = \twonorm{y}^2$.
\item
$Y$ is $\psi_2$ with a constant $\alpha$ if for every $y \in \R^p$,
\beq
\norm{\ip{Y, y}}_{\psi_2} := \;
\inf \{t: \expct{\exp(\ip{Y,y}^2/t^2)} \leq 2 \}
\; \leq \; \alpha \twonorm{y}.
\eeq
\enum
The  $\psi_2$ condition on a scalar random variable $V$ is equivalent to
the subgaussian tail decay of $V$, which means
$\prob{|V| >t} \leq 2 \exp(-t^2/c^2), \; \; \text{for all} \; \; t>0.$
\end{definition}
Throughout this paper, we use $\psi_2$ vector, a vector with subgaussian marginals
and subgaussian vector interchangeably.

\noindent{\bf The model.}
Let $Z$ be an $n \times m$ random matrix with independent entries $Z_{ij}$ satisfying
$\E Z_{ij} = 0$, $1 = \E Z_{ij}^2 \le \norm{Z_{ij}}_{\psi_2} \leq K$.  
Let $Z_1, Z_2$ be independent copies of $Z$.
Let 
\ben
\label{eq::subgdata}
X = X_0 + W
\een
such that $X_0 = Z_1 A^{1/2}$ is the design matrix with independent subgaussian
row vectors, and  $W =B^{1/2} Z_2$ is a random noise matrix with independent subgaussian
column vectors. 

Assumption (A1) allows the covariance model in~\eqref{eq::dataplus}
and its subgaussian variant in~\eqref{eq::subgdata} to be identifiable. 
\bnum
\item[(A1)]
We assume $\tr(A) = m$ is a known parameter, where $\tr(A)$ denotes
the trace of matrix $A$.
\enum
In the Kronecker sum model, we could assume we know 
$\tr(B)$, in order not to assume knowing $\tr(A)$. 
Assuming one or the other is known is unavoidable as the covariance model is not identifiable otherwise. 
Moreover, by knowing $\tr(A)$, we can construct an estimator for
$\tr(B)$:
\ben
\label{eq::trBest}
\hat\tr(B) &= &
\onem \big(\fnorm{X}^2 - n \tr(A)\big)_{+} \; \; \;
\text{ and define } \; \; 
\hat\tau_B  := \onen \hat\tr(B) \ge 0 
\een
where $(a)_{+} = a \vee 0$ and 
$\fnorm{X}^2 := \sum_i \sum_j X_{ij}^2$. 
We first introduce the corrected Lasso estimator, adapted from those as
considered in~\cite{LW12}.

Suppose that $\hat\tr(B)$ is an estimator for $\tr(B)$; for example, 
as constructed in~\eqref{eq::trBest}. Let 
\ben
\label{eq::hatGamma}
\hat\Gamma & = & 
\inv{n} X^T X - \inv{n} \hat\tr(B)  I_{m}\;\;
\; \text{ and } \;
\hat\gamma \; = \; \onen X^T y.
\een
For a chosen penalization parameter $\lambda \geq 0$, and parameters
$b_0$ and $d$, we consider the following  regularized estimation with the $\ell_1$-norm penalty,
\begin{eqnarray}
\label{eq::origin} \; \; 
\hat \beta & = & \argmin_{\beta: \norm{\beta}_1 \le b_0 \sqrt{d}} \frac{1}{2} \beta^T \hat\Gamma \beta 
- \ip{\hat\gamma, \beta} + \lambda \|\beta\|_1, \; \;
\end{eqnarray}
which is a variation of the Lasso \citep{Tib96} or the Basis
Pursuit~\citep{Chen:Dono:Saun:1998} estimator.
Although in our analysis, we set $b_0 \ge \twonorm{\beta^*}$ and
$d = \size{\supp(\beta^*)} :=  \size{\{j: \beta_j^* \not=0\}}$ for
simplicity, in practice, both $b_0$
and $d$ are understood to be parameters chosen to provide an upper bound on the $\ell_2$ norm and the sparsity of  the true $\beta^*$.

For a vector $\beta \in \R^m$, denote by $\norm{\beta}_{\infty} :=
\max_j \abs{\beta_j}$.
Recently,~\cite{BRT14} discussed the following conic programming
compensated matrix uncertainly (MU) selector
, which is a variant of the Dantzig selector~\citep{CT07,RT10,RT13}. Adapted to our setting, it is
defined as follows. Let $\lambda, \mu, \tau >0$, 
\begin{eqnarray}
\label{eq::Conic} \; \; 
\hat \beta & = & \arg\min\big\{\norm{\beta}_1 +\lambda t\; :\; (\beta, t) \in \U\big\} 
\text{ where } \; \\
\nonumber
\U &= & \left\{(\beta, t) \; : \; \beta \in \R^m, \norm{\hat\gamma - \hat\Gamma \beta}_{\infty}
\le \mu t + \omega, \twonorm{\beta} \le t\right\}
\end{eqnarray}
where $\hat\gamma$ and $\hat\Gamma$ are as defined in \eqref{eq::hatGamma} 
with $\mu \sim \sqrt{\frac{\log m}{n}}$, $\omega \sim \sqrt{\frac{\log m}{n}}$. We refer to this estimator as the 
Conic programming estimator from now on.

\subsection{Gradient descent algorithms}

In order to obtain fast, approximate solutions
to the optimization goal as in~\eqref{eq::origin2}, we adopt the computational framework
of~\cite{ANW12,LW12}, namely, the composite gradient descent method
due to Nesterov~\cite{Nesterov07} to analyze our computational and statistical errors in
an integrated manner. First we denote the population and empirical loss functions by
\ben
\label{eq::poploss}
\mathcal{L}(\beta) = \half \beta^T \Sigma_x \beta - \beta^{*T} \Sigma_x \beta 
\; \quad \text{ and } \quad 
\loss(\beta) = \half \beta^T \hat\Gamma \beta - \hat\gamma^{T} \beta
\een
respectively.
We consider regularizers that are separable across all coordinates and
write 
\bens
\pen(\beta) = \sum_{i=1}^m \pen(\beta_i).
\eens
Throughout this paper, we denote by 
$$\phi(\beta) =\half \beta^T \hat\Gamma \beta - \hat\gamma^{T}
\beta + \rho_{\lambda}(\beta).$$
From the formulation~\eqref{eq::origin}, the corrected linear
regression estimator is given by minimizing the penalized loss
function $\phi(\beta)$ subject to the constraint that $g(\beta) \le R$:
\ben
\label{eq::origin2}
\hat\beta \in \argmin_{\beta \in \R^m, g(\beta) \le R} 
\left\{\half \beta^T \hat\Gamma \beta - \hat\gamma^{T} \beta + \rho_{\lambda}(\beta)\right\}
\een
where $g(\beta)$ is a convex function, which is allowed to be
identical to $\onenorm{\beta}$ and $R$ is a second tuning parameter that is
chosen to confine the solution $\hat\beta$ within the $\ell_1$ ball of
radius $R$, while at the same time ensuring that $\beta^*$ is a
feasible solution. 
The gradient descent method generates a sequence
$\{\beta^t\}_{t=0}^{\infty}$ of iterates by first initializing to some
parameter $\beta^0 \in \R^m$, and then for $t=0, 1, 2, \ldots$,
applying the recursive updates:
\ben
\label{eq::optlinear}
\beta^{t+1} = \argmin_{\beta \in \R^m, g(\beta) \le R} 
\left\{\loss(\beta^t) + \ip{\grad \loss(\beta^t),
  \beta-\beta^t} + \frac{\zeta}{2} \twonorm{\beta -\beta^t}^2 + \rho_{\lambda}(\beta) 
\right\}
\een
where $\zeta$ is the step size parameter.

More generally, we consider loss function $\loss: \R^m \to \R$ and $\pen$ which are possibly
nonconvex and consider the regularized M-estimator of the form
\ben
\label{eq::global}
\hat\beta \in \argmin_{\beta \in \R^m, g(\beta) \le R} 
\left\{\loss(\beta; X) + \rho_{\lambda}(\beta)\right\}
\een
where $\pen: \R^m \to \R$ is a regularizer depending on a tuning
parameter $\lambda > 0$.
Because of this potential nonconvexity, we also include a side constraint in the form of $g(\beta) \le R$, where
\ben
\label{eq::gconv}
g(\beta) :=\inv{\lambda}\left\{\rho_{\lambda}(\beta)   +\frac{ \mu}{2} \twonorm{\beta}^2 \right\}
\een
so that this choice of $g$ is convex for properly chosen parameter
$\mu \ge 0$ for a class of weakly convex penalty functions $\rho$~\cite{Vial82};
See Assumption~1 in~\cite{LW15} where properties of $g$ and
$\rho_{\lambda}$ are stated in terms of the univariate function $\rho_{\lambda}: \R
\to \R$ and the parameter $\mu \ge 0$. 
While our results hold for the general nonconvex penalty
$\rho_{\lambda}$ that is weakly convex in the sense that \eqref{eq::gconv} holds for
some parameter $\mu > 0$, we focus our discussion to the choice of $\rho_{\lambda}(\beta) =
\lambda \onenorm{\beta}$ and $\mu =0$ in the present paper.

\subsection{Our contributions}
We provide a unified analysis of the rates of convergence for both the
corrected Lasso estimator~\eqref{eq::origin} and the Conic programming estimator \eqref{eq::Conic},
which is a Dantzig selector-type, although under slightly different conditions. 
We will show the rates of convergence in the $\ell_q$ norm for $q =1, 2$
for estimating a sparse vector $\beta^* \in \R^m$ in the
model~\eqref{eq::oby} and~\eqref{eq::obX} using the corrected Lasso
estimator~\eqref{eq::origin} 
in Theorems~\ref{thm::lasso} and~\ref{thm::lassora},  
and the Conic programming estimator~\eqref{eq::Conic} in
Theorems~\ref{thm::DS} and~\ref{thm::DSoracle} 
for $1\le q\le 2$.
We also show bounds on the predictive errors for the Conic programming estimator.
The bounds we derive in Theorems~\ref{thm::lasso} and~\ref{thm::DS}  focus on cases
where the errors in $W$ are not too small in their magnitudes  in the
sense that $\tau_B := \tr(B)/n$ 
is bounded from below. For the extreme case when $\tau_B$
approaches $0$, one hopes to recover bounds close to those for the
regular Lasso or the Dantzig selector since the effect of the noise in matrix $W$ on the procedure 
becomes negligible. We show in Theorems~\ref{thm::lassora} 
and~\ref{thm::DSoracle} that this is indeed the case. These results are new to the best of our knowledge.

Let $Z_1, Z_2$ be independent
subgaussian random matrices with independent entries
(cf.~\eqref{eq::subgdata}). 
In Theorems~\ref{thm::lasso} to~\ref{thm::DSoracle}, 
we consider the regression model in \eqref{eq::oby} and \eqref{eq::obX} with subgaussian random design, 
where $X_0 = Z_1 A^{1/2}$ is a subgaussian random matrix with independent row
vectors, and $W =B^{1/2} Z_2$ is an $n \times m$ random noise matrix with
independent column vectors, This model is significantly different from
those analyzed in the literature. For example, unlike the present work, the authors in~\cite{LW12} apply 
Theorem~\ref{thm::main} which states a general result on statistical
convergence properties of the estimator~\eqref{eq::origin} to cases
where $W$ is composed of independent subgaussian row vectors, 
when the row vectors of $X_0$ are either independent or follow a
Gaussian vector auto-regressive model.  See also~\cite{RT10,RT13,BRT14} for the corresponding results on
the compensated MU selectors, variations on the Conic programming estimator~\eqref{eq::Conic}. 

The second key difference between our framework and
the existing work is that we assume that only one observation matrix $X$ with the single measurement
error matrix $W$ is available. Assuming (A1) allows us to estimate $\E W^T W$ as required
in the estimation procedure \eqref{eq::hatGamma} directly, given the
knowledge that $W$ is composed of independent column vectors.  
In contrast, existing work needs to assume that the covariance matrix $\Sigma_W := \onen \E W^T W$ of the independent row vectors
of $W$ or its functionals are either known a priori, or can be 
estimated from a dataset independent of $X$, or from replicated $X$
measuring the same $X_0$; see for example~\cite{RT10,RT13,BRT14,LW12,
  carr:rupp:2006}. Although the model we consider is different from those in the
literature, the identifiability issue, which arises from the fact that
we observe the data under an additive error model, is common.
Such repeated measurements are not always available or costly to 
obtain in practice~\cite{carr:rupp:2006}. We will explore such
tradeoffs in future work.

A noticeable exception is the work of~\cite{CC13}, which deals with
the scenario when the noise covariance is not assumed to be known. 
We now elaborate on their result, which is a variant of the orthogonal
matching pursuit (OMP) algorithm~\citep{Tropp:04,TG07}.
Their support recovery result, that is, recovering the support set of
$\beta^*$, applies only to the case when both signal matrix and the
measurement error matrix have isotropic subgaussian row vectors. In
other words, they assume independence among both rows and columns in $X$ 
($X_0$ and $W$). Moreover, their algorithm requires the knowledge
of the sparsity parameter $d$, which is the number of non-zero entries
in $\beta^*$, as well as a $\beta_{\min}$ condition: $\min_{j \in
  \supp{\beta^*}}  \abs{\beta^*_j} = \Omega\left(\sqrt{\frac{\log
      m}{n}}(\twonorm{\beta^*}+1)\right)$. 
Under these conditions, they recover essentially the same $\ell_2$-error bounds
as in the current work, and~\cite{LW12}, where the covariance $\Sigma_W$
is assumed to be known.

Finally, we present in Theorems~\ref{thm::opt-linear} and~\ref{thm::corrlinear}
the optimization error for the gradient descent algorithms in
solving~\eqref{eq::global} and more specifically~\eqref{eq::origin}.
Let $\hat\beta$ be a global optimizer of~\eqref{eq::global}.
Let $\lambda_{\max}(A)$ and $\lambda_{\min}(A)$ be the
largest and smallest eigenvalues, and  $\kappa(A)$ be the condition
number for matrix $A$. Let $0 < \kappa< 1$ be a contraction factor to
be defined in \eqref{eq::definekappa}.
Similar to the work of~\cite{ANW12,LW12}, we show that the geometric convergence is
not guaranteed to an arbitrary precision, but only to an accuracy
related to statistical precision of the problem,  measured by the
$\ell_2$ error: $\shtwonorm{\hat\beta - \beta^*}^2 =:\ve_{\static}^2$ between the global
optimizer $\hat\beta$ and the true parameter $\beta^*$.

More precisely, our analysis guarantees geometric convergence of the sequence 
 $\{\beta^t\}_{t=0}^{\infty}$ to a parameter $\beta^*$
up to a neighborhood of radius defined through the statistical error bound $\ve_{\static}^2$
\bens
\delta^2 \asymp \frac{\ve_{\static}^2}{1-\kappa} \frac{d \log  m}{n},
\eens
where $\kappa$ is a contraction coefficient to be
defined~\eqref{eq::definekappa},  so that for all $t \ge T^*(\delta)$
as in~\eqref{eq::Tstar}, $\al \asymp \lambda_{\min}(A)$ and $\au \asymp \lambda_{\max}(A)$,
\bens
\twonorm{\beta^t - \hat\beta}^2 \le 
\frac{4 \delta^2 }{\al} + \frac{\al \ve_{\static}^2}{4}
+\frac{4 \delta^4}{b_0^2 \al \lambda_{\max}(A)} = O(\ve_{\static}^2)
\eens
for $\lambda, \zeta \ge \au$ appropriately chosen, $R =\tilde{O}(\sqrt{\frac{n}{\log
    m}})$ and $n =\tilde\Omega\left(d \log m\right)$,  where the
$\tilde{O}(\cdot)$ and $\tilde\Omega(\cdot)$ symbols hide spectral parameters regarding $A$ and $B$.
To quantify such results, we first need to introduce some conditions in Section~\ref{sec::conditions}.
See Theorem~\ref{thm::opt-linear}  and Corollary~\ref{coro::deer} for the precise conditions and statements.

\subsection{Discussion}
\label{sec::discuss}
The theory on matrix variate normal data show that having replicates will allow one to estimate more
complicated graphical structures and achieve faster rates of
convergence under less restrictive assumptions~\cite{Zhou14a}.
Our consistency results in the present work deal with only a single
random matrix following the model \eqref{eq::subgdata}, assuming that
$\tr(A)$ is known. With replicates, this assumption can be lifted off immediately.
Assume there exists a replicate
\ben
\label{eq::error}
 \tilde{X} = X_0 + \tilde{W},
\een
then we can use  $\tilde{X} -X =  \tilde{W} - W$
to estimate $B$ using existing methods.
The rationale for considering such an option is one may have a
repeated measurement of $X_0$ for which the errors $W$ and $\tilde{W}$
follow the same error distribution.
Such external data or knowledge of the noise distribution is needed
in order to do inference under such additive measurement error
model~\cite{carr:rupp:2006}.

The second key modeling question is: would each row vector in $W$ for a
particular patient across all time points be a correlated normal or
subgaussian vector as well? It is our conjecture that combining the
newly developed techniques, namely, the concentration of measure
inequalities we have derived in the current framework with techniques
from existing work~\cite{Zhou14a}, we can handle the case when $W$ follows a matrix
normal distribution with a separable covariance matrix  $\Sigma_W = C
\otimes B$, where $C$ is an $m \times m$ positive semi-definite
covariance matrix.  Moreover, for this type of ''seasonal effects'' as 
the measurement errors, the time varying covariance model 
would make more sense to model $W$, which we elaborate in the second example.

In neuroscience applications, population encoding 
refers to the information contained in the combined activity of
multiple neurons~\cite{KassVB05}.
The relationship between population encoding and correlations is complicated and is an area of 
active investigation, see for example~\cite{RC14a,CK11}.
It becomes more often that repeated measurements (trials) simultaneously recorded
across a set of neurons and over an ensemble of stimuli are
available. In this context, one can use a random matrix
$X_0 \sim \N_{n,m}(\mu, A \otimes B)$  which follows a matrix-variate normal
distribution, or its subgaussian correspondent, 
to model the ensemble of mean response variables, e.g.,
the membrane potential, corresponding to the cross-trial average 
over a set of experiments. Here we use $A$ to model the task correlations and $B$ 
to model the baseline correlation structure among all pairs of neurons at the {\it signal } level.
It has been observed that the onset of stimulus and task events not only
change the cross-trial mean response in $\mu$, but also alter the structure
and correlation of the {\it noise } for a set of neurons,
which correspond to the trial-to-trial fluctuations of the neuron responses.
We use $W$ to model such task-specific trial-to-trial fluctuations of
a set of neurons  recorded over the time-course of a variety of tasks.
Models as in~\eqref{eq::oby} and~\eqref{eq::obX} are useful in predicting the response of set of 
neurons based on the current and past mean responses of all neurons.
Moreover, we could incorporate non-i.i.d. non-Gaussian $W = [w_1,
\ldots, w_m]$ with $w_t = B^{1/2}(t) z(t)$, where $z(1), \ldots, z(m)$ 
are independent isotropic subgaussian random vectors and $B(t) \succ 0$ for all $t$,
to model the time-varying correlated noise as observed in the
trial-to-trial fluctuations. 
It is possible to combine the techniques developed in the present
paper with those in~\cite{ZLW08,Zhou14a} to develop estimators for
$A$, $B$ and the time varying $B(t)$, 
which is itself an interesting topic, however, beyond the scope of the current work.

In summary, oblivion in $\Sigma_W$ and a general dependency
condition in the data matrix $X$ are not simultaneously allowed in
existing work. In contrast, while we assume that $X_0$ is composed of
independent subgaussian row vectors, we allow rows of $W$ to be
dependent, which brings dependency to the row vectors of the
observation matrix $X$.
In the current paper, we focus on the proof-of-the-concept on using
the Kronecker sum covariance and additive model to model two way
dependency in data matrix $X$, and derive bounds in statistical and
computational convergence for~\eqref{eq::origin} and~\eqref{eq::Conic}.
In some sense, we are considering a parsimonious model for fitting
observation data with two-way dependencies:  we use the signal
matrix to encode column-wise dependency among covariates in $X$, 
and error matrix $W$ to explain its row-wise dependency.
When replicates of $X$ or $W$ are available, we are able to study more
sophisticated models and inference problems, some of which are
described earlier in this section.

We leave the investigation of this more general modeling framework and
relevant statistical questions to future work. 
We refer to~\cite{carr:rupp:2006}
for an excellent survey of the classical as well as modern
developments in measurement error models. 
In future work, we will also extend the estimation methods to the settings where the covariates are
measured with multiplicative errors which are shown to be reducible to
the additive error problem as studied in the present work~\cite{RT13,LW12}. Moreover, we are interested in applying the
analysis and concentration of measure results developed in the current
paper and in our ongoing work to the more general contexts and settings where measurement
error models are introduced and investigated; see for
example~\cite{DLR77,CGG85,Stef:1985,HWang86,Full:1987,Stef:1990,CW91,CGL93,Cook:Stef:1994,Stef:Cook:1995,ICF99,LHC99,Str03,XY07,HM07,LL09,ML10,AT10,SSB14,SFT14,SFT14b} and references therein.

\noindent{\bf Notation.}
Let $e_1,\ldots, e_p$ be the canonical basis of $\R^p$.
For a set $J \subset \{1, \ldots, p\}$, denote
$E_J = \spin\{e_j: j \in J\}$.
For a matrix $A$, we use $\twonorm{A}$ to denote its operator norm.
For a set $V \subset \R^p$,
we let $\conv V$ denote the convex hull of $V$. For a finite set
$Y$, the cardinality is denoted by $|Y|$. Let $\Ball_1^p$, $\Ball_2^p$
and $S^{p-1}$ be the unit $\ell_1$ ball, the unit Euclidean ball and the
unit sphere respectively.
For a matrix $A = (a_{ij})_{1\le i,j\le m}$, let $\norm{A}_{\max} =
\max_{i,j} |a_{ij}|$ denote  the entry-wise max norm. 
Let $\norm{A}_{1} = \max_{j}\sum_{i=1}^m\abs{a_{ij}}$ 
denote the matrix $\ell_1$ norm.
The Frobenius norm is given by $\norm{A}^2_F = \sum_i\sum_j a_{ij}^2$. 
Let $|A|$ denote the determinant and ${\rm tr}(A)$ be the trace of $A$.
The operator or $\ell_2$ norm $\twonorm{A}^2$ is given by
$\lambda_{\max}(AA^T)$.  
For a matrix $A$, denote by $r(A)$ the effective rank
$\tr(A)/\twonorm{A}$. Let ${\fnorm{A}^2 }/{\twonorm{A}^2}$ denote the
stable rank for matrix $A$.
We write $\diag(A)$ for a diagonal matrix with the same diagonal as
$A$.  For a symmetric matrix $A$, let $\Upsilon(A) =
\left(\upsilon_{ij}\right)$ where $\upsilon_{ij}   = \ind(a_{ij}
\not=0)$, 
where $\mathbb{I}(\cdot)$ is the indicator function.
Let $I$ be the identity matrix.  For two numbers $a, b$, $a \wedge b
:= \min(a, b)$ and $a \vee b := \max(a, b)$. 
For a function $g: \R^m \to \R$, we write $\grad g$ to denote a
gradient or subgradient, if it exists. 
We write $a \asymp b$ if $ca \le b \le Ca$ for some positive absolute
constants $c,C$ which are independent of $n, m$ or sparsity
parameters. Let $(a)_+ := a \vee 0$.
We write $a =O(b)$ if $a \le Cb$ for some positive absolute
constants $C$ which are independent of $n, m$ or sparsity
parameters.
The absolute constants $C, C_1, c, c_1, \ldots$ may change line by line.

\section{Assumptions and preliminary results}
\label{sec::conditions}
We will now define some parameters related to the restricted and
sparse eigenvalue conditions that are needed to state our main
results. We also state a preliminary result in
Lemma~\ref{lemma::REcomp} regarding the relationships between the
two conditions in Definitions~\ref{def:memory} and~\ref{def::lowRE}.
\begin{definition}
\label{def:memory}
\textnormal{\bf (Restricted eigenvalue condition $\RE(s_0, k_0, A)$).}
Let $1 \leq s_0 \leq p$, and let $k_0$ be a positive number.
We say that a $q \times p$ matrix $A$ satisfies $\RE(s_0, k_0, A)$
 condition with parameter $K(s_0, k_0, A)$ if for any $\upsilon
 \not=0$,
\beq
\inv{K(s_0, k_0, A)} := 
\min_{\stackrel{J \subseteq \{1, \ldots, p\},}{|J| \leq s_0}}
\min_{\norm{\upsilon_{J^c}}_1 \leq k_0 \norm{\upsilon_{J}}_1}
\; \;  \frac{\norm{A \upsilon}_2}{\norm{\upsilon_{J}}_2} > 0.
\eeq
where $\upsilon_{J}$ represents the subvector of $\upsilon \in \R^p$
confined to a subset $J$ of $\{1, \ldots, p\}$.
\end{definition}
It is clear that when $s_0$ and $k_0$ become smaller,
this condition is easier to satisfy.
We also consider the following variation of the baseline $\RE$ condition.
\begin{definition}{\textnormal{(Lower-$\RE$ condition)~\cite{LW12}}}
\label{def::lowRE}
The matrix $\Gamma$ satisfies a Lower-$\RE$ condition with curvature
$\alpha >0$ and tolerance $\tau > 0$ if 
\bens
\theta^T \Gamma \theta \ge 
\alpha \twonorm{\theta}^2 - \tau \onenorm{\theta}^2 \; \;  \forall \theta \in \R^m.
\eens
\end{definition}
where $\onenorm{\theta} := \sum_{j} \abs{\theta_j}$.
As $\alpha$ becomes smaller, or as $\tau$ becomes larger, the
Lower-$\RE$ condition is easier to be satisfied.
\begin{lemma}
\label{lemma::REcomp}
Suppose that the Lower-$\RE$ condition holds for $\Gamma := A^T A$
with $\alpha, \tau > 0$ such that $\tau (1 + k_0)^2 s_0 \le \alpha/2$.
Then the $\RE(s_0, k_0, A)$ condition holds for $A$ with 
\bens
\inv{K(s_0, k_0, A)} \ge \sqrt{\frac{\alpha}{2}} >0.
\eens
Assume that $\RE((k_0+1)^2, k_0, A)$ holds.
Then the Lower-$\RE$ condition holds for  $\Gamma = A^T A$ with 
\bens
\alpha =\inv{(k_0 + 1)K^2(s_0, k_0, A)} > 0
\eens
where $s_0 = (k_0+1)^2$, and $\tau > 0$ which satisfies
\ben
\label{eq::lemmatauchoice}
\lambda_{\min}(\Gamma) \ge \alpha - \tau s_0/4.
\een
The condition above holds for any 
$\tau \ge \frac{4}{(k_0 + 1)^3K^2(s_0, k_0, A)} - \frac{4 \lambda_{\min}(\Gamma)}{(k_0+1)^2}$.
\end{lemma}
The first part of Lemma \ref{lemma::REcomp} means that, if 
$k_0$ is fixed, then smaller values of $\tau$ guarantee $\RE(s_0, k_0,
A)$ holds with larger $s_0$, that is, a stronger $\RE$ condition. 
The second part of the Lemma implies that a weak $\RE$ condition
implies that the Lower-$\RE$ (LRE) holds with a large $\tau$. 
On the other hand, if one assumes $\RE((k_0+1)^2,k_0,A)$ holds with a
large value of  $k_0$ (in other words, a strong $\RE$ condition), this
would imply LRE with a small $\tau$. In short, the two
conditions are similar but require tweaking the parameters. 
Weaker $\RE$ condition implies LRE condition holds with a
larger $\tau$, and Lower-$\RE$ condition with a smaller $\tau$, that
is, stronger LRE implies stronger $\RE$.
We prove Lemma \ref{lemma::REcomp} in Section~\ref{sec::proofoflemmaREcomp}.
\begin{definition}{\textnormal{(Upper-$\RE$ condition)~\cite{LW12}}}
\label{def::upRE}
The matrix $\Gamma$ satisfies an upper-$\RE$ condition with smoothness $\tilde\alpha >0$ and tolerance $\tau > 0$ if 
\bens
\theta^T \Gamma \theta \le \tilde\alpha \twonorm{\theta}^2 + \tau
\onenorm{\theta}^2 \; \;  \forall \theta \in \R^m.
\eens
\end{definition}

\begin{definition}
\label{def::sparse-eigen}
Define the largest and smallest
$d$-sparse eigenvalue of a $p \times q$ matrix $A$ to be
\ben
\label{eq::eigen-Sigma}
\rho_{\max}(d, A) & := &
\max_{t \not= 0; d-\text{sparse}} \; \;\shtwonorm{A
  t}^2/\twonorm{t}^2, \text{ where } \; \; d< p, \\
\label{eq::eigen-Sigma-min}
 \text{ and } \; \;
\rho_{\min}(d, A) & := &
\min_{t \not= 0; d-\text{sparse}} \; \;\shtwonorm{A t}^2/\twonorm{t}^2.
\een
\end{definition}

Before stating some general result for 
the optimization program~\eqref{eq::global} and its implications for
the Lasso-type estimator~\eqref{eq::origin} in terms of statistical
and optimization errors, we need to introduce some more notation and
the following assumptions.
Let $a_{\max} = \max_{i} a_{ii}$ and $b_{\max} =\max_{i} b_{ii}$ be the maximum diagonal entries of $A$ and $B$ respectively.
In general, under (A1), one can think of $\lambda_{\min}(A) \le 1$ and 
for $s \ge 1$,
\ben
\label{eq::eigencond}
1 \le a_{\max} \le \rho_{\max}(s, A) \le \lambda_{\max}(A),
\een
where $\lambda_{\mathrm{max}}(A)$ denotes the maximum eigenvalue of
$A$. 
\bnum
\item[(A2)]
The minimal eigenvalue $\lambda_{\min}(A)$
of the covariance matrix $A$ is bounded: $1 \ge \lambda_{\min}(A) > 0$.
\item[(A3)]
Moreover, we assume that the condition number $\kappa(A)$ is upper bounded
by $O\left(\sqrt{\frac{n}{\log m}}\right)$ and $\tau_B = O(\lambda_{\max}(A))$.
\enum
Throughout the rest of the paper, $s_0 \ge 32$ is understood to 
be the largest integer chosen such that the following inequality still holds:
\ben
 \label{eq::s0cond}
\sqrt{s_0} \vp(s_0) \le \frac{\lambda _{\min}(A)}{32 C}\sqrt{\frac{n}{\log m}}
\; \text{ where  }\; \vp(s_0) := \rho_{\max}(s_0, A)+\tau_B
\een 
where we denote by $\tau_B = \tr(B)/n$ and $C$ is to be defined.
Denote by
\ben
\label{eq::defineM}
M_A = \frac{64 C \vp(s_0)}{\lambda_{\min}(A)} \ge 64 C.
\een
Throughout this paper, 
we denote by $\A_0$ the event that the modified gram
matrix $\hat\Gamma$ as defined in~\eqref{eq::hatGamma}
satisfies the Lower as well as Upper $\RE$ conditions with
\bens
\text{curvature} && 
\alpha = \frac{5}{8}\lambda_{\min}(A), \;\text{smoothness} \; \; \tilde\alpha
= \frac{11}{8}\lambda_{\max}(A) \\
\text{ and tolerance } && 
\frac{384 C^2 \vp(s_0)^2}{\lambda_{\min}(A)}\frac{\log m}{n}
 \le  \tau :=  \frac{\lambda_{\min}(A) - \alpha}{s_0}  
\le \frac{396 C^2 \vp^2(s_0+1)}{\lambda_{\min}(A)}\frac{\log m}{n}
\eens
for $\alpha, \tilde\alpha$ and $\tau$ as defined in
Definitions~\ref{def::lowRE} and \ref{def::upRE}, and $C, s_0,
\vp(s_0)$ in~\eqref{eq::s0cond}. 

To bound the optimization errors, we show that the corrected linear
regression loss function~\eqref{eq::poploss} satisfies the following {\it Restricted
  Strong Convexity (RSC)} and {\it Restricted Smoothness (RSM)} conditions
when the sample size and effective rank of matrix $B$ satisfy certain
lower bounds (cf. Theorem~\ref{thm::lasso}); namely, for all vectors
$\beta_0, \beta_1 \in \R^m$ and
\bens
\T(\beta_1, \beta_0) & := & \loss(\beta_1) - \loss(\beta_0) - \ip{\grad\loss(\beta_0), \beta_1 - \beta_0},
\eens
we show that for some parameters $(\al,  \tau_{\ell}(\loss))$ and  $(\au, \tau_{u}(\loss))$,
\begin{eqnarray}
\label{eq::lowerRSC}
\T(\beta_1, \beta_0) 
& \ge & \frac{\alpha_{\ell}}{2} \twonorm{\beta_1 - \beta_0}^2 -
\tau_{\ell}(\loss) \onenorm{\beta_1 - \beta_0}^2 \quad \text{ and } \quad \\
\label{eq::upperRSC}
\T(\beta_1, \beta_0) &  \le &  
\frac{\alpha_u}{2} \twonorm{\beta_1 - \beta_0}^2  + \tau_{u}(\loss)
\onenorm{\beta_1 - \beta_0}^2.
\end{eqnarray}
Applied to~\eqref{eq::global}, the
composite gradient descent procedure of~\cite{Nesterov07} produces
a sequence of iterates $\{\beta^t\}_{t=0}^{\infty}$ via the updates
\ben
\label{eq::optupdatelinear}
\beta^{t+1} = \argmin_{\beta \in \R^m, g(\beta) \le R} 
\left\{
\half\twonorm{\beta-\left(\beta^t - \frac{\grad
        \loss(\beta^t)}{\zeta} \right)}^2   + \frac{\rho_\lambda(\beta)}{\zeta}
\right\}
\een
where $\inv{\zeta}$ is the step size. 
Let  $\nl = 64 d \tau_{\ell}(\loss)$ and  $\bar{\alpha}_{\ell} := \al - \nl$.
We show that the composite gradient updates exhibit a type of globally
geometric convergence in terms of the compound contraction coefficient
\ben
\label{eq::definekappa}
\kappa & = & \frac{1-\frac{\bar\alpha_{\ell}}{4\zeta} + \z}{1-\z}, \; \; 
\text{ where } \; \;
 \z := \frac{2\nu(d, m, n)}{\alpha_{\ell}- \nl} :=\frac{128 d \tau_u(\loss)}{\bar\alpha_{\ell}}
\een  
where $\nl < \al/C$ for some $C > 1$ to be specified. 
Let $\tau(\loss) = \tau_{\ell}(\loss)  \vee \tau_u(\loss)$. Define
\ben
\label{eq::definexi}
\xi & := &
\frac{2  \tau(\loss)}{1-\z}  \left(
\frac{\bar\alpha_{\ell}}{4 \zeta}  + 2 \z + 5\right)  > 10 \tau(\loss).
\een
For simplicity, we present in Theorem~\ref{thm::opt-linear} the case for
$\rho_{\lambda}(\beta) =\lambda \onenorm{\beta}$ only. 
\begin{theorem}
\label{thm::opt-linear}
Consider the optimization program~\eqref{eq::global} 
for a radius $R$ such that $\beta^*$ is feasible.
Let $g(\beta) = \inv{\lambda} \rho_{\lambda}(\beta)$ where 
$\rho_{\lambda}(\beta) = \lambda \onenorm{\beta}$.
Suppose that the loss function $\loss$ satisfies the RSC/RSM
conditions~\eqref{eq::lowerRSC} and~\eqref{eq::upperRSC} with parameters 
$(\alpha_{\ell}, \tau_{\ell}(\loss))$ and $(\alpha_{u}, \tau_{u}(\loss))$ respectively. 
Let $\z$, $\kappa$ and $\xi$ be defined as in~\eqref{eq::definekappa}
and~\eqref{eq::definexi} respectively.
Suppose that the regularization parameter is chosen such that for
$\zeta \ge \au$ 
\ben
\label{eq::optlambda}
\lambda & \ge &  
\max\left\{12 \maxnorm{\grad\loss(\beta^*)}, 
\frac{16 R \xi}{(1 - \kappa)}\right\}.
\een
Suppose that $\kappa < 1$. 
Suppose that $\hat\beta$ is a global minimizer of~\eqref{eq::global}.
Then for any step size parameter $\zeta \ge \alpha_u$ and 
tolerance parameter 
\ben
\label{eq::statloss}
\delta^2 & \ge &   \frac{ c \ve_{\static}^2}{1-\kappa} \frac{d \log m}{n} =:\bar{\delta}^2,
\quad \text{ where }  \; \; \ve_{\static}^2  = \twonorm{\hat\beta -
  \beta^*}^2,
\een
the following hold for all $t \ge T^*(\delta)$
\ben
\label{eq::loss}
\phi(\beta^t) - \phi(\hat\beta) & \le &  \delta^2, \quad \text{and for } \; \e^2 = \frac{16
  \delta^4}{\lambda^2} \wedge 4 R^2, \\
\label{eq::twoerror}
\twonorm{\beta^t - \hat\beta}^2
&  \le &  \frac{2}{\bar\alpha_{\ell}}
\left(\delta^2 + 4 \nu \ve_{\static}^2 +
4 \tau(\loss) \e^2 \right),
\een
where $\nu= 64 d \tau(\loss)$,
$\tau(\loss) \asymp \frac{\log m}{n}$,
and
\ben
\label{eq::Tstar}
T^*(\delta) =   \frac{2 \log(\frac{\phi(\beta^0) -
    \phi(\hat\beta)}{\delta^2})}{\log (1/\kappa)} + 
\log\log\left(\frac{\lambda R}{\delta^2}\right)\left(1 +
  \frac{\log 2}{\log (1/\kappa)} \right).
\een
\end{theorem}
We prove Theorem~\ref{thm::opt-linear} in
Section~\ref{sec::proofofoptlinear}. 
Theorem~\ref{thm::opt-linear} is
similar in spirit to the main result Theorem 2 in~\cite{ANW12} that
deals with a convex loss function, and 
Theorem 3 in~\cite{LW15} on a similar setting to the present work. 
Compared to~\cite{LW15}, we simplified the
condition on $\lambda$ by not imposing an upper bound. Moreover, we
present refined analysis on the sample requirement and illuminate its
dependence upon the condition number $\kappa(A)$ and the tolerance
parameter $\tau$ when applied to the corrected linear regression
problem~\eqref{eq::origin2}. It is understood throughout the paper that for the same $C$ as in \eqref{eq::defineM},
\ben
\label{eq::definetau0}
\tau \asymp \tau_0 \frac{\log m}{n}, \; \; \text{ where } \; \; 
\tau_0 \asymp \frac{400  C^2 \vp(s_0  +  1)^2}{\lambda_{\min}(A)}
\approx M_A^2 \lambda_{\min}(A)/10
\een
and it is helpful to consider $M_A$ as being upper bounded by
$O(\kappa(A))$ in view of \eqref{eq::eigencond}  and (A3).
Toward this end, we prove in Section~\ref{sec::opt} that under event
$\A_0 \cap \B_0$, the RSC and RSM conditions as stated in
Theorem~\ref{thm::opt-linear} hold with $\alpha_{\ell} \asymp
\lambda_{\min}(A)$ and $\alpha_u  \asymp \lambda_{\max}(A)$ and
$\tau_{\ell}(\loss) = \tau_{u}(\loss) \asymp \tau$;
then we have  for all $t \ge T^*(\delta)$ as defined in~\eqref{eq::Tstar} and
for $\delta^2 \asymp \frac{\ve_{\static}^2}{1-\kappa} \frac{d \log
  m}{n}$,
\ben
\label{eq::errortconclude}
\twonorm{\beta^t - \hat\beta}^2 
& \le &  
\frac{4}{\alpha_{\ell}}
\delta^2 + \frac{\alpha_{\ell}}{4}  \ve_{\static}^2 + 
O\left(\frac{\delta^2 \ve_{\static}^2}{b_0^2}
\right),
\een
where $0< \kappa < 1$ so long as $\zeta \asymp
\lambda_{\max}(A)$ and  $n =\Omega(\kappa(A)  M_A^2 d  \log m)$.

We now check the conditions on $\lambda$ in
Theorem~\ref{thm::opt-linear}.
First, we note that both types of conditions on $\lambda$ are also required in 
the present paper for the statistical error bounds shown in Theorems~\ref{thm::lasso} and~\ref{thm::lassora}.
We state in Theorem~\ref{thm::main} a deterministic result from~\cite{LW12} on the
statistical error for the corrected linear model, which requires that
\ben
\label{eq::statlambda}
\lambda \ge 2
\maxnorm{\grad\loss(\beta^*)}  \; \text{ and } \quad
\lambda \ge 4 b_0 \sqrt{d} \tau
\asymp 4 R \tau \; \text{for } \tau := \tau_0 \frac{\log  m}{n}
\een
as defined in \eqref{eq::definetau0} and $d \tau \le  \frac{\alpha_{\ell}}{32}$
in order to obtain the statistical error bound for the 
corrected linear model  at the order of 
\ben
\label{eq::statloss}
\ve_{\static}^2  = \twonorm{\hat\beta - \beta^*}^2  \asymp
\frac{400}{\al^2}\lambda^2 d.
\een 
Under suitable conditions on the sample size $n$ and the effective 
rank of matrix $B$ to be stated in Theorem~\ref{thm::lasso}, we show 
that  for the loss function~\eqref{eq::poploss},  the RSC and RSM conditions hold under event $\A_0$
 (cf. Lemma~\ref{lemma::lowerREI}) following the Lower and
Upper-RE conditions as derived in Lemma~\ref{lemma::lowerREI},
\bens
\bal \approx \al \asymp\frac{\lambda_{\min}(A)}{2}, \quad
\au \asymp \frac{3\lambda_{\max}(A) }{2}, \quad \text{and} \quad  \tau(\loss) \asymp \tau.
\eens 
Compared with the lower bound imposed on $\lambda$ as
in~\eqref{eq::statlambda} that we use to derive statistical error bounds,
the penalty now involves a term $\frac{\xi}{1-\kappa}$ that crucially depends on the condition
number $\kappa(A)$ in \eqref{eq::optlambda};
Assuming that $\zeta \ge \au$, then the second condition in \eqref{eq::optlambda} on
$\lambda$ implies that  
\ben
\nonumber
\lambda &  = & \Omega(R \tau(\loss) \kappa(A))
\; \; \text{ given } \; \; \\ 
\label{eq::lowerkappa}
\frac{\xi}{1-\kappa} 
& \ge &  40 \tau(\loss) \frac{\zeta}{\bal} + 2 \tau(\loss)  \asymp
\tau \kappa(A),
 \een
which now depends explicitly on the condition number $\kappa(A)$
in addition to the radius $R \asymp b_0 \sqrt{d}$ and the tolerance
parameter $\tau$. 
This is expected given that both RSC and RSM conditions are needed in order to derive
the computational convergence bounds, while for the statistical error,
we only require the RSC (Lower RE) condition to hold.

\noindent{\bf Remarks.}
Consider the regression model in  \eqref{eq::oby} and \eqref{eq::obX}
with independent random matrices $X_0, W$ as in~\eqref{eq::subgdata},  and an error vector $\e \in \R^n$
independent of $X_0, W$,  with independent entries $\e_{j}$ satisfying
$\E \e_{j} = 0$ and  $\norm{\e_{j}}_{\psi_2} \leq M_{\e}$.
Theorem~\ref{thm::oracle} and its corollaries provide an upper bound
on the $\ell_{\infty}$ norm of the gradient $\grad\loss(\beta^*) =  \hat\Gamma \beta^* - \hat\gamma$ of the
loss function in the corrected linear model, where $\hat\Gamma$ and
$\hat\gamma$ are as defined in~\eqref{eq::hatGamma}.
Let
\ben
\label{eq::defineD0}
D_0' = \twonorm{B}^{1/2} + a_{\max}^{1/2}, \; \; \text{ and }  \; \; \;  D_{\ora} =  2(\twonorm{A}^{1/2} + \twonorm{B}^{1/2}).
\een
Specializing to the case of corrected linear models,  we have 
by Corollary \ref{coro::D2improv},  on event $\B_0$ as defined therein,
\bens
\norm{\grad\loss(\beta^*)}_{\infty} = 
\norm{\hat\Gamma \beta^* - \hat\gamma}_{\infty} \le \psi \sqrt{\frac{\log m}{n}}
\eens
where $\psi:= C_0 D_0' K \left(M_{\e}  + \tau_B^{+/2} K  \twonorm{\beta^*}
\right)$ and $\tau_B^{+/2} = \tau_B^{1/2}  +\frac{D_{\ora}}{\sqrt{m}}$ for $D_0', D_{\ora}$ as defined in \eqref{eq::defineD0}.

The bound \eqref{eq::loss} characterizes the excess loss
$\phi(\beta^t) - \phi(\hat\beta)$ for solving~\eqref{eq::origin} using
the composite gradient algorithm; moreover,  for
any iterate $\beta^t$ such that \eqref{eq::loss} holds, the
following bound on the optimization error $\beta^t -\hat\beta$ follows
immediately:
\bens
\twonorm{\beta^t - \hat\beta}^2
&  \le &  \frac{2}{\bar\alpha_{\ell}}
\left(\delta^2 + 4 \nu \ve_{\static}^2 +
\frac{64  \tau_{\ell}(\loss)\delta^4}{\lambda^2} \right),
\eens
where $\nu = 64 d \tau(\loss)$ and $4 \tau(\loss) \e^2 = 64
\tau(\loss) \frac{\delta^4}{\lambda^2}$ by definition of $\e^2$ in
view of \eqref{eq::statloss}.
Finally, we note that Theorem~\ref{thm::opt-linear} holds for a class of weakly convex penalties as
considered in~\cite{LW15} with suitable adaptation of RSC and
parameters and conditions to involve $\mu$, following exactly the same
sequence of arguments. Notable examples of such weakly convex penalty
functions are SCAD~\citep{FL01} and MCP~\citep{Zhang10}.

The rest of the paper is organized as follows.  In
Section~\ref{sec::tworesults}, we present two main results in
Theorems~\ref{thm::lasso} and~\ref{thm::DS}.  In Section~\ref{sec::smallW}, we
state more precise results which improve upon Theorems~\ref{thm::lasso}
and~\ref{thm::DS}; these results are more precise in the sense that our bounds
and penalty parameters now take $\tr(B)$, the parameter that measures the
magnitudes of errors in $W$, into consideration.  In Section~\ref{sec::opt}, we
show that the RSC and RSM conditions hold for the corrected linear loss
function and present our computational convergence bounds with regard
to~\eqref{eq::origin} in Theorem~\ref{thm::corrlinear} and
Corollary~\ref{coro::deer}.  In Section~\ref{sec::proofall}, we outline the
proof of the main theorems.  In particular, we outline the proofs for
Theorems~\ref{thm::lasso},~\ref{thm::DS},~\ref{thm::lassora}
and~\ref{thm::DSoracle} in Section~\ref{sec::proofall},~\ref{sec::lassooracle}
and~\ref{sec::DSoracle} respectively.  In Section~\ref{sec::AD}, we show a
deterministic result as well as its application to the random matrix
$\hat\Gamma - A$ for  $\hat\Gamma$ as in~\eqref{eq::hatGamma} with regards to
the upper and Lower $\RE$ conditions. 
In Section~\ref{sec::exp}, we present results from numerical simulations designed to
validate the theoretical predictions in previous sections. 
The technical details of proofs are
collected at the end of the paper.  We prove Theorem~\ref{thm::lasso} in
Section~\ref{sec::proofofthmlasso}.  We prove Theorem~\ref{thm::DS} in
Section~\ref{sec::proofofDSthm}.  We prove Theorems~\ref{thm::lassora}
and~\ref{thm::DSoracle} in Section~\ref{sec::classoproof} and
Section~\ref{sec::DSoraproof} respectively.  We defer the proof of
Theorem~\ref{thm::opt-linear} to Section~\ref{sec::proofofoptlinear}. The paper
concludes with a discussion of the results  in Section~\ref{sec::conclude}.  We
list a set of symbols we use throughout the paper in Table \ref{tab::symbol}.
Additional proofs and theoretical results are collected in the Appendix.

\begin{table}[h]
\begin{center}
\caption{symbols we used throughout the proof}
\label{tab:fpfn}
\begin{tabular}{l|l} 
\hline 
Symbol & Definition \\ \hline 
$\alpha$ & curvature: $\alpha := \frac{5}{8}\lambda_{\min}(A)$ \\
$\alpha_{\ell}$ & Lower RE/ RSC curvature  parameter: $\alpha_{\ell} = \alpha$ \\
$\alpha_{u}$ & Upper RE/ RSM parameter $\alpha_{u} \asymp
\frac{3}{2}\lambda_{\max}(A)$ \\
$\bare$  &
 $\bare = 8\sqrt{d} \ve_{\static}$ where 
$\ve_{\static} = \twonorm{\hat\beta - \beta^*}$ \\ 
$\tau_0$  & $\tau_0 \asymp \frac{400  C^2 \vp(s_0  +  1)^2}{\lambda_{\min}(A)}$ \\
$\tau = \frac{\lambda_{\min}(A) - \alpha}{s_0}$  &  tolerance parameter $\tau = \tau_0
\frac{\log m}{n}$ in Lower/Upper RE conditions \\
$\tau_B$ & $\tau_B = \tr(B)/n$ \\
$s_0 \ge 1$ & the largest integer chosen such that the following inequality still
holds: \\
& $\sqrt{s_0} \vp(s_0) \le \frac{\lambda _{\min}(A)}{32
  C}\sqrt{\frac{n}{\log m}}$ \\
$ \vp(s_0)$ & $\rho_{\max}(s_0, A)+\tau_B$ \\
$\tau_{\ell}(\loss)$
 &  tolerance parameter in RSC condition:  
$\tau_{\ell}(\loss) \asymp \tau_0 \log m /n$ \\
$\tau_{u}(\loss)$  &  tolerance parameters in RSM condition:
$\tau_{u}(\loss) \asymp \tau_0 \log m /n$ \\
$\nl$ &  $\nl = 64 d \tau_{\ell}(\loss) < \frac{\alpha_{\ell}}{60}$ \\ 
$\nu(d, m, n)$ & $\nu(d, m, n) = 64 d \tau_{u}(\loss)$ \\
$\bal$ & effective RSC coefficient $\bal = \al - \nl$ \\
$\phi(\beta)$ & loss function: $\phi(\beta) =\half \beta^T \hat\Gamma
\beta - \hat\gamma^{T} + \rho_{\lambda}(\beta)$ \\
$\grad\loss(\beta)$ & Gradient of the loss function $\hat\Gamma \beta
- \hat\gamma$ \\ 
$\rho_n$ & $\rho_n = C_0 K \sqrt{\frac{\log m}{n}}$ \\
$r_{m,n}$ & $r_{m,n} = 2 C_0 K^2 \sqrt{\frac{\log m}{mn}}$ \\
$\zeta$ & step size parameter: $\zeta \ge \alpha_{u} = 11\lambda_{\max}/8$ \\
$\z$ & contraction parameter 
$\z:= \frac{2\nu(d, m, n)}{\bar\alpha_{\ell}}=\frac{128 d \tau_u(\loss)}{\bar\alpha_{\ell}} <
\frac{\bar\alpha_{\ell}}{8\zeta}$ \\
$\kappa$ &  contraction coefficient as $\kappa := {(1-\frac{\bar\alpha_{\ell}}{4\zeta}+ \z)}{(1-\z)^{-1}} < 1$ \\
$\delta^2$ & tolerance parameter in computational errors $\delta^2 \ge  
\frac{c \ve_{\static}^2}{1-\kappa} \frac{d \log
  p}{n}$ \\
$M_A$ &  $M_A = \frac{64 C  \vp(s_0)}{\lambda_{\min}(A)}$ where $\vp(s_0) = \rho_{\max}(s_0, A)
+ \tau_B$.\\
$M_+$ &  $M_+ = \frac{32 C  \vp(s_0+1)}{\lambda_{\min}(A)}$ where $\vp(s_0+1) = \rho_{\max}(s_0+1, A)
+ \tau_B$.\\
$\xi$ &  $\xi = {2(   \tau_{\ell}(\loss) \vee \tau_{u}(\loss))} \left(\frac{\bar\alpha_{\ell}}{4 \zeta}  + 2 \z + 5\right)(1-\z)^{-1}$  \\
$\V$ &  $\V= 3eM_A^3 /2$ \\
\hline
\end{tabular}
\label{tab::symbol}
\vskip-10pt
\end{center}
\end{table} 

\section{Main results on the statistical error}
\label{sec::tworesults}
In this section,  we will state our main results in Theorems~\ref{thm::lasso}
and~\ref{thm::DS} where we consider the regression model in \eqref{eq::oby}
and \eqref{eq::obX} with random matrices $X_0, W \in \R^{n \times m}$
as defined in~\eqref{eq::subgdata}.
For the corrected  Lasso estimator, we are interested in the case where the smallest eigenvalue
of the column-wise covariance matrix $A$ does not approach $0$ too quickly
and the effective rank of the row-wise covariance matrix $B$ is 
bounded from below (cf.~\eqref{eq::trBLasso}).
More precisely, (A2) thus ensures that the Lower-$\RE$ condition as in
Definition~\ref{def::lowRE} is not vacuous. 
(A3) ensures that~\eqref{eq::s0cond} holds for some $s_0 \ge 1$.
Throughout this paper, for the corrected 
Lasso estimator, we will use the
expression 
\bens
\label{eq::definetau}
\tau := 
\frac{\lambda_{\min}(A) - \alpha}{s_0},  
\; \; \text{where} \; \; \alpha =
\frac{5}{8}\lambda_{\min}(A) \; \; \text{ and } \; \; s_0 \asymp \frac{4
  n}{M_A^2 \log m}
\eens
where $M_A$ is as defined in~\eqref{eq::defineM}.
Let 
\ben
\label{eq::defineD2}
D_0 = \sqrt{\tau_B} + a_{\max}^{1/2}  \; \; \text{ and }\; \; D_2 = 2 (\twonorm{A} + \twonorm{B}).
\een
\begin{theorem}{\textnormal{(\bf{Estimation for the corrected Lasso estimator})}}
\label{thm::lasso}
Consider the regression model in  \eqref{eq::oby} and \eqref{eq::obX}
with independent random matrices $X_0, W$ as in~\eqref{eq::subgdata},  and an error vector $\e \in \R^n$
independent of $X_0, W$,  with independent entries $\e_{j}$ satisfying
$\E \e_{j} = 0$ and  $\norm{\e_{j}}_{\psi_2} \leq M_{\e}$.
Set $n = \Omega(\log m)$.
Suppose $n \le ({\V}/{e}) m \log m$, 
where $\V$ is a constant which depends on $\lambda_{\min}(A)$,
$\rho_{\max}(s_0, A)$ and $\tr(B)/n$. 
Suppose $m$ is sufficiently large. 

Suppose (A1), (A2) and (A3) hold. 
Let $C_0, c', c_2, c_3 > 0$ be  some absolute constants. 
Suppose that  $\fnorm{B}^2/\twonorm{B}^2 \ge \log m$.
Suppose that $c' K^4 \le 1$ and
\ben
\label{eq::trBLasso}
\quad
r(B) := \frac{\tr(B)}{\twonorm{B}} & \ge & 16c' K^4 \frac{n}{\log m}
\log \frac{\V m \log m }{n}.
\een
Let $b_0, \phi$ be numbers which satisfy
\ben
\label{eq::snrcond}
\frac{M^2_{\e}}{K^2 b_0^2}   \le \phi  \le 1.
\een
Assume that the sparsity of $\beta^*$ satisfies for some $0 < \phi \le
1$
\ben
\label{eq::dlasso}
&& d:= \abs{\supp(\beta^*)} \le 
\frac{c' \phi K^4}{40 M_+^2} \frac{n}{\log  m}< n/2, \\
&& \quad
\text{ where } \quad M_{+} =  \frac{32 C \vp(s_0+1)}{\lambda_{\min}(A)}
\een
for $\vp(s_0+1) = \rho_{\max}(s_0+1,A) + \tau_B$.
 
Let $\hat\beta$ be an optimal solution to the corrected  Lasso estimator as in~\eqref{eq::origin} with 
\ben
\label{eq::psijune}
&& \lambda \ge 4 \psi \sqrt{\frac{\log m}{n}} \; \; \text{ where } \;\;
\psi  := C_0 D_2 K \left(K \twonorm{\beta^*}+ M_{\e}\right).
\een
Then for any $d$-sparse vectors $\beta^* \in \R^m$, such that
\ben
\label{eq::range}
\phi b_0^2 \le \twonorm{\beta^*}^2 \le b_0^2,
\een
we have with probability at least
$1- 4\exp\left(-\frac{c_3 n}{M_A^2 \log m}
  \log\left(\frac{\V m \log m}{n}\right)\right) -2\exp\left(- \frac{4c_2 n}{M_A^2 K^4 }\right) - 22/m^3$,
\bens
\twonorm{\hat{\beta} -\beta^*} \leq \frac{20}{\alpha}  \lambda \sqrt{d} \; \;
\text{ and } \; \norm{\hat{\beta} -\beta^*}_1 \leq \frac{80}{\alpha}
\lambda d.
\eens
\end{theorem}
We give an outline of the proof
of Theorem~\ref{thm::lasso} in Section~\ref{sec::lassooutline}.
We prove Theorem~\ref{thm::lasso} in Section~\ref{sec::proofofthmlasso}.
We defer discussions on conditions appearing 
Theorem~\ref{thm::lasso} in Section~\ref{sec::discusslasso}.

For the Conic programming estimator, we impose a restricted eigenvalue condition as formulated
in~\cite{BRT09,RZ13} on $A$ and assume that the sparsity of $\beta^*$
is bounded by $o(\sqrt{n/\log m})$. These conditions will be relaxed
in Section~\ref{sec::smallW} where we allow $\tau_B$ to approach 0.
\begin{theorem}
\label{thm::DS}
Suppose (A1) holds.
Set $0< \delta < 1$. Suppose that $n < m \ll \exp(n)$ and $1\le  d_0 < n$.
Let $\lambda >0$ be the same parameter as in~\eqref{eq::Conic}.
Suppose that $\fnorm{B}^2/\twonorm{B}^2 \ge \log m$.
Suppose that the sparsity of $\beta^*$ is bounded by 
\ben
\label{eq::sqrt-sparsity}
d_0 := \abs{\supp(\beta^*)} \le c_0 \sqrt{n/\log m}
\een
for some constant $c_0>0$.
Suppose
\ben
\label{eq::samplebound}
n & \geq & \frac{2000 d K^4}{\delta^2} \log \left(\frac{60 e m}{d  \delta}\right) \; \; \text{ where} \; \;  \\
\label{eq::sparse-dim-Ahalf}
d & = & 2d_0 + 2d_0 a_{\max} \frac{16 K^2(2d_0, 3k_0, A^{1/2}) (3k_0)^2
  (3k_0 + 1)}{\delta^2}.
\een
Consider the regression model in  \eqref{eq::oby} and \eqref{eq::obX}
with $X_0$, $W$ as in~\eqref{eq::subgdata} and an error vector $\e \in
\R^n$, independent of $X_0, W$,  with independent entries $\e_{j}$ satisfying
$\E \e_{j} = 0$ and  $\norm{\e_{j}}_{\psi_2} \leq M_{\e}$.
Let $\hat\beta$ be an optimal solution to the Conic programming estimator
as  in \eqref{eq::Conic} with input $(\hat\gamma, \hat\Gamma)$ as
defined in~\eqref{eq::hatGamma}.
Recall $\tau_B := \tr(B)/n$.  Choose for $D_0, D_2$ as in~\eqref{eq::defineD2} 
and 
\bens
\mu \asymp D_2 K^2  \sqrt{\frac{\log m}{n}} \;\;\text{ and } \; \;  \omega \asymp D_0 K M_{\e} \sqrt{\frac{\log m}{n}}.
\eens
Then with probability at least $1-\frac{c'}{m^2} - 2 \exp(-\delta^2 n/2000 K^4)$, 
\ben
\label{eq::ellqnorm}
\norm{\hat{\beta} -\beta^*}_q \le C D_2  K^2  d_0^{1/q} \sqrt{\frac{\log m}{n}}
\left(\twonorm{\beta^*} + \frac{M_{\e}}{K}\right)
\een
for $2 \ge q \ge 1$.
Under the same assumptions, the predictive risk admits the following
bounds with the same probability as above,
\bens
\inv{n} \twonorm{X (\hat{\beta} -\beta^*)}^2 \le C' D_2^2  K^4  d_0 \frac{\log m}{n}
\left(\twonorm{\beta^*} + \frac{M_{\e}}{K}\right)^2
\eens
where $c', C_0, C, C' > 0$ are some absolute constants.
\end{theorem}
We give an outline of the proof of Theorem~\ref{thm::DS} in
Section~\ref{sec::proofall} while leaving the detailed proof in
Section~\ref{sec::proofofDSthm}.

\subsection{Regarding the $M_A$ constant}
Denote by
\bens
M_A = \frac{64 C \vp(s_0)}{\lambda_{\min}(A)} \asymp \frac{\rho_{\max}(s_0, A)+\tau_B}{\lambda_{\min}(A)}
\eens
\bit
\item
(A3) ensures that $M_A$ and $M_{+}$ are upper bounded by the condition number of $A$:
$\kappa(A) := \frac{\lambda_{\max}(A)}{\lambda_{\min}(A)} = 
O\left(\sqrt{\frac{n}{\log m}}\right)$ given that $\tau_B := \tr(B)/n
= O(\lambda_{\max}(A))$. 
\item
So the condition~\eqref{eq::dlasso} 
in Theorem~\ref{thm::lasso} allows $d \asymp n/\log
  m$ in the optimal setting when the condition number $\kappa(A)$ is understood to be a
  constant. As $\kappa(A)$ increases, the conservative worst case upper bound on
$d$ needs to be adjusted correspondingly. Moreover, this adjustment is also crucial in order
to ensure the composite gradient algorithm to converge in the sense of
Theorem~\ref{thm::opt-linear}. 
We will illustrate such dependencies on $\kappa(A)$ in numerical examples in Section~\ref{sec::exp}.
\item
The condition $\tau_B = O(\lambda_{\max}(A))$ puts an upper bound on
how large the measurement error in $W$ can be. 
We do not allow the measurement error to overwhelm the signal entirely. When
$\tau_B \to 0$, we recover the ordinary Lasso bound
in~\cite{BRT09}, which we elaborate in the next two sections.
\eit
Throughout this paper, we assume that $M_{A}
\asymp M_+$, where recall $M_+ = \frac{32 C
  \vp(s_0+1)}{\lambda_{\min}(A)}$.

\subsection{Discussions}
\label{sec::discusslasso}
Throughout our analysis, we set the parameter $b_0 \ge \twonorm{\beta^*}$ and
$d = \size{\supp(\beta^*)} :=  \size{\{j: \beta_j^* \not=0\}}$ for the corrected 
Lasso estimator. In practice, both $b_0$
and $d$ are understood to be parameters chosen to provide an upper
bound on the $\ell_2$ norm and the sparsity of  the true $\beta^*$.
The parameter $0< \phi < 1$ is a parameter that we use to describe the gap between
$\twonorm{\beta^*}^2$ and its upper bound $b_0^2$. 
Denote the Signal-to-noise ratio by 
\bens
\SNR := {K^2 \twonorm{\beta^*}^2}/{M^{2}_{\e}},
\; \text{ where } \; \; \noise := M^{2}_{\e} 
 \; \text{ and
} \; \; \phi K^2 b_0^2  \le \signal := K^2 \twonorm{\beta^*}^2 \le K^2 b_0^2.
\eens
The two conditions \eqref{eq::snrcond} and \eqref{eq::range} on $b_0$
and $\phi$  imply that $\noise \le K^2 \phi b_0^2 \le \signal$. 
Notice that this could be restrictive if $\phi$ is small.
We will show in Section~\ref{sec::lassooutline} that
condition~\eqref{eq::snrcond} is not needed in order for the
$\ell_p, p=1, 2$ errors as stated in the Theorem~\ref{thm::lasso} to hold.  It was
indeed introduced so as to further simplify the expression for the condition on $d$ as shown in
\eqref{eq::dlasso}. Therefore we provide slightly more general conditions on $d$ in
\eqref{eq::dlassoproof} in Lemma~\ref{lemma::dmain},
where~\eqref{eq::snrcond} is not required.
We introduce the parameter $\phi$ so that the conditions on $d$ depend on $\phi$ and  $b_0^2$ rather than
the true signal $\twonorm{\beta^*}$ (cf. Proof of Lemmas~\ref{lemma::dmain} and~\ref{lemma::dmainoracle}).
It will also become clear in the sequel from the proof of Lemma~\ref{lemma::dmain}
(cf. \eqref{eq::dphicondition}) that
we could use $\twonorm{\beta^*}$ rather than its the lower bound $b_0^2 \phi$
in the expression for $d$. However, we choose to state the condition on $d$ as in Theorem~\ref{thm::lasso} for clarity of
our exposition. See also Theorem~\ref{thm::lassora} and Lemma~\ref{lemma::dmainoracle}.

In fact, we prove that Theorem~\ref{thm::lasso} holds with $\noise = M_{\e}^2$ and $\underline{\signal}
=\phi K^2 b_0^2$ in arbitrary orders, so long as conditions 
\eqref{eq::trBLasso} and~\eqref{eq::dlasso} or \eqref{eq::dlassoproof}
hold.  For both cases, we require that 
$\lambda  \asymp (\twonorm{A} + \twonorm{B}) K \sqrt{\signal+\noise}
\sqrt{\frac{\log m}{n}}$ as expressed in \eqref{eq::psijune}.
That is, when either the noise level $M_{\e}$ or the signal strength
$K \norm{\beta^*}$ increases, we need to increase
$\lambda$ correspondingly; moreover, when $\noise$ dominates the
signal $K^2 \twonorm{\beta^*}^2$, we have for $d \asymp \inv{M_A^2}
\frac{n}{\log m}$ as in~\eqref{eq::dlasso}, 
\bens
\twonorm{\hat{\beta} -\beta^*} /\twonorm{\beta^*} 
 =O_P\left( D_2 K^2 \sqrt{\frac{\noise}{\signal}}  \inv{\vp(s_0+1)}\right),
\eens
which eventually becomes a vacuous bound when $\noise \gg \signal$.
This bound appears a bit crude as it does not entirely discriminate
between the noise, measurement error, and the signal strength.
We further elaborate on the relationships among these three elements
in Section~\ref{sec::smallW}.
We will then present an improved bound in Theorem~\ref{thm::lassora}.
\bnum
\item
The choice of $\lambda$ for the Lasso estimator and parameters
$\mu, \omega$ for the DS-type  estimator satisfy
\bens
\lambda \asymp \mu \twonorm{\beta^*} + \omega.
\eens
This relationship is made clear through Theorem~\ref{thm::main}
regarding the corrected  Lasso estimator, which follows from Theorem
1 by~\cite{LW12}, and Lemmas~\ref{lemma::DS}
and~\ref{lemma::DSimprov} for the Conic programming estimator.
The penalty parameter $\lambda$ is chosen to bound
$\norm{\hat\gamma -  \hat\Gamma \beta^*}_{\infty}$  from above, 
which is in turn bounded in Theorem~\ref{thm::oracle}.
See Corollaries~\ref{coro::low-noise} and~\ref{coro::D2improv}, 
which are the key results in proving Theorems~\ref{thm::lasso},~\ref{thm::DS},~\ref{thm::lassora},
and~\ref{thm::DSoracle}.
\item
Throughout our analysis of Theorems
~\ref{thm::lasso} and~\ref{thm::DS}, our error bounds are stated in a
way assuming the errors in $W$ are sufficiently large 
in the sense that these bounds are optimal only when $\tau_B$ is
bounded from below by some absolute constant. 
For example,  when
$\twonorm{B}$ is bounded away from $0$, the lower bound on the
effective rank $r(B) = \tr(B)/\twonorm{B}$
implies that $\tau_B$ must also be bounded away from $0$.
More precisely,  by the condition on the effective rank as in \eqref{eq::trBLasso}, we have 
\bens
\tau_{B} = \frac{\tr(B)}{n} & \ge &
 16c' K^4 \frac{\twonorm{B}}{\log m} \log \frac{\V m \log m }{n} \; \;
 \text{ where} \; \; \V = 3eM_A^3/2.
\eens
Later, we will state our results with $\tau_B  = \tr(B)/n> 0$ being
explicitly included in the error bounds as well as the penalization parameters
and sparsity constraints.
\item
In view of the main Theorems~\ref{thm::lasso} and~\ref{thm::DS}, at
this point, we do not really think one estimator is preferable to the
other.
While the $\ell_q$ error bounds we obtain for the two estimators are 
at the same order for $q=1, 2$, the conditions under which these error
bounds are obtained are somewhat different.  In Theorem~\ref{thm::DS}, we only require that 
$\RE(2d_0,  3 k_0, A^{1/2})$ 
holds for $k_0 = 1+\lambda$ where $\lambda \asymp 1$, while in
Theorem~\ref{thm::lasso} we need the minimal eigenvalue of $A$ to be
bounded from below, namely, we need to assume that (A2) holds.
As mentioned earlier, (A2) ensures that the Lower-$\RE$ condition as in
Definition~\ref{def::lowRE} is not vacuous while (A3) ensures
that~\eqref{eq::s0cond} holds for some $s_0 \ge 1$.
Th condition \eqref{eq::trBLasso}  on the effective rank of the row-wise covariance matrix $B$ 
is also needed to establish the Lower and Upper RE conditions in
Lemma~\ref{lemma::lowerREI} for the corrected 
Lasso estimator. Moreover, for the sparsity parameter $d_0$ in 
\eqref{eq::sqrt-sparsity}, we show in Lemma~\ref{lemma::translation}
that (A2) is a sufficient condition for a type of $\RE(2d_0, 3k_0)$
condition to hold on non positive definite $\hat\Gamma$ as defined in \eqref{eq::hatGamma}.
See also Theorem~\ref{thm::AD}.
\item
In some sense, the assumptions in Theorem~\ref{thm::lasso} appear to be
slightly stronger, while at the same time yielding correspondingly
stronger results in the following sense: 
The corrected  Lasso procedure can recover a sparse model using 
$O(\log m)$ number of measurements per nonzero component 
despite the measurement error in $X$ and the stochastic noise $\e$,
while the Conic programming estimator allows only $d \asymp \sqrt{n/\log
  m}$ to achieve the error rate at the same order as the corrected  Lasso estimator.
Hence, while Conic programming estimator is conceptually more
 adaptive by not fixing an upper bound on $\twonorm{\beta^*}$ a priori,
 the price we pay seems to be a more stringent upper bound on the sparsity level. 
\item
We note that following Theorem~2 as in~\cite{BRT14}, one can
show that without the relatively restrictive sparsity
condition~\eqref{eq::sqrt-sparsity}, a bound similar to that
in~\eqref{eq::ellqnorm} holds, however, with $\twonorm{\beta^*}$ being
replaced by $\onenorm{\beta^*}$,  so long as the sample size satisfies
the condition as in \eqref{eq::samplebound}. 
However, we show in Theorem~\ref{thm::DSoracle} in Section
\ref{sec::DSoracle} that this restriction on the sparsity can be
relaxed for the Conic programming estimator~\eqref{eq::Conic}, 
when we make a different choice for the parameter $\mu$ based on a
more refined analysis.
\enum
Results similar to Theorems~\ref{thm::lasso} and~\ref{thm::DS} 
have been derived in~\cite{LW12,BRT14}, however, under different assumptions on the distribution of the noise
matrix $W$. When $W$ is a random matrix with i.i.d. subgaussian noise,
our results in Theorems~\ref{thm::lasso} and~\ref{thm::DS} will essentially recover the results in~\cite{LW12}
and~\cite{BRT14}. We compare with their results in
Section~\ref{sec::smallW} in case $B =  \tau_B I$ after we present our
improved bounds in Theorems~\ref{thm::lassora}
and~\ref{thm::DSoracle}. We refer to the paper of~\cite{BRT14} for a concise
summary of these and some earlier results.

Finally, one reviewer asked about the dependence of the tuning
parameter on properties of $A$ and $B$, namely parameters $D_0 =
\sqrt{\tau_B} + a_{\max}^{1/2}$, $D_0' = \twonorm{B}^{1/2} +
a_{\max}^{1/2}$ and $D_2 = \twonorm{A} + \twonorm{B}$.
We now state in Lemma~\ref{lemma::trBest} a sharp bound on
estimating $\tau_B$ using $\hat\tau_B$ as in \eqref{eq::trBest}, which
will provide a natural plug-in estimate for parameters such as $D_0$ that involve $\tau_B$.
\begin{lemma}
\label{lemma::trBest}
Let $m \ge 2$. Let $X$ be defined as in~\eqref{eq::subgdata} and
$\hat\tau_B$ be as defined in \eqref{eq::trBest}. Denote by $\tau_B = \tr(B)/n$ and $\tau_A = \tr(A)/m$.
Suppose that $n \vee (r(A)  r(B)) > \log m$.
Denote by $\B_6$ the event such that 
\bens
\abs{\hat\tau_B - \tau_B} & \le &  
2 C_0 K^2 \sqrt{\frac{\log m}{m n}}
\left(\frac{\fnorm{A}  }{\sqrt{m}} +\frac{\fnorm{B}  }{\sqrt{n}}
\right) =:  D_1 r_{m,m},
\eens 
where $D_1 = \frac{\fnorm{A}}{\sqrt{m}} + \frac{\fnorm{B}}{\sqrt{n}}$ and
$ r_{m,m} = 2 C_0 K^2 \sqrt{\frac{\log m}{m n}}$.
Then $\prob{\B_6} \ge 1-\frac{3}{m^3}$.

If we replace $\sqrt{\log m}$ with $\log m$ in the definition of event
$\B_6$, then we can drop the condition on $n$ or $r(A)r(B) =
\frac{\tr(A)}{\twonorm{A}} \frac{\tr(B)}{\twonorm{B}}$ to achieve 
the same bound on event $\B_6$.
\end{lemma}
In an earlier version of the present work by the same authors~\cite{RZ15}, 
we presented the rate of convergence for using the corrected gram 
matrix $\hat{B} := \inv{m}X X^T - \frac{\tr(A)}{m} I_m$ 
to estimate $B$ and proved isometry properties in the operator norm
once the effective rank of $A$ is sufficiently large compared to $n$; 
one can then use such estimated  $\hat{B}$ and its operator norm in
$D_2$ and $D_0'$. See Theorem 21 and Corollary 22 therein. 
As mentioned, we use the estimated $\hat{\tau}_B$ (cf. Lemma~\ref{lemma::trBest}) in $D_0$.
The dependencies on $A$, $\twonorm{\beta^*}$ and $\e$ are known problems in
the Lasso and corrected Lasso literature; see~\cite{BRT09,LW12}. For example, the RE condition as stated in
Definition~\ref{def:memory} and its subgaussian concentration
properties as shown~\cite{RZ13}
clearly depend on unknown parameter $a_{\max}$ related to
covariance matrix $A$. See Theorem~\ref{thm:subgaussian-T-intro} in
the present paper.
We prove Lemma~\ref{lemma::trBest} in Section~\ref{sec::proofofTR}.
Lemma~\ref{lemma::trBest}  provides the powerful technical insight
and one of the key ingredients leading to the tight analysis in
Theorems~\ref{thm::lassora} and~\ref{thm::DSoracle} for the corrected
Lasso estimator~\eqref{eq::origin} as well as the
 Conic programming estimator~\eqref{eq::Conic} in
 Section~\ref{sec::smallW}, where  we
 also present theory for which the dependency on $\twonorm{A}$ becomes extremely mild.

\section{Improved bounds when the measurement errors are small}
\label{sec::smallW}
Although the conclusions of Theorems~\ref{thm::lasso}
and~\ref{thm::DS} apply to cases when $\twonorm{B} \to 0$, the error bounds
are not as tight as the bounds we are about to derive in this section.
So far, we have used more crude approximations on the error bounds in terms of
estimating $\norm{\hat\gamma -  \hat\Gamma \beta^*}_{\infty}$ for the
sake of reducing the amount of unknown parameters we need to consider. 
The bounds we derive in this section take the magnitudes of the measurement errors in $W$ into
consideration. As such, we allow the error bounds to depend on the
parameter $\tau_B$ explicitly, which become much tighter as $\tau_B$
becomes smaller. For the extreme case when $\tau_B$
approaches $0$, one hopes to recover a bound close to the
regular Lasso or the Dantzig selector as the effect of the noise on
the procedure should become negligible. 
We show in Theorems~\ref{thm::lassora}
and~\ref{thm::DSoracle} that this is indeed the case.
Denote by 
\ben
\label{eq::defineDtau}
&&  \tau_B^{+/2}  :=    \sqrt{\tau_B} + \frac{D_{\ora}}{\sqrt{m}},\;
\text{  where } \;
D_{\ora} \;  = \; 2(\twonorm{A}^{1/2} + \twonorm{B}^{1/2}).
\een
We first state a more refined result for the Lasso-type estimator, for
which we now only require that
$$\lambda  \asymp (a^{1/2}_{\max} + \twonorm{B}^{1/2}) K
\sqrt{\noise + \tau_B \signal} \sqrt{\frac{\log m}{n}}.$$
That is, we replace $\sqrt{\noise + \signal}$ in
$\lambda$~\eqref{eq::psijune} now with $\sqrt{\noise +
\tau_B \signal}$, 
which leads to significant improvement on the rates of
convergence for estimating $\beta^*$ when $\tau_B \to 0$.
\begin{theorem}
\label{thm::lassora}
Suppose all conditions in Theorem~\ref{thm::lasso} hold, except that
we drop \eqref{eq::snrcond} and replace~\eqref{eq::psijune} with
\ben
\label{eq::psijune15}
&& \lambda \ge 4 \psi \sqrt{\frac{\log m}{n}}, \text{ where } \;\;
\psi:= C_0 D_0' K \left( M_{\e}  + \tau_B^{+/2} K \twonorm{\beta^*}  \right)
\een
for $D_0'$ and $\tau_B^{+/2}$ as defined in \eqref{eq::defineD0}
and \eqref{eq::defineDtau} respectively.
Let $c', \phi, b_0, M_{\e}$, $K$ and  $M_{+}$
be as defined in Theorem~\ref{thm::lasso}. Let $\tau_B^+ = (\tau_B^{+/2})^2$.

Suppose that for $0< \phi \le 1$ and $C_A := \inv{160 M_{+}^2}$,
\ben
\label{eq::doracle}
d & := & \abs{\supp(\beta^*)} \le C_A \frac{n}{\log m} \left\{c' c'' D_{\phi}
  \wedge 8 \right\} =: \bar{d}_0, \; \; \text{ where } \\
&& c'' =  \frac{\twonorm{B} +  a_{\max}}{\vp(s_0+1)^2} \; \; \text{ and } \; \;  
D_{\phi}  = \frac{K^2 M^2_{\e}}{b_0^2} +  \tau_B^{+} K^4 \phi
\een
Then for any $d$-sparse vectors $\beta^* \in \R^m$, such that
$\phi b_0^2 \le \twonorm{\beta^*}^2 \le b_0^2$, 
we have 
\ben
\label{eq::lassobounds}
\twonorm{\hat{\beta} -\beta^*} \leq \frac{20}{\alpha}  \lambda \sqrt{d} \; \;
\text{ and } \; \norm{\hat{\beta} -\beta^*}_1 \leq \frac{80}{\alpha}
\lambda d
\een
with probability  at least 
$1- 4\exp\left(-\frac{c_3 n}{M_A^2 \log m}
  \log\left(\frac{\V m \log m}{n}\right)\right)
 -2\exp\left(- \frac{4c_2 n}{M_A^2 K^4 }\right) - 22/m^3$.
\end{theorem}
We give an outline for the proof of Theorem~\ref{thm::lassora} in
Section~\ref{sec::lassooracle}, and show the actual proof in
Section~\ref{sec::classoproof}.

We next state in Theorem~\ref{thm::DSoracle} 
an improved bounds for the Conic programming
estimator~\eqref{eq::Conic}, which dramatically improve upon 
those in Theorem~\ref{thm::DS} when $\tau_B$ is small, where 
an ``oracle'' rate for estimating $\beta^*$ with the 
Conic programming estimator $\hat\beta$~\eqref{eq::Conic} is defined
and the predictive error $\twonorm{X v}^2$ when $\tau_B = o(1)$ is derived.

Let $C_0$ satisfy \eqref{eq::defineC0} for $c$ 
as defined in Theorem~\ref{thm::HW}. 
Throughout the rest of the paper, we denote by:
\ben
\label{eq::rhon}
\rho_{n} & = &  C_0 K \sqrt{\frac{\log m}{n}} \; \; \text{ and } \; \
r_{m,m} =  2 C_0 K^2 \sqrt{\frac{\log m}{m n}}; \\
\tau_B^{\dagger/2} 
& = & (\tau_B^{1/2} + \frac{3}{2}C_{6} r_{m,m}^{1/2}) \; \; \text{
  and } \; \; \tau_B^{\ddagger} \asymp 2 \tau_B+ 3 C_{6}^2 r_{m,m}.
\een
\begin{theorem}
\label{thm::DSoracle}
Let $D_0 = \sqrt{\tau_B} + a_{\max}^{1/2}$, and  $D_0', D_{\ora}$ be as defined in~\eqref{eq::defineD0}.
Let  $C_{6} \ge D_{\ora}$. Let $\rho_n$ and $r_{m,m}$ be as defined in~\eqref{eq::rhon}.
Suppose all conditions in Theorem~\ref{thm::DS} hold, except that
we replace the condition on $d$ as in~\eqref{eq::sqrt-sparsity} with
the following.

Suppose that the sample size $n$ and the size of the support of
$\beta^*$ satisfy the following requirements:
\ben
\label{eq::ora-sparsity}
d_0 & = & O \left(\tau_B^-\sqrt{\frac{n}{\log m}} \right), 
\; \;  \text{ where } \; 
\tau_B^- \le \inv{\tau_B^{1/2} + 2C_{6} r_{m,m}^{1/2}}, \\
\label{eq::samplebound}
\; \; \text{ and } \; \; 
n & \geq & \frac{2000 d K^4}{\delta^2} \log \left(\frac{60 e m}{d \delta}\right), \; \text{where } \\
\label{eq::sparse-dim-Ahalforacle}
d & = & 2d_0 + 2d_0 a_{\max} \frac{16 K^2(2d_0, 3k_0, A^{1/2}) (3k_0)^2
  (3k_0 + 1)}{\delta^2}.
\een
Let $\hat\tau_B$ be as defined in defined in \eqref{eq::trBest}.  
Let $\hat\beta$ be an optimal solution to the Conic programming estimator
as in \eqref{eq::Conic} with input $(\hat\gamma, \hat\Gamma)$ as
defined in~\eqref{eq::hatGamma}.
Suppose
\ben
\label{eq::muchoice}
\omega & \asymp & D_0 M_{\e}  \rho_{n} \; \;
\; \text{ and } \quad \mu \; \asymp \;  D_0' \tilde\tau_B^{1/2} K \rn,
\\
\nonumber
&& \text{ where} \; \; \tilde\tau_B^{1/2}  := \PaulBhalf+ C_{6} r_{m,m}^{1/2}.
\een
Then with probability at least $1-\frac{c''}{m^2} - 2 \exp(-\delta^2
n/2000 K^4)$, 
\ben
\label{eq::ellqnormimp}
\text{for } \; \; 2 \ge q \ge 1, \; \; \; 
\norm{\hat{\beta} -\beta^*}_q  \le  C' D_0' K^2 d_0^{1/q} \sqrt{\frac{\log m}{n}} 
\left(\tau_B^{\dagger/2} \twonorm{\beta^*} + \frac{M_{\e}}{K}\right);
\een
Under the same assumptions, the predictive risk admits the following bound
\bens
\onen \twonorm{X (\hat{\beta} -\beta^*)}^2 \le 
C'' (\twonorm{B} + a_{\max}) K^2 d_0 \frac{\log m}{n} \left(\tau_B^{\ddagger} 
 K^2 \twonorm{\beta^*}^2 + M_{\e}^2\right),
\eens
with the same probability as above, where $c'', C', C'' > 0$ are some
absolute constants.
\end{theorem}
We give an outline for the proof of Theorem~\ref{thm::DSoracle} in
Section~\ref{sec::DSoracle}, and show the actual proof in
Section~\ref{sec::DSoraproof}.

\subsection{Oracle results on the Lasso-type estimator}
 We now discuss the improvement  being made in
 Theorem~\ref{thm::lassora} and Theorem~\ref{thm::DSoracle}.

\noindent{\bf The Signal-to-noise ratio.} 
Let us redefine the Signal-to-noise ratio by 
\bens
\SMR  &:= & 
\frac{K^2 \twonorm{\beta^*}^2}{\tau_B^+ K^2 \twonorm{\beta^*}^2 +
  M^{2}_{\e} }, \; \; \text{ where } \;\\
\signal & := & K^2 \twonorm{\beta^*}^2 \; \text{ and
}  \; \; \bignoise := M^{2}_{\e} + \tau_B^+ K^2 \twonorm{\beta^*}^2.
\eens
When either the noise level $M_{\e}$ or the measurement error strength 
in terms of $\tau_B^{+/2} K \twonorm{\beta^*}$ increases, 
we need to increase the penalty parameter $\lambda$ correspondingly;
moreover, when $d \asymp \inv{M_A^2} \frac{n}{\log m}$, we have
\bens
\frac{\twonorm{\hat{\beta} -\beta^*}}{\twonorm{\beta^*} }
= O_P\left(D_0' K^2 \sqrt{\frac{\bignoise}{\signal}}
\inv{\vp(s_0+1)}\right),
\eens
which eventually becomes a vacuous bound when $\bignoise \gg \signal$.

Finally, suppose $B =\sigma_w^2 I$, we have 
 $\twonorm{B}^{1/2} = \sigma_w$ and $\tau_B = \sigma_w^2$.
In this setting, we recover essentially the same $\ell_2$ error bound
as that in Corollary 1 of~\cite{LW12} in case $\twonorm{\beta^*} \asymp
1$, as we have on event $\A_0 \cap \B_0$,
 \ben
\label{eq::relative}
\twonorm{\hat{\beta} -\beta^*}
\leq 
\frac{C (\sigma_w + a_{\max}^{1/2})}{\lambda_{\min}(A)} 
 \sqrt{\sigma_{\e}^2+ \sigma_w^2 \twonorm{\beta^*}^2} \sqrt{\frac{d \log m}{n}}
\een
where $\sigma_{\e}^2 \asymp M_{\e}^2$ and $K^2 \asymp 1$. 
However, when $\twonorm{\beta^*} =\Omega(1)$, our statistical precision appears to be  sharper as we allow the term
$\twonorm{\beta^*}$ to be  removed entirely
 from the RHS when $\sigma_w
\to 0$ and hence recover the regular Lasso rate of convergence.

\noindent{\bf The penalization parameter.}
We focus now on the penalization parameter $\lambda$ in \eqref{eq::origin}. 
The effective rank condition  in \eqref{eq::trBLasso} implies that for
$n = O(m \log m)$
\ben
\twonorm{B} \le \frac{\tau_B}{16c' K^4} \frac{\log m}{\log (3eM_A^3/2)
  + \log ( m \log m) - \log n} \le  C_B \tau_B \log m
\een
where $C_B = \inv{16c' K^4 \log (3eM_A^3/2)}$ given that  
$\log (m \log m) - \log n > 0$. This bound is very crude given that in
practice,  we focus on cases where $n \ll m \log m$.
Note that under (A1) (A2) and (A3), we have for $n =O(m \log m)$,
\bens
\tau_B^+ & \asymp &  \tau_B + \frac{\twonorm{A} +
   \twonorm{B}}{m} \\
& \le &   \tau_B + \inv{m}(\kappa(A) \lambda_{\min}(A)+ C_B \tau_B \log m) \asymp
 \tau_B + O\left(\frac{\lambda_{\min}(A)}{\sqrt{m}}\right).
\eens
Without knowing $\tau_B$, we will use $\hat\tau_B$ as defined in \eqref{eq::trBest}. 
Notice that we know neither $D_0'$ nor $D_{\ora}$ in the definition of
$\lambda$, where $D_{\ora}^2 \asymp  D_2$; Indeed,
\bens
2 D_2 \le D^2_{\ora} \le 4 D_2.
\eens
However, assuming that we normalize 
the column norms of the design matrix $X$ to be roughly at the same scale,
we have  for $\tau_B =O(1)$ and $m$ sufficiently large,
\bens
D_0' \asymp 1 \; \; \text{ while } \; \; D_{\ora}/\sqrt{m} = o(1) \;
\; \text{in case} \; \; \twonorm{A}, \twonorm{B} \le M
\eens
for some large enough constant $M$. In summary, 
compared to Theorem~\ref{thm::lasso}, in $\psi$, we replace 
$D_2 =2 (\twonorm{A} + \twonorm{B})$ with
$D'_0 :={\twonorm{B}}^{1/2} + a_{\max}^{1/2}$ so that the dependency on
$\twonorm{A}$ becomes much weaker.
As mentioned in Section \ref{sec::discusslasso}, we may use the
plug-in estimate $\shnorm{\hat{B}}$ in $D_0'$, where $\hat{B}$ is  the corrected gram
matrix $\inv{m}X X^T - \frac{\tr(A)}{m} I_m$.
Finally, the concentration of measure bound for the estimator  $\hat\tau_B$ as in~\eqref{eq::trBest} is 
stated in Lemma~\ref{lemma::trBest}, which ensures that $\hat\tau_B$
is indeed a good proxy for $\tau_B$ (cf. Lemma~\ref{lemma::tauB}).

\noindent{\bf The sparsity parameter.}
The condition on $d$ (and  $D_{\phi}$) for the Lasso estimator as defined in~\eqref{eq::doracle} 
suggests that as $\tau_B \to 0$, and thus $\tau^+_B \to 0$,  
the constraint on the sparsity parameter $d$ becomes slightly more
stringent when $K^2 M_{\e}^2/b_0^2 \asymp 1$ and much more restrictive when
$K^2 M_{\e}^2/b_0^2 = o(1)$.
Moreover, suppose we require
\bens
M_{\e}^2 = \Omega(\tau_B^+ K^2 \twonorm{\beta^*}^2),
\eens 
that is, the stochastic error $\e$ in the response variable $y$ as in~\eqref{eq::oby} 
does not converge to $0$ as quickly as the measurement error $W$ in
\eqref{eq::obX} does, then the sparsity constraint becomes essentially 
unchanged as $\tau_B^+ \to 0$ as we show now.
\bnum
\item[Case 1.]
Suppose $\tau_B \to 0$ and  $M_{\e} = \Omega(\tau_B^{+/2} K
\twonorm{\beta^*})$.
In this case, essentially, we require that 
\ben
\label{eq::dupper}
&&
d  \le  \frac{c_0 \lambda^2_{\min}(A)}{\vp^2(s_0+1)} \frac{n}{\log m}
\left\{\frac{c' c''K^2 M^2_{\e}}{b_0^2} 
  \wedge 1 \right\} \\
&& 
\nonumber
\text{ where }\;  D_{\phi}  \asymp \frac{K^2  M^2_{\e}}{b_0^2} 
\; \; \text{ given that } \; \; 
\tau_B^+ K^4 \phi \le  \frac{\tau_B^+ K^4 \twonorm{\beta^*}^2}{b_0^2} \ll \frac{K^2 M^2_{\e}}{b_0^2}
\een
where $c_0, c'$ are absolute constants and $c'' := \frac{\twonorm{B} + a_{\max}}{\vp^2(s_0+1)} \asymp
1$ where $\vp(s_0+1) = \rho_{\max}(s_0+1, A) + \tau_B$.
In this case,  the sparsity constraint becomes essentially 
unchanged as $\tau_B^+ \to 0$. 
\item[Case 2.]
Analogous to~\eqref{eq::dlasso}, when $M_{\e}^2 \le  \tau_B^+ \phi K^2
b_0^2$,
we could represent the condition on $d$ as follows:
\bens
d & \le & C_A c' c'' \tau_B^+ K^4  \phi \frac{n}{\log m}
\le  C_A c' c'' D_{\phi}\frac{n}{\log m}
\eens
which is sufficient for~\eqref{eq::doracle} to hold for $\tau_B \to
0$;
Indeed, by assumption that $c' K^4 \le 1$  and $M_{\e}^2 \le  \tau_B^+
\phi K^2 b_0^2$, we have
\bens
8 > 2 c'  K^4  \tau_B^{+} \phi \ge c' D_{\phi} 
 \asymp c' \tau_B^{+}  K^4  \phi.
\eens
Hence, for $c' \tau_B^{+} K^4  \le 1$, we have
\bens
d  & \le  & C_A (c' c'' \tau_B^+ K^4 \phi \wedge 8)\frac{n}{\log m} 
\asymp C_A c'' (c' \tau_B^+ K^4 \phi \wedge 8)\frac{n}{\log m} \\
& \le & 
 C_A  c'' c' \tau_B^+ K^4 \phi \frac{n}{\log  m}
\asymp C_A  c'' c' D_{\phi} \frac{n}{\log  m}
\eens
This condition, however, seems to be unnecessarily strong, when
$\tau_B \to 0$ (and $M_{\e} \to 0$ simultaneously).
We focus on the following Case 2 in the present work.
\enum
For both cases, it is clear that sample size needs to satisfy
\ben
\label{eq::statsample}
n = \tilde\Omega\left(d \log m \frac{(\rho_{\max}(s_0+1, A) +
    \tau_B)^4}{\lambda_{\min}(A)^2 (\twonorm{B} + a_{\max})}\right),
\een
where $\tilde\Omega(\cdot)$ notation hides 
parameters $K, M_{\e}, \phi$ and $b_0$, which we treat as absolute
constants that do not change as $\tau_B \to 0$.
These tradeoffs are somehow different from the behavior of the Conic 
programming estimator (cf \eqref{eq::ora-sparsity-rem}).
We will provide a more detailed analysis in
Sections~\ref{sec::stoc} and \ref{sec::lassooracle}.

\subsection{Oracle results on the Conic programming estimator}
In order to exploit the oracle bound as stated in  
Theorem~\ref{thm::oracle}  regarding $\norm{\hat\gamma -  \hat\Gamma
  \beta^*}_{\infty}$, we need to know the noise level  $\tau_B :=
\tr(B)/n$ in  $W$ and then we can set
\bens
\mu  & \asymp & D_0' (\tau_B^{1/2}  +\frac{ D_{\ora}}{\sqrt{m}}) K  \rho_{n}
\; \; \text{ while retaining} \; \; \omega \asymp D_0 M_{\e} \rho_{n},
\\ 
&& 
\; \; \text{ where recall} \quad \rho_{n} = C_0 K \sqrt{\frac{\log
    m}{n}} 
\; \; \text{ and } \; \; D_0 =\sqrt{\tau_B} + \sqrt{a_{\max}}.
\eens
This will  in turn lead to improved bounds in
Theorems~\ref{thm::lassora} and~\ref{thm::DSoracle}.

\noindent{\bf The penalization parameter.}
Without knowing the parameter $\tau_B$, we
rely on the estimate from $\hat\tau_B$ as in~\eqref{eq::trBest}, as discussed in Section~\ref{sec::tworesults}.
For a chosen parameter $C_{6} \asymp D_{\ora}$, we use $\PaulBhalf +
C_{6} r_{m,m}^{1/2}$ to replace $\tau_B^{+/2} := \tau_B^{1/2}  +
D_{\ora}/\sqrt{m}$ and set
\bens
\mu & \asymp & 
C_0 D_0'  K^2  (\PaulBhalf + D_{\ora} r^{1/2}_{m,m})
\sqrt{\frac{\log m}{n}} 
\eens
in view of Corollary~\ref{coro::D2improv}, where an improved error bound over 
$\norm{\hat\gamma -   \hat\Gamma \beta^*}_{\infty}$ is stated.
Without knowing $D_{\ora}$, we could replace it with an upper bound;
for example, assuming that 
$D^2_{\ora} \asymp  \twonorm{A} + \twonorm{B} =
O\left(\sqrt{\frac{n}{\log m}}\right)$, we could set
\bens
\mu \asymp C_0 D_0' K^2 (\hat\tau_B^{1/2} + O(m^{-1/4})) \sqrt{\frac{\log m}{n}}.
\eens
\noindent{\bf The sparsity parameter.}
Roughly speaking, for the Conic programming estimator~\eqref{eq::Conic}, one can think of $d_0$ as being bounded: 
\ben
\label{eq::ora-sparsity-rem}
d_0 & = & O \left(\tau_B^-\sqrt{\frac{n}{\log m}}  \bigwedge
  \frac{n}{ \log(m/d_0) }\right)
\; \;  \text{ where } \; 
\tau_B^- \asymp \tau_B^{-1/2}
\een
That is,  when $\tau_B$ decreases, we allow larger values of $d_0$; however, when 
$\tau_B \to 0$, the sparsity level of $d = O\left(n/ \log (m/d)\right)$ starts to 
dominate, which enables the Conic programming estimator to achieve results similar
to the Dantzig Selector when the design matrix $X_0$ is a subgaussian 
random matrix satisfying the Restricted Eigenvalue conditions;
See for example~\cite{CT07,BRT09,RZ13}.

In particular, when $\tau_B \to 0$, Theorem~\ref{thm::DSoracle} 
allows us to recover a rate close to that of the Dantzig selector with an exact 
recovery if $\tau_B =0$ is known a priori; see Section~\ref{sec::conclude}.
Moreover the constraint~\eqref{eq::sqrt-sparsity} on the sparsity
parameter $d_0$ appearing in Theorem~\ref{thm::DS} can now be
relaxed as in \eqref{eq::ora-sparsity}. 
In summary, our results in Theorem~\ref{thm::DSoracle} are
stronger than those in~\cite{BRT14} (cf. Corollary~1) as their rates
as stated therein are at the same order as ours in Theorem~\ref{thm::DS}.
We illustrate this dependency on $\tau_B$ in Section \ref{sec::exp}
with numerical examples, where we clearly show an advantage by taking
the noise level into consideration when choosing the penalty
parameters for both the Lasso and the Conic programming estimators.

\section{Optimization error on the gradient descent algorithm}
\label{sec::opt}
We now present our computational convergence bounds. First we present  
Lemma~\ref{lemma::reclaim} regarding the RSC and RSM conditions on the loss 
function~\eqref{eq::origin}.
Lemma~\ref{lemma::reclaim} follows from Lemma~\ref{lemma::lowerREI} immediately.
\begin{lemma}
\label{lemma::reclaim}
Suppose all conditions as stated in Theorem~\ref{thm::lasso} hold. 
Suppose event $\A_0$ holds. Then~\eqref{eq::lowerRSC} and \eqref{eq::upperRSC} hold with
$\alpha_{\ell}  = \frac{5}{8} \lambda_{\min}(A)$, $\alpha_u =  \frac{11}{8} \lambda_{\max}(A)$  and
\ben
\label{eq::taup}
\tau_{\ell}(\loss)  = \tau_{u}(\loss) = \tau_0 \frac{\log m}{n},
\quad \text{ where } \quad
\tau_0 \asymp \frac{400 C^2 \vp(s_0  +  1)^2}{\lambda_{\min}(A)}.
\een
\end{lemma}

\begin{theorem}
\label{thm::corrlinear}
Suppose all conditions in Theorem~\ref{thm::lassora} hold and let
$\psi$ be defined therein.
Let $g(\beta) = \inv{\lambda} \rho_{\lambda}(\beta)$ where 
$\rho_{\lambda}(\beta) = \lambda \onenorm{\beta}$.
Consider the optimization program~\eqref{eq::origin2}
for a radius $R$ such that $\beta^*$ is feasible and a regularization parameter 
chosen such that
\ben
\label{eq::labound}
\lambda & \ge & \left(\frac{16 R \xi}{1-\kappa} \right)
 \bigvee\left(12 \psi \sqrt{\frac{\log m}{n}}\right).
\een
Suppose that the step size parameter $\zeta \ge \au \asymp
\frac{3}{2}\lambda_{\max}(A)$. 
Suppose that the sparsity parameter and sample size further satisfy
the following relationship: 
\ben
\label{eq::dupperthm}
d & < &  \frac{n }{512\tau_0\log m}\left(\frac{\lambda_{\min}(A)^2}{12 \lambda_{\max}(A)} 
\bigwedge \frac {(\alpha_{\ell})^2}{5\zeta}\right) =: \bar{d}.
\een
Then on event $\A_0 \cap \B_0$, 
the conclusions in Theorem~\ref{thm::opt-linear} hold, where 
$$\prob{\A_0 \cap \B_0} \ge 1- 4\exp\left(-\frac{c_3 n}{M_A^2 \log m}  \log\left(\frac{\V m \log
      m}{n}\right)\right) -2\exp\left(- \frac{4c_2  n}{M_A^2 K^4 }\right) - 22/m^3.$$
\end{theorem}

\begin{corollary}
\label{coro::deer}
Suppose all conditions as stated in Theorem~\ref{thm::corrlinear} hold
and event $\A_0 \cap \B_0$ defined therein holds.
Consider for some constant $M \le 400 \tau_0$ and $\bar{\delta}^2$ as
defined in Theorem~\ref{thm::opt-linear},
\bens
\delta^2 \asymp  \frac{ c \ve_{\static}^2}{1-\kappa}
\frac{d \log m}{n} =:\bar{\delta}^2 \; \; \text{ and } \; \; \delta^2 \le  M \bar{\delta}^2  \le 400 \tau_0 \bar{\delta}^2.
\eens
Then for all $t \ge T^*(\delta)$ as in~\eqref{eq::Tstar} and $R =\Omega(b_0 \sqrt{d})$,
\ben
\label{eq::errort2}
\twonorm{\beta^t - \hat\beta}^2 
& \le &  \frac{3}{\alpha_{\ell}} \delta^2 + \frac{\alpha_{\ell}}{4}  \ve_{\static}^2 + 
O\left(\frac{\delta^2 \ve_{\static}^2}{b_0^2}
\right).
\een
Finally,  suppose we fix for $M_+ = \frac{32 C \vp(s_0+1)}{\lambda_{\min}(A)}$,
\bens
R \asymp \sqrt{\bar{d}} b_0 \asymp \frac{b_0}{20 M_+ \sqrt{6 \kappa(A)} }\sqrt{\frac{n}{\log  m}},
\eens
in view of the upper bound $\bar{d}$~\eqref{eq::dupperthm}.
Then for all $t \ge T^*(\delta)$ as in~\eqref{eq::Tstar}, 
\ben
\label{eq::error3}
\twonorm{\beta^t - \hat\beta}^2 
& \le &  
\frac{3}{\alpha_{\ell}}\delta^2 + \frac{\alpha_{\ell}}{4}  \ve_{\static}^2
+ \frac{2}{\alpha_{\ell}}\frac{\delta^4}{b_0^2 \twonorm{A}}.
\een
\end{corollary}
We prove Theorem~\ref{thm::corrlinear} and Corollary~\ref{coro::deer}
in Section~\ref{sec::optproof}.

\subsection{Discussions}
\label{sec::optdiscuss}
Throughout this section, we assume $\psi$ \eqref{eq::psijune15} is as
defined in Theorem~\ref{thm::lassora}.
Assume that $\zeta \ge \au \ge \bal$.
In addition, suppose that the radius $R \asymp b_0 \sqrt{d}$ as we set
in~\eqref{eq::origin}.
Let $\bar{d}_0 \le \frac{n}{160 M_+^2 \log m}$ be as defined in
\eqref{eq::doracle}, where recall that we require the following condition on $d$:
\bens d & \le &   C_A \left\{c' C_{\phi} \wedge 8 \right\}
\frac{n}{\log m} =: \bar{d}_0, \; \; \text{ where }\; C_A = \inv{160
  M_+^2}, \\
C_{\phi} &  = & \frac{\twonorm{B} + a_{\max} }{\vp(s_0+1)^2}  D_{\phi}
\; \; \text{ and } \; \; b_0^2 \ge \twonorm{\beta^*}^2 \ge \phi b_0^2. 
\eens
Then by  the proof of Lemma~\ref{lemma::dmainoracle}, 
\ben
\label{eq::d0}
b_0 \sqrt{\bar{d}_0}  & \le &  \frac{5 s_0}{3 \alpha} \sqrt{\frac{\log m}{n}}
\psi =: \frac{\psi}{\tau}\sqrt{\frac{\log m}{n}},
 \quad \text{ where }\quad \tau = \frac{3\alpha}{5s_0}.
\een
In contrast, under~\eqref{eq::dupperthm}, the following upper bound
holds on $d$, which is slightly more restrictive in the sense that the maximum level
of sparsity allowed on $\beta^*$ has decreased by a factor
proportional to $\kappa(A)$ compared to the upper bound
$\bar{d}_0$~\eqref{eq::doracle}  in Theorem~\ref{thm::lassora};
Now we require that $\size{\supp(\beta^*)} \le \bar{d}$,
where for $C_A = \inv{160 M_+^2}$,
\ben
\label{eq::bard}
\bar{d} & \asymp & 
 \frac{n \lambda_{\min}(A)^2}{1024 C^2 \vp(s_0  +  1)^2\log m}
\frac{1}{2400 \kappa(A)}  \\
\nonumber
& \approx & 
C_A \frac{n}{\log m} \left(\frac{\lambda_{\min}(A)}{15\lambda_{\max}(A)} \right) \asymp
\bar{d}_0 \inv{ \kappa(A)}.
\een
To consider the general cases as stated in Theorem~\ref{thm::lassora},
we consider the ideal case when we set 
$$\zeta = \au = \frac{11}{8}\lambda_{\max}(A)$$ such that 
$$\frac{\zeta}{\bal} \approx \au /(\frac{59}{60} \al) \asymp
\kappa(A), \; \; \text{ where} \; \; \al = \frac{5}{8}\lambda_{\min}(A).$$
Following the derivation in Remark~\ref{rem::upperxi}, we have
\ben
\label{eq::xibound}
\frac{\xi}{1-\kappa} \le 
6 \tau(\loss)  +\frac {80 \zeta}{\bal} \tau(\loss) \approx  
200 \kappa(A) \tau(\loss).
\een
Combining~\eqref{eq::d0} and \eqref{eq::xibound}, it is clear that
one can set
\ben
\label{eq::neighbor}
\lambda & = & \Omega\left( \kappa(A) \psi \sqrt{\frac{\log m}{n}} \right) 
\een
in order to satisfy the condition \eqref{eq::labound} on $\lambda$ in
Theorem~\ref{thm::opt-linear} when we set
\ben
\label{eq::Rbound}
R \asymp b_0 \sqrt{\bar{d}_0} &  = & O\left(\frac{\psi}{\tau}
  \sqrt{\frac{\log m}{n}}\right) \\
\nonumber
\text{ and hence } \quad  R \tau \kappa(A)  & = &  O\left(\kappa(A) \psi \sqrt{\frac{\log m}{n}}\right).
\een
This choice is potentially too conservative because we are setting $R$
in~\eqref{eq::Rbound} with respect to the upper sparsity level $\bar{d}_0$ chosen to guarantee statistical
convergence, leading to a larger than necessary penalty
parameter as in~\eqref{eq::neighbor}. 
Similarly, when we choose step size parameter $\zeta$ to be too large, 
we need to increase the penalty parameter $\lambda$ correspondingly
given the following lower bound: $\lambda = \Omega\left(\frac{R
    \xi}{1-\kappa} \right)$ where 
\bens
\frac{R \xi}{1-\kappa}
& = & R \left( 2\tau(\loss)\left(\frac{\frac{\bal}{4\zeta} + 2 \z}{\frac{\bal}{4 \zeta} - 2\z}
 +\frac{5}{\frac{\bal}{4 \zeta} - 2\z} \right) \right) \\
& \ge &  40 R \tau(\loss) \frac{\zeta}{\bal} + 2 R \tau(\loss)  \asymp
R \tau(\loss) \kappa(A).
\eens
Suppose we set $\zeta = \frac{3}{2}\lambda_{\max}(A)$ and $\frac{\zeta}{\bal}
\approx 3 \kappa(A)$ as in Theorem~\ref{thm::corrlinear}.
It turns out that the less conservative choice of $\lambda$ as
in~\eqref{eq::lambdaprelude} 
\ben
\label{eq::lambdaprelude}
\lambda &\asymp  & \left(b_0 \sqrt{\kappa(A)} \vp(s_0) \bigvee \psi
 \right) \sqrt{\frac{\log m}{n} }
\een
is sufficient, for example when $\tau_B = \Omega(1)$,
for which we now set 
\bens
R \asymp b_0 \sqrt{d} \asymp \frac{b_0}{20 M_{+}}\inv{\sqrt{6 \kappa(A)}} \sqrt{\frac{n}{\log m}}
\eens
as in Corollary~\ref{coro::deer}.  We will discuss the two scenarios as considered in
Section~\ref{sec::smallW}. See  the detailed discussions in
Section~\ref{sec::optproof}.

\section{Proof of theorems}
\label{sec::proofall}
In Section \ref{sec::stoc}, we develop in Theorem~\ref{thm::oracle}
the crucial large deviation bound on $\norm{\hat\gamma - \hat\Gamma
  \beta^*}$. 
This entity appears in the constraint set in the Conic programming estimator
\eqref{eq::Conic}, and is directly related to the choice of $\lambda$ 
for the corrected Lasso estimator in view of Theorem~\ref{thm::main}.
 Its corollaries are stated in Corollary~\ref{coro::low-noise}
 and Corollary~\ref{coro::D2improv}.
In section~\ref{sec::lassooutline}, we provide an outline and
additional Lemmas~\ref{lemma::lowerREI} and~\ref{lemma::dmain} to prove Theorem~\ref{thm::lasso}. 
The full proof of Theorem~\ref{thm::lasso} appears in
Section~\ref{sec::proofofthmlasso}.
In Section~\ref{sec::lassooracle}, 
we give an outline illustrating the improvement for the Lasso error bounds as stated in Theorem~\ref{thm::lassora}.
We emphasize the impact of this improvement over sparsity parameter
$d$, which we restate in Lemma~\ref{lemma::dmainoracle}.
In Section~\ref{sec::DSoutline}, we provide an outline
as well as technical results for Theorem~\ref{thm::DS}.
In Section~\ref{sec::DSoracle}, we give an outline illuminating the
improvement in error bounds for the Conic programming estimator
as stated in Theorem~\ref{thm::DSoracle}.

\subsection{Stochastic error terms}
\label{sec::stoc}
In this section, we first develop stochastic error bounds in
Lemma~\ref{lemma::Tclaim1}, where we also define some events $\B_4, \B_5,
\B_{10}$. Recall that  $\B_6$ was defined in
Lemma~\ref{lemma::trBest}.
Putting the bounds in Lemma~\ref{lemma::Tclaim1} 
together with that in Lemma~\ref{lemma::trBest} yields Theorem~\ref{thm::oracle}.
\begin{lemma}
\label{lemma::Tclaim1}
Assume that the stable rank of $B$, $\fnorm{B}^2/\twonorm{B}^2 \ge \log m$.
Let $Z, X_0$ and $W$ as defined in Theorem~\ref{thm::lasso}.
Let $Z_0, Z_1$ and  $Z_2$ be independent copies of $Z$.
Let $\e^T \sim Y  M_{\e}/K$ where $Y := e_1^T Z_0^T$.
Denote by $\B_4$ the event such that  for $\rho_{n} := C_0 K  \sqrt{\frac{\log m}{n}}$,
\bens
\onen \norm{A^{\half} Z_1^T \e}_{\infty}
& \le & \rho_{n} M_{\e} a_{\max}^{1/2} \\
 \; \text{ and } \; 
\onen \norm{ Z_2^T B^{\half} \e}_{\infty}
& \le &   \rho_{n} M_{\e}\sqrt{\tau_B} \;  \; \text{ where } \; \; \tau_B = \frac{\tr(B)}{n}.
 \eens
Then $\prob{\B_4} \ge  1 - 4/m^3$.
Moreover, denote by $\B_5$ the  event such that 
\bens
\onen \norm{(Z^T B Z- \tr(B) I_{m}) \beta^*}_{\infty} & \le &  
\rho_{n} K  \twonorm{\beta^*}  \frac{\fnorm{B}}{\sqrt{n}} \\
\text{and} \;\;\onen \norm{X_0^T W \beta^*}_{\infty}
& \le &   \rho_{n} K  \twonorm{\beta^*}  \sqrt{\tau_B} a^{1/2}_{\max}.
\eens
Then $\prob{\B_5} \ge 1 -  4 /m^3$.

Finally, denote by $\B_{10}$ the event such that
\bens
\onen \norm{(Z^T B Z- \tr(B) I_{m})}_{\max} & \le &  \rho_{n} K \frac{\fnorm{B}}{\sqrt{n}} \\
\text{and} \;\; \onen \norm{X_0^T W}_{\max}
& \le &  \rho_{n} K \sqrt{\tau_B} a^{1/2}_{\max}.
\eens
Then $\prob{\B_{10}} \ge 1 -  4/m^2$.
\end{lemma}
We prove  Lemma~\ref{lemma::Tclaim1}
in Section~\ref{sec::proofTclaim1}.
Denote by $\B_0 := B_4 \cap \B_5 \cap \B_6$, which we use throughout
this paper. 
\begin{theorem}
\label{thm::oracle}
Suppose (A1) holds.   Let $\rho_n = C_0 K \sqrt{\frac{\log m}{n}}$.
Suppose that
$$\fnorm{B}^2/\twonorm{B}^2 \ge \log m \; \; \text{ where } \; m \ge
16.$$
Let $\hat\Gamma$ and $\hat\gamma$ be as in~\eqref{eq::hatGamma}.
Let $D_0 =\sqrt{\tau_B} + \sqrt{a_{\max}}$ and $D_0'$ be as defined in \eqref{eq::defineD0}. 
Let  $D_1=\frac{\fnorm{A}  }{\sqrt{m} }+ \frac{\fnorm{B}}{\sqrt{n}}$.
On event $\B_0$, for which $\prob{\B_0} \ge 1-16/m^3$, 
\ben
\norm{\hat\gamma - \hat\Gamma \beta^*}_{\infty}
& \le &
\label{eq::oracle}
\left(D_0' K  \tau_B^{1/2} \twonorm{\beta^*} + \frac{2D_1 K}{\sqrt{m}}
  \norm{\beta^*}_{\infty} + D_0 M_{\e}\right) \rho_{n}.
\een
\end{theorem}

We next state the first  Corollary~\ref{coro::low-noise} of
Theorem~\ref{thm::oracle}, 
which we use in  proving Theorems~\ref{thm::lasso} and~\ref{thm::DS}.
Here we state a somewhat simplified bound on 
$\norm{\hat\gamma -  \hat\Gamma \beta^*}_{\infty}$ 
for the sake of reducing the number of unknown parameters involved with a slight worsening of the
statistical error bounds when $\tau_B \asymp 1$.
On the other hand, the bound in \eqref{eq::oracle} provides a significant
improvement over the error bound in Corollary~\ref{coro::low-noise} in
case $\tau_B =o(1)$. 

\begin{corollary}
\label{coro::low-noise}
Suppose all conditions in Theorem~\ref{thm::oracle} hold.
Let $\hat\Gamma$ and $\hat\gamma$ be as
in~\eqref{eq::hatGamma}. 
On event $\B_0$, we have for 
$D_2 = 2(\twonorm{A} +\twonorm{B})$ and some absolute constant $C_0$ 
\bens
&&
\norm{\hat\gamma - \hat\Gamma \beta^*}_{\infty}
\le  \psi \sqrt{\frac{\log m}{n}},  \quad \text{ where } \;
\psi   =    C_0 D_2 K \left(K \twonorm{\beta^*} +  M_{\e} \right)
\eens
is as defined in Theorem~\ref{thm::lasso}.
\end{corollary}

In particular, Corollary~\ref{coro::low-noise} ensures that for the
corrected Lasso estimator,
\eqref{eq::psimain} holds with high probability for $\lambda$  chosen as in~\eqref{eq::psijune}. 
We prove Corollary~\ref{coro::low-noise} in Section~\ref{sec::lownoiseproof}. 

\noindent{\bf What happens when $\tau_B \to 0$?}
Recall $D_0 = \sqrt{\tau_B} + a_{\max}^{1/2}$ and $D'_0 :=\sqrt{\twonorm{B}} + a_{\max}^{1/2}$.
When $\tau_B \to 0$, we have by Theorem~\ref{thm::oracle}
\bens
\norm{\hat\gamma - \hat\Gamma \beta^*}_{\infty}
& = &  
O\left(D_1 K\inv{\sqrt{m}}\norm{\beta^*}_{\infty} +  D_0 K M_{\e}\right)
K \sqrt{\frac{\log m}{n}}
\eens
where $D_0 \to a_{\max}^{1/2}$ and $D_1 = \frac{\fnorm{A}}{\sqrt{m}} +
\frac{\fnorm{B}}{\sqrt{n}} 
\to \twonorm{A}^{1/2}$  under (A1), given that $\fnorm{B}/
\sqrt{n} \le \tau_B^{1/2} \twonorm{B}^{1/2} \to 0$.
In this case,  the error term involving $\twonorm{\beta^*}$
in~\eqref{eq::psijune15}  vanishes, and we only need to set (cf. Theorem~\ref{thm::main})
\ben
\label{eq::psijune15lasso}
&& \lambda \ge 2 \psi \sqrt{\frac{\log m}{n}} \; \; \text{ for } \;\;
\psi  \asymp a_{\max}^{1/2} K M_{\e} + \twonorm{A}^{1/2} K^2 
\norm{\beta^*}_{\infty}/m^{1/2},
\een
where the second term in $\psi$ defined immediately above 
comes from the estimation error in Lemma~\ref{lemma::trBest}; this
term vanishes if we were to assume that (1) $\tr(B)$ is also known
or (2) $\norm{\beta^*}_{\infty} = o(M_{\e} m^{1/2}/K)$. 
For both cases, by setting $\lambda \asymp 4 a_{\max}^{1/2} K M_{\e} \sqrt{\frac{\log  m}{n}}$,
we can recover the regular Lasso rate of 
\bens
\norm{\hat{\beta} -\beta^*}_q =
O_p(\lambda d^{1/q}), \quad \text{ for } \; \;  q =1, 2,
\eens
when the design matrix $X$  is almost free of measurement errors.

Finally, we state a second Corollary~\ref{coro::D2improv} of
Theorem~\ref{thm::oracle}.
Corollary~\ref{coro::D2improv} is essentially a restatement of the bound
in~\eqref{eq::oracle}. 
\begin{corollary}
\label{coro::D2improv}
Suppose all conditions in Theorem~\ref{thm::oracle} hold.
Let $D_0, D_0', D_{\ora}$, and $\tau_B^{+/2} :=  \tau_B^{1/2}  +  \frac{D_{\ora}}{\sqrt{m}}$  
be as defined in~\eqref{eq::defineD0} and~\eqref{eq::defineDtau}.
On event $\B_0$,
\ben
\nonumber
\norm{\hat\gamma - \hat\Gamma \beta^*}_{\infty}
& \le &  \psi \sqrt{\frac{\log m}{n}},\quad \text{ where } \;  \\
\label{eq::psioracle} 
&& \psi:= C_0 K \left(D_0' \tau_B^{+/2} K \twonorm{\beta^*} + D_0  M_{\e}
\right).
\een
Then $\prob{\B_0} \ge 1- 16/m^3$.
\end{corollary}
We mention in passing that Corollaries~\ref{coro::low-noise}
and~\ref{coro::D2improv} are crucial in proving
Theorems~\ref{thm::lasso},~\ref{thm::DS},~\ref{thm::lassora}
and~\ref{thm::DSoracle}.

\subsection{Outline for proof of Theorem~\ref{thm::lasso}}
\label{sec::lassooutline}
In this section, we state
Theorem~\ref{thm::main}, and two Lemmas~\ref{lemma::lowerREI} and~\ref{lemma::dmain}.
Theorem~\ref{thm::lasso} follows from Theorem~\ref{thm::main} in view
of Corollary~\ref{coro::low-noise},
 Lemmas~\ref{lemma::lowerREI} and~\ref{lemma::dmain}. 
In more details,  Lemma~\ref{lemma::lowerREI} checks the Lower and the 
Upper $\RE$ conditions on the modified gram matrix, 
\ben
\label{eq::modifiedgramA}
\hat \Gamma_A := \onen(  X^T X - \hat\tr(B) I_{m}),
\een
while Lemma~\ref{lemma::dmain} checks condition \eqref{eq::taumain} as
stated in Theorem~\ref{thm::main} for curvature $\alpha$ and tolerance
$\tau$ regarding  the lower $\RE$ condition as derived in
Lemma~\ref{lemma::lowerREI}. 

First, we replace (A3) with (A3') which reveals some additional information
regarding the constant hidden inside the $O(\cdot)$ notation.
\bnum
\item[(A3')]
Suppose (A3) holds; moreover, 
$m n \ge 4096 C_0^2 D_2^2 K^4 \log
m/{\lambda_{\min}^2(A)}$ for $D_2 = 2(\twonorm{A} +
\twonorm{B})$,  or equivalently,
\bens
\frac{\lambda_{\min}(A)}{\twonorm{A} + \twonorm{B}}
> C_K \sqrt{\frac{\log m}{m n}}  \text{ for some large enough contant $C_K$}.
\eens
\enum
\begin{lemma}{{\bf(Lower and Upper-RE conditions)}}
\label{lemma::lowerREI} 
Suppose (A1), (A2) and (A3') hold. 
Denote by $\V := 3eM_A^3 /2$, where $M_A$ is as defined
in~\eqref{eq::defineM}. 
Let $s_0 \ge 32$  be as defined in \eqref{eq::s0cond}.
Recall that we denote by $\A_0$ the event that the modified gram
matrix $\hat\Gamma$ as defined in~\eqref{eq::hatGamma}
satisfies the Lower as well as Upper $\RE$ conditions with
\bens
\text{curvature} && 
\alpha = \frac{5}{8}\lambda_{\min}(A), \;\text{smoothness} \; \; \tilde\alpha
= \frac{11}{8}\lambda_{\max}(A) \\
\text{ and tolerance } && 
\frac{384 C^2 \vp(s_0)^2}{\lambda_{\min}(A)}\frac{\log m}{n}
 \le  \tau :=  \frac{\lambda_{\min}(A) - \alpha}{s_0}  
\le \frac{396 C^2 \vp^2(s_0+1)}{\lambda_{\min}(A)}\frac{\log m}{n}
\eens
for $\alpha, \tilde\alpha$ and $\tau$ as defined in Definitions~\ref{def::lowRE} and \ref{def::upRE}, and $C, s_0,
\vp(s_0)$ in~\eqref{eq::s0cond}. 
Suppose  that for some $c' > 0$ and $c' K^4 < 1$,
\ben
\label{eq::trBlem}
\frac{\tr(B)}{\twonorm{B}} 
& \ge & c' K^4 \frac{s_0}{\ve^2} \log\left(\frac{3 e m}{s_0
    \ve}\right)\; \; \;\text{ where } \; \ve =\inv{2 M_A}.
\een
Then 
$\prob{\A_0} \ge 1- 4\exp\left(-\frac{c_3 n}{M_A^2 \log m}
  \log\left(\frac{\V m \log m}{n}\right)\right) -2\exp\left(- \frac{4c_2 n}{M_A^2 K^4 }\right) - 6/m^3$.
\end{lemma}
The main focus of the current section is then to apply Theorem~\ref{thm::main} to show Theorem~\ref{thm::lasso}.
Theorem~\ref{thm::main} follows from Theorem 1 by~\cite{LW12}.

\begin{theorem}
\label{thm::main}
Consider the regression model in 
\eqref{eq::oby} and \eqref{eq::obX}.  Let $d \le n/2$.
Let $\hat\gamma, \hat\Gamma$ be as constructed in \eqref{eq::hatGamma}.
Suppose that the matrix $\hat\Gamma$ satisfies the Lower-$\RE$ condition with curvature $\alpha >0$ and
tolerance $\tau > 0$,
\ben
\label{eq::taumain}
\sqrt{d} \tau \le \min
 \left\{\frac{\alpha}{32 \sqrt{d}}, \frac{\lambda}{4b_0} \right\},
\een
where $d, b_0$ and $\lambda$ are as defined in \eqref{eq::origin}.
Then for any $d$-sparse vectors $\beta^* \in \R^m$, such that
$\twonorm{\beta^*} \le b_0$ and 
\ben
\label{eq::psimain}
&& \norm{\hat\gamma - \hat\Gamma \beta^*}_{\infty}
\le  \half \lambda, 
\een
the following bounds hold:
\ben
\label{eq::2-loss}
&& \twonorm{\hat{\beta} -\beta^*} \leq \frac{20}{\alpha}  \lambda
\sqrt{d} \; \;
\text{ and } \; \norm{\hat{\beta} -\beta^*}_1 \leq \frac{80}{\alpha} \lambda d,
\een
where $\hat\beta$ is an optimal solution to the corrected Lasso
estimator as in~\eqref{eq::origin}.
\end{theorem}

We include the proof of Theorem~\ref{thm::main} for the sake of
self-containment and  defer it to Section~\ref{sec::proofofmain} for clarity of presentation.
\begin{lemma}
  \label{lemma::dmain}
Let $c', \phi, b_0, M_{\e}$,  $M_{+}$ and $K$ be as defined in
Theorem~\ref{thm::lasso}, where we assume that
$b_0^2 \ge \twonorm{\beta^*}^2 \ge \phi b_0^2  \; \text{ for some } \; 0 < \phi
\le 1$.
Suppose all conditions in Lemma~\ref{lemma::lowerREI} hold.
Suppose that $s_0 \ge 32$ and
\ben
\label{eq::dlassoproof}
&& d:= \abs{\supp(\beta^*)} \le C_A \frac{n}{\log m} \left\{c' D_\phi
  \wedge 2\right\} \;\;\text{ where }  C_A := \inv{40 M_{+}^2} \\
\nonumber
&\text{and} & D_{\phi} = K^4 \left(\frac{M^2_{\e}}{K^2 b_0^2}  + \phi \right) \ge
 K^4 \phi \ge \phi.
\een
Then the following condition holds
\ben
\label{eq::dcond} 
d \le  \frac{\alpha}{32 \tau} \bigwedge
\inv{\tau^2} \frac{\log m}{n} \left(\frac{\psi}{b_0}\right)^2,
\een
where $\psi = C_0 D_2 K (K \twonorm{\beta^*}+ M_{\e})$ is as defined in~\eqref{eq::psijune},
$\alpha = 5\lambda_{\min}(A)/8$, and $\tau$ is as defined in Lemma~\ref{lemma::lowerREI}.
\end{lemma}

We prove Lemmas~\ref{lemma::lowerREI} and \ref{lemma::dmain} in 
Sections~\ref{sec::records} and~\ref{sec::dmain} respectively.
Lemma~\ref{lemma::lowerREI} follows immediately from Corollary~\ref{coro::BC}.
We prove Lemmas~\ref{lemma::lowerREI} and Corollary~\ref{coro::BC} 
 in Sections~\ref{sec::records} and~\ref{sec::appendLURE}  respectively.  

\begin{remark}
Clearly for $d, b_0, \phi$ as bounded in Theorem~\ref{thm::lasso}, 
we have by assumption~\eqref{eq::snrcond}
the following upper and lower bound on $D_{\phi}$:
\bens
2 K^4 \phi \ge D_{\phi} := \left(\frac{M^2_{\e}K^2}{b_0^2}   +  K^4 \phi \right) \ge
K^4 \phi.
\eens
In this regime, the conditions on $d$ as in~\eqref{eq::dlassoproof}
can be conveniently expressed as that in \eqref{eq::dlasso} instead. 
\end{remark}

\subsection{Improved bounds for the corrected Lasso estimator}
\label{sec::lassooracle}
The proof of Theorem~\ref{thm::lassora} follows exactly the same
line of arguments as in Theorem~\ref{thm::lasso}, except that we now
use the improved bound on the error term 
$\norm{\hat\gamma -  \hat\Gamma \beta^*}_{\infty}$ given in Corollary~\ref{coro::D2improv},
instead of that in Corollary~\ref{coro::low-noise}.
Moreover, we replace Lemma~\ref{lemma::dmain} with Lemma~\ref{lemma::dmainoracle}, 
the proof of which follows from Lemma~\ref{lemma::dmain} with 
$d$ now being bounded as in \eqref{eq::doracle}  and $\psi$ being
redefined as in \eqref{eq::psioracle}.
The proof of Lemma~\ref{lemma::dmainoracle} appears in Section~\ref{sec::dmainoracle}.
See Section~\ref{sec::classoproof} for the proof of Theorem~\ref{thm::lassora}.

\begin{lemma}
\label{lemma::dmainoracle}
Let $c', \phi, b_0, M_{\e}$, $M_+$ and $K$ be as defined in
Theorem~\ref{thm::lasso}.
Suppose all conditions in Lemma~\ref{lemma::lowerREI} hold.
Suppose that \eqref{eq::doracle} holds:
\ben
\label{eq::doraclelocal}
&& d:= \abs{\supp(\beta^*)} \le C_A \frac{n}{\log m} \left\{c' c'' D_\phi
  \wedge 8 \right\}, \;\;\text{ where }  C_A := \inv{160 M_{+}^2}, \\
\nonumber
&& c'' := \frac{\twonorm{B} + a_{\max}}{\vp(s_0+1)^2} \le
\left(\frac{D_0'}{\vp(s_0+1)}\right)^2 \; \; \text{ and } \;  
D_{\phi} = \frac{K^2M^2_{\e}}{b_0^2}   +  \tau_B^+ K^4 \phi.
\een
Then~\eqref{eq::dcond}  holds with $\psi$ as defined in Theorem~\ref{thm::lassora} and 
$\alpha = \frac{5}{8}\lambda_{\min}(A)$.
\end{lemma}

\subsection{Outline for proof of Theorem~\ref{thm::DS} }
\label{sec::DSoutline}
We provide an outline and state the technical lemmas needed for proving Theorem~\ref{thm::DS}.
Our first goal is to show that the following holds with high probability,
\bens
\norm{\hat\gamma - \hat\Gamma  \beta^*}_{\infty} 
& = & 
\norm{\onen X^T(y - X \beta^*) + \onen \hat\tr(B) \beta^*}_{\infty} \;
\le \; \mu \twonorm{\beta^*}  + \omega,
\eens
where  $\mu, \omega$ are chosen as in \eqref{eq::paraDS}.
This forms the basis for proving the $\ell_q$ convergence, where $q
\in [1, 2]$,  for the Conic programming estimator \eqref{eq::Conic}.
This follows immediately from Theorem~\ref{thm::oracle} and  
Corollary~\ref{coro::low-noise}. More
explicitly, we will state it in Lemma~\ref{lemma::DS}.

\begin{lemma}
\label{lemma::DS}
Let $D_0 = \sqrt{\tau_B} + \sqrt{a_{\max}}$  and $D_2 = 2 (\twonorm{A} + \twonorm{B})$ be as in Theorem~\ref{thm::DS}.
Suppose all conditions in Theorem~\ref{thm::oracle} hold.
Then on event $\B_0$ as defined therein,
the pair $(\beta, t) = (\beta^*, \twonorm{\beta^*})$ 
belongs to the feasible set of the minimization problem
\eqref{eq::Conic} with
\ben
\label{eq::paraDS}
\mu \asymp 2 D_2 K \rho_{n} \; \; \text{ and } \; \;  \omega \asymp
D_0 M_{\e} \rho_{n}, 
\; \quad \text{ where} \quad \rho_{n} := C_0 K  \sqrt{\frac{\log m}{n}}.
\een
\end{lemma}

Before we proceed, we first need to introduce some notation and definitions.
Let $X_0 = Z_1 A^{1/2}$ be defined as in~\eqref{eq::subgdata}. 
Let $k_0 = 1+\lambda$. 
First we need to define the $\ell_q$-sensitivity parameter for $\Psi
:= \onen X_0^T X_0$ following~\cite{BRT14}: 
\ben
\label{eq::sense}
\kappa_{q}(d_0, k_0) 
& = & \min_{J : \abs{J} \le d_0} \min_{\Delta \in \Cone_J(k_0)}
\frac{\norm{\Psi \Delta}_{\infty}}{\norm{\Delta}_q}, \; \; \text{ where
} \; \\
\W_J(k_0) &  = & \left\{x \in \R^m \;| \; \mbox{ s.t. } \; \norm{x_{J^c}}_1 \leq k_0 \norm{x_{J}}_1 \right\}.
\een
See also~\cite{GT11}.
Let $(\hat\beta, \hat{t})$ be the optimal solution
to~\eqref{eq::Conic} and denote by $v = \hat\beta-\beta^*$.
We will state the following auxiliary lemmas, the first of which is 
deterministic in nature.
The two lemmas reflect the two geometrical constraints 
on the optimal solution to \eqref{eq::Conic}. 
The optimal solution $\hat\beta$ satisfies:
\bnum
\item
The vector $v$ obeys the following cone constraint:
$\onenorm{v_{\Sc}} \le   k_0 \onenorm{v_S}$, and $\hat{t} \le
\inv{\lambda} \onenorm{v} + \twonorm{\beta^*}$.
\item
$\norm{\Psi v}_{\infty}$ is upper bounded by a 
quantity at the order of $O\left(\mu (\twonorm{\beta^*} + \onenorm{v}) + \omega\right)$.
\enum
\begin{lemma}
\label{lemma::DS-cone}
Let $\mu, \omega >0$ be set.
Suppose that the pair $(\beta, t) = (\beta^*, \twonorm{\beta^*})$ belongs to 
the feasible set of the minimization problem \eqref{eq::Conic}, for which
$(\hat\beta, \hat{t})$ is an optimal solution.
Denote by $v = \hat\beta-\beta^*$. Then
\bens
\onenorm{v_{\Sc}} & \le &  (1+\lambda) \onenorm{v_S} \; \text{ and }
\; \hat{t} \; \le \; \inv{\lambda}\onenorm{v} + \twonorm{\beta^*}.
\eens
\end{lemma}

\begin{lemma}
\label{lemma::grammatrix}
On event $\B_0\cap \B_{10}$, 
\bens
\norm{\Psi v}_{\infty} \le \mu_1 \twonorm{\beta^*} + \mu_2
\onenorm{v} + \omega',
\eens
where $\mu_1 = 2 \mu$, $\mu_2  = \mu(\inv{\lambda} + 1)$ and $\omega' = 2 \omega$
for $\mu, \omega$ as defined in \eqref{eq::paraDS}.
\end{lemma}
Now combining Lemma 6 of~\cite{BRT14} and an earlier result of the two
authors (cf. Theorem~\ref{thm:subgaussian-T-intro}~\cite{RZ13}), we can show that 
 the $\RE(2d_0,  3(1+ \lambda), A^{1/2})$ condition and the sample requirement as in~\eqref{eq::samplebound}
 are enough to ensure that the $\ell_q$-sensitivity parameter satisfies the following
lower bound for all $1\le q \le 2$: for some contant $c$,
\bens
\kappa_{q}(d_0, k_0) & \ge & c d_0^{-1/q},
\eens
which ensures that for $v =  \hat\beta-\beta^*$ and $\Psi = \onen X_0^T X_0$,
\ben
\label{eq::lqcond}
{\norm{\Psi v}_{\infty}}  &\ge & \kappa_{q}(d_0,
k_0) {\norm{v}_q} \ge  c d_0^{-1/q} \norm{v}_q.
\een
Combining \eqref{eq::lqcond} with Lemmas \ref{lemma::DS},~\ref{lemma::DS-cone} 
and~\ref{lemma::grammatrix} gives us both the lower and upper bounds on 
$\norm{\Psi v}_{\infty}$, with the lower bound being 
$\kappa_{q}(d_0, k_0) \norm{v}_q$ and the upper bound as specified in Lemma~\ref{lemma::grammatrix}.
Following some algebraic manipulation, this yields the bound on the
$\norm{v}_q$ for all $1\le q \le 2$. 
We prove Theorem~\ref{thm::DS} in Section~\ref{sec::proofofDSthm} and Lemmas~\ref{lemma::DS},~\ref{lemma::DS-cone}
and~\ref{lemma::grammatrix} in Section~\ref{sec::proofofDSlemma}.
The proof of Lemma~\ref{lemma::DS-cone} follows  the same line of
arguments in~\cite{BRT14} in view of Lemma~\ref{lemma::DS}.

\subsection{Improved bounds for the DS-type estimator}
\label{sec::DSoracle}
Lemma~\ref{lemma::DSimprov} follows directly from Corollary~\ref{coro::D2improv}.
\begin{lemma}
\label{lemma::DSimprov}
Suppose all conditions in Corollary~\ref{coro::D2improv} hold.
Let $D_0 =  \sqrt{\tau_B} + \sqrt{a_{\max}} \asymp 1$ under (A1).
Then on event $\B_0$, the pair $(\beta, t) = (\beta^*, \twonorm{\beta^*})$ 
belongs to the feasible set $\U$ of the minimization problem
\eqref{eq::Conic} with
\ben
\label{eq::paraDSimprov}
\mu \ge 
D_0' \tau_B^{+/2} K \rho_{n}  \; \; \text{ and } \; \; \omega \ge D_0 M_{\e} \rho_{n},
\een
where  $\tau_B^{+/2} :=  \tau_B^{1/2}  +  \frac{D_{\ora}}{\sqrt{m}}$  is as defined in~\eqref{eq::defineDtau}.
\end{lemma}

\begin{lemma}
\label{lemma::tauB}
On event $\B_6$ and (A1), the choice of $\tilde\tau_B^{1/2}  := \PaulBhalf+ C_{6} r_{m,m}^{1/2}$ as in
\eqref{eq::muchoice}, where recall $r_{m,m} =  2 C_0 K^2
\sqrt{\frac{\log m}{m n}}$,  satisfies for $m \ge 16$ and $C_0 \ge 1$,
\ben
\label{eq::tildetauB}
\tau_B^{+/2}    &  \le &   \tilde\tau_B^{1/2} 
\le \tau_B^{1/2} + \frac{3}{2} C_{6} r_{m,m}^{1/2} =:
\tau^{\dagger/2}_B, \\
\label{eq::tildetauBbound}
\tilde\tau_B  & \le &  2 \tau_B +3 C_{6}^2 r_{m,m} \asymp \tau^{\ddagger}_B, \; \text{ and
  moreover  }\; \;  \tilde\tau_B^{1/2}  \tau_B^- \le  1.
\een
\end{lemma}
We next state an updated result in Lemma~\ref{lemma::grammatrixopt}.
\begin{lemma}
\label{lemma::grammatrixopt}
On event $\B_0\cap \B_{10}$,  the solution $\hat\beta$ to \eqref{eq::Conic} with
$\mu, \omega$ as in \eqref{eq::muchoice}
satisfies for $v := \hat\beta - \beta^*$
\bens
\norm{\onen X_0^T X_0 v}_{\infty} \le \mu_1 
\twonorm{\beta^*}+ \mu_2 \onenorm{v} + \omega',
\eens
where $\mu_1 = 2 \mu $, $\mu_2  =2 \mu (1+\inv{2 \lambda})$ and 
$\omega' = 2 \omega$.
\end{lemma}

\section{Lower and Upper RE conditions}
\label{sec::AD}
The goal of this section is to show that for $\Delta$ defined
in~\eqref{eq::defineAD}, the presumption in 
Lemmas~\ref{lemma::bigcone} and~\ref{lemma::bigconeII} as 
restated in \eqref{eq::Deltacond} holds with high probability (cf Theorem~\ref{thm::AD}). 
We first state a deterministic result showing that the Lower and Upper $\RE$
conditions hold for $\hat\Gamma_A$ under condition
\eqref{eq::Deltacond} in Corollary~\ref{coro::BC}.
This allows us to prove Lemma~\ref{lemma::lowerREI} in  Section~\ref{sec::records}.
See Sections~\ref{sec::geometry} and~\ref{sec::appendLURE},
where we show that Corollary~\ref{coro::BC} follows immediately from
the geometric analysis result as stated in Lemma~\ref{lemma::bigconeII}.
\begin{corollary}
\label{coro::BC}
Let $1/8 > \delta > 0$. Let $1 \le k < m/2$.
Let $A_{m \times m}$ be a symmetric positive semidefinite 
covariance matrice. Let $\hat\Gamma_A$ be an $m \times m$ symmetric
matrix and $\Delta = \hat\Gamma_A - A$.
Let $E=\cup_{|J| \leq k} E_J$, where $E_J = \spin\{e_j: j \in J\}$. 
Suppose that $\forall u, v \in E \cap S^{m-1}$
\ben
\label{eq::Deltacond}
\abs{u^T \Delta v} \le \delta \le \frac{3}{32}\lambda_{\min}(A).
\een
Then the Lower and Upper $\RE$ conditions hold: for all $\up \in \R^m$,
\ben
\label{eq::ADlow}
\up^T \hat\Gamma_A \up & \ge & \frac{5}{8} \lambda_{\min}(A)
 \twonorm{\up}^2 -\frac{3 \lambda_{\min}(A)}{8 k} \onenorm{\up}^2 \\
\label{eq::ADup}
\text{and} \quad \up^T \hat\Gamma_A \up & \le &  \frac{11}{8} \lambda_{\max}(A) \twonorm{\up}^2 
+\frac{3\lambda_{\min}(A)}{8 k} \onenorm{\up}^2.
\een
\end{corollary}

\begin{theorem}
\label{thm::AD}
Let $A_{m \times m}, B_{n \times n}$ be symmetric positive definite 
covariance matrices. Let $E=\cup_{|J| \leq k} E_J$ for $1 \le k < m/2$. Let $Z, X$ be $n \times m$
random matrices defined as in Theorem~\ref{thm::lasso}.
Let $\hat\tau_B$ be defined as in~\eqref{eq::trBest}.
Let
\ben
\label{eq::defineAD}
\Delta := \hat\Gamma_{A} -A := \onen X^TX - \hat\tau_B I_{m} -A.
\een
Suppose that for some absolute constant $c' > 0$ and $0 < \ve \le \inv{C}$,
\ben
\label{eq::trB}
\frac{\tr(B)}{\twonorm{B}} & \ge & 
\left(c' K^4 \frac{k}{\ve^2} \log\left(\frac{3e m}{k \ve}\right)\right) \bigvee \log m,
\een
where $C = C_0/\sqrt{c'}$ for $C_0$ as chosen to satisfy~\eqref{eq::defineC0}.

Then with probability at least $1- 4 \exp\left(-c_2\ve^2 \frac{\tr(B)}{K^4\twonorm{B}}\right)-2
  \exp\left(-c_2\ve^2 \frac{n}{K^4}\right) - 6/m^3$ for $c_2 \ge 2$, 
we have for all $u, v \in E \cap S^{m-1}$,
\bens
\abs{u^T \Delta v} \le  8 C \vp(k) \ve + 4 C_0 D_1 K^2 \sqrt{\frac{\log m}{m n}},
\eens
where
  $\vp(k) = \tau_B+  \rho_{\max}(k, A)$, and $D_1 \le \frac{\fnorm{A}}{\sqrt{m}}  + \frac{\fnorm{B}}{\sqrt{n}}$,
\end{theorem}
We prove Theorem~\ref{thm::AD} in Section~\ref{sec::proofofthmAD}.

\section{Numerical results}
\label{sec::exp}

In this section, we present results from numerical simulations designed to
validate the theoretical predictions as presented in previous sections. 
We implemented the composite gradient descent algorithm as
described in \cite{ANW12,LW12,LW15} for solving the corrected Lasso
objective function~\eqref{eq::origin} with $(\hat\Gamma, \hat\gamma)$
as defined in~\eqref{eq::hatGamma}.
For the Conic programming estimator, we use the implementation
provided by the authors~\cite{BRT14} with the same 
input $(\hat\Gamma, \hat\gamma)$~\eqref{eq::hatGamma}.
Throughout our experiments, $A$ is a correlation matrix with $a_{\max} =1$.
We set the following as our default parameters: $D_0'=
\twonorm{B}^{1/2} +1$, $D_0 = \sqrt{\tau_B} + 1$ and 
$R=\twonorm{\beta^*} \sqrt{d}$, where $d$ is the
sparsity parameter, the number of non-zero entries in $\beta^*$.
In one set of simulations, we also vary $R$. 

In our simulations, we look at three different models from which
$A$ and $B$ will be chosen.  Let $\Omega = A^{-1} = (\omega_{ij})$ and 
$\Pi = B^{-1} = (\pi_{ij})$.  Let  $E$ denote edges in $\Omega$, and
$F$ denote edges in $\Pi$. We choose $A$ from one of these two models:
\begin{itemize}
\item 
$\AR$ model.
In this model, the covariance matrix is of the form $A =\{\rho^{|i-j|}\}_{i,j}$.
The graph corresponding to the precision matrix $A^{-1}$ is a chain.
\item 
Star-Block model. In this model the covariance matrix is block-diagonal with equal-sized blocks
whose inverses correspond to star structured graphs, where $A_{ii} = 1$, for all $i$.
We have 32 subgraphs, where in each subgraph, 16 nodes are connected to a central hub node
with no other connections. The rest of the nodes in the graph are singletons.
The covariance matrix for each block $S$ in $A$ is generated by setting
$S_{ij} = \rho_A$ if $(i,j) \in E$, and $S_{ij} = \rho_A^2$ otherwise.
\end{itemize}
We choose $B$ from one of the following models. Recall that $\tau_B = \tr(B)/n$.
\begin{itemize}
\item 
For $B$ and $B^* = B /\tau_B = \rho(B)$, 
we consider the $\AR$ model with two parameters.
First we choose the $\AR$ parameter $\rho_{B^*} \in \{0.3, 0.7\}$ for
the correlation matrix $B^*$. We then set $B = \tau_B B^{*}$, where
$\tau_B \in \{0.3, 0.7, 0.9\}$ depending on the experimental setup.
\item
We also consider a second model based on $\Pi = B^{-1}$, where we use
the random concentration matrix model in \cite{ZLW08}. 
The graph is generated according to a type of Erd\H{o}s--R\'{e}nyi
random graph model. Initially, we set $\Pi= c I_{n \times n}$, and $c$
is a constant. Then we randomly select $n \log n$ edges and update $\Pi$ as follows:
for each new edge $(i, j)$, a weight $w >0$ is chosen uniformly at
random from $[w_{\min}, w_{\max}]$ where $w_{\max} > w_{\min} > 0$;
we subtract $w$ from $\pi_{ij}$ and $\pi_{ji}$, and increase $\pi
_{ii}$ and $\pi_{jj}$ by $w$.\vadjust{\goodbreak} This keeps $\Pi$
positive definite. We then rescale $B$ to have a certain desired trace
parameter $\tau_B$.
\end{itemize}
For a given $\beta^*$, we first generate matrices $A$ and $B$, where $A$
is $m \times m$ and $B$ is $n \times n$.  
For the given covariance matrices $A$ and $B$, 
we repeat the following steps to estimate $\beta^*$ in the errors-in-variables model as in \eqref{eq::oby} and \eqref{eq::obX},
\begin{enumerate}
\item
We first generate random matrices
$X_0 \sim \N_{f,m} (0, A \otimes I)$ and $W \sim \N_{f,m} (0, I
\otimes B)$ independently from the matrix variate normal distribution as follows.
Let $Z \in \R^{n \times m}$ be a Gaussian random ensemble 
with independent entries $Z_{ij}$ satisfying $\E Z_{ij} = 0$, $\E
Z_{ij}^2 =1$. Let $Z_1, Z_2$ be independent copies of $Z$. Let
$X_0 = Z_1 A^{1/2}$ and $W =B^{1/2} Z_2$, where $A^{1/2}$ and
$B^{1/2}$ are the unique square root of the positive definite matrix $A$ and $B = \tau_B B^*$ respectively.
\item
We then generate $X = X_0 + W$ and $y = X_0 \beta^* + \e$, where
$\e_i \;  \text{i.i.d.} \; \sim \; N(0, 1)$. 
We compute $\hat\tau_B$, $\hat\gamma$ and $\hat\Gamma$ according to \eqref{eq::trBest} and 
\eqref{eq::hatGamma} using $X, y$, where 
by \eqref{eq::trBest}, $\hat\tau_B  := \onen \hat\tr(B)   = \inv{mn} \big(\fnorm{X}^2 - n \tr(A)\big)_{+}$.
\item
Finally, we feed $X$ and $y$ to the Composite Gradient Descent algorithm as
described in \cite{ANW12,LW12} to solve the Lasso
program~\eqref{eq::origin} to recover $\beta^*$, where we set 
the step size parameter to be $\zeta$.
The output of this step is denoted by $\hat\beta$, the estimated $\beta^*$.
We then compute the relative error of $\hat\beta$: 
${\stnorm{\hat\beta-\beta^*}}/{\norm{\beta^*}}$, where $\norm{\cdot}$ denotes either the $\ell_1$ or the $\ell_2$ norm.
\end{enumerate}
The final relative error is the average of 100 runs for each set of
tuning and step-size parameters; for the Conic programming estimator, we solve~\eqref{eq::Conic} instead
of~\eqref{eq::origin} to recover $\beta^*$.

\subsection{Relative error}
In the first experiment, $A$ and $B$ are generated using the $\AR$
model with parameters $\rho_A, \rho_{B^*} \in \{0.3, 0.7\}$ and trace parameter
$\tau_B \in \{0.3, 0.7, 0.9\}$.   
We see in Figures~\ref{fig:gd-rescaled-chain-chain1}
and~\ref{fig:gd-rescaled-chain-chain2} that a larger sample size 
is required when $\rho_A$,  $\rho_{B^*}$ or $\tau_B$ increases.  
To explain these results, we first recall the following definition of the Signal-to-noise ratio, where we
take $K = M_{\ve} \asymp 1$ 
\bens
\SMR  & \asymp & 
\frac{\twonorm{\beta^*}^2}{\tau_B \twonorm{\beta^*}^2 +1} =
\inv{\tau_B + (1/\twonorm{\beta^*}^2)}, \; \; \text{ where } \;\\
\signal & := &\twonorm{\beta^*}^2 \; \text{ and
}  \; \; \bignoise := 1 + \tau_B \twonorm{\beta^*}^2,
\eens
which clearly increases as $\twonorm{\beta^*}^2$ increases or as the
measurement error metric $\tau_B$ decreases. 
We keep $\twonorm{\beta^*} = 5$ throughout our simulations. 
The corrected Lasso recovery problem thus becomes more difficult
as $\tau_B$ increases. 
Indeed, we observe that a larger sample size $n$ is needed when $\tau_B$ increases from $0.3$ to $0.9$ in order
to control the relative $\ell_2$ error to stay at the same level.  
Moreover, in view of  Theorem~\ref{thm::lassora}, 
we can express the relative error as follows: for $\alpha \asymp
\lambda_{\min}(A)$ and $K \asymp 1$, 
 \ben
\label{eq::relative}
\frac{\twonorm{\hat{\beta} -\beta^*}}{\twonorm{\beta^*} }
= O_P\left(
\frac{(\twonorm{B}^{1/2} + 1)}{\lambda_{\min}(A)} 
 \sqrt{\frac{\bignoise}{\signal}}
\sqrt{\frac{d \log m}{n}}\right).
\een
Note that when $\twonorm{\beta^*}$ is large enough and $\tau_B =
\Omega(1)$, the factor preceding $\sqrt{\frac{d \log m}{n}}$ on the RHS of
\eqref{eq::relative} is proportional to $\frac{(\twonorm{B}^{1/2} + 1) \sqrt{\tau_B}}{\lambda_{\min}(A)}$.
\begin{figure}
\begin{center}
\begin{tabular}{cc}
\begin{tabular}{c}\includegraphics[width=0.48\textwidth]{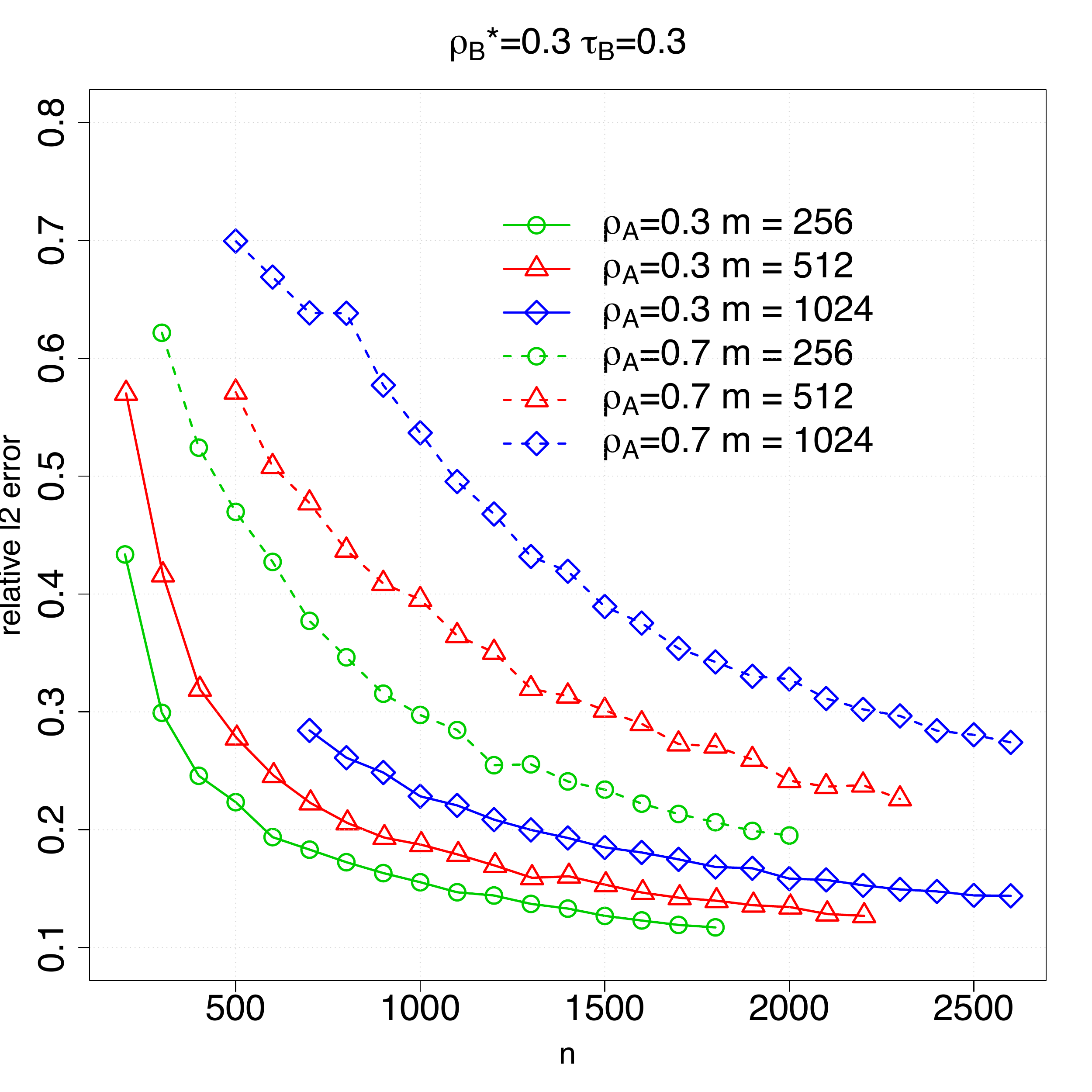}
\end{tabular}&
\begin{tabular}{c}\includegraphics[width=0.48\textwidth]{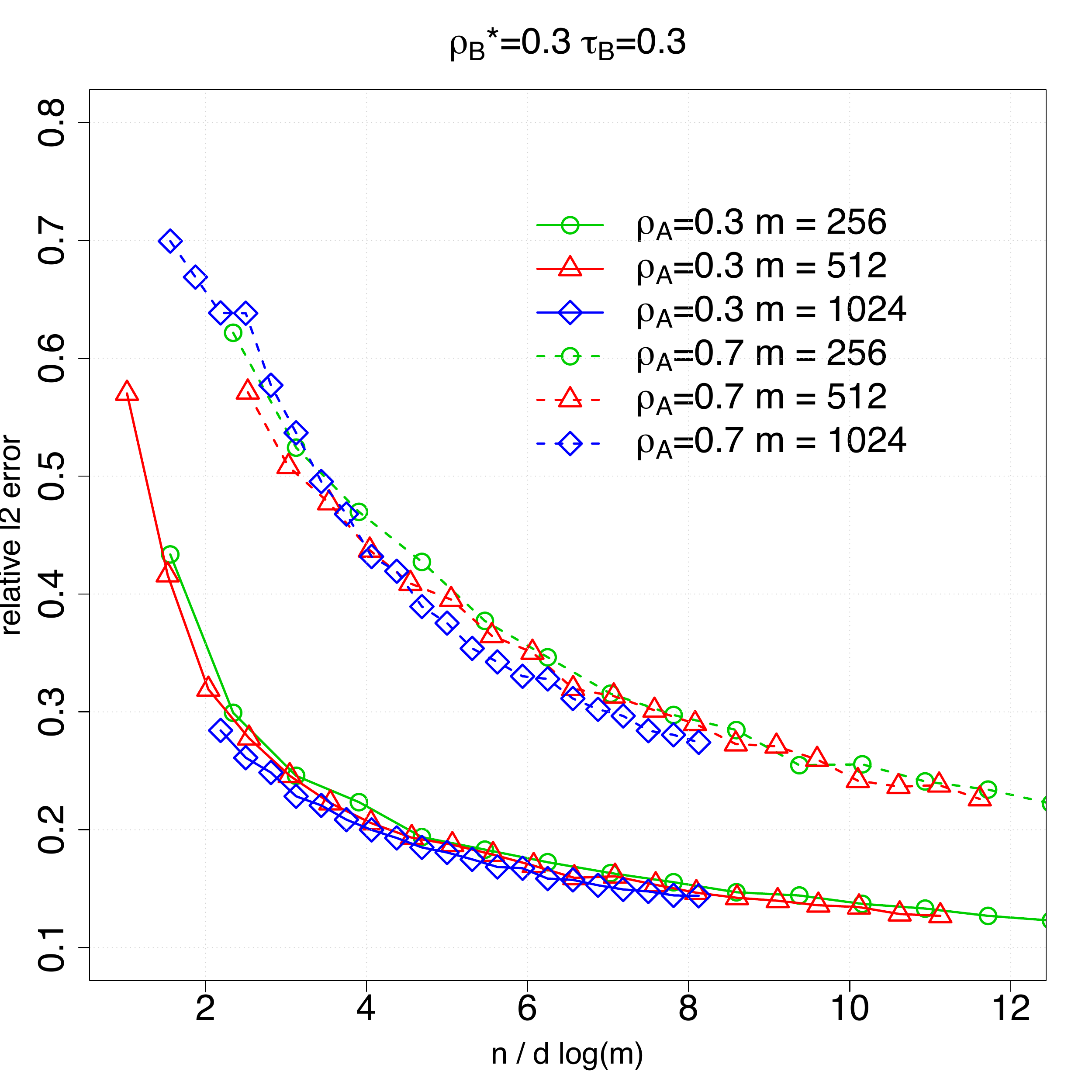}
\end{tabular} \\
(a) &(b) \\
\begin{tabular}{c}\includegraphics[width=0.48\textwidth]{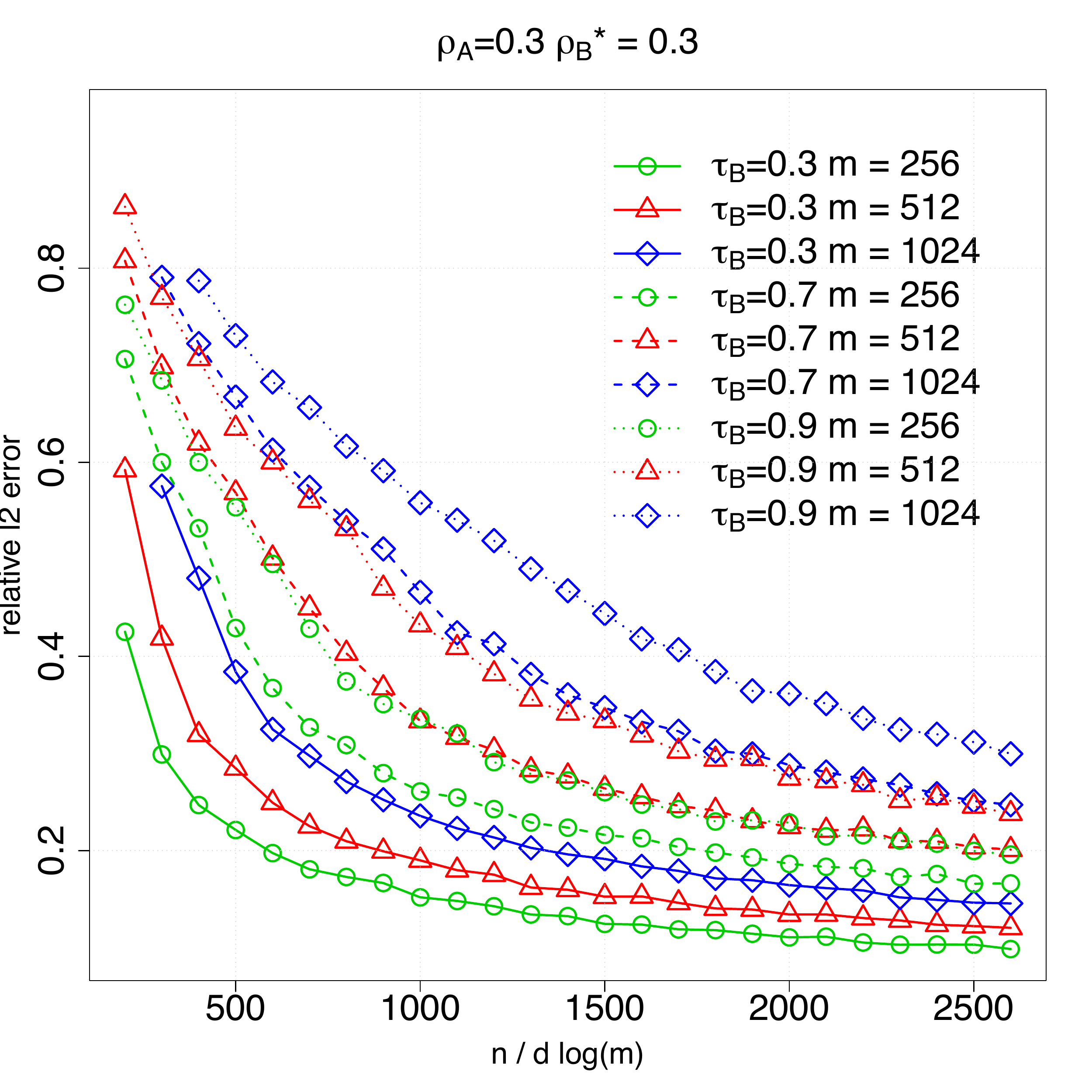}
\end{tabular}&
\begin{tabular}{c}\includegraphics[width=0.48\textwidth]{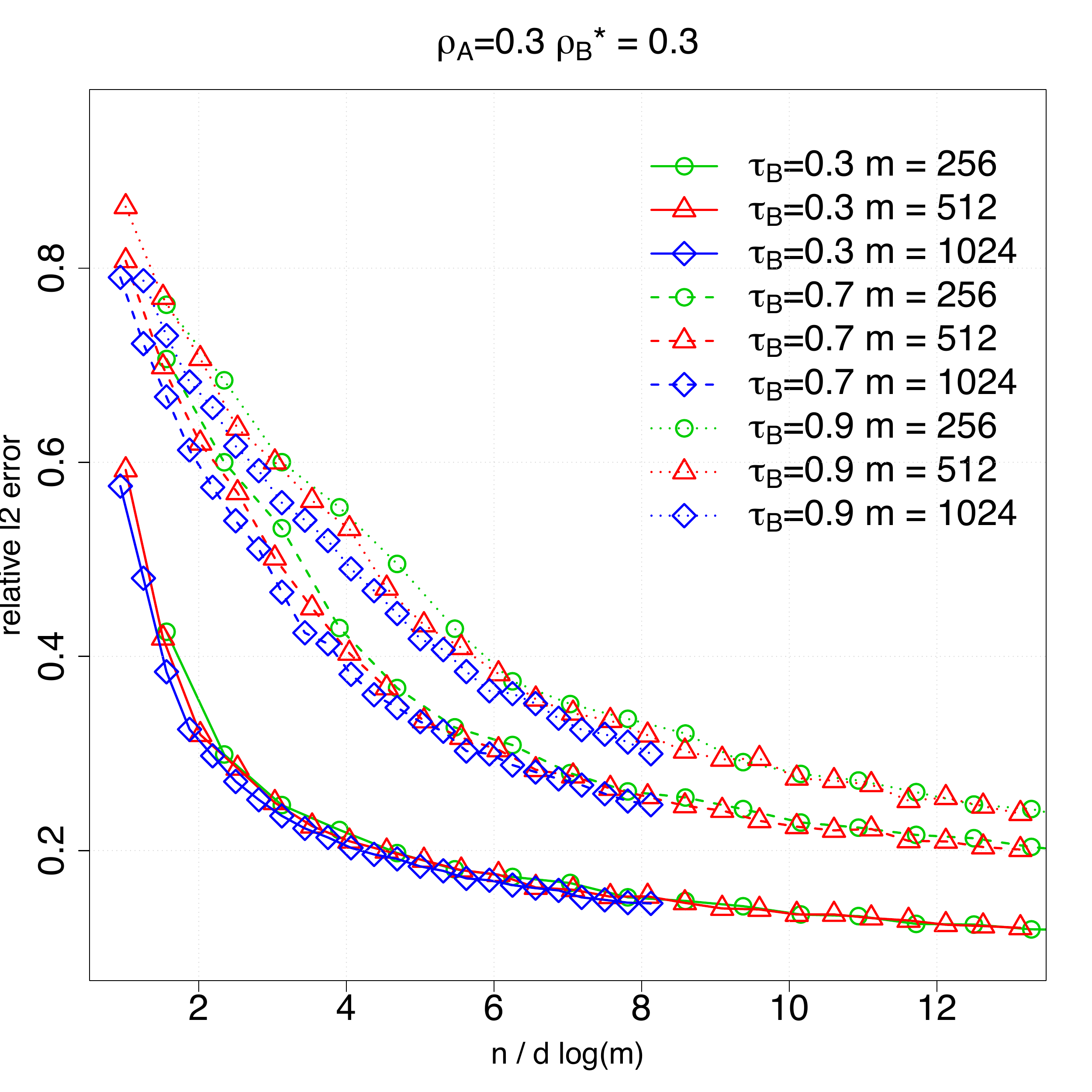}
\end{tabular}\\
(c) &(d) \\
\end{tabular}
\caption{
Plots of the relative $\ell_2$ error
after running composite gradient descent algorithm on recovering
$\beta^*$ using the corrected Lasso objective function with sparsity
parameter $d = \lfloor \sqrt{m}\rfloor$, where we vary $m \in \{256,
512, 1024\}$. We generate $A$ and $B$ using the $\AR$ model with 
$\rho_A, \rho_{B^{*}} \in \{0.3, 0.7\}$ and $\tau_B  = \{0.3, 0.7, 0.9\}$.
In the left and right column, we plot the relative $\ell_2$ error with
respect to sample size $n$ as well as  the rescaled
sample size $n/(d \log m)$.   As $n$ increases, we see that the statistical error decreases.
In the top row, we vary the $\AR$ parameter $\rho_A \in \{0.3, 0.7\}$,
while holding $(\tau_B, \rho_B^*)$ and $\twonorm{\beta^*}$ invariant.
Plot (a) shows the relative $\ell_2$ error versus $n$ for $m=256, 512, 1024$. 
In Plots (c) and (d), we vary  the trace parameter $\tau_B \in \{0.3, 0.7, 0.9\}$,
while fixing the $\AR$ parameters $\rho_A, \rho_B^* = 0.3$.
Plot (b) and (d) show the relative $\ell_2$ error
versus the rescaled sample size $n/(d \log m)$. 
The curves now align for different values of $m$ in the rescaled plots, consistent with the
theoretical prediction in Theorem~\ref{thm::lassora}.}
\label{fig:gd-rescaled-chain-chain1}
\end{center}
\end{figure}

\begin{figure}
\begin{center}
\begin{tabular}{cc}
\begin{tabular}{c}\includegraphics[width=0.48\textwidth]{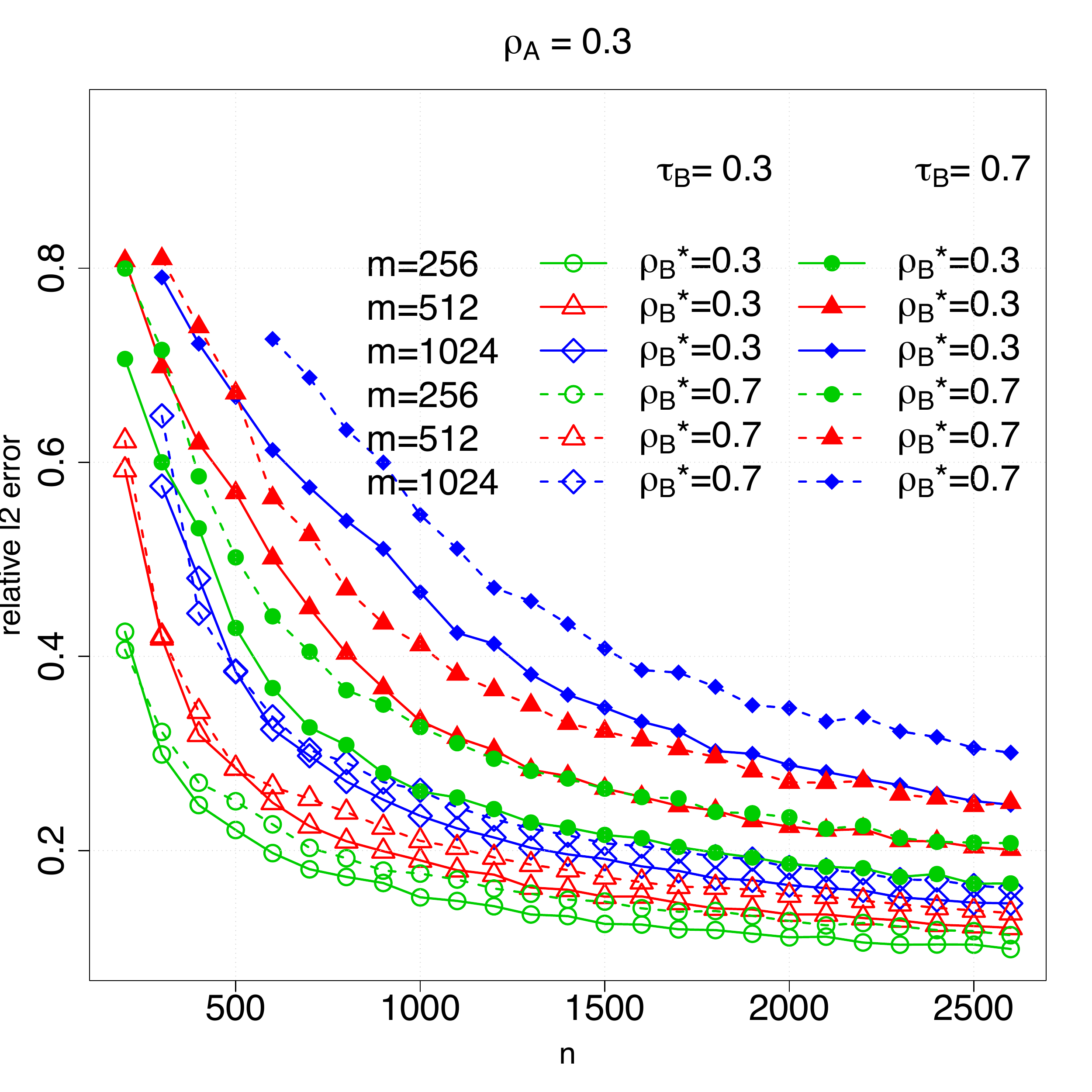}
\end{tabular}&
\begin{tabular}{c}\includegraphics[width=0.48\textwidth]{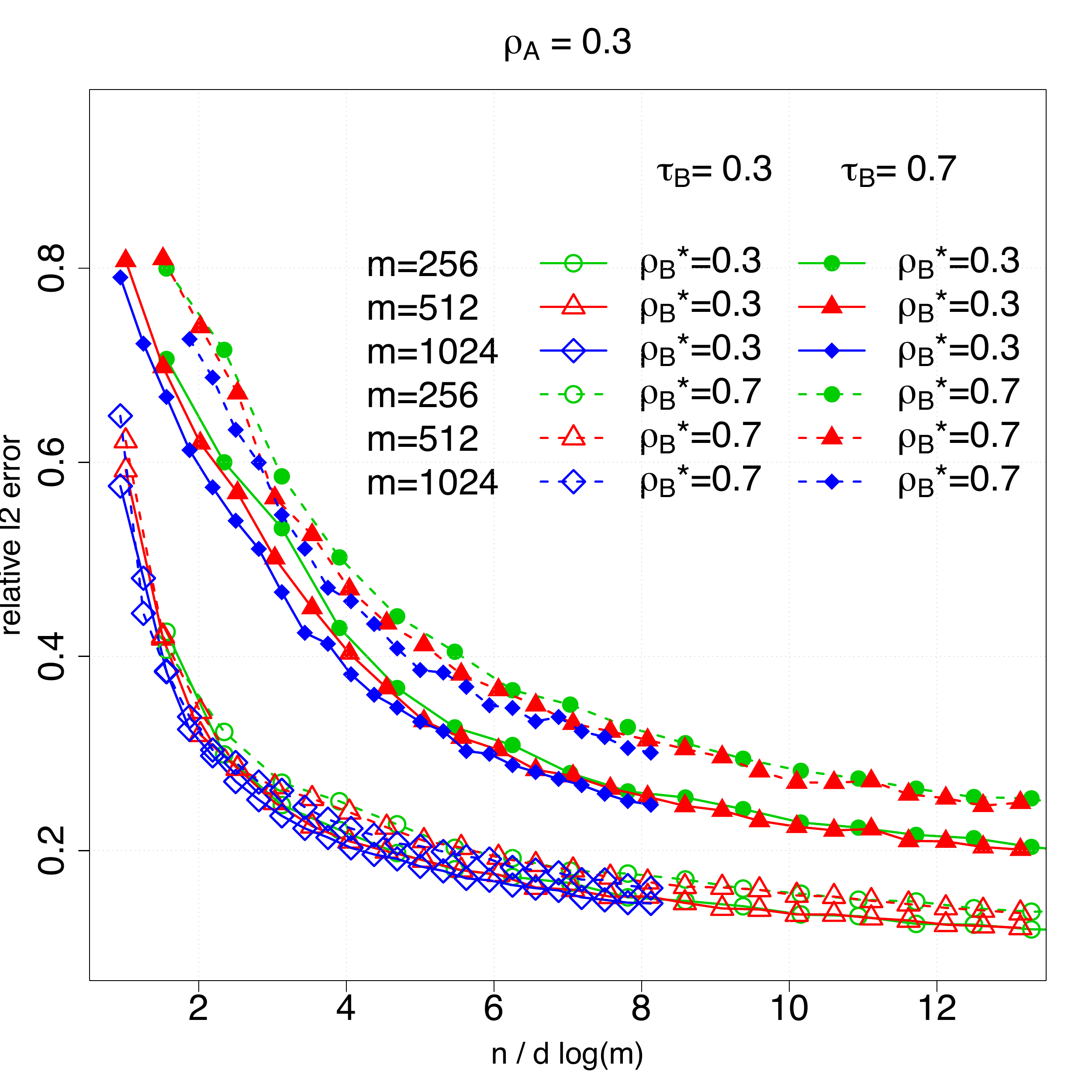}
\end{tabular}\\
(a) &(b) \\
\end{tabular}
\caption{
Plots of the relative $\ell_2$ error
after running composite gradient descent algorithm on recovering
$\beta^*$ using the corrected Lasso objective function with sparsity
parameter $d = \lfloor \sqrt{m}\rfloor$, where we vary $m \in \{256,
512, 1024\}$. We generate $A$ and $B$ using the $\AR$ model with 
$\rho_A=0.3$ and $\rho_B^* \in \{0.3, 0.7\}$.
We set $B = \tau_B B^*$ and vary the trace parameter
$\tau_B \in \{0.3, 0.7\}$ for each value of $\rho_{B^*}$.
The parameters $\tau_B$ and $\rho_B^{*}$ affect the rate of
convergence through $D_0' = \twonorm{B}^{1/2} + a_{\max}^{1/2}$ and $\tau_B^{1/2}$.
Plot (b) shows the relative $\ell_2$ error
versus the rescaled sample size $n/(d \log m)$.
We observe that as $\tau_B$ increases from $0.3$ to $0.7$, the two sets of curves
corresponding to $\rho_{B^*} = 0.3, 0.7$ become visibly more 
separated. As $n$ increases, all curves converge to $0$.}
\label{fig:gd-rescaled-chain-chain2}
\end{center}
\end{figure}
When we plot the relative $\ell_2$ error  $\shtwonorm{\hat\beta -\beta^*}/\shtwonorm{\beta^*}$
versus the rescaled sample size $\frac{n}{d \log m}$ under the same
$\SMR$ ratio, the two sets of curves corresponding to
$\rho_{A} = 0.3$ and $\rho_{A} = 0.7$ indeed line up in
Figure~\ref{fig:gd-rescaled-chain-chain1}(b), as predicted by
\eqref{eq::relative}.
We observe in Figure~\ref{fig:gd-rescaled-chain-chain1}(b),
the rescaled curves overlap well for different values of $(m, d)$ for
each $\rho_A$ when we keep $(\rho_{B^*}, \tau_B)$ and the length
$\twonorm{\beta^*} = 5$ invariant.
Moreover, the upper bound on the relative $\ell_2$ error~\eqref{eq::relative}
characterizes the relative positions
of these two sets of curves in that the ratio between the
$\ell_2$ error corresponding to $\rho_{A} = 0.7$ and that for
$\rho_{A} = 0.3$ along the $y$-axis roughly falls within the interval
$(2, 3)$ for each $n$,  while $\lambda_{\min}(\AR, 0.3)/\lambda_{\min}(\AR,0.7) = 3$.
These results are consistent with the theoretical predictions in Theorems~\ref{thm::lasso}
and~\ref{thm::lassora}.

In Figure~\ref{fig:gd-rescaled-chain-chain1}(c) and (d), we also show
the effect of $\tau_B$ when $\tau_B$ is chosen from
$\{0.3, 0.7, 0.9\}$, while fixing the $\AR$ parameters $\rho_A = 0.3$ and  $\rho_{B^*} =0.3$.
As predicted by our theory, as the measurement error magnitude
$\tau_B$ increases, $\bignoise$ increases, resulting in a larger 
relative $\ell_2$ error for a fixed sample size $n$.

While the effect of $\rho_A$ as shown in \eqref{eq::relative}
through the minimal eigenvalue of $A$ is directly visible
in Figure~\ref{fig:gd-rescaled-chain-chain1}(b), the effect of $\rho_B{^*}$ is more subtle, as it is modulated by $\tau_B$ as
shown in Figure~\ref{fig:gd-rescaled-chain-chain2}(a) and (b).
When $\tau_B$ is fixed, our theory predicts that $\twonorm{B}$ plays a
role in determining the $\ell_p$ error,  $p=1, 2$, through the penalty
parameter $\lambda$ in view of~\eqref{eq::relative}. 
The effect of $\rho_{B^*}$, which goes into the parameter $D_0'=
\twonorm{B}^{1/2} + a_{\max}^{1/2} \asymp 1$, is not changing the
sample requirement or the rate of convergence as significantly as that
of $\rho_A$ when $\tau_B = 0.3$. 
This is shown in the bottom set of curves in
Figure~\ref{fig:gd-rescaled-chain-chain2}(a) and (b).
On the other hand, the trace parameter $\tau_B$ plays a dominating role in
determining the sample size as well as the $\ell_p$ error for $p=1,
2$, especially when the length of the signal $\beta^*$ is large:
$\twonorm{\beta^*} =\Omega(1)$.  In particular, the separation between the two sets of curves in
Figure~\ref{fig:gd-rescaled-chain-chain2}(b),
which correspond to the two choices of
$\rho_B{^*}$, is clearly modulated by $\tau_B$ and becomes 
more visible when $\tau_B = 0.7$.

These findings are also consistent with our theoretical prediction
that in order to guarantee statistical and computational convergence, the sample size
needs to grow according to the following relationship to be specified
in \eqref{eq::nlower}. 
We will show in the proof of Theorem~\ref{thm::corrlinear} that the
condition on sparsity $d$
as stated in \eqref{eq::dupperthm} implies that as $\rho_A$, or
$\tau_B$, or the step size parameter $\zeta$ increases, we need to
increase the sample size in order to guarantee computational convergence for the composite gradient descent
algorithm given the following lower bound: 
\ben
\label{eq::nlower}
n &  = & \Omega\left({d \tau_0 \log m} 
\left\{\frac{\lambda_{\max}(A)}{\lambda_{\min}(A)^2} 
\right\}  \bigvee \left\{\frac{\zeta}{(\bar\alpha_{\ell})^2}\right\}
\right),  \; \; \text{ where } \; \\
\nonumber
\quad 
\tau_0 & \asymp &  \frac{(\rho_{\max}(s_0, A) + \tau_B)^2}{\lambda_{\min}(A)}.
\een
We illustrate the effect of the penalty and step size parameters in Section~\ref{sec::GCplots}.

\subsection{Corrected Lasso via GD versus Conic programming estimator}
\label{sec::GCplots}
In the second experiment, both $A$ and $B$ are generated using the $\AR$ model with
parameters $\rho_A = 0.3$, $\rho_{B^*}=0.3$,  and $\tau_B \in \{ 0.3, 0.7\}$.
We set $m=1024$, $d=10$ and  $\twonorm{\beta^*} =5$.
We then compare the performance of the corrected Lasso estimator~\eqref{eq::origin}
using the composite gradient descent algorithmic with the Conic
programming estimator, which is a convex program designed and implemented by authors of~\cite{BRT14}. 

We consider three choices for the step size parameter for the composite gradient descent algorithm:
 $\zeta_1 = \lambda_{\max}(A) +\half\lambda_{\min}(A)$, $\zeta_2 = \frac{3}{2}\lambda_{\max}(A)$  and
$\zeta_3 = 2 \lambda_{\max}(A)$.
We observe that the gradient descent algorithm consistently produces
an output such that its statistical error in $\ell_2$ norm is lower
than the best solution produced by the Conic programming estimator,
when both methods are subject to optimal tuning after we fix upon the
radius $R = \sqrt{d} \twonorm{\beta^*}$ for~\eqref{eq::origin2} and
($\omega, \lambda$) in~\eqref{eq::Conic} as follows.
As illustrated in our theory, one can think of the parameter $\lambda$
in~\eqref{eq::origin} and parameters $\mu, \omega$ in~\eqref{eq::Conic} satisfying
\bens
\lambda \asymp \mu \twonorm{\beta^*} + \omega,
\eens
where we set $\omega = 0.1 D_0 \sqrt{\frac{\log m}{n}}$, where the
factor $0.1$ is chosen without loss of generality, as we will sweep
over  $f \in (0, 0.8]$ to run through a sufficiently large range of values of the tuning
parameters:
\begin{itemize}
\item 
For the corrected Lasso estimator, we set $\lambda= f D_0' {\hat\tau_B}^{1/2}
  \twonorm{\beta^*}\sqrt{\frac{\log m}{n}} + \omega$;
\item 
For the Conic programming estimator, we use $\mu = f {D'_0}
{\hat\tau_B}^{1/2} \sqrt{\frac{\log m}{n}}$. 
We set $\lambda = 1$ in~\eqref{eq::Conic}, which is independent of
the Lasso penalty.
\end{itemize}
The factor $f$ is chosen to
reflect the fact that in practice, we do not know the exact value of
$\twonorm{\beta^*}$ or $\onenorm{\beta^*}$, $D_0$
or $D_0'$, or other parameters related to the spectrum properties of
$A, B$; moreover, in practice, we wish to understand the whole-path
behavior for both estimators.

In Figures~\ref{fig:conic-gd-chain-chain-tau03} and ~\ref{fig:conic-gd-chain-chain-tau07},
we plot the relative error in $\ell_1$ and $\ell_2$ norm as $n$ increases from $100$ to
$2500$, while sweeping over penalty factor $f \in [0.05, 0.8]$
for $\tau_B = 0.3$ and $\tau_B = 0.7$ respectively.
For both estimators,  the relative $\ell_2$ and $\ell_1$ error versus the
scaled sample size $n/(d\log m)$ are also plotted.
In these figures, green dashed lines are for the corrected Lasso estimator
via gradient descent algorithm, and blue dotted lines are for the Conic programming estimator.
These plots allow us to observe the behaviors of the two
estimators across a set of tuning parameters. 
Overall, we see that both methods are able to achieve low relative
error $\ell_p, p=1, 2$ norm when $\lambda$ and $\mu$
are chosen from a suitable range.

For the corrected Lasso estimator, we display results where the step
size parameter $\zeta$ is set to $\zeta_2 = \frac{3}{2}\lambda_{\max}(A)$ 
and $\zeta_3 = 2 \lambda_{\max}(A)$ in the left and right column respectively.
We mention in passing that the algorithm starts to converge even
when we set $\zeta = \zeta_1 =\lambda_{\max}(A) +
\half\lambda_{\min}(A)$ 
as we observe quantitively similar behavior as
the displayed cases. For both estimators, we observe that we need a larger sample size $n$ in case
$\tau_B = 0.7$ in order to control the error at the same level as in
case $\tau_B = 0.3$. 
\begin{figure}
\begin{center}
\hskip-20pt
\begin{tabular}{cc}
\begin{tabular}{c}\includegraphics[width=0.50\textwidth]{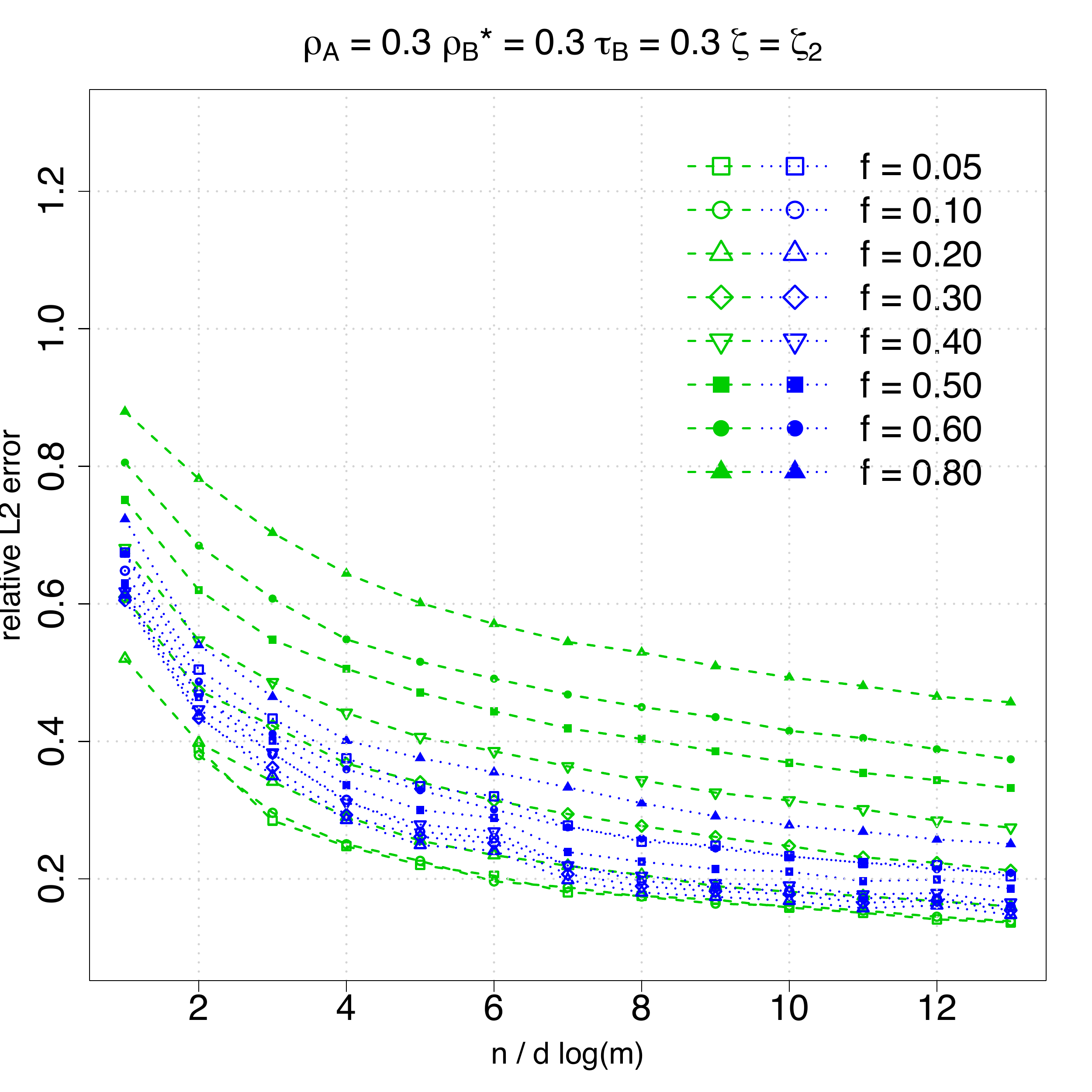}
\end{tabular}&
\hskip-15pt
\begin{tabular}{c}\includegraphics[width=0.50\textwidth]{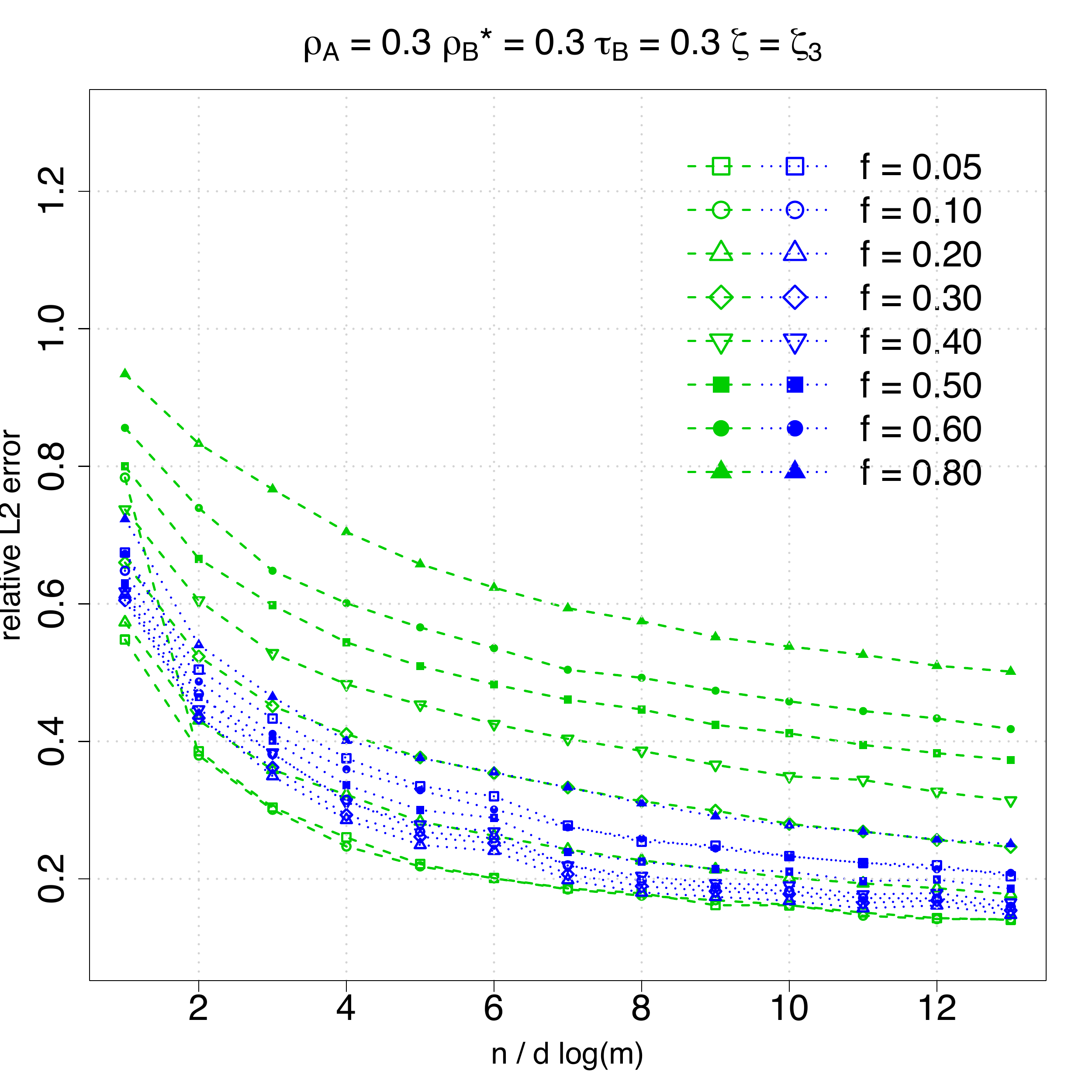}
\end{tabular} \\
\begin{tabular}{c}\includegraphics[width=0.50\textwidth]{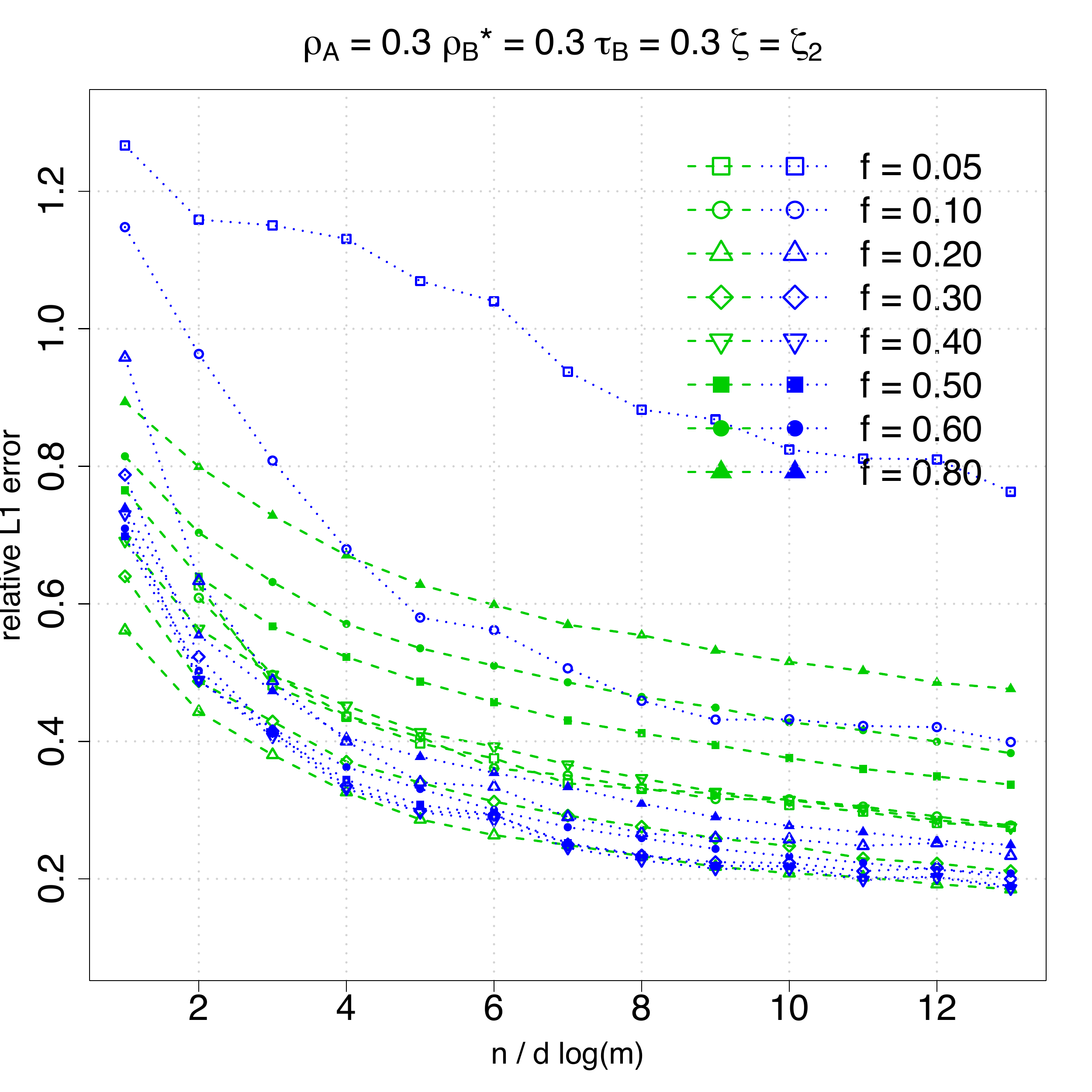}
\end{tabular}&
\hskip-15pt
\begin{tabular}{c}\includegraphics[width=0.50\textwidth]{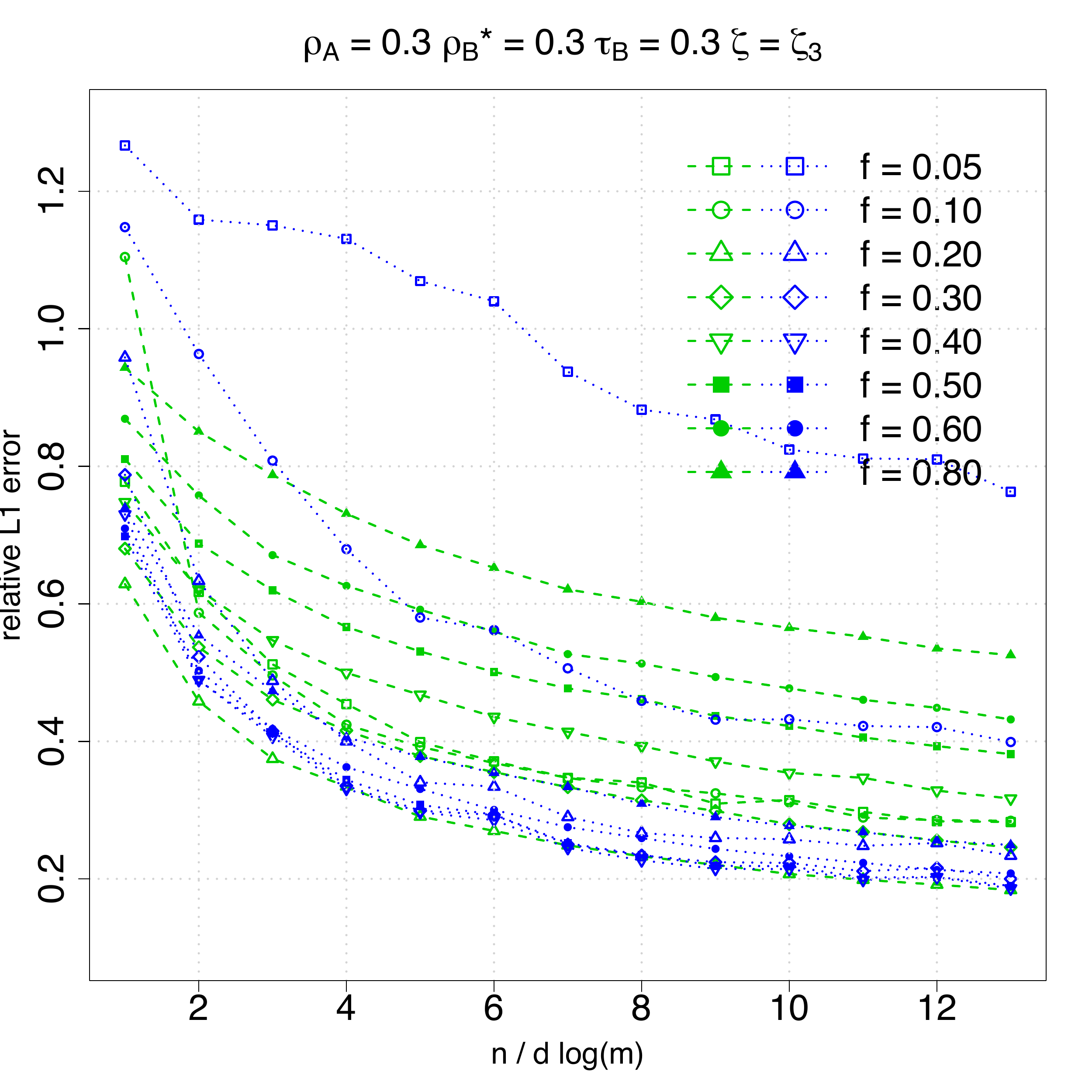} 
\end{tabular} \\
\end{tabular}
\caption{Plots of the relative $\ell_1$ and $\ell_2$ error
$\norm{\hat\beta-\beta^*}/\norm{\beta^*}$ for the 
Conic programming estimator and the corrected Lasso estimator obtained via running the
composite gradient descent algorithm on (approximately) recovering
$\beta^*$. Set parameters $d = 10$ and $m = 1024$ while varying $n$.
Generate $A$ and $B$ using the $\AR$ model with parameters $\rho_A =0.3$, $\rho_{B^*}=0.3$ and $\tau_B = 0.3$.
Set $\zeta \in \{\zeta_1, \zeta_2, \zeta_3\}$. 
We compare the performance of the corrected Lasso and the Conic
programming estimators over choices of $\lambda$ and  $\mu$ while
sweeping through $f \in (0, 0.8]$.
In the top row, we plot the relative $\ell_2$ error for the Conic
programming estimator (blue dotted lines) and 
the corrected Lasso (green dashed lines) via the composite gradient descent algorithm 
with step size parameter set to be $\zeta_2 = \frac{3}{2}\lambda_{\max}(A)$ and $\zeta_3 = 2
\lambda_{\max}(A)$; in the bottom row, we plot the relative $\ell_1$ error
under the same settings.
 We note that the composite gradient descent
algorithm starts to converge even when we
set  the step size parameter to be $\zeta_1 = \lambda_{\max}(A) + \half\lambda_{\min}(A)$.}
\label{fig:conic-gd-chain-chain-tau03}
\end{center}
\end{figure}

\begin{figure}
\begin{center}
\hskip-20pt
\begin{tabular}{cc}
\begin{tabular}{c}\includegraphics[width=0.50\textwidth]{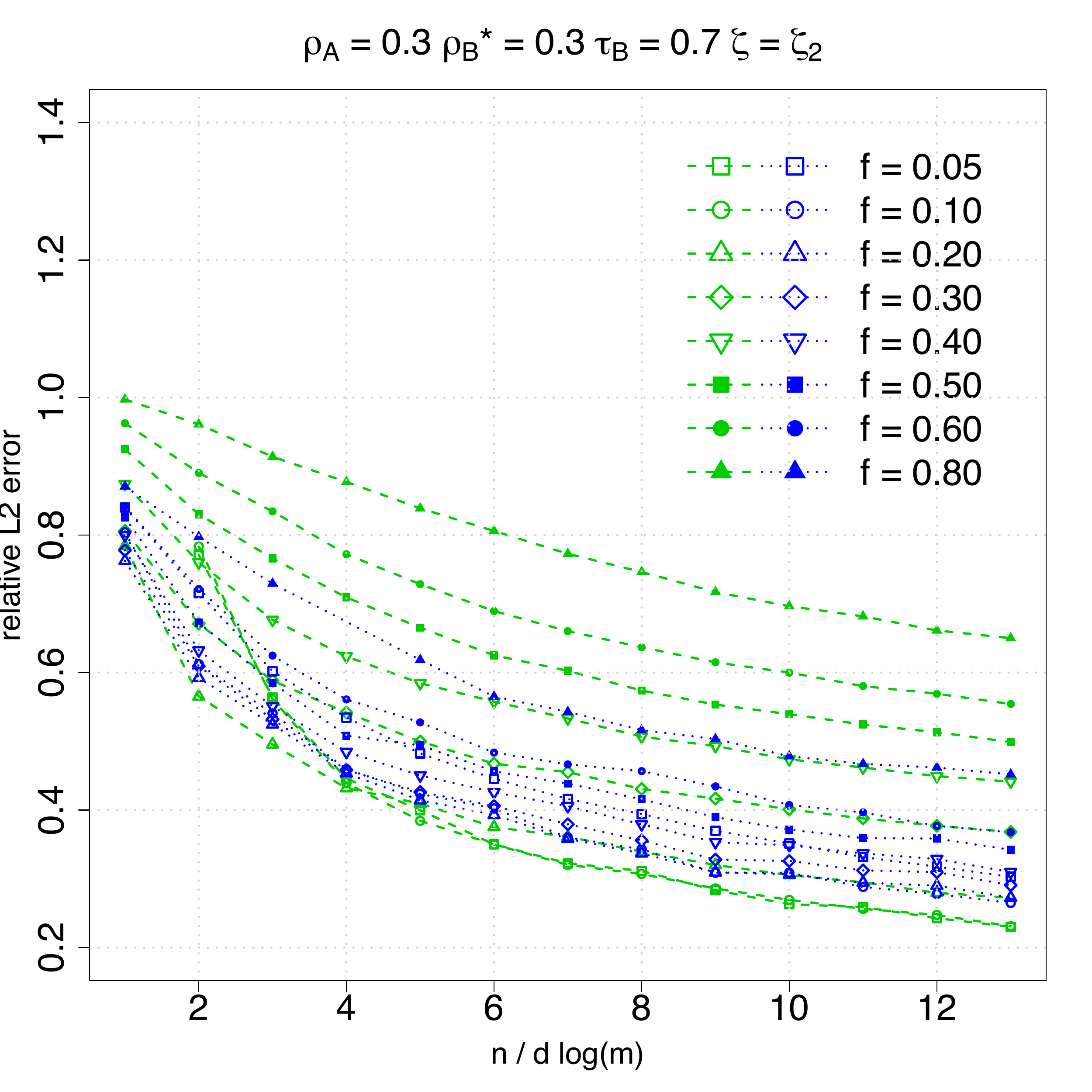}
\end{tabular}&
\hskip-15pt
\begin{tabular}{c}\includegraphics[width=0.50\textwidth]{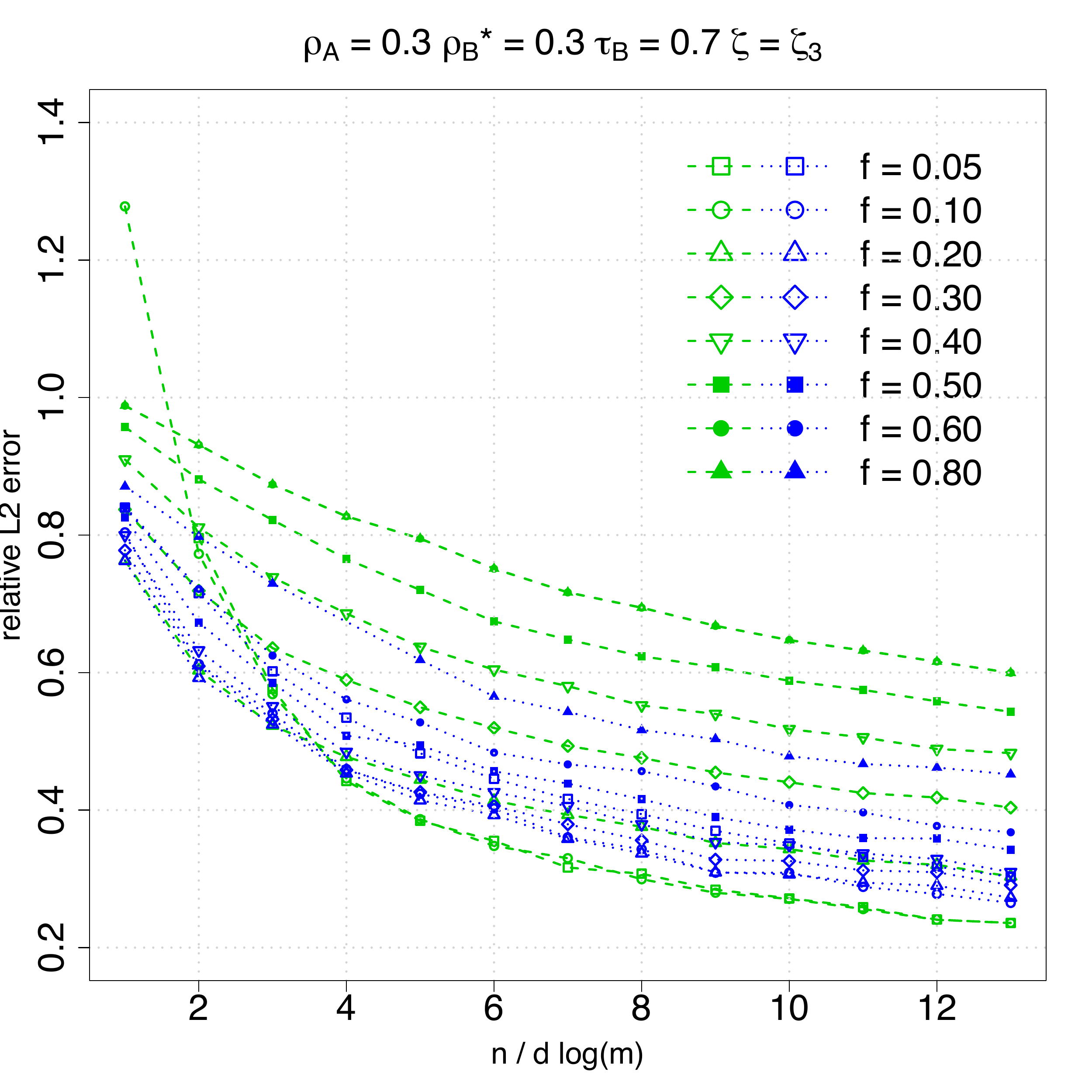}
\end{tabular} \\
\begin{tabular}{c}\includegraphics[width=0.50\textwidth]{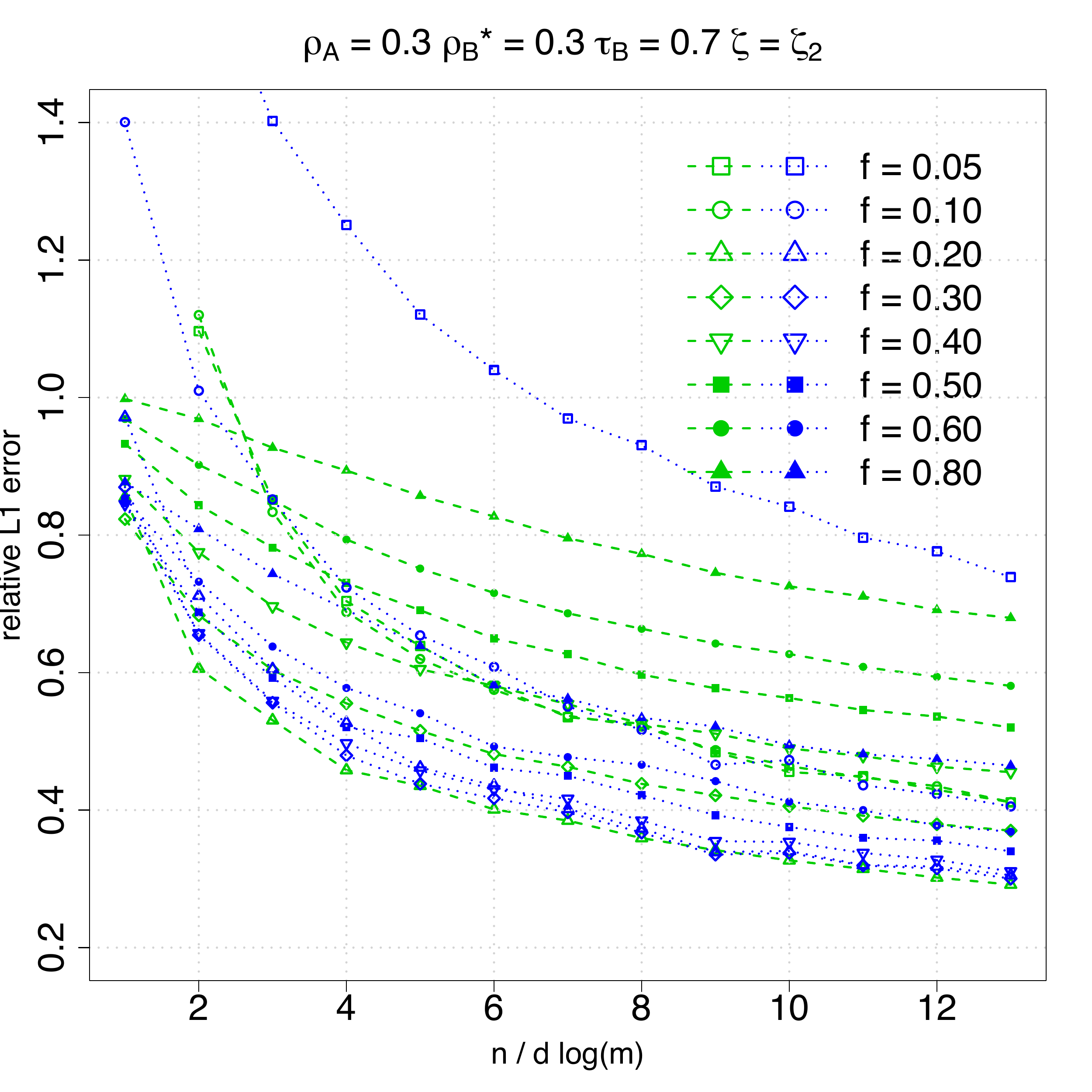}
\end{tabular}&
\hskip-15pt
\begin{tabular}{c}\includegraphics[width=0.50\textwidth]{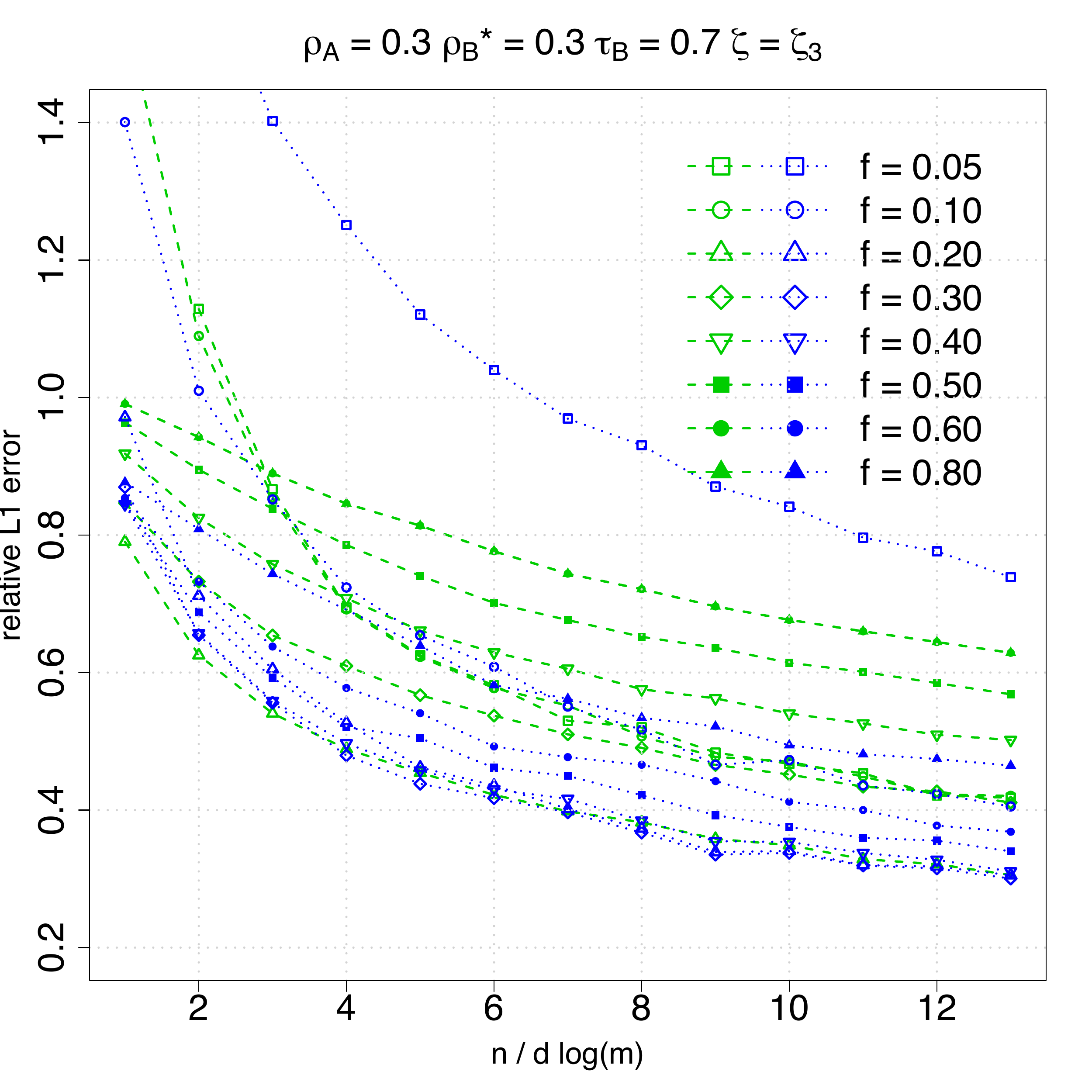} 
\end{tabular} \\
\end{tabular}
\caption{Plots of the relative $\ell_1$ and $\ell_2$ error
$\norm{\hat\beta-\beta^*}/\norm{\beta^*}$
after running the Conic programming estimator and composite gradient descent algorithm on recovering
$\beta^*$ using the corrected Lasso objective function with sparsity
parameter $d = 10$ and $m = 1024$ while varying $n$.
Both $A$ and $B$ are generated using the $\AR$ model with
parameters $\rho_A =0.3$, $\rho_{B^*}=0.3$ and $\tau_B = 0.7$.
We compare the performance of the corrected Lasso (green dashed lines) 
and the Conic programming estimators (blue dotted lines) over choices of $\lambda$ and  $\mu$ while
sweeping through $f \in (0, 0.8]$. For the composite gradient descent
algorithm, we choose $\zeta$ from $\{\zeta_1, \zeta_2, \zeta_3\}$. 
In the top row, we plot the $\ell_2$ error for the Conic and the
corrected Lasso with $\zeta_2 = \frac{3}{2}\lambda_{\max}(A)$ and
$\zeta_3 = 2 \lambda_{\max}(A)$, 
while in the bottom row, we plot the $\ell_1$ error corresponding to the two step size parameters.
} 
\label{fig:conic-gd-chain-chain-tau07}
\end{center}
\end{figure}

\begin{figure}
\begin{center}
\begin{tabular}{cc}
\begin{tabular}{c}\includegraphics[width=0.48\textwidth]{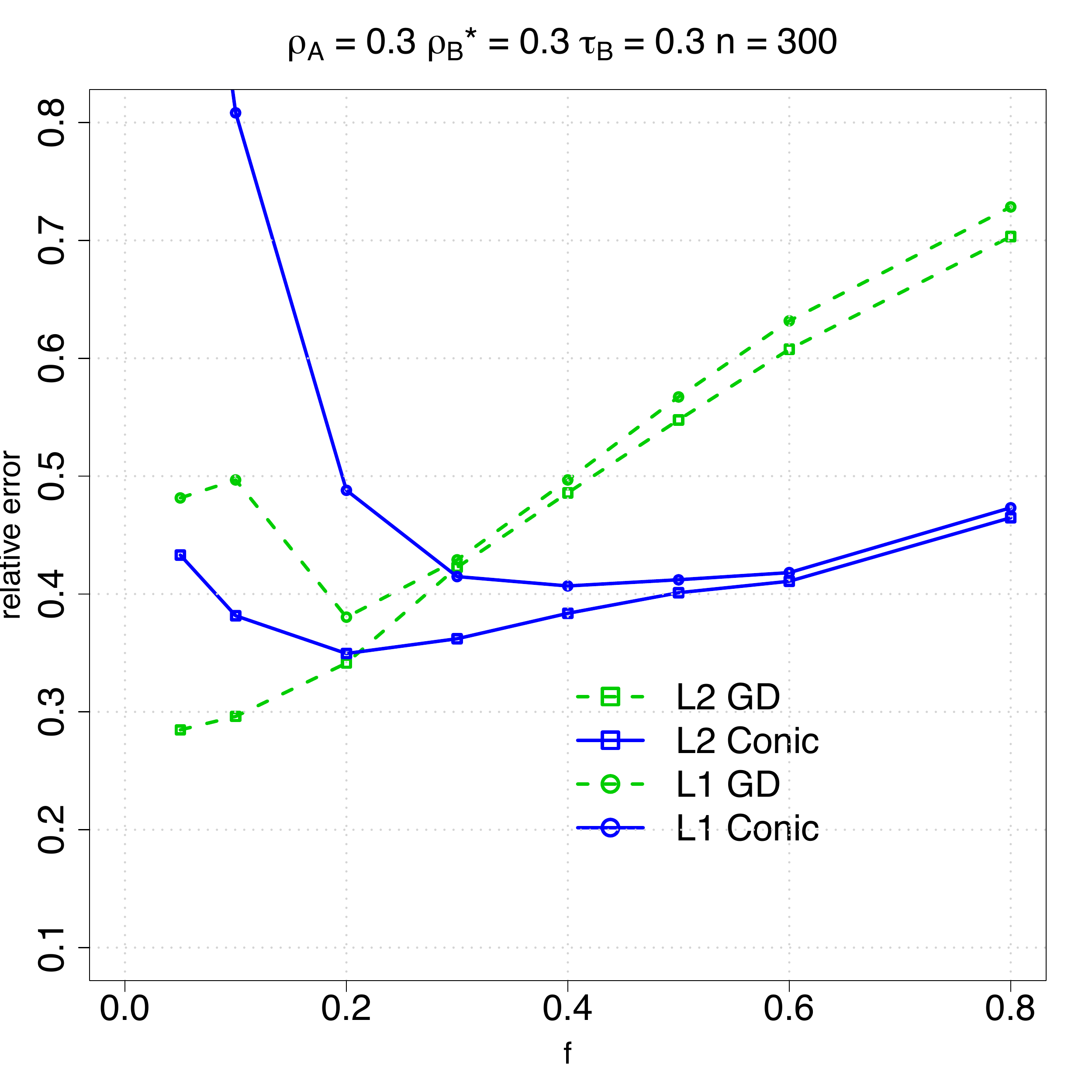}
\end{tabular}&
\begin{tabular}{c}\includegraphics[width=0.48\textwidth]{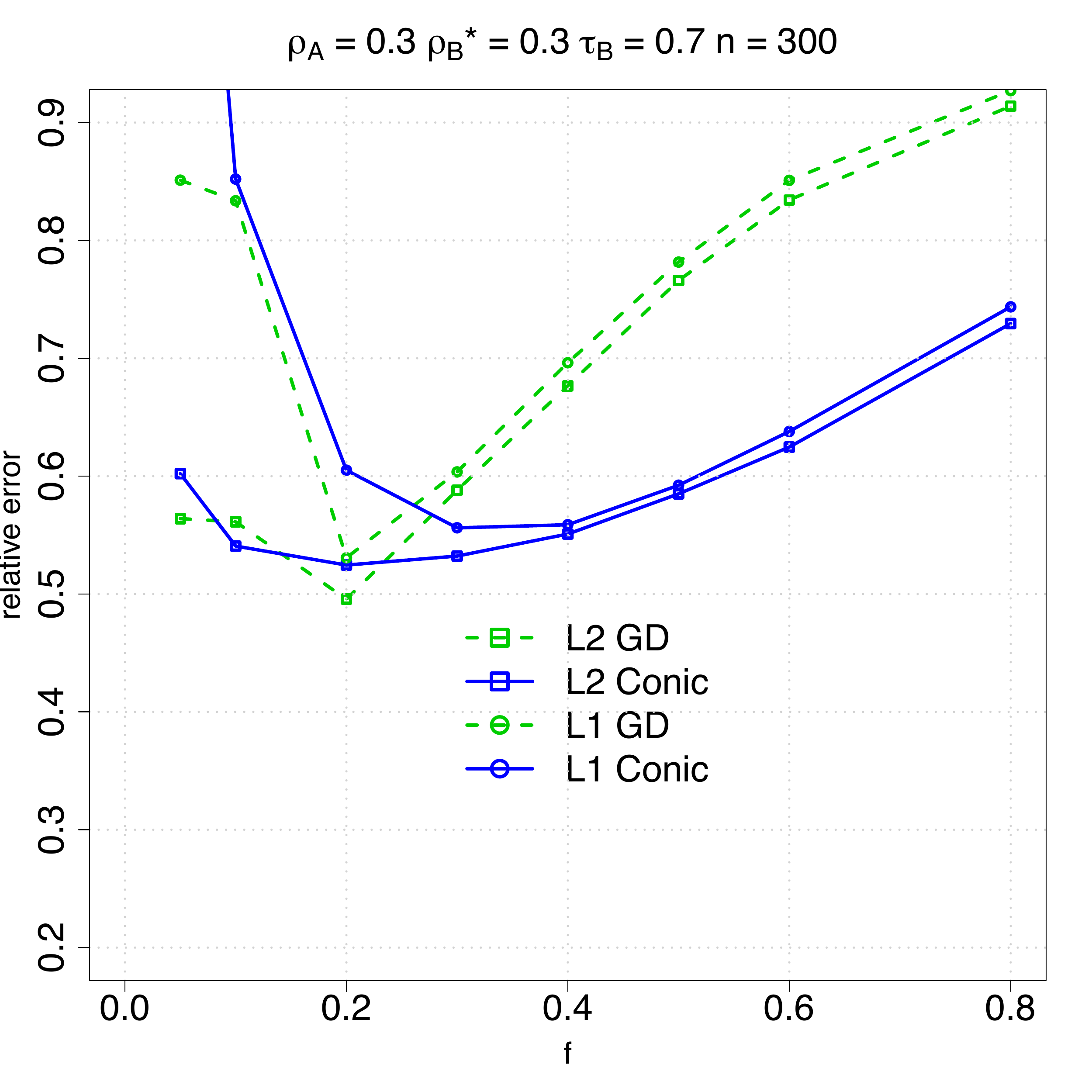}
\end{tabular} \\
\begin{tabular}{c}\includegraphics[width=0.48\textwidth]{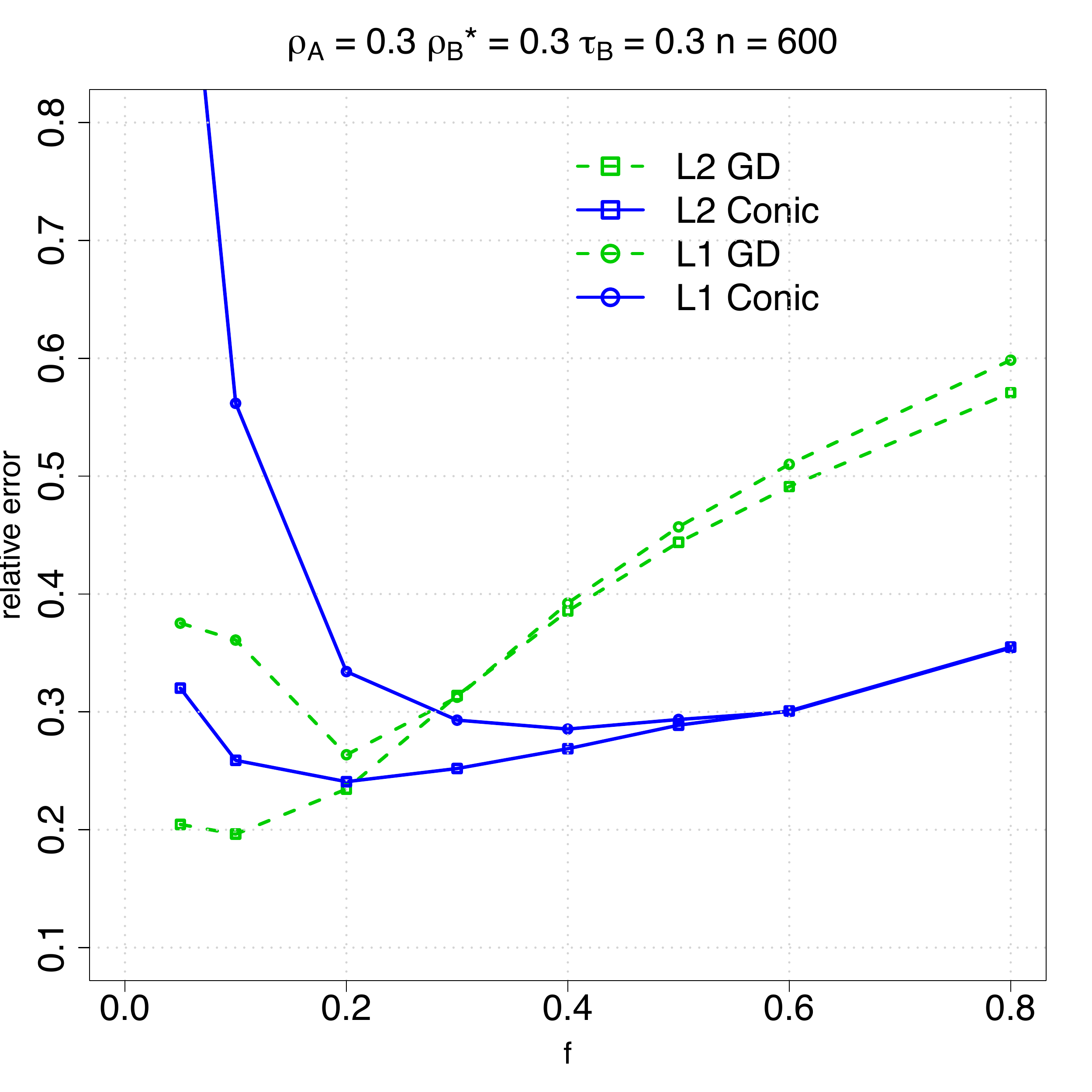} 
\end{tabular}&
\begin{tabular}{c}\includegraphics[width=0.48\textwidth]{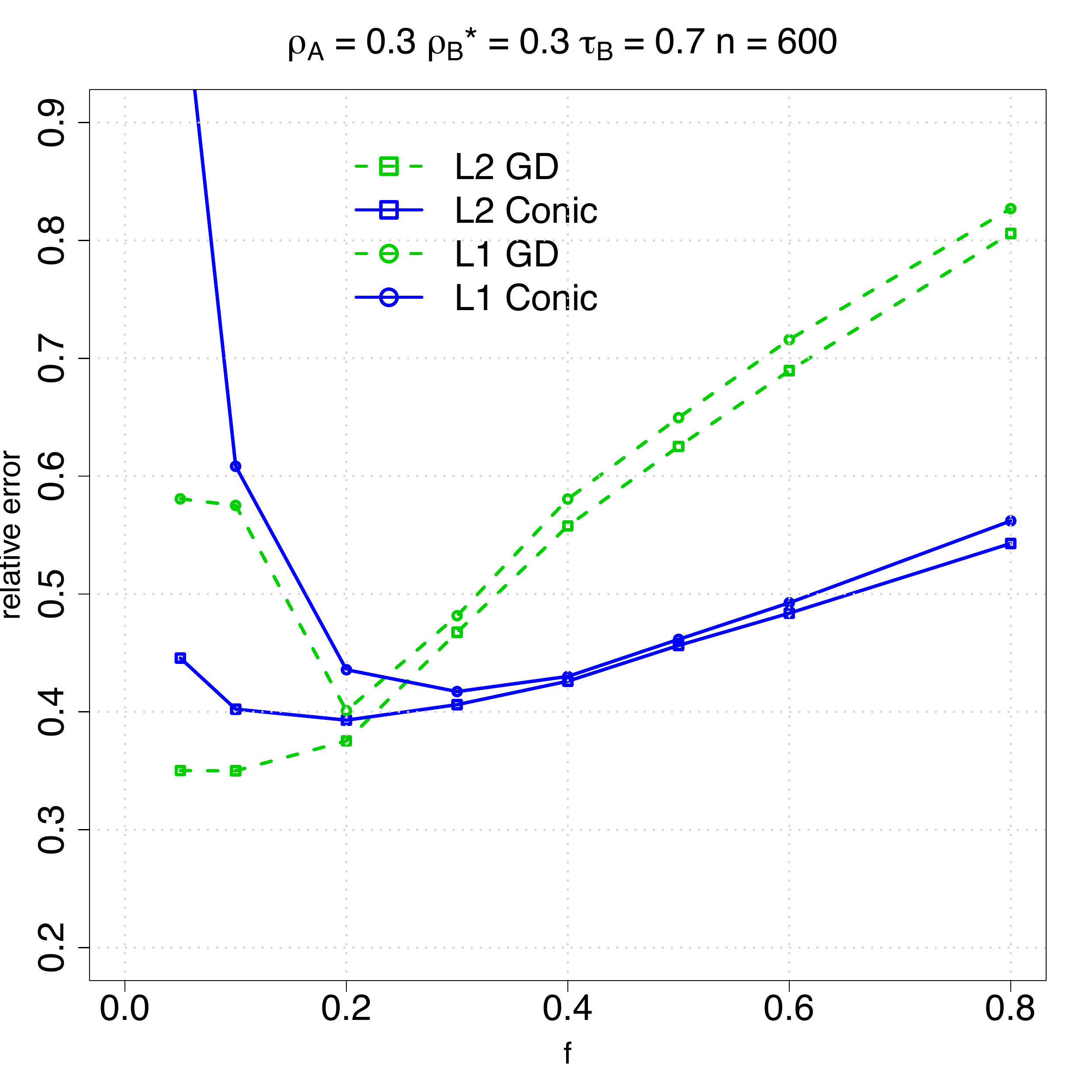} 
\end{tabular} \\
\begin{tabular}{c}\includegraphics[width=0.48\textwidth]{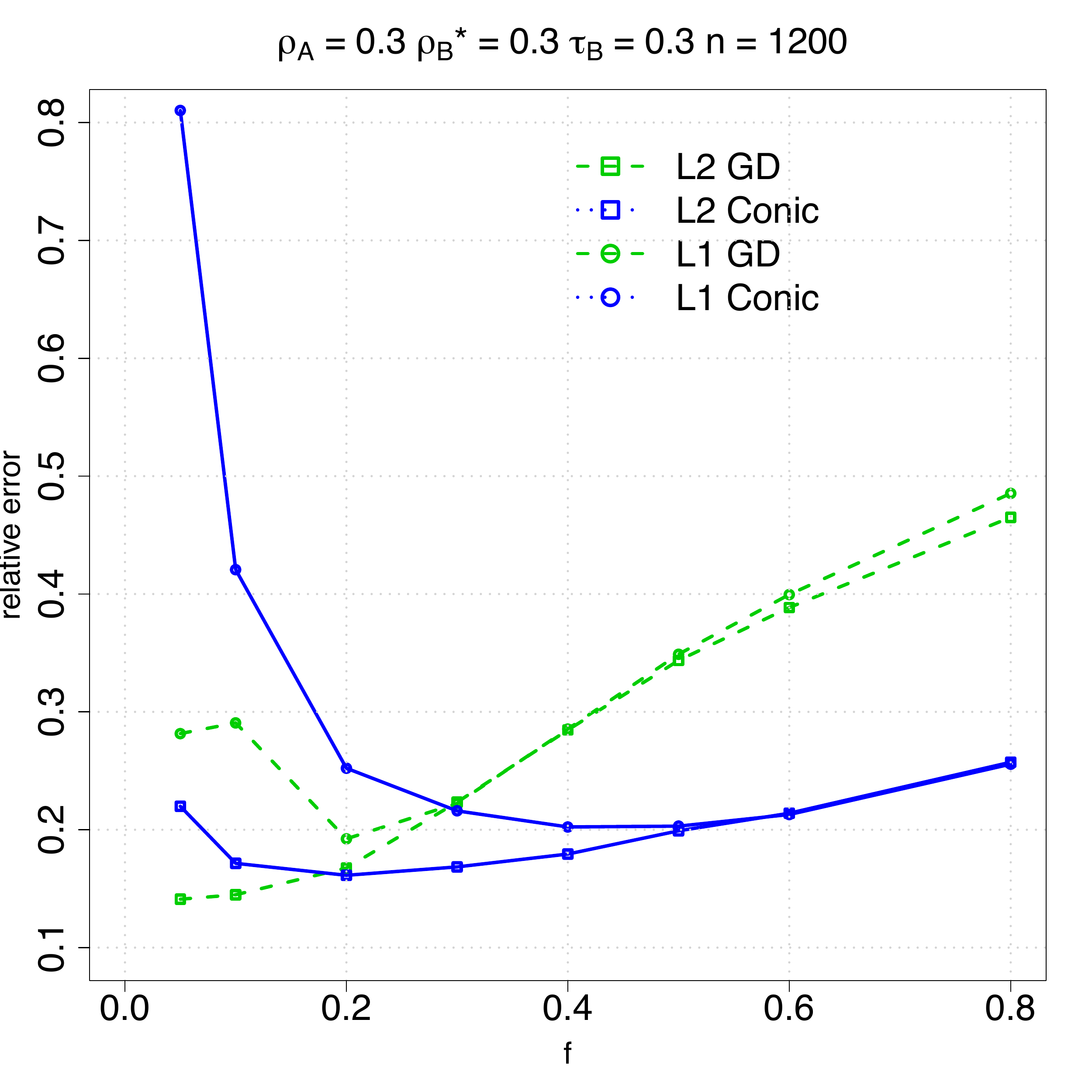} 
\end{tabular} & 
\begin{tabular}{c}\includegraphics[width=0.48\textwidth]{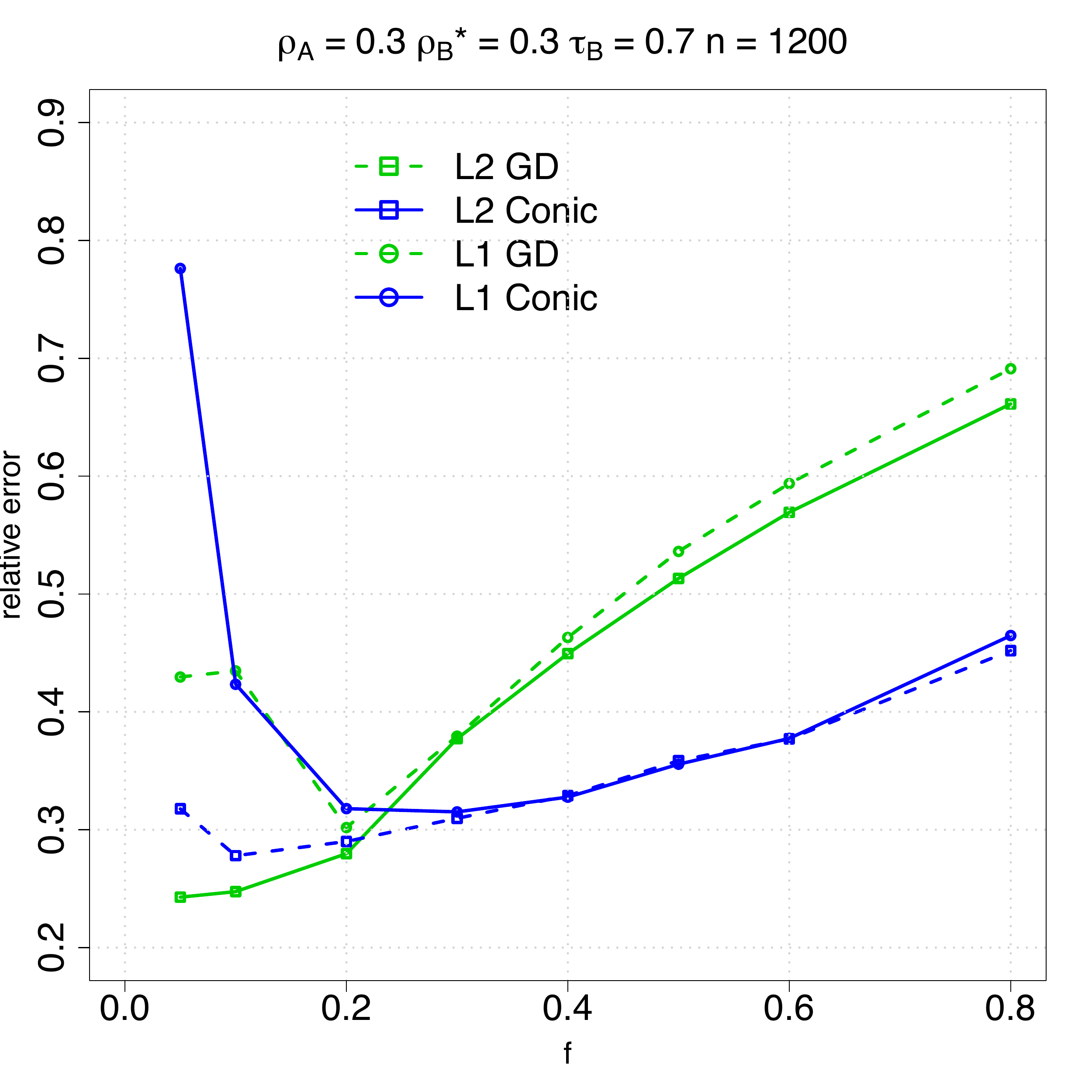} 
\end{tabular}\\
\end{tabular}
\caption{Plot of the relative error in $\ell_2$ and $\ell_1$ norm versus the penalty factor $f
  \in (0, 0.8]$ as we change the sample size $n$. Set $m = 1024$ and $d = 10$.
Both $A$ and $B$ are generated using the $\AR$ model with
parameters $\rho_A=0.3$ and $\rho_{B^*}=0.3$.  
We plot the relative error in $\ell_1$ and $\ell_2$ norm
versus the penalty parameter factor $f \in (0, 0.8]$ for $n = 300,
600, 1200$ when $\zeta = \frac{3}{2}\lambda_{\max}(A)$.
In the left  column, $\tau_B = 0.3$. In the right column, we set $\tau_B = 0.7$.
}
\label{fig:conic-gd-chain-chain-3}
\end{center}
\end{figure}
In Figure~\ref{fig:conic-gd-chain-chain-3}, we plot the $\ell_2$ and $\ell_1$ 
error versus the penalty factor $f \in [0.05, 0.8]$ for sample size $n
\in \{300, 600, 1200\}$. We plot results for $\tau_B = 0.3$ and
$\tau_B = 0.7$ in the left and right column respectively.
For these plots, we focus on cases when $n > d \kappa(A) \log m$, 
by choosing $n \in \{300, 600, 1200\}$; Otherwise, the gradient descent algorithm does not yet
reach the sample requirement \eqref{eq::nlower} that guarantees computational
convergence. In Figure~\ref{fig:conic-gd-chain-chain-3}, we observe that the Conic programming estimator is relatively
stable over the choices of $\mu$ once $f \ge 0.2$. 
The composite gradient algorithm favors smaller 
penalties such as $f \in [0.05, 0.2]$, leading to smaller relative error in 
the $\ell_1$ and $\ell_2$ norm, consistent with our theoretical 
predictions. These results also confirm our theoretical prediction that the Lasso
and Conic programming penalty parameters $\lambda$ and $\mu$ need to be
adaptively chosen based on the noise level $\tau_B$, because a larger than
necessary amount of penalty will cause larger relative error in both 
$\ell_1$ and $\ell_2$ norm.

\subsection{Sensitivity to tuning parameters}
In the third experiment, we change the $\ell_1$-ball radius $R \in
\{R^*, 5 R^*, 9 R^*\}$ in~\eqref{eq::origin2}, where $R^* =
\shnorm{\beta^*} \sqrt{d}$, while running through different penalties
for the composite gradient descent algorithm.
In the left column in Figure~\ref{fig:changeR},
$A$ and $B$ are generated using the $\AR$ model with
$\rho_A=0.3$, $\rho_{B^*}=0.3$ and $\tau_B = 0.7$. 
In the right column, we set $\tau_B = 0.3$, while keeping other
parameters invariant. 

As predicted by our theory, a larger radius
demands correspondingly larger penalty to ensure consistent estimation
using the composite gradient descent algorithm; this in turn will
increase the relative error when  $R$ is too large, for
example, when $R = \tilde{\Omega}(\sqrt{\frac{n}{\log  m}})$, 
where the $\tilde{\Omega}(\cdot)$ notation hides parameters involving
$\tau_B$ and $\kappa(A)$.
This is observed in Figure~\ref{fig:changeR}.
When $n$ is sufficiently large relative to $\tau_B$ and $\kappa(A)$, 
the optimal $\ell_1$ and $\ell_2$ error become less sensitive with 
regard to the choice of $R$, so long as $R = \tilde{O}(\sqrt{\frac{n}{\log m}})$,
where $\tilde{O}(\cdot)$ hides parameters involving $\tau_B$ and
$\kappa(A)$,  
as shown in Figure~\ref{fig:changeR}.

\begin{figure}
\begin{center}
\begin{tabular}{cc}
\begin{tabular}{c}\includegraphics[width=0.48\textwidth]{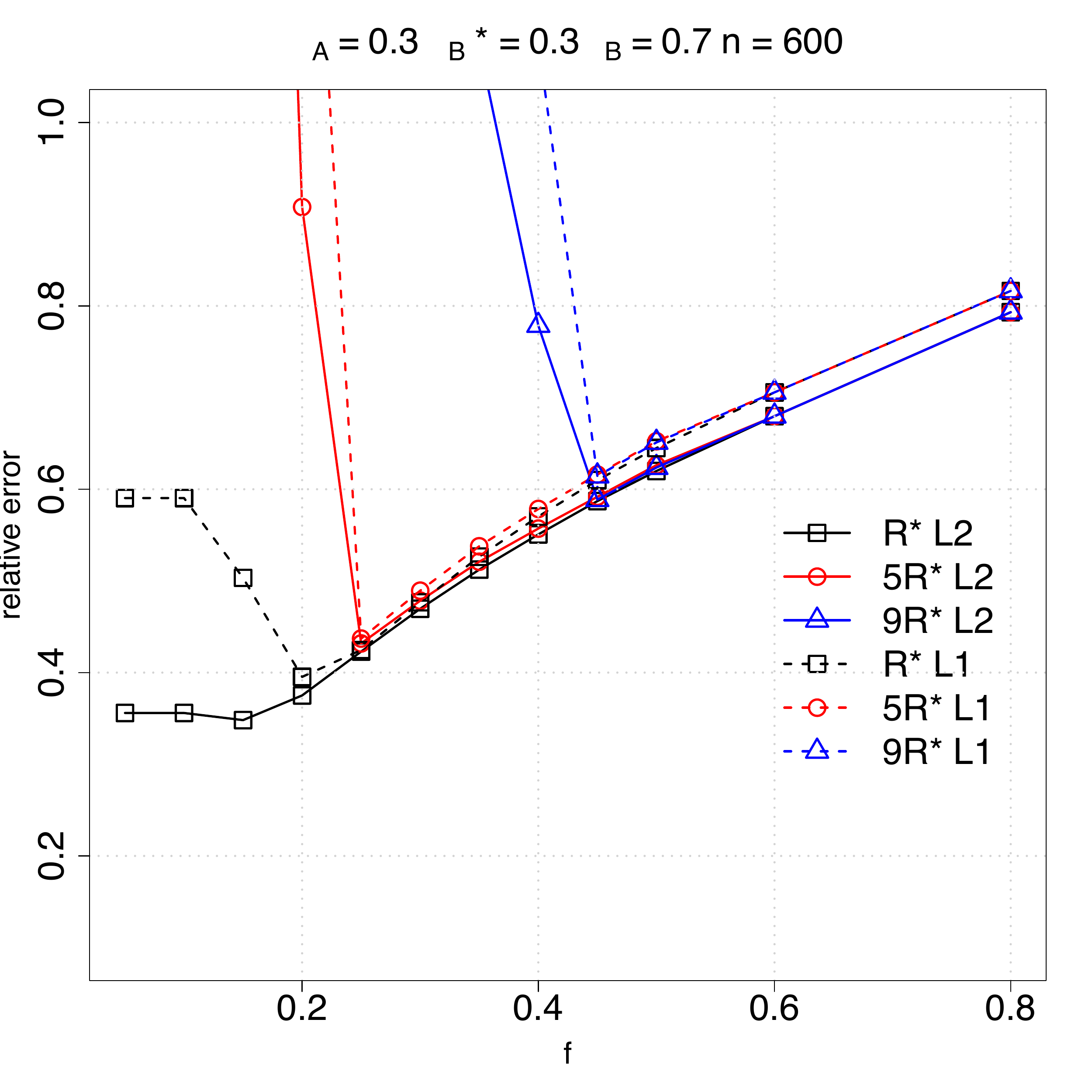}
\end{tabular}&
\begin{tabular}{c}\includegraphics[width=0.48\textwidth]{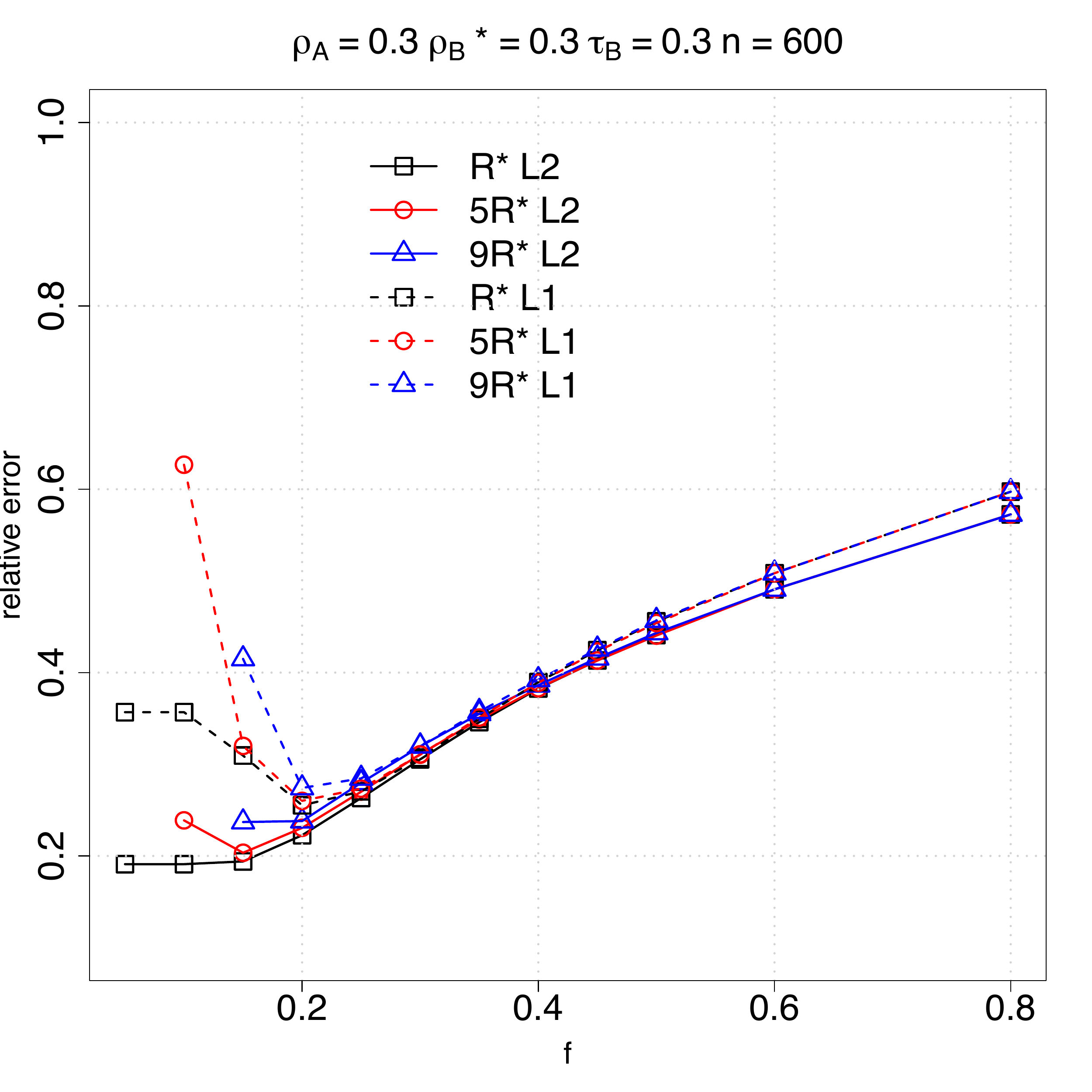}
\end{tabular} \\
\begin{tabular}{c}\includegraphics[width=0.48\textwidth]{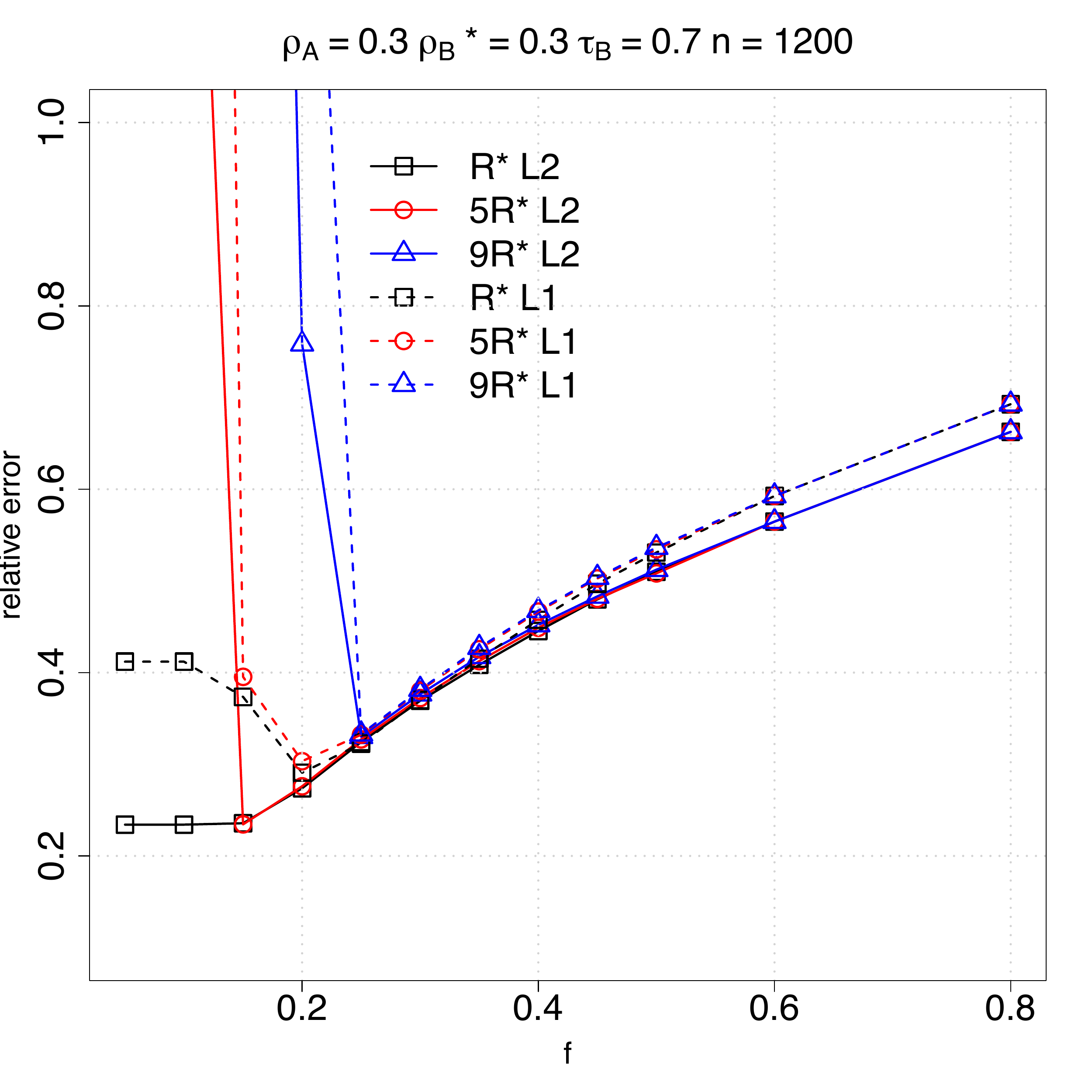}
\end{tabular}&
\begin{tabular}{c}\includegraphics[width=0.48\textwidth]{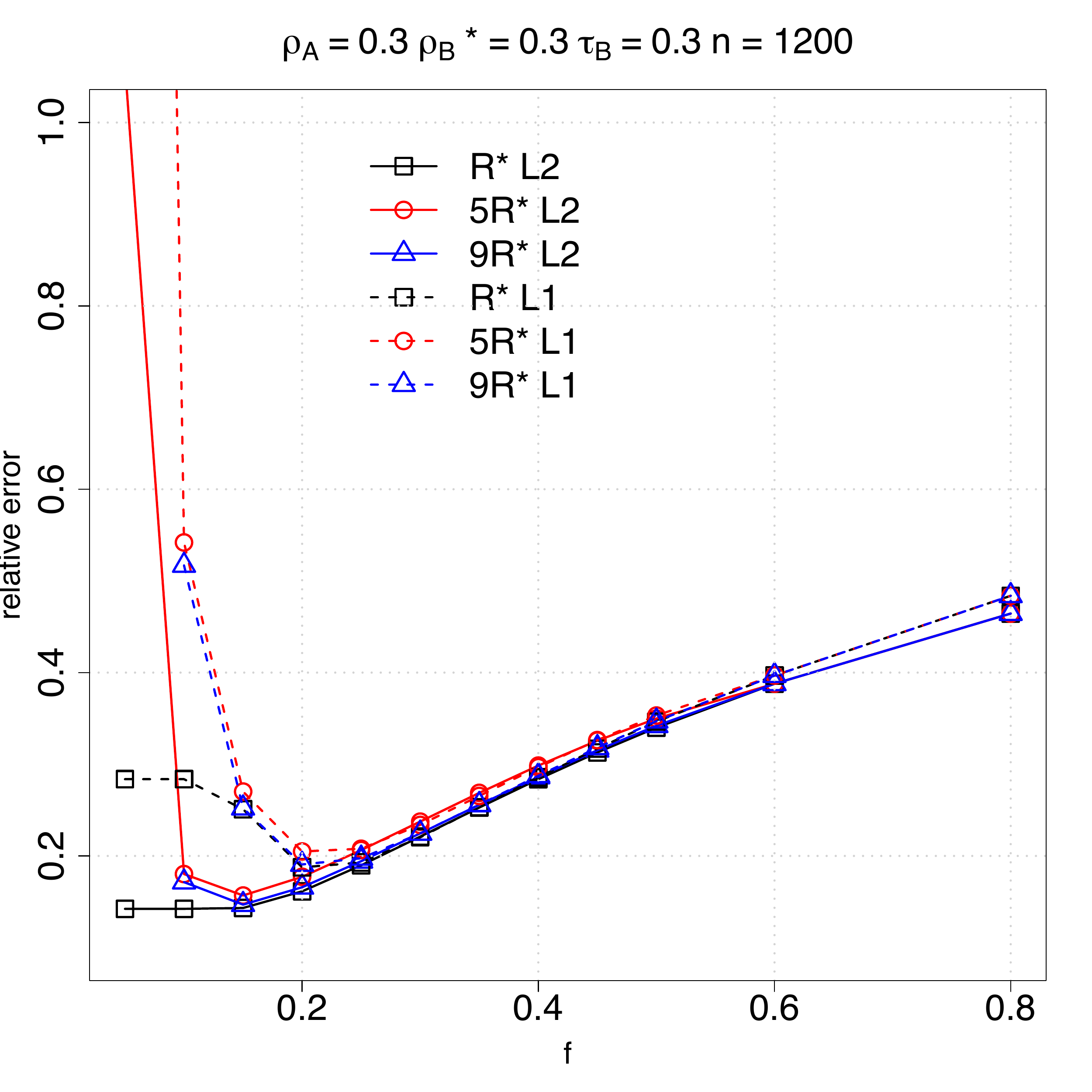}
\end{tabular} \\
\begin{tabular}{c}\includegraphics[width=0.48\textwidth]{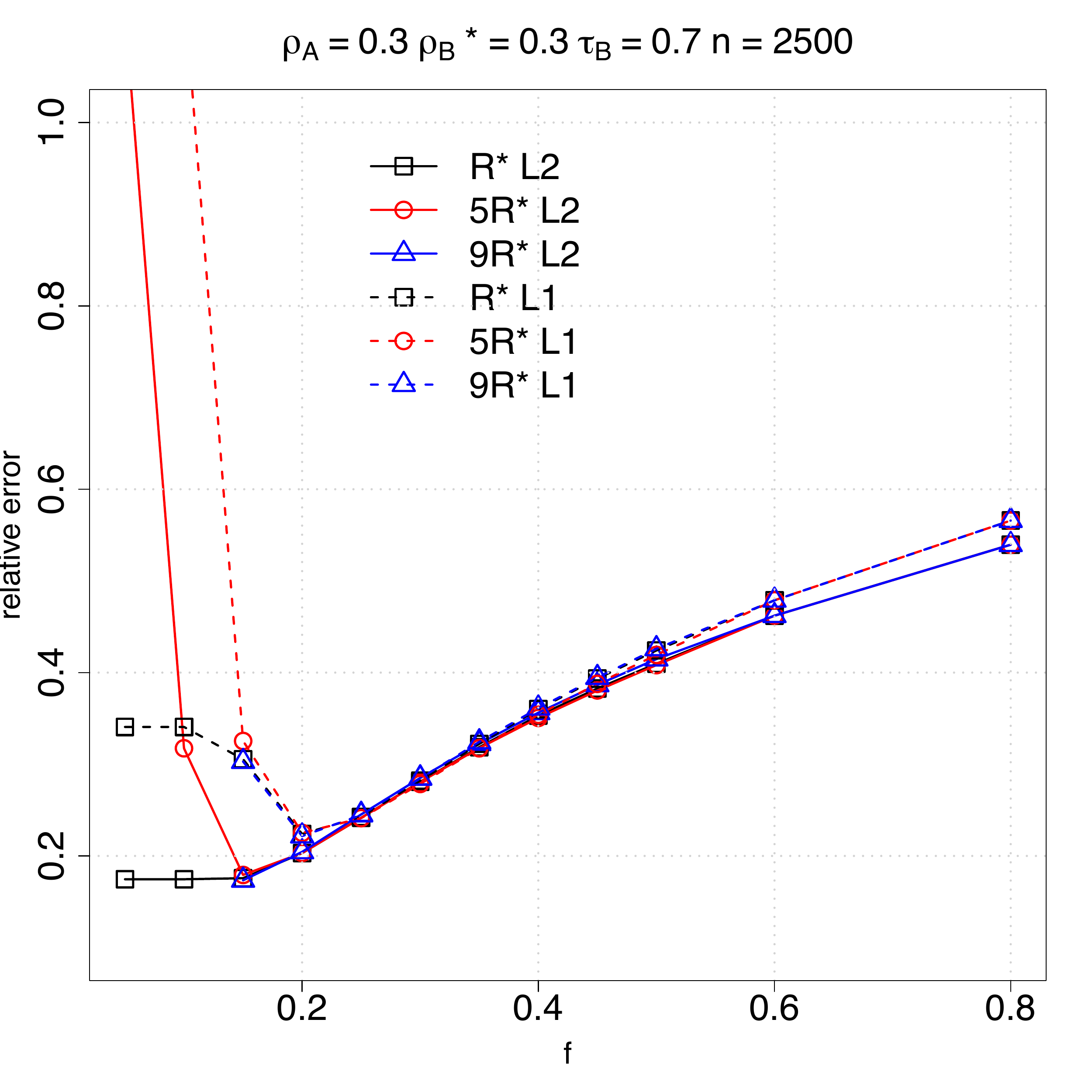}
\end{tabular} &
\begin{tabular}{c}\includegraphics[width=0.48\textwidth]{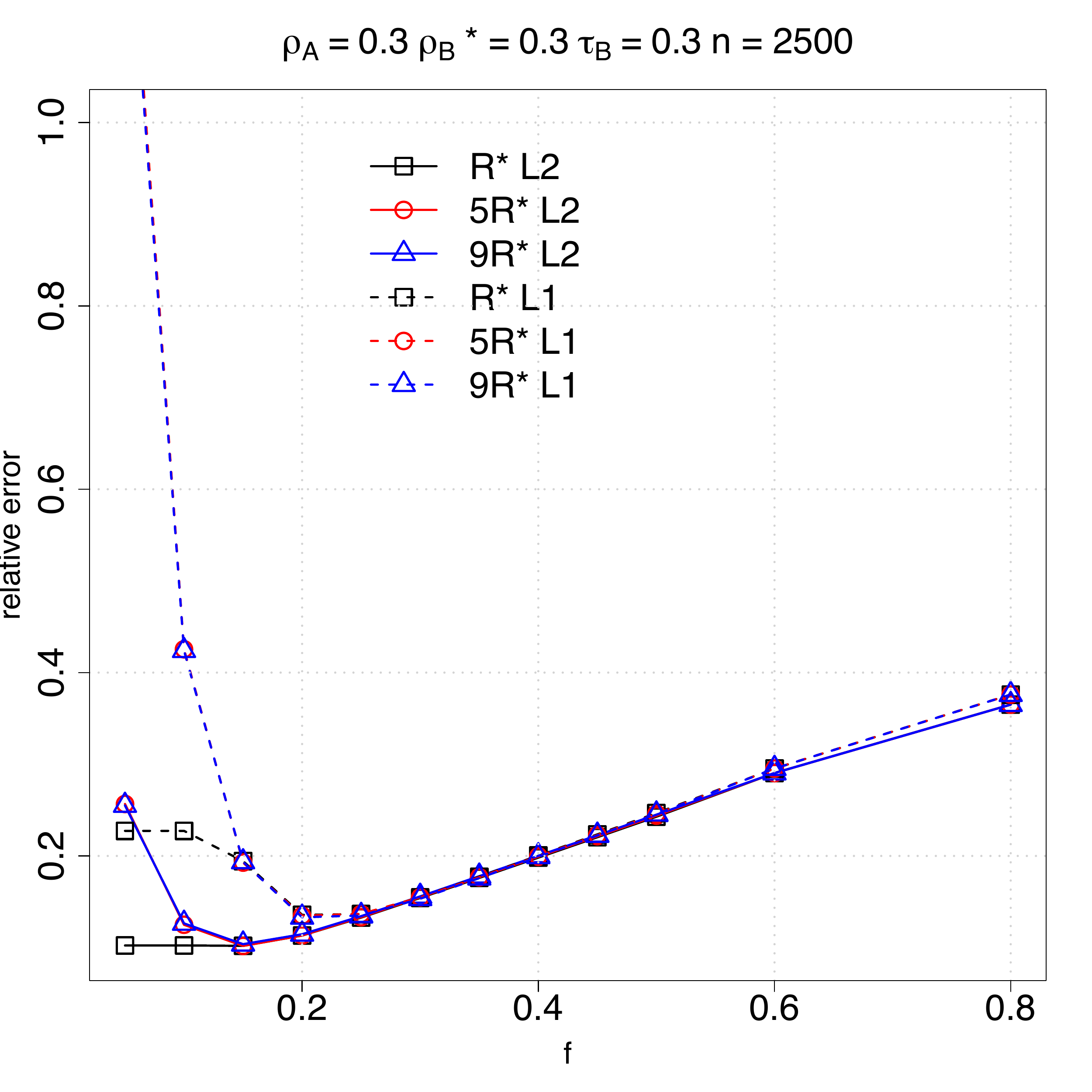} 
\end{tabular}\\
\end{tabular}
\caption{
Plot of the relative error in $\ell_2$ and $\ell_1$ norm versus the penalty factor $f
  \in (0, 0.8]$ as we change the radius $R$. Set $m = 1024$, $d = 10$
  and  $n \in \{ 600, 1200, 2500\}$.
We change the $\ell_1$-ball radius $R \in \{R^*, 5 R^*, 9 R^*\}$,
where $R^* = \shnorm{\beta^*} \sqrt{d}$,  
while running through different penalties for the composite gradient descent algorithm.
In the left column, $A$ and $B$ are generated using the $\AR$ model with
$\rho_A=0.3$, $\rho_{B^*}=0.3$ and $\tau_B = 0.7$. 
In the right column, we set $\tau_B = 0.3$, while keeping other parameters invariant. 
}
\label{fig:changeR}
\end{center}
\end{figure}

\begin{figure}
\begin{center}
\begin{tabular}{cc}
\begin{tabular}{c}\includegraphics[width=0.48\textwidth]{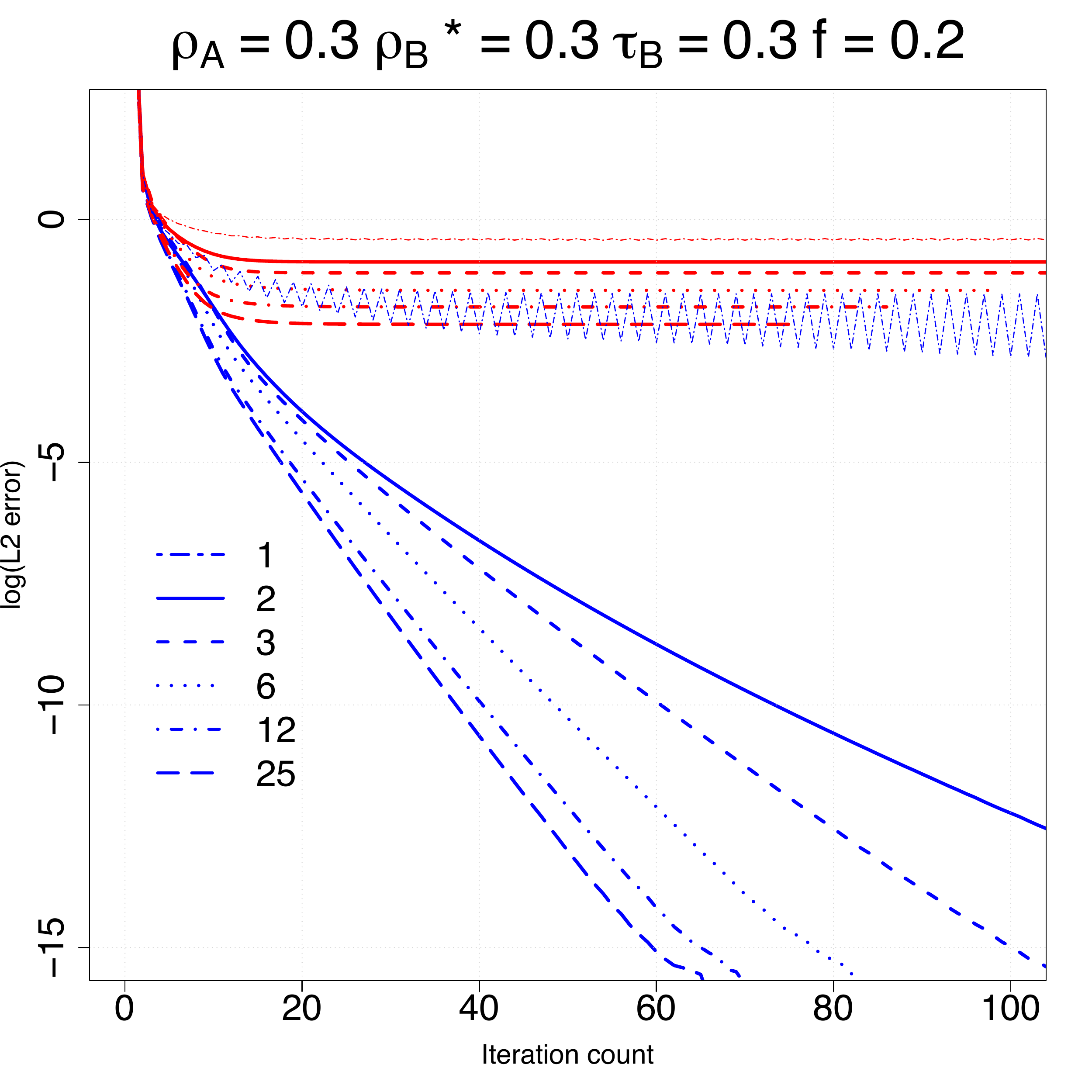}
\end{tabular}&
\begin{tabular}{c}\includegraphics[width=0.48\textwidth]{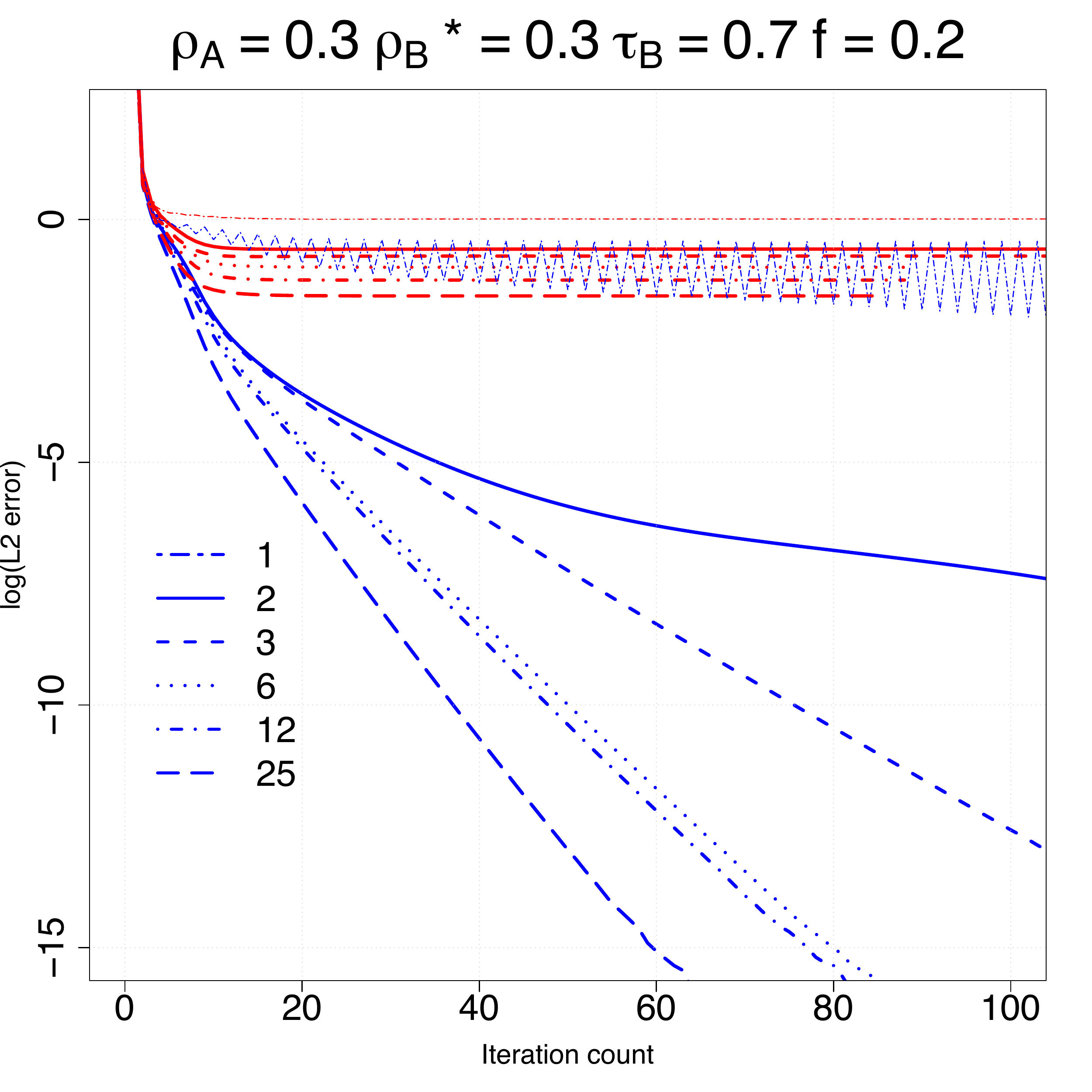}
\end{tabular}\\
(a) &(b)\\
\end{tabular}
\caption{Plots of the statistical error $\log (\shnorm{\beta^t -
    \beta^*})$, and the optimization error $\log (\shnorm{\beta^t -
    \hat\beta})$ versus iteration number $t$, generated by running the
  composite gradient descent algorithm on the corrected Lasso objective function.
Each curve represents an average over 10 random trials, each with a different
initialization point of $\beta^0$.
In Plots (a) and (b), $B$ is generated using the $\AR$ model with
$\rho_{B^*}=0.3$ and $A$ is generated using the $\AR$ model with
$\rho_A=0.3$. We set $\tau_B = 0.3, 0.7$ in Plot (a) and (b) respectively.
We set $n = \lceil \rho d \log m \rceil$, where we vary  $\rho \in
\{1, 2, 3, 6, 12, 25\}$.}
\label{fig:stat-opt-R}
\end{center}
\end{figure}

\subsection{Statistical and optimization  error in Gradient Descent}
In the last set of experiments, we study the statistical error and optimization error for each
iteration within the composite gradient descent algorithm.
We observe a geometric convergence of the optimization error
$\shnorm{\beta^t - \hat\beta}$. 

For each experiment, we repeat the following procedure 10 times: we start with a
random initialization point $\beta^0$ and apply the composite
gradient descent algorithm to compute an estimate $\hat\beta$; 
we compute the optimization error $\log (\shnorm{\beta^t - \hat\beta})$,
which records the difference between $\beta^t$ and $\hat\beta$, where
$\hat\beta$ is the final solution.
In all simulations, we plot the log error $\log (\shnorm{\beta^t - \hat\beta})$
between the iterate $\beta^t$ at time $t$ versus the final solution
$\hat\beta$, as well as the statistical error $\log(\shnorm{\beta^t -
  \beta^*})$, which is the difference between $\beta^t$ and $\beta^*$  at time $t$.
 Each curve plots the results averaged over ten random instances.

In the first experiment, both $A$ and $B$ are generated using the $\AR$ model with
parameters $\rho_A=0.3$ and $\rho_{B^*}=0.3$.  We set $m=1024$,  $d=10$
and $\tau_B \in \{0.3, 0.7\}$. 
These results are shown in Figure~\ref{fig:stat-opt-R}.
Within each plot, the red curves show the statistical error and the blue curves show the optimization error. 
We can see the optimization error $\shnorm{\beta^t - \hat\beta}$ decreases exponentially for each
iteration, obeying a geometric convergence.
To illuminate the dependence of convergence rate on the sample size $n$, we
study the optimization error $\log (\shnorm{\beta^t - \hat\beta})$ when
 $n = \lceil \rho d \log m \rceil$, where we vary 
$\rho \in \{1, 2, 3, 6, 12, 25\}$.
When $n = d\log m$,  the composite gradient algorithm fails to
converge since the sample size is too small for the RSC/RSM conditions to
hold, resulting in the oscillatory behavior of the algorithm for a constant step size.
As the factor $\rho$ increases, the lower and upper RE curvature
$\alpha$ and smoothness parameter $\tilde{\alpha}$
become more concentrated around
$\lambda_{\min}(A)$ and $\lambda_{\max}(A)$ respectively, and the
tolerance parameter $\tau$ decreases at the rate of $\frac{\log
  m}{n}$. 
Hence we observe faster rates of convergence for 
$\rho = 25, 12, 6$ compared to  $\rho = 2, 3$.
This is well aligned with our theoretical prediction that 
once $n =\Omega(\kappa(A) \frac{\tau_0}{\lambda_{\min}(A)} d \log m)$ (cf.~\eqref{eq::nlower}), 
we expect to observe a geometric convergence of the computational
error $\shnorm{\beta^t -  \hat\beta}$.

For the statistical error, we first observe the geometric contraction,
and then the curves flatten out after a certain number of iterations,
confirming the claim that $\beta^t$ converges to $\beta^*$ only up to a
neighborhood of radius defined through the statistical error bound
$\ve_{\static}^2$;  that is, the geometric convergence is
not guaranteed to an arbitrary precision, but only to an accuracy
related to statistical precision of the problem measured by
$\ell_2$ error: $\shtwonorm{\hat\beta - \beta^*}^2 =:\ve_{\static}^2$ between the global
optimizer $\hat\beta$ and the true parameter $\beta^*$.

In the second experiment, $A$ is generated from the Star-Block model,
where we have 32 subgraphs and each subgraph has 16 edges; $B$ is
generated using the random graph model with $n \log n$ edges and adjusted to have $\tau_B = 0.3$. 
We set $m = 1024$, $n = 2500$ and $d=10$.  
We then choose $\rho_A \in \{0.3, 0.5, 0.7, 0.9\}$. The results are shown in
Figure~\ref{fig::topology}(b).
As we increase $\rho_A$, we need larger sample size to control the
statistical error. Hence for a fixed $n$, the statistical error is
bigger for $\rho_A = 0.7$, compared to cases where $\rho_A = 0.5$ or
$\rho_A = 0.3$, for which we have  $\kappa(A) = 42.06$ and $\kappa(A) = 10.2$ (for $\rho_A = 0.3$) respectively;
Moreover, the rates of convergence are faster for the latter two
compared to $\rho_A = 0.7$, where $\kappa(A) = 169.4$. 
When $\rho_A =0.9$, the composite gradient descent algorithm fails to
converge as $\rho(A)$ is too large (hence not plotted here) with
respect to the sample size we fix upon. In
Figure~\ref{fig::topology}(a), we show results of $A$ being generated using the $\AR$ model with
four choices of $\rho_A \in \{0.3, 0.5, 0.7, 0.9\}$ and $B$ being 
generated using the $\AR$ model with $\rho_{B^*}=0.7$ and $\tau_B = 0.3$.
We observe quantitively similar behavior as in Figure~\ref{fig::topology}(b).
\begin{figure}
\begin{center}
\begin{tabular}{cc}
\begin{tabular}{c}\includegraphics[width=0.48\textwidth]{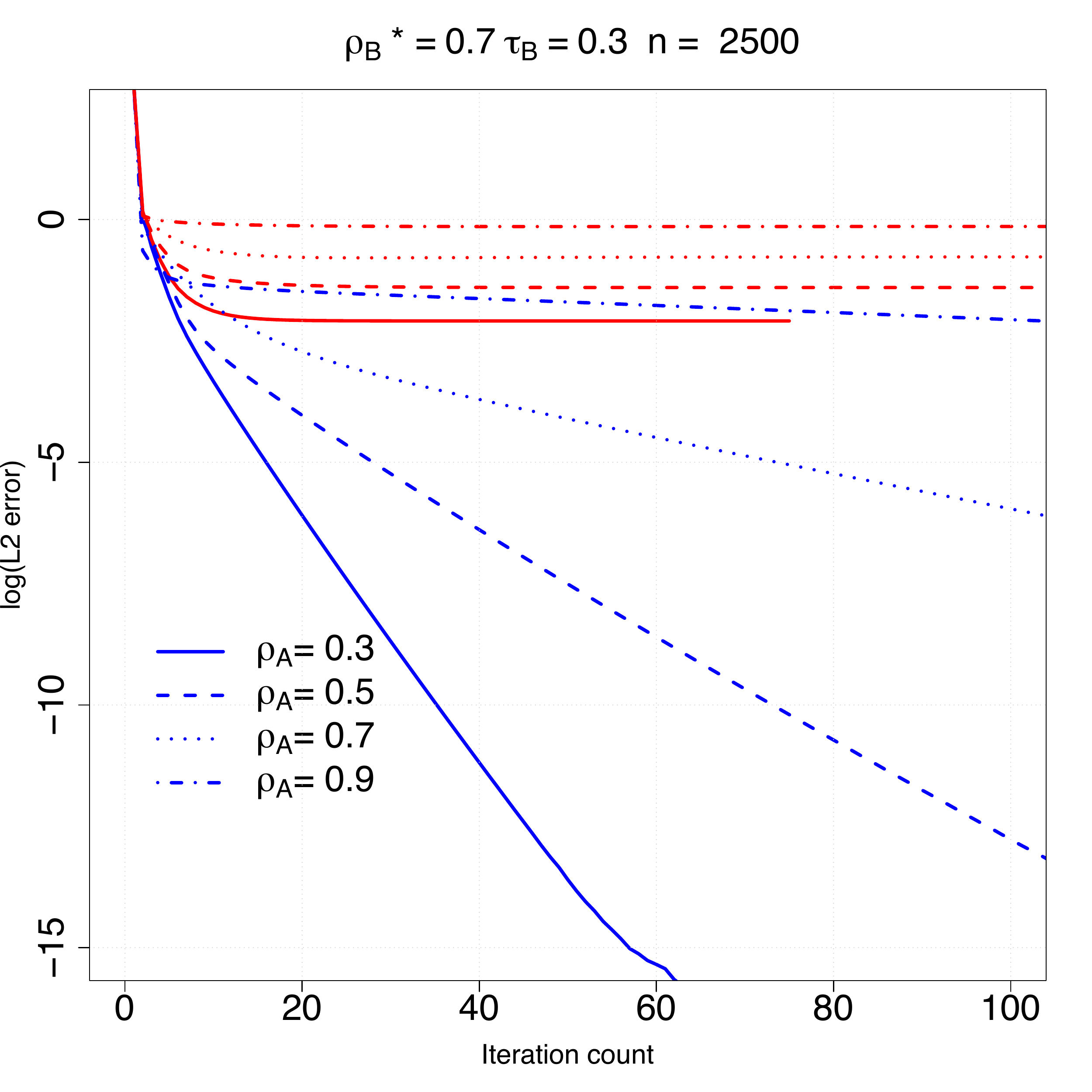}
\end{tabular}&
  \begin{tabular}{c}\includegraphics[width=0.48\textwidth]{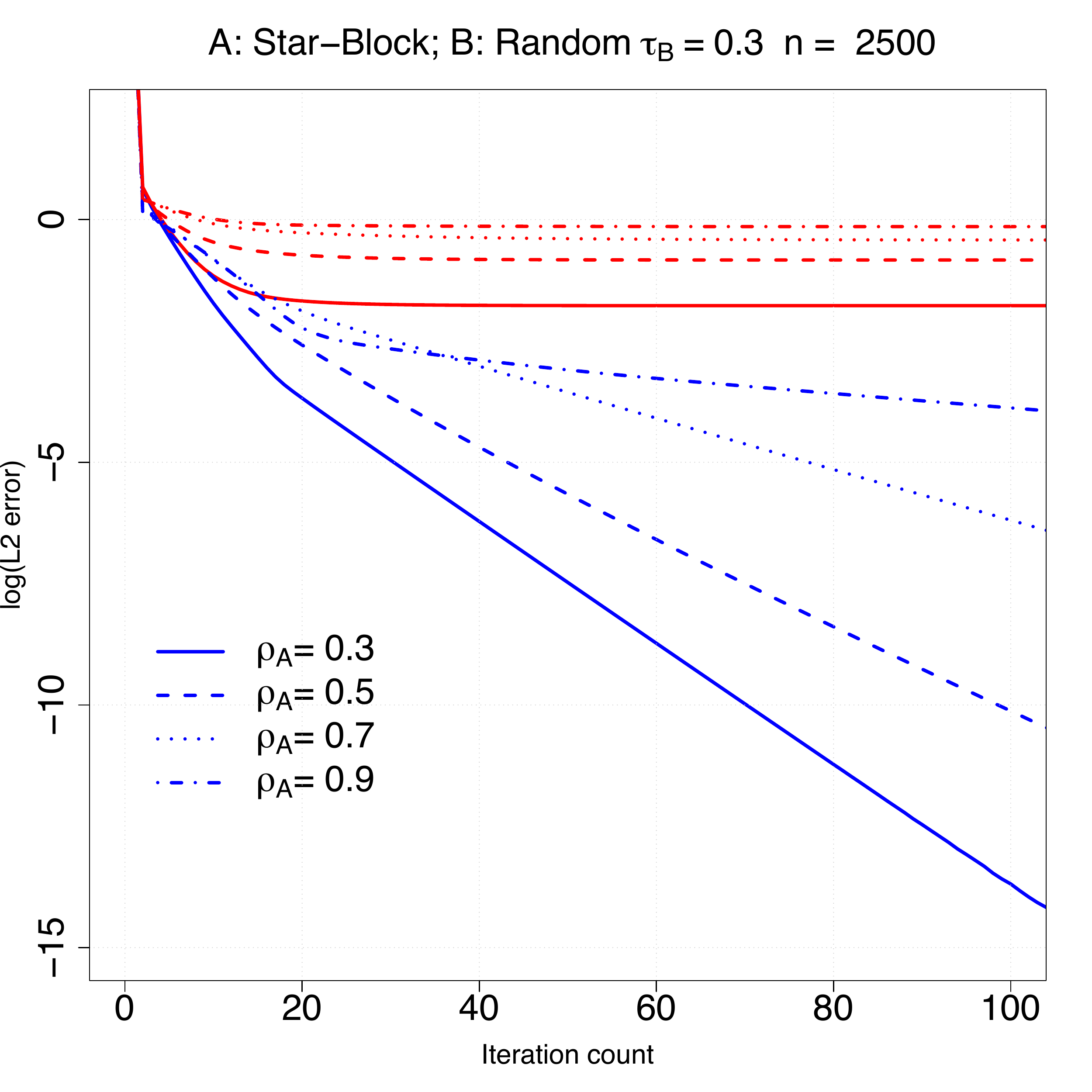}
\end{tabular}\\
(a) &(b)\\
\end{tabular}
\caption{Plots of the statistical error $\log (\shnorm{\beta^t - \beta^*})$, and the optimization error when we change the
  topology. In the last experiment, we have $m = 1024$, $d=10$ and $n = 2500$.
In (a), $A$ is generated using the $\AR$ model with
four choices of $\rho_A \in \{0.3, 0.5, 0.7, 0.9\}$ and 
$B$ is generated using $\AR$ model with $\rho_{B^*}=0.7$ and $\tau_B = 0.3$.
In (b), $A$ follows the Star-Block model and $B$ follows
the random graph model. We show four choices of $\rho_A \in \{0.3,
0.5, 0.7, 0.9\}$}
\label{fig::topology}
\end{center}
\end{figure}

\section{Proof of Lemma~\ref{lemma::REcomp}}
\label{sec::proofoflemmaREcomp}
\begin{proofof}{Lemma~\ref{lemma::REcomp}}
Part I: Suppose that the Lower-$\RE$ condition holds for $\Gamma := A^T A$.
Let $x \in \W(s_0, k_0)$. Then 
\bens
\onenorm{x} \le (1+k_0) \onenorm{x_{T_0}} \le  (1 + k_0)\sqrt{s_0} \twonorm{x_{T_{0}}}.
\eens
Thus for $x \in \W(s_0, k_0) \cap S^{p-1}$ and  $\tau (1 + k_0)^2 s_0 \le \alpha/2$, 
we have 
\bens
\twonorm{A x} = (x^T A^T A x)^{1/2} & \ge &  
\left(\alpha \twonorm{x}^2 - \tau \onenorm{x}^2\right)^{1/2} \\
& \ge &  
\left(\alpha \twonorm{x}^2 - \tau (1 + k_0)^2 s_0 \twonorm{x_{T_{0}}}^2\right)^{1/2} \\& \ge &  
\left(\alpha - \tau (1 + k_0)^2 s_0 \right)^{1/2}\ge
\sqrt{\frac{\alpha}{2}}.
\eens
Thus the $\RE(s_0, k_0, A)$ condition holds with
\bens
\inv{K(s_0, k_0, A)} & := & 
\min_{x \in \W(s_0, k_0)}\frac{\twonorm{Ax}}{\twonorm{x_{T_0}}}
\ge \sqrt{\frac{\alpha}{2}} 
\eens
where we use the fact that for any $J \in \{1, \ldots, p\}$ such that $\abs{J} \le
s_0$, $\twonorm{x_J} \le \twonorm{x_{T_0}}$.
We now show the other direction.

Part II. 
Assume that $\RE(4R^2, 2R-1, A)$ holds for some integer $R > 1$.
Assume that for some $R  >1$
\bens
\onenorm{x} \le R \twonorm{x}.
\eens
Let $(x_i^*)_{i=1}^p$ be non-increasing arrangement of
$(\abs{x_i})_{i=1}^p$. Then
\bens
\onenorm{x} 
& \le & R \left(\sum_{j=1}^s (x^*_j)^2 +
  \sum_{j=s+1}^{\infty}\left(\frac{\onenorm{x}}{j}\right)^2
\right)^{1/2} \\
& \le & 
R \left(\twonorm{x^*_{J}}^2 + \onenorm{x}^2 \inv{s}\right)^{1/2}  
\le R \left(\twonorm{x^*_{J}} + \onenorm{x}  \inv{\sqrt{s}}\right)
\eens
where $J := \{1, \ldots, s\}$.
Choose $s = 4R^2$.   Then
\bens
\onenorm{x} \le R\twonorm{x^*_{J}} + \half \onenorm{x}.
\eens
Thus we have 
\ben
\label{eq::T02bound}
\onenorm{x} & \le & 2 R\twonorm{x^*_{J}} \le  2 R\onenorm{x^*_{J}} \;
\; \text{ and hence } \; \; \\
\onenorm{x^*_{{J}^c}} & \le &  (2 R - 1) \onenorm{x^*_{J}}.
\een
Then $x \in \W(4R^2, 2R-1)$.
Then for all $x \in S^{p-1}$ such that $\onenorm{x} \le R \twonorm{x}$,
we have for $k_0 =2R-1$ and $s_0 := 4R^2$,
\bens
x^T \Gamma x \ge \frac{\twonorm{x_{T_0}}^2 }{K^2(s_0, k_0, A)} 
\ge 
\frac{\twonorm{x}^2}{\sqrt{s_0} K^2(s_0, k_0, A)} =:
\alpha \twonorm{x}^2
\eens
where we use the fact that $(1 + k_0) \twonorm{x_{T_{0}}}^2 \ge
\twonorm{x}^2$ by Lemma~\ref{lemma::lower-bound-Az} with $x_{T_0}$ as
defined therein. 
Otherwise, suppose that $\onenorm{x} \ge R \twonorm{x}$.
Then for a given $\tau > 0$,
\ben
\label{eq::alphatau}
\alpha\twonorm{x}^2 - \tau \onenorm{x}^2 \le
(\inv{\sqrt{s_0}K^2(s_0, k_0, A)} - \tau R^2) \twonorm{x}^2.
\een
Thus we have by the choice of $\tau$ as in~\eqref{eq::lemmatauchoice}  and \eqref{eq::alphatau}
\bens
x^T \Gamma x  \ge  
\lambda_{\min}(\Gamma) \twonorm{x}^2 &\ge & 
(\inv{\sqrt{s_0}K^2(s_0, k_0, A)} - \tau R^2) \twonorm{x}^2 \\
& \ge & 
\alpha \twonorm{x}^2 - \tau \onenorm{x}^2.
\eens
The Lemma thus holds.
\end{proofof}

\section{Proof of Theorem~\ref{thm::lasso}}
\label{sec::proofofthmlasso}
\begin{proofof2}
Throughout this proof, we assume that $\A_0 \cap \B_{0}$ holds. 
First we note that it is sufficient to have~\eqref{eq::trBLasso} in order for \eqref{eq::trBlem} to hold. 
Condition~\eqref{eq::trBLasso} guarantees that
for $\V=3 e M_A^3/2$,
\ben
\nonumber
r(B) := \frac{\tr(B)}{\twonorm{B}} & \ge & 16 c' K^4 \frac{n}{\log m}
\log \frac{\V m \log m }{n}  \\
\nonumber
&\ge &
16 c' K^4 \frac{n}{\log m} 
\log\left(\frac{3 e m M_A^3\log m }{2n}\right) \\
\nonumber
&= & c' K^4 \inv{\ve^2} \frac{4}{M_A^2}\frac{n}{\log m} 
\log\left(\frac{6 e m M_A }{\frac{4}{M_A^2}(n/\log m)}\right) \\
\label{eq::trBlemup}
& \ge & 
c' K^4 \inv{\ve^2} s_0
\log\left(\frac{6 e m M_A }{s_0} \right) = c' K^4\frac{s_0}{\ve^2} \log\left(\frac{3 e m}{s_0 \ve }\right)
\een
where $\ve = \inv{2M_A} \le \inv{128 C}$, and the last inequality
holds given that $k \log (c m/k)$ on the 
RHS of~\eqref{eq::trBlemup} is a 
monotonically increasing function of $k$,
\bens
s_0 & \le & \frac{4 n}{M_A^2 \log m} 
\; \text{  and } \;  M_A =
  \frac{64 C (\rho_{\max}(s_0,A) +
  \tau_B)}{\lambda_{\min}(A)} \ge 64 C.
\eens
Next we check that the choice of $d$ as in~\eqref{eq::dlasso} ensures
that \eqref{eq::dlassoproof} holds for $D_{\phi}$ defined there.
Indeed, for $c' K^4 \le 1$, we have
\bens
d  & \le  & C_A (c' K^4 \wedge 1)\frac{\phi n}{\log m}
\le C_A  \left(c' D_{\phi} \wedge  1\right) \frac{n}{\log  m}.
\eens
By Lemma~\ref{lemma::lowerREI}, we have on event $\A_0$, the
modified gram matrix $\hat \Gamma_A :=   \onen (X^T X - \hat\tr(B) I_{m})$
satisfies the Lower $\RE$ conditions with $\alpha$ and $\tau$ as in \eqref{eq::localtau}.
Theorem~\ref{thm::lasso} follows from Theorem~\ref{thm::main}, so long
as we can show that condition~\eqref{eq::taumain} holds for $\lambda \ge 4 \psi
\sqrt{\frac{\log m}{n}}$, where  the parameter $\psi$ is as
defined~\eqref{eq::psijune},
\ben
\label{eq::localtau}
\text{ curvature}\; \;
\alpha = \frac{5}{8} \lambda_{\min}(A)
\text{ and tolerance } \; \; \tau = \frac{\lambda_{\min}(A) - \alpha}{s_0} = \frac{3\alpha}{5 s_0}.
\een
Combining~\eqref{eq::localtau} and~\eqref{eq::taumain},
we need to show~\eqref{eq::dcond} holds.
This is precisely the content of Lemma~\ref{lemma::dmain}. 
This is the end of the proof for Theorem~\ref{thm::lasso}
\end{proofof2}

\section{Proof of Theorem~\ref{thm::DS}}
\label{sec::proofofDSthm}
For the set $\Cone_J(k_0)$ as in~\eqref{eq::cone-init}, 
\bens
\kappa_{\RE}(d_0, k_0)  & := &  \min_{J : \abs{J} \le d_0}
\min_{\Delta \in \Cone_J(k_0)} 
\frac{\abs{\Delta^T \Psi \Delta}}{\twonorm{\Delta_J}^2} 
= \left(\inv{K(d_0, k_0,  (1/\sqrt{n}) Z_1 A^{1/2})}\right)^2.   
\eens
Recall the following Theorem \ref{thm:subgaussian-T-intro} from~\cite{RZ13}.
\begin{theorem}{(\textnormal{\cite{RZ13}})}
\label{thm:subgaussian-T-intro}
Set $0< \delta < 1$,  
$k_0 > 0$, and $0< d_0 < p$.
Let $A^{1/2}$ be an $m \times m$ matrix satisfying $\RE(d_0, 3k_0, A^{1/2})$ condition
as in Definition~\ref{def:memory}. 
Set 
\bens
d & = & d_0 + d_0 \max_j  \twonorm{A^{1/2} e_{j}}^2 \frac{16 K^2(d_0, 3k_0, A^{1/2}) (3k_0)^2 (3k_0 + 1)}{\delta^2}.
\eens
Let $\Psi$ be an $n \times m$ matrix whose rows are
independent isotropic $\psi_2$ random vectors in $\R^m$ with constant $\alpha$.
Suppose the sample size satisfies
\ben
\label{eq::UpsilonSampleBound-intro}
n \geq \frac{2000 d \alpha^4}{\delta^2} \log \left(\frac{60 e m}{d \delta}\right).
\een
Then with probability at least $1- 2 \exp(-\delta^2 n/2000 \alpha^4)$,
$\RE(d_0, k_0, (1/\sqrt{n})\Psi A^{1/2})$ condition holds for matrix $(1/\sqrt{n}) \Psi A$
with
\ben
\label{eq::RE-subg}
0< K(d_0, k_0,  (1/\sqrt{n}) \Psi A^{1/2}) \leq \frac{K(d_0, k_0, A^{1/2})}{1-\delta}.
\een
\end{theorem}

\begin{proofof}{Theorem~\ref{thm::DS}}
Suppose $\RE(2d_0, 3k_0, A^{1/2})$ holds.
Then for $d$ as defined in~\eqref{eq::sparse-dim-Ahalf} and $n
= \Omega(d K^4 \log (m/d))$, we have with probability at least $1 - 2
\exp(\delta^2 n/2000 K^4)$, 
$\RE(2d_0, k_0, \inv{\sqrt{n}} Z_1 A^{1/2})$ condition holds with 
\bens
\kappa_{\RE}(2d_0, k_0)  & = &  \left(\inv{K(2d_0, k_0,  (1/\sqrt{n}) Z_1 A^{1/2})}\right)^2  \ge
\left(\inv{2 K(2d_0, k_0,  A^{1/2})}\right)^2  
\eens
by Theorem~\ref{thm:subgaussian-T-intro}.

The rest of the proof follows from~\cite{BRT14} Theorem 1 and thus we only provide
a sketch. In more details, in view of the lemmas shown in
Section~\ref{sec::proofall}, we need
\bens
\kappa_{q}(d_0, k_0) \ge c d_0^{-1/q}
\eens
to hold for some constant $c$ for $\Psi := \onen X_0^T X_0$.
It is shown in Appendix C in~\cite{BRT14} that under the $\RE(2d_0,
k_0, \inv{\sqrt{n}} Z_1 A^{1/2})$ 
condition, for any $d_0 \le m/2$ and $1\le q\le 2$, 
\ben
\nonumber
\kappa_{1}(d_0, k_0) & \ge & c d_0^{-1} \kappa_{\RE}(d_0, k_0) \quad
\text{ and } \\
\label{eq::kqbound}
\kappa_{q}(d_0, k_0) & \ge & c(q) d_0^{-1/q} \kappa_{\RE}(2d_0, k_0),
\een
where $c(q) > 0$ depends on $k_0$ and $q$.
The theorem is thus proved following exactly the same line of
arguments as in the proof of Theorem 1 in~\cite{BRT14} in view of the
$\ell_q$ sensitivity condition derived immediately above, 
in view of Lemmas~\ref{lemma::DS},~\ref{lemma::DS-cone} and~\ref{lemma::grammatrix}.
Indeed, for $v := \hat\beta - \beta^*$, we have by definition
of $\ell_q$ sensitivity as in~\eqref{eq::sense},
\ben
\nonumber
 c(q) d_0^{-1/q} \kappa_{\RE}(2d_0, k_0) \norm{v}_q
&  \le &  
\nonumber
\kappa_{q}(d_0, k_0) \norm{v}_q \le  \norm{\onen X_0^T X_0 v}_{\infty}
\\
\nonumber
& \le & \mu_1 \twonorm{\beta^*} + \mu_2 \onenorm{v} + \omega  \\
\nonumber
& \le & \mu_1 \twonorm{\beta^*} + \mu_2 (2+\lambda) \onenorm{v_S} +
\omega 
\\
\nonumber
& \le & \mu_1 \twonorm{\beta^*} + \mu_2 (2+\lambda) d_0^{1-1/q}
\norm{v_S}_q + \omega  \\
\label{eq::prelude}
& \le & \mu_1 \twonorm{\beta^*} + \mu_2 (2+\lambda) d_0^{1-1/q} \norm{v}_q + \omega.
\een
Thus we have for $d_0 = c_0 \sqrt{n/\log m}$, where $c_0$ is sufficiently small,
\bens
&& d_0^{-1/q} 
(c(q) \kappa_{\RE}(2d_0, k_0)-  \mu_2 (2+\lambda)
  d_0) \norm{v}_q
\le \mu_1 \twonorm{\beta^*} + \omega  \\
\text{and hence }  && 
\norm{v}_q  \le  C( 4 D_2  \rho_{n} K \twonorm{\beta^*} + 2 D_0 M_{\e} \rho_{n} ) d_0^{1/q} \\
\nonumber
&  & \quad \quad \quad \le  4 C D_2  \rho_{n} (K \twonorm{\beta^*} + M_{\e}) d_0^{1/q}
\eens
for some constant $C = 1/\left(c(q) \kappa_{\RE}(2d_0, k_0)-  \mu_2 (2+\lambda)
  d_0\right) \ge 1/\left(2 c(q) \kappa_{\RE}(2d_0, k_0)\right)$, where
\bens
\mu_2(2+\lambda) d_0& = & 2 D_2 K \rho_{n}(\inv{\lambda} + 1)
(2+\lambda) c_0 \sqrt{n/\log m} \\
& = &  2 c_0 C_0 D_2 K^2 (2+\lambda) (\inv{\lambda} + 1)
\eens 
is sufficiently small and thus \eqref{eq::ellqnorm} holds.
The prediction error bound follows exactly the same line of arguments
as in~\cite{BRT14} which we omit here. 
See proof of Theorem~\ref{thm::DSoracle} in Section~\ref{sec::DSoracle} for details.
\end{proofof}

\section{Proof of Theorem~\ref{thm::lassora}}
\label{sec::classoproof}
\begin{proofof2}
Throughout this proof, we assume that $\A_0 \cap \B_{0}$ holds. 
The proof is also identical to the proof of Theorem~\ref{thm::lasso} up till
~\eqref{eq::localtau}, except that we replace the condition on $d$ 
as in the theorem statement by~\eqref{eq::doracle}.
Theorem~\ref{thm::lassora}  follows from Theorem~\ref{thm::main}, so long
as we can show that condition~\eqref{eq::taumain} holds for 
$\alpha$ and $\tau = \frac{\lambda_{\min}(A) - \alpha}{s_0}$ as defined 
in~\eqref{eq::localtau}, and $\lambda \ge 2 \psi
\sqrt{\frac{\log m}{n}}$, where the parameter $\psi$ is as
defined~\eqref{eq::psioracle}.
Combining~\eqref{eq::localtau} and~\eqref{eq::taumain}, 
we need to show~\eqref{eq::dcond} holds.
This is precisely the content of Lemma~\ref{lemma::dmainoracle}.
This is the end of the proof for Theorem~\ref{thm::lassora}.
\end{proofof2}

\section{Proof of Theorem~\ref{thm::DSoracle}}
\begin{proofof2}
\label{sec::DSoraproof}
Throughout this proof, we assume that $\B_0 \cap \B_{10}$ holds. 
The rest of the proof follows that of Theorem~\ref{thm::DS}, except for the last
part.
Let $\mu_1, \mu_2, \omega$ be as defined in
Lemma~\ref{lemma::grammatrix}. 
We have for $\mu_2  :=  2 \mu (1+\inv{2\lambda}) $, where  $\mu = D_0' K \rho_{n}  \tilde\tau_B^{1/2}$
and $d_0 =  c_0  \tau_B^-\sqrt{n/\log m}$, 
\ben
\label{eq::half}
\mu_2(2+\lambda) d_0 & = &
 2 C_0 D_0' K^2  \tilde\tau_B^{1/2} 
(\inv{2\lambda} + 1)
(2+\lambda)  c_0  \tau_B^- \\
\nonumber
&  \le &  2 c_0 C_0 D_0' K^2 (2+\lambda) (\inv{2\lambda} + 1)  
\le \half c(q) \kappa_{\RE}(2d_0, k_0),
\een
which holds when $c_0$ is sufficiently small,  where 
$\tau_B^{-}   \tilde\tau_B^{1/2} \le 1$
by \eqref{eq::tildetauBbound}.
Hence
\bens
\mu_2  d_0 \le \frac{c(q) \kappa_{\RE}(2d_0, k_0)}{2(2+\lambda)}.
\eens
Thus for $c_0$ sufficiently small, $\mu_1 = 2 \mu$, we have
by \eqref{eq::kqbound}, \eqref{eq::half}, \eqref{eq::prelude} and~\eqref{eq::tildetauB},
\ben
\nonumber
\lefteqn{ d_0^{-1/q} \half (c(q) \kappa_{\RE}(2d_0, k_0)) \norm{v}_q }\\
& =  & 
\nonumber
d_0^{-1/q} (c(q) \kappa_{\RE}(2d_0, k_0)-  \mu_2 (2+\lambda)  d_0) \norm{v}_q \\
\nonumber
&\le & (\kappa_{q}(d_0, k_0) - \mu_2 (2+\lambda) d_0^{1-1/q})
\norm{v}_q \le \mu_1 \twonorm{\beta^*} + \omega   \\
\label{eq::middle}
& \le & 
 2 D_0' \rho_{n} K^2 ((\tau_B^{1/2} + (3/2)C_{6} r_{m,m}^{1/2})\twonorm{\beta^*} + M_{\e}/K)
\een
and thus \eqref{eq::ellqnormimp} holds, following the proof in Theorem~\ref{thm::DS}.
The prediction error bound follows exactly the same line of arguments
as in~\cite{BRT14}, which we now include for the sake completeness.
Following~\eqref{eq::ellqnormimp}, we have by \eqref{eq::middle}, 
\bens
\onenorm{v} & \le &  
C_{11} d_0 (\mu_1 \twonorm{\beta^*} + \omega), \; \text{ where } \;
C_{11} = 2/\left(c(q) \kappa_{\RE}(2d_0, k_0)\right), \\
\text{ and hence } \; 
\mu_2 \onenorm{v} & \le & C_{11} \mu_2 d_0 (\mu_1 \twonorm{\beta^*} + \omega) \\
& \le & 
C_{11} \inv{2(2+\lambda)}  \left(c(q) \kappa_{\RE}(2d_0, k_0)\right) (\mu_1\twonorm{\beta^*} + \omega) \\
& = & 
\inv{2  + \lambda} (\mu_1\twonorm{\beta^*} + \omega). 
\eens
Thus we have by \eqref{eq::middle},~\eqref{eq::tildetauBbound} and the bounds immediately above,
\bens
\onen \twonorm{X (\hat{\beta} -\beta^*)}^2 & \le & 
\onenorm{v} \norm{\onen X_0^T X_0 v}_{\infty} \\
& \le &  
C_{11}  d_0(\mu_1 \twonorm{\beta^*}+ \omega)  \left( \mu_1 \twonorm{\beta^*}+
  \mu_2 \onenorm{v} + 2 \omega \right) \\ 
& \le &  
C_{11} d_0(\mu_1 \twonorm{\beta^*}+ \omega )(1+\inv{2+\lambda})  \left( \mu_1 \twonorm{\beta^*}+ 2 \omega  \right) \\ 
& = &  
C' (D_0')^2  K^4 d_0 \frac{\log m}{n} \left( \tilde\tau_B^{1/2} \twonorm{\beta^*} +
  \frac{M_{\e}}{K}\right)^2\\
& \le &  
C'' (\twonorm{B} + a_{\max}) 
K^2 d_0 \frac{\log m}{n} \left( (2\tau_B+ 3 C_6^2 r_{m,m}) K^2 \twonorm{\beta^*}^2 + M_{\e}^2\right),
\eens
 where $(D_0')^2 \le 2\twonorm{B} + 2a_{\max}$.
The theorem is thus proved.
\end{proofof2}

\section{Proof of Theorem~\ref{thm::corrlinear}}
\label{sec::optproof}
\begin{proofof2}
Suppose that event $\A_0 \cap \B_0$ holds.
The condition on $d$ in \eqref{eq::dupperthm} implies that 
\ben
\label{eq::n3}
n & > & {512 d \tau_0 \log m} 
\left\{\frac{12\lambda_{\max}(A)}{\lambda_{\min}(A)^2} 
\right\}  
\bigvee \left\{\frac{4\zeta}{(\bar\alpha_{\ell})^2}\right\}, \; \text{ where } \\
\tau_0 & \asymp &  \frac{400 C^2 \vp(s_0  +  1)^2}{\lambda_{\min}(A)}.
\een
To see this, note that the following holds by the first bound in~\eqref{eq::dupperthm}:
\ben
\label{eq::n2}
\nl = \frac{64 d \tau_0 \log m}{n} \le  
64 d \tau_0 \log m \frac{\lambda_{\min}(A)}{256 d \tau_0 \log m *24 \kappa(A)}
=\frac{\lambda_{\min}(A)}{96 \kappa(A)} \le \frac{\alpha_{\ell}}{60},
\een
where $\alpha_{\ell} = \frac{5}{8} \lambda_{\min}(A)$ by Lemma~\ref{lemma::reclaim},
and hence $\bal \ge \frac{59 \alpha_{\ell}}{60}$.
Thus we have 
\bens
\frac{\al^2}{5 \zeta} \le \frac{\bal^2}{4\zeta},  \quad 
\text{ where } \; \; 
\zeta \ge \au > \lambda_{\max}(A) \ge  \kappa(A) \bar\alpha_{\ell}.
\eens
Now, by definition of $\nu(d, m, n)$ and the second bound on $n$ in~\eqref{eq::n3},
\bens
2 \nu(d, m, n)
 = 128 d \tau_u(\loss) & := &  \frac{128 d \tau_0 \log m}{n} 
\le \frac{(\bar\alpha_{\ell})^2}{16 \zeta} 
\eens
Then
\bens
2 \z &  := &   \frac{4 \nu(d, m, n)}{\bar\alpha_{\ell}}
 = \frac{256 d \tau_u(\loss)}{\bar\alpha_{\ell}} 
\le   \frac{\bar\alpha_{\ell}}{8\zeta}.
\eens
That is, we actually need to have  for $2 \z \le \frac{\bar\alpha_{\ell}}{8\zeta}$
\bens
\frac{\xi}{1-\kappa}  & =  & 
 \inv{1- \z} 2 \tau(\loss)\left(\frac{\bal}{4\zeta} +  2 \z+ 5 \right) 
\frac {1- \z}{\frac{\bal}{4 \zeta} - 2 \z} \\
& =  &  
2\tau(\loss)\left(\frac{\bal}{4\zeta} +  2 \z+ 5 \right)  \inv{\frac{\bal}{4 \zeta} - 2\z} \\
& = &  
2\tau(\loss)\left(\frac{\frac{\bal}{4\zeta} + 2 \z}{\frac{\bal}{4 \zeta} - 2\z}
 +\frac{40 \zeta}{\bal} \right) \le 2\tau(\loss)\left(3 +\frac{40
   \zeta}{\bal} \right) \\
& \le &  
 6 \tau(\loss) + \frac{80 \zeta}{\bal} \tau(\loss),
\eens 
where we use the second bound in \eqref{eq::n3}, and hence
\bens
\frac{\bal}{4\zeta} +  2 \z 
& \le & \frac{3}{2}\frac{\bal}{4\zeta}  \quad \text{ and } \\
(1-\kappa)(1-\z) =  \frac{\bal}{4\zeta} -  2 \z   & \ge &
\half \frac{\bal}{4\zeta}.
\eens 
Finally, putting all bounds in~\eqref{eq::definekappa}, we have $0<
\kappa < 1$. Thus the conclusion of Theorem~\ref{thm::opt-linear} hold.
\end{proofof2}

\subsection{Proof of Corollary~\ref{coro::deer}}
 \begin{proofof2}
Suppose that event $\A_0 \cap \B_0$ holds.
We first show that 
\bens
&& 4 \nu \ve_{\static}^2 + 4 \tau(\loss) \e^2 
 \asymp 64 \tau(\loss) \left(4 d \ve_{\static}^2 +
   \frac{\delta^4}{\lambda^2} \right) \; \; \text{ in case } \; \;
 \delta^2 \le M \bar\delta^2.
\eens
Recall that $\xi \ge 10 \tau_{\ell}(\loss)$ by definition of $\xi$
in~\eqref{eq::definexi}. The condition~\eqref{eq::labound}  on $\lambda$ as stated in
Theorem~\ref{thm::opt-linear} indicates that 
\ben
\label{eq::imme}
\lambda \ge \frac{160  b_0 \sqrt{d} \tau_{\ell}(\loss)}{1-\kappa}
 \; \text{ where
} \; \; R \asymp b_0 \sqrt{d}.
\een
We first show that  for the choice of $\lambda$ and $R$ as in~\eqref{eq::imme},
\bens
\twonorm{\hat\Delta^t}^2
&  \le & 
\frac{2}{\bar\alpha_{\ell}} \left(\delta^2 + 64 \tau \left(4 d \ve_{\static}^2  +  \frac{\delta^4}{\lambda^2}  \right) \right) \\
&  \le &  
\frac{3}{\alpha_{\ell}}\delta^2 
+ \frac{\al\ve_{\static}^2}{4} + 
\frac{2}{\alpha_{\ell}} O\left(\frac{\delta^2}{\tau_0} \frac{M\ve_{\static}^2}{400b_0^2}
\right).
\eens
Then \eqref{eq::errort2} holds.

For the second term on the RHS of \eqref{eq::twoerror}, we have by
\eqref{eq::n3},
\ben
\label{eq::nlocal} 
n & \ge &
\frac{256 d \tau_0 \log m}{\bar\alpha_{\ell}}
\frac{8 \zeta}{\bar\alpha_{\ell}}, \quad \text{ where } \quad 
\tau_0 \asymp  \frac{400 C^2 \vp(s_0  + 1)^2}{\lambda_{\min}(A)}.
\een
Thus
\bens
\nonumber
4 \nl &  = & 256 d \tau_0 \frac{\log m}{n} 
\le \bar\alpha_{\ell} \frac {\bar\alpha_{\ell}} {8
  \zeta} \; \quad\text{ and } \quad
\frac{2}{\bar\alpha_{\ell}} 4 \nu \ve_{\static}^2  \le   \frac {\bar\alpha_{\ell}} {4 \zeta}\ve_{\static}^2  
\le  \frac{\al\ve_{\static}^2}{4 \zeta}.
\eens
Consider the choice of $\bar\eta = \delta^2$, where
$M \bar{\delta}^2 \ge \bar{\eta} = \delta^2 \ge 
  \frac{ c \ve_{\static}^2}{1-\kappa}\frac{d \log  m}{n}   =:
  \bar{\delta}^2$.

Thus we have for~\eqref{eq::imme},
\bens
\frac{2 \delta^2}{\lambda} 
& \le &  \frac{M c \ve_{\static}^2}{1-\kappa}\frac{d \log m}{n}
\frac{1-\kappa}{160 b_0 \sqrt{d} \tau_{\ell}(\loss)}  = 
\frac{M c \ve_{\static}^2}{160 b_0 } \frac{\sqrt{d}}{\tau_0} < R 
\eens
and hence $\e=\frac{4 \delta^2}{\lambda}$.

Then for the last term on the RHS of \eqref{eq::twoerror}, we have 
for $\tau_{\ell}(\loss) \asymp \tau$,
\bens
4 \tau_{\ell}(\loss) 
\e^2
 & = & 16 \tau_{\ell}(\loss)  \min\left(\frac{2 \delta^2}{\lambda},
   R\right)^2  \\
& = & 64 \tau \frac{\delta^4}{\lambda^2}  \le 
\frac{\delta^4(1-\kappa)^2}{400 b_0^2 \tau_0}\frac{n}{d \log m} \\
& \le & 
\delta^2 \frac{c M \ve_{\static}^2}{1-\kappa}
\frac{(1-\kappa)^2}{400b_0^2\tau_0}  \\
& = & 
\frac{c\delta^2}{\tau_0} \frac{M\ve_{\static}^2 (1-\kappa)}{400b_0^2} \\
& =&  O\left(\frac{c\delta^2}{\tau_0} \frac{M\ve_{\static}^2}{400b_0^2}
\right)
\eens
where  
$\delta^2 \le  \frac{M c \ve_{\static}^2}{1-\kappa}\frac{d \log
  m}{n}$.

Finally,  suppose we fix
   \bens
R \asymp \frac{b_0}{20 M_+ \sqrt{6\kappa(A)} }\sqrt{\frac{n}{\log  m}}
\eens
in view of the upper bound $\bar{d}$~\eqref{eq::bard}.
Then in order for
\bens
\lambda \ge 16 R \frac{\xi}{1-\kappa}
\eens
to hold, we need to set
\bens
\lambda & \ge & 640 R \tau(\loss) \kappa(A),
\eens 
because of the following lower bound $\frac{\xi}{1-\kappa} \ge 40
\tau(\loss) \kappa(A)$ as shown in~\eqref{eq::lowerkappa}.

Then \eqref{eq::error3} holds given that the last term on
the RHS of \eqref{eq::twoerror} is now bounded by
\bens
\frac{2}{\bal} 
4 \tau_{\ell}(\loss)
 \e^2 & = & \frac{2}{\bal}  16 \tau_{\ell}(\loss) \min\left(\frac{2\delta^2}{\lambda}, R\right)^2
\le \frac{60}{59}\frac{2}{\al}\frac{64 \delta^4}{640^2 R^2 \kappa(A)^2 \tau_{\ell}(\loss)} \\
& \le & 
\frac{60}{59}\frac{2}{\al} \frac{\delta^4}{6400 \kappa(A) \tau_0} \left(\frac{20 M_+
    \sqrt{6} }{b_0}\right)^2 \\
& = & 
\frac{60}{59}\frac{12}{\al} \frac{\delta^4 M^2_+}{16{b_0^2} \kappa(A) \tau_0} 
\le \frac{60}{59}\frac{3 \delta^4}{4 \al} \frac{1024}{400 b_0^2 \twonorm{A}} 
\le \frac{2 \delta^4}{b_0^2 \al \twonorm{A}}.
\eens
\end{proofof2}

\begin{remark}
\label{rem::upperxi}
First we obtain an upper bound on $\frac{\xi}{1-\kappa}$ for $\zeta =
\alpha_u \asymp \frac{3\lambda_{\max}(A) }{2}$ and $\frac{59}{60} \frac{5}{8} \lambda_{\min}(A) 
\le \bal$
\bens
\frac{\xi}{1-\kappa} 
& \le & 
 6 \tau(\loss) + \frac{80 \zeta}{\bal} \tau(\loss)  \\
& \le & 
 6 \tau(\loss) + \frac{80 \zeta}{\frac{59}{60} \frac{5}{8}
   \lambda_{\min}(A) } \tau(\loss)  \\
& \approx &  200 \tau(\loss) \kappa(A).
\eens
\end{remark}
Now we obtain an upper bound using $R \le b_0 \sqrt{d}$ for $d \le \bar{d}$ as in \eqref{eq::bard},
\bens
R \frac{\xi}{1-\kappa}  & \le &  
200 \kappa(A) \tau b_0 \sqrt{d}  
\le  200 \kappa(A) \tau_0 \sqrt{\frac{\log m}{n}} \frac{b_0}{20 M_+ \sqrt{6 \kappa(A)}}\\
\nonumber
& = &   200 \kappa(A) \tau_0 \frac{ \lambda_{\min}(A)}{ \vp(s_0+1)}
\sqrt{\frac{\log m}{n}} \frac{b_0}{640 C \sqrt{6\kappa(A)}} \\ 
&\le & 10 b_0\frac{ \tau_0}{M_+ \sqrt{6}} 
\sqrt{\kappa(A)} \sqrt{ \frac{\log m}{n} } \\
& = & 
\frac{125 b_0}{\sqrt6} C \vp(s_0 +  1) \sqrt{\kappa(A)}\sqrt{\frac{\log m}{n} },
\eens
where we use \eqref{eq::dupperthm} and the fact that  $\frac{\tau_0}{M_+} = 12.5 C \vp(s_0+1)$.
We now discuss the implications of this bound on the choice of $\lambda$
in Section~\ref{sec::optdiscuss}. We consider two cases.
\bit
\item
When $\tau_B = \Omega(1)$.
It is sufficient to have for 
$\twonorm{\beta^*} \le b_0$ 
and $\tau_B \asymp 1$,
\bens
\lambda & \ge &
16 C b_0\left(50 \sqrt{\kappa(A)} \vp(s_0+1)
\bigvee \left(D_0' K( K \tau_B^{1/2} + \frac{M_{\e}}{b_0}) \right) \right) 
\sqrt{\frac{\log m}{n} } 
\eens
following the discussions in Section~\ref{sec::smallW},
where the first and the second term on the RHS are at the same order
except that the new lower bound involves the condition number
$\kappa(A)$, while the original bound in Theorem~\ref{thm::lassora}
involves only $D'_0 = \twonorm{B}^{1/2} + a_{\max}^{1/2}$.
\item
When $\tau_B = o(1)$ and  $M_{\e} = \Omega(\tau_B^{+/2} K
\twonorm{\beta^*})$. Now $d$ satisfies~\eqref{eq::dupper} and hence
$$b_0 \sqrt{d} \le \inv{4 \sqrt{5} M_+}  \sqrt{\frac{n}{\log  m}} \left\{\frac{\sqrt{c'} D_0' K M_{\e} }{\vp(s_0+1)}
  \wedge b_0\right\}.$$
Now combining this with the condition on $d$ as
in~\eqref{eq::dupperthm} implies that it is sufficient to set $R$ such
that 
\bens
R \frac{\xi}{1-\kappa} & \asymp &  
\frac{\kappa(A) \tau}{M_+} 
\left(\frac{b_0}{\sqrt{\kappa(A)}} \bigwedge \frac{D_0' K
    M_{\e}}{\vp(s_0 + 1)}\right) 
\sqrt{\frac{n}{\log  m}}\\
& = &  
 \kappa(A) \vp(s_0 + 1)
 \left(\frac{b_0}{\sqrt{\kappa(A)}} \bigwedge  \frac{D_0' K  M_{\e}}{\vp(s_0 + 1)}\right) \sqrt{\frac{\log m}{n} }\\
& \approx & 
\left(b_0 \vp(s_0 +  1) \sqrt{\kappa(A)} \wedge \kappa(A)  D_0' K  M_{\e} \right) \sqrt{\frac{\log  m}{n}} =: \bar{U}.
\eens
Hence it is sufficient to have for $\psi \asymp D_0' K \left(M_{\e} +  K \tau_B^{+/2} \twonorm{\beta^*}\right)$ as in
\eqref{eq::psijune15},
\bens
\lambda 
& \ge & \left(\bar{U} \vee \psi \right)\sqrt{\frac{\log  m}{n} }.
\eens
\eit
\section{Proof of Theorem~\ref{thm::oracle}}
\label{sec::thmoracle}

\begin{proofof2}
Clearly the condition on the stable rank of $B$ guarantees that 
$$n \ge r(B) = \frac{\tr(B)}{\twonorm{B}} =
\frac{\tr(B)\twonorm{B}}{\twonorm{B}^2}\ge \fnorm{B}^2 /\twonorm{B}^2 
 \ge \log m.$$
Thus the conditions in
Lemmas~\ref{lemma::Tclaim1} and \ref{lemma::trBest} hold.
First notice that 
\bens
\hat\gamma & = & 
\onen \left(X_0^T X_0 \beta^*
 +  W^T X_0 \beta^* + X_0^T \e +  W^T \e\right) \\
(\onen X^T X - \frac{\hat\tr(B)}{n} I_{m})\beta^* 
& = & 
\onen (X_0^T X_0 + W^T X_0 + X_0^T W 
+ W^T W - \frac{\hat\tr(B)}{n} I_{m})\beta^*.
\eens
Thus 
\bens
\norm{\hat\gamma - \hat\Gamma \beta^*}_{\infty} 
& \le &
\norm{\hat\gamma - \onen \left(X^T X - \hat\tr(B) I_{m}\right)\beta^*}_{\infty} \\
 & = & 
\onen \norm{ X_0^T \e +  W^T \e- \left(W^T W + X_0^T W - \hat\tr(B) I_m\right)\beta^* }_{\infty} \\
 & \le & 
\onen \norm{ X_0^T \e +  W^T \e}_{\infty}  +\onen \norm{(W^T W-
  \hat\tr(B) I_{m}) \beta^*}_{\infty} 
+\norm{\onen X_0^T W \beta^*}_{\infty}  \\
 & \le & 
 \onen \norm{ X_0^T \e +  W^T \e}_{\infty} + \onen(\norm{(Z^T B
     Z- \tr(B) I_{m}) \beta^*}_{\infty}) + \onen \norm{X_0^T W \beta^*}_{\infty} \\
&& + \onen \abs{\hat\tr(B) - \tr(B)} \norm{\beta^*}_{\infty} =:  U_1+  U_2 + U_3 + U_4.
\eens
By Lemma~\ref{lemma::Tclaim1} we have  on $\B_4$ for $D_0 := \sqrt{\tau_B} + a_{\max}^{1/2}$,
\bens
U_1 = \onen \norm{ X_0^T \e +  W^T \e}_{\infty}   = 
\onen \norm{A^{\half} Z_1^T \e +  Z_2^T B^{\half} \e}_{\infty}  \le
\rho_{n}M_{\e} D_0,
\eens
and on event $\B_5$ for  $D'_0 :=\sqrt{\twonorm{B}} + a_{\max}^{1/2}$,
\bens
\lefteqn{U_2 +  U_3 = 
\onen \norm{(Z^T B Z- \tr(B) I_{m}) \beta^*}_{\infty} +
\onen \norm{X_0^T W \beta^*}_{\infty}} \\
& \le & \rho_{n} K \twonorm{\beta^*}
\left(\frac{\fnorm{B}}{\sqrt{n}} + \sqrt{\tau_B}
  a^{1/2}_{\max}\right) \le  K \rho_{n} \twonorm{\beta^*}  \tau_B^{1/2} D_0',
\eens
where recall $\fnorm{B} \le \sqrt{\tr(B)}\twonorm{B}^{1/2}$.
Denote by $\B_0 := \B_4 \cap \B_5 \cap \B_6$.
We have on $\B_0$ and under (A1),  by Lemmas~\ref{lemma::Tclaim1}
and~\ref{lemma::trBest} and $D_1$ defined therein,
\bens
\nonumber
\norm{\hat\gamma - \hat\Gamma \beta^*}_{\infty} 
& \le &
U_1+  U_2 + U_3 + U_4 \\
& \le & 
\nonumber
\rho_{n} M_{\e} D_0 + D_0' \tau_B^{1/2} K \rho_{n} \twonorm{\beta^*} 
+ \onen \abs{\hat\tr(B) - \tr(B)} \norm{\beta^*}_{\infty} \\
& \le & 
\label{eq::oracleone}
D_0 M_{\e} \rho_{n} + D_0' K  \tau_B^{1/2} \twonorm{\beta^*} \rho_{n}  +
D_1 \norm{\beta^*}_{\infty} r_{m,m} \\
& \le & 
D_0 M_{\e} \rho_{n} + D_0' K  \tau_B^{1/2} \twonorm{\beta^*} \rho_{n}
+ 2 D_1 K \inv{\sqrt{m}} \rho_{n}.
\eens
Finally, we have by the union bound, $\prob{\B_0} \ge 1 -  16/m^3$.
This is the end of the proof of Theorem~\ref{thm::oracle}.
\end{proofof2}

\section{Conclusion}
\label{sec::conclude}
In this paper, we provide a unified analysis on the rates of convergence for both the
corrected Lasso estimator~\eqref{eq::origin} and the Conic programming
estimator \eqref{eq::Conic}.
As $n$ increases or as the measurement error metric $\tau_B$
decreases, we see performance gains over the entire paths for both
$\ell_1$ and $\ell_2$ error for both estimators as expected. 
When we focus on the lowest $\ell_2$ error along the paths as we vary the penalty factor
$f \in [0.05, 0.8]$, the corrected Lasso via the composite gradient descent algorithm
performs slightly better than the Conic programming estimator as shown
in  Figure~\ref{fig:conic-gd-chain-chain-3}.

For the Lasso estimator, when we require that the stochastic error
$\e$ in the response variable $y$ as in~\eqref{eq::oby}
does not approach $0$ as quickly as the measurement error $W$ in
\eqref{eq::obX} does, then the sparsity constraint becomes essentially
unchanged as $\tau_B \to 0$. 
These tradeoffs are somehow different from the behavior of the Conic
programming estimator versus the Lasso estimator; however, we believe the
differences are minor. 
Eventually, as $\tau_B \to 0$, the relaxation
on $d$ as in \eqref{eq::ora-sparsity-rem} enables the Conic
programming estimator to achieve bounds which are essentially
identical to  the Dantzig Selector when the design matrix $X_0$ is a
subgaussian random matrix satisfying the Restricted Eigenvalue conditions;
See for example~\cite{CT07,BRT09,RZ13}.

When $\tau_B \to 0$ and $M_{\e} = \Omega(\tau_B^+ K \twonorm{\beta^*})$, we set 
\ben
\label{eq::ordlassobound}
\lambda \ge 2 \psi \sqrt{\frac{\log m}{n}}, \; \; \text{ where } \;\;
\psi:= 4 C_0 D_0' K  M_{\e},
\een 
so as to recover the regular lasso bounds in
$\ell_q$ loss for $q = 1, 2$ in~\eqref{eq::lassobounds} in Theorem~\ref{thm::lassora}.
Moreover, suppose that $\tr(B)$ is given, then one can drop the second
term  in $\psi$  as in \eqref{eq::psijune15}  involving
$\twonorm{\beta^*}$  entirely and hence recover
the lasso bound as well.

Finally, we note that the bounds corresponding to the Upper $\RE$ 
condition as stated in Corollary~\ref{coro::BC}, Theorem~\ref{thm::AD}
and Lemma~\ref{lemma::lowerREI} are  not needed for
Theorems~\ref{thm::lasso} and~\ref{thm::lassora}.
They are useful to ensure algorithmic convergence and to bound the optimization error for the gradient
descent-type of algorithms as considered in~\cite{ANW12, LW12}, when one is
interested in approximately solving the nonconvex optimization
function~\eqref{eq::origin}. Our Theorem~\ref{thm::corrlinear}
illustrates this result. 
Our theory in Theorem~\ref{thm::corrlinear} predicts the dependencies of
the computational and statistical rates of convergence for the
corrected Lasso via gradient descent algorithm on the
condition number $\kappa(A)$, 
the trace parameter $\tau_B$ and the radius $R$ as
\bens
\lambda \asymp \frac{R \xi}{1-\kappa} \asymp \tau _0\kappa(A) \frac{R \log m}{n},
\; \;\text{ where } \; \;  \tau_0 \asymp  \frac{(\rho_{\max}(s_0, A) +
  \tau_B)^2}{\lambda_{\min}(A)}
\eens
depends on $\tau_B$, sparse and minimal eigenvalues of $A$.
Therefore, we need to increase the penalty when we
increase the $\ell_1$-ball radius $R$ in~\eqref{eq::origin2} in order
to ensure algorithmic and statistical convergence as predicted in
Theorem~\ref{thm::corrlinear}.
This is well-aligned with the observation in Figure~\ref{fig:changeR}.
Our numerical results validate such algorithmic and statistical convergence properties.

\section*{Acknowledgements}
The authors thank the Editor, the Associate editor and two referees
for their constructive comments. 
We thank Rob Kass, Po-Ling Loh, Seyoung Park, Kerby Shedden and Martin Wainwright for
helpful discussions. We thank Professor Alexander Belloni for providing us the code
implementing the Conic programming estimator.

\appendix

\section{Outline}
We prove  Theorem~\ref{thm::opt-linear} in
Section~\ref{sec::proofofoptlinear}. In Sections~\ref{sec::aux}, we 
present variations of  the Hanson-Wright inequality
as recently derived in~\cite{RV13} (cf. Lemma~\ref{lemma::oneeventA}).
We prove Lemma~\ref{lemma::Tclaim1} in
Section~\ref{sec::proofTclaim1}.
In Sections~\ref{sec::lassoall} and~\ref{sec::proofofDSlemma}, we prove the
technical lemmas for Theorems~\ref{thm::lasso} and~\ref{thm::DS} respectively. 
In Section~\ref{sec::DSoracleproof}, we prove the Lemmas needed for Proof of Theorem~\ref{thm::DSoracle}.
In order to prove Corollary~\ref{coro::BC},
we need to first state some geometric analysis results Section~\ref{sec::geometry}.
We prove Corollary~\ref{coro::BC} in Section~\ref{sec::appendLURE} and
Theorem~\ref{thm::AD} in Section~\ref{sec::proofofthmAD}.

\section{Proof of Theorem~\ref{thm::opt-linear}}
\label{sec::proofofoptlinear}
Let us first define the following shorthand notation
\bens
\hat\Delta^t & = & \beta^t - \hat{\beta} 
\; \;\text{ and } \; \; \delta^t \; = \; 
\phi(\beta^t) -\phi(\hat{\beta}).
\eens
The proof of the theorem requires two technical
Lemmas~\ref{lemma::oneextra} and 
~\ref{lemma::twoextra}.
Both are stated under assumption~\eqref{eq::etatollone}, which is stated
in terms of a given tolerance $\bar\eta > 0$ and integer $T > 0$ such that 
\ben
\label{eq::etatollone}
\phi(\beta^t) - \phi(\hat\beta) \le \bar\eta, \; \; \forall t \ge T,
\een
where the distance between $\beta^t$ and the global optimizer $\hat\beta$ is
measured in terms of  the objective function $\phi$, namely, $\delta^t
= \phi(\beta^t) -\phi(\hat{\beta})$. 

We first show Lemma~\ref{lemma::oneextra}, which ensures that 
the vector $\hat\Delta^t := \beta^t - \hat{\beta}$ satisfies a
certain cone-type condition. The proof is omitted, as it is 
a shortened proof of Lemma~1 of~\cite{LW15}.

\begin{lemma}{(\bf Iterated Cone Bound)}
\label{lemma::oneextra}
Under the conditions of Theorem~\ref{thm::opt-linear}, 
suppose there exists a pair
$(\bar\eta, T)$ such that \eqref{eq::etatollone} holds. 
Then for any iteration $t \ge T$, we have
\bens
\onenorm{\beta^t - \hat\beta} \le 
4\sqrt{d} \twonorm{\beta^t - \hat\beta} + 8 \sqrt{d}
\twonorm{\hat\beta - \beta^*} + 2 \cdot \min\left(\frac{2 \bar\eta}{\lambda}, R\right). 
\eens
\end{lemma} 
 We next state the following auxiliary result on the loss function. We
 use Lemma~\ref{lemma::Delta1} in the proof of
 Lemma~\ref{lemma::oneextra} and Corollary~\ref{coro::deer}.
\begin{lemma}
\label{lemma::Delta1}
Denote by $\tau_{\ell}(\loss) := \tau_0 \frac{\log m}{n}$ and
$\nu_{\ell} = 64 d  \tau_{\ell}(\loss)$.
Let $\bare = 8\sqrt{d} \ve_{\static}$, where 
$\ve_{\static} = \twonorm{\hat\beta - \beta^*}$ and 
$\e =2 \cdot \min\left(\frac{2 \bar\eta}{\lambda}, R\right)$.
Under the assumptions of Lemma~\ref{lemma::oneextra},
we have for $\hat\Delta^t := \beta^t - \hat{\beta}$ and $t > T$,
\ben
\label{eq::Tlower1}
\T(\hat\beta, \beta^t)   
& \ge &
  \frac{\alpha_{\ell} - \nu_{\ell}}{2} \twonorm{\hat\Delta^t}^2  - 2 \tau_{\ell}(\loss)(\bare + \e)^2 \; \; \text{ and } \; \\
\label{eq::Tupper1}
\phi(\beta^t) - \phi(\hat\beta) &  \ge &  
\T(\beta^t, \hat\beta)  \ge 
\frac{\alpha_{\ell} - \nu_{\ell}}{2} \twonorm{\hat\Delta^t}^2
  - 2 \tau_{\ell}(\loss)(\bare + \e)^2.
\een
\end{lemma}

\begin{lemma}{\bf (Lemma~3 of Loh-Wainwright (2015))}
\label{lemma::twoextra}
Suppose the RSC and RSM conditions as stated in~\eqref{eq::lowerRSC} 
and~\eqref{eq::upperRSC} hold 
with parameters  $(\alpha_{\ell}, \tau_{\ell}(\loss))$ and $(\alpha_{u},
\tau_{u}(\loss))$ respectively.  
Under the conditions of Theorem~\ref{thm::opt-linear}, 
suppose there exists a pair $(\bar\eta, T)$ such that
\eqref{eq::etatollone} holds.
Then for any iteration $t \ge T$, we have for $0< \kappa < 1$,
\bens
\phi(\beta^t) - \phi(\hat\beta) & \le &  \kappa^{t-T}( \phi(\beta^T) -
\phi(\hat\beta)) + \frac{\xi}{1-\kappa}(\bare^2 + \e^2) \; \; \text{ for } \\ 
\bare  & := & 8\sqrt{d} \ve_{\static} \; \; \text{ and } \; \; 
\e = 2 \cdot \min\left(\frac{2 \bar\eta}{\lambda}, R\right),
\eens
where the quantities $\kappa$ and $\psi$ are as defined in Theorem~\ref{thm::opt-linear} (cf.~\eqref{eq::definekappa} and \eqref{eq::definexi}).
\end{lemma}

\begin{proofof}{Theorem~\ref{thm::opt-linear}} 
We are now ready to put together the final argument for the theorem.
First notice that \eqref{eq::twoerror}  follows from~\eqref{eq::loss}
directly in view of \eqref{eq::Tupper1} and
Lemma~\ref{lemma::oneextra}, where we set $\bar\eta = \delta^2$, 
$\bar\e_{\static} = 8 \sqrt{d} \ve_{\static}$ and
$\e = 2 \min\left(\frac{2 \delta^2}{\lambda}, R\right)$.

Following \eqref{eq::Tupper1}, we have for  $\nl = 64 d \tau_{\ell}(\loss)$,
\bens
\nonumber
\frac{\alpha_{\ell} - \nu_{\ell}}{2} 
\twonorm{\hat\Delta^t}^2
&  \le &  
\phi(\beta^t) - \phi(\hat\beta) +  2 \tau_{\ell}(\loss)(\bare + \e)^2,
\eens
and thus
\ben
\nonumber
\twonorm{\hat\Delta^t}^2
&  \le &  
\frac{2}{\bar\alpha_{\ell}}
(\phi(\beta^t) - \phi(\hat\beta) +  \frac{4}{\bar\alpha_{\ell}}
 \tau_{\ell}(\loss)(\bare + \e)^2 \\
\nonumber
&  \le &  
\frac{2}{\bar\alpha_{\ell}}
\left(\delta^2+  2  \tau_{\ell}(\loss)(2\bare^2 + 2\e^2) \right)\\
&  \le &  
\nonumber
\frac{2}{\bar\alpha_{\ell}} \left(\delta^2 +  2  \tau_{\ell}(\loss)(128 d \ve_{\static}^2 + 2\e^2) \right)\\
&  \le &  
\label{eq::twoerrorlocal}
\frac{2}{\bar\alpha_{\ell}} \left(\delta^2 + 4 \nu_{\ell} \ve_{\static}^2 +
4 \tau_{\ell}(\loss) \e^2 \right).
\een
The remainder of the proof follows an argument in~\cite{ANW12}. We
first prove the following inequality:
\bens
\phi(\beta^t) - \phi(\hat\beta) \le \delta^2, \; \;\quad \forall t \ge T^*(\delta).
\eens
We divide the iterations $t \ge 0$ into a series of epochs $[T_{\ell},
T_{\ell+1}]$ and defend the tolerances $\bar\eta_0 > \bar\eta_1 >
\ldots$ such that 
\bens
\phi(\beta^t) - \phi(\hat\beta) \le \bar\eta_{\ell}, \; \;
\quad \forall t \ge T_{\ell}.
\eens
In the first iteration, we apply Lemma~\ref{lemma::twoextra} with 
$\bar\eta_0 := \phi(\beta^0) - \phi(\hat\beta)$ to obtain
\bens
\phi(\beta^t) - \phi(\hat\beta) \le \kappa^{t}( \phi(\beta^0) -
\phi(\hat\beta)) + \frac{\xi}{1-\kappa}(\bare^2 + 4R^2)
\quad \text{ for any iteration $t \ge 0$}.
\eens 
Set
\bens
\bar\eta_1 :=\frac{2 \xi}{1 - \kappa}(\bare^2 + 4R^2) \; \; \text{ and
} \quad T_1 := \left\lceil{\frac{\log(2 \bar{\eta}_0/\bar{\eta}_1)}{\log
    (1/\kappa)} }\right\rceil.
\eens
Then we have for any iteration $t \ge T_1$
\bens
\phi(\beta^t) - \phi(\hat\beta)  \le \bar\eta_1 :=\frac{4 \xi}{1 - \kappa}
\max\left\{\bare^2,4R^2\right\}.
\eens
The same argument can be now be applied in a recursive manner.
Suppose that for some $\ell \ge 1$, we are given a pair
$(\bar\eta_{\ell},T_{\ell})$ such that
\ben
\label{eq::etatollproof}
\phi(\beta^t) - \phi(\hat\beta) \le \bar\eta_{\ell}, \; \; \forall t \ge T_{\ell}.
\een
We now define 
\bens
\bar\eta_{\ell+1} :=
\frac{2 \xi}{1 - \kappa}(\bare^2 + \e_{\ell}^2) \; \; \text{ and
} \quad T_{\ell+1} := 
\left\lceil{
\frac{\log(2 \bar{\eta}_{\ell}/\bar{\eta}_{\ell+1})}{\log
    (1/\kappa)} }\right\rceil + T_{\ell}.
\eens
We can apply Lemma~\ref{lemma::twoextra} to obtain 
 for any iteration $t \ge T_{\ell}$ and $\ve_{\ell} := 2 \min\{\frac{\bar\eta_{\ell}}{\lambda}, R\}$,
\bens
\phi(\beta^t) - \phi(\hat\beta) \le 
 \kappa^{t-T_{\ell}}( \phi(\beta^{T_{\ell}}) - \phi(\hat\beta)) 
+ \frac{\xi}{1-\kappa}(\bare^2 + \ve_{\ell}^2),
\eens
which implies that  for all $t \ge T_{\ell+1}$,
\bens
\phi(\beta^t) - \phi(\hat\beta) \le 
\bar\eta_{\ell+1}\le \frac{4\xi}{1-\kappa}\max\{\bare^2, \ve_{\ell}^2\}
\eens
by our choice of $\{\eta_{\ell}, T_{\ell}\}_{\ell\ge 1}$.
Finally, we use the recursion 
\ben
\label{eq::iterate}
\bar\eta_{\ell+1} 
\le \frac{4\xi}{1-\kappa}
\max(\bare^2, \ve_{\ell}^2) \quad  \text{and} \quad
T_{\ell} \le \ell +  
\frac{\log(2^{\ell} \bar{\eta}_{0}/\bar{\eta}_{\ell})}{\log (1/\kappa)} 
\een
to establish the recursion that 
\ben
\label{eq::induction}
\bar\eta_{\ell+1} \le \frac{\bar\eta_{\ell}}{4^{2^{\ell-1}}}  \quad \text{and} \quad
\ve_{\ell+1} := \frac{\bar\eta_{\ell+1}}{\lambda} \le \frac{R}{4^{2^{\ell}}} \; \;
\forall \ell=1, 2, \ldots.
\een
Taking these statements as given, we need to have 
\bens
\bar\eta_{\ell} \le \delta^2.
\eens
It is sufficient to establish that 
\bens
\frac{\lambda R}{4^{2^{\ell-1}}} \le \delta^2.
\eens
Thus we find that the error drops below $\delta^2$ after at most 
\bens
\ell_{\delta} \ge \log \left(\log(R \lambda/\delta^2) / \log (4)\right)
/\log 2 + 1 = \log \log(R \lambda/\delta^2)
\eens
epochs. Combining the above bound on $\ell_{\delta}$ with the
recursion~\eqref{eq::iterate} 
\bens
T_{\ell} \le \ell +  
\frac{\log(2^{\ell} \bar{\eta}_{0}/\bar{\eta}_{\ell})}{\log (1/\kappa)}, 
\eens
we conclude that
\bens
\phi(\beta^t) - \phi(\hat\beta) \le \delta^2
\eens
is guaranteed to hold for all iterations
\bens
t> \ell_{\delta}\left(1+ \frac{\log 2}{\log(1/\kappa)}\right) + 
\frac{\log (\bar\eta_0/\delta^2)}{\log(1/\kappa)}.
\eens
To establish \eqref{eq::induction}, we start with $\ell=0$ and establish 
that for $\bare = 8 \sqrt{d} \e_{\static} =o(\sqrt{d}) = o(R)$
\ben
\label{eq::baseR}
\frac{\bar\eta_1}{\lambda}
& := & \frac{4 \xi}{(1 - \kappa)\lambda}\max(\bare^2, 4R^2)  =\frac{16 R
  \xi}{(1 - \kappa)\lambda} R  \le \frac{R}{4}\; \\
 \text{ and thus } \; \; \ve_1 &  :=  & 2 \min\{\frac{\bar\eta_1}{\lambda}, R\}
= R/2 \le \ve_0 = R.
\een 
Assume that $\bare \le \ve_1$ (otherwise, we are
done at the first iteration).
First, we  obtain for $\ell=1$, 
\bens
\bar\eta_{2} & \le &  
\frac{4\xi}{1-\kappa} \max(\bare^2, \ve_{1}^2) = \frac{4\xi}{1-\kappa}
\ve_{1}^2 = \frac{4\xi}{1-\kappa} \left(\frac{2 \bar\eta_1}{\lambda}\right)^2 \\
& \le & 
\frac{16 \xi}{1-\kappa} \frac{\bar\eta_1^2}{\lambda^2} 
\le \frac{16 \xi R}{1-\kappa} \frac{\bar\eta_1}{4 \lambda} \le 
 \frac{\bar\eta_1}{4},  \\
\text{ and } \quad
 \frac{ \bar\eta_{2}}{\lambda}  & \le &  \frac{ \bar\eta_{1}}{4
   \lambda}   \le \frac{R}{16},  
\eens
where in the last three steps, we use the fact that $\lambda \ge
\frac{16 R \xi}{(1 - \kappa)}$ and  \eqref{eq::baseR}.
Thus~\eqref{eq::iterate} holds for $\ell=1$.
 
Now assume that \eqref{eq::induction} holds for $d \le \ell$.
In the induction step, we again use the assumption that $\ve_{\ell} :=
2 \frac{\bar\eta_{\ell}}{\lambda} \ge \bare$ and~\eqref{eq::iterate} to obtain 
\bens
\bar\eta_{\ell+1} & \le &   
\frac{4\xi}{1-\kappa} \max(\bare^2, \ve_{\ell}^2)
 =
\frac{16 \xi}{1-\kappa} \frac{\bar\eta_{\ell}^2}{\lambda^2}  \\
& \le & \frac{16 \xi}{1-\kappa} \frac{R}{4^{2^{(\ell-1)}}} 
\frac{\bar\eta_{\ell}}{\lambda} 
  = \frac{16 R \xi}{1-\kappa} \inv{\lambda} \frac{\bar\eta_{\ell}}
{4^{2^{(\ell-1)}}} \\
& \le &   \frac{\bar\eta_{\ell}}{4^{2^{(\ell-1)}}}.
\eens
Finally,  by the induction assumption
\bens
\frac{\bar\eta_{\ell}}{\lambda} \le \frac{R}{4^{2^{\ell-1}}},
\eens
we use the bound immediately above to obtain 
\bens
\frac{\bar\eta_{\ell+1}}{\lambda}  & \le &   
\frac{\bar\eta_{\ell}}{4^{2^{(\ell-1)}}} \inv{\lambda}
\le \frac{R}{4^{2^{\ell-1}}}\inv {4^{2^{(\ell-1)}}} \le 
\frac{R}{4^{2^{\ell}}}.
\eens
The rest of the proof follows from that of Corollary~\ref{coro::deer}.
This is the end of the proof for Theorem~\ref{thm::opt-linear}.
\end{proofof}

It remains to prove Lemma~\ref{lemma::Delta1}.

\begin{proofof}{Lemma~\ref{lemma::Delta1}}
Using the RSC condition, we have for $\tau_{\ell}(\loss) := \tau_0
\frac{\log m}{n}$ and 
$\nl  =  64 d \tau_{\ell}(\loss) \le
\frac{\alpha_{\ell}}{48}$,
\ben
\label{eq::TRSC1}
\T(\hat\beta, \beta^t) 
& \ge &
\frac{\alpha_{\ell}}{2} 
\twonorm{\hat\Delta^t}^2  - \tau_{\ell}(\loss)
\onenorm{\hat\Delta^t}^2  \\
\nonumber
& \ge & \frac{\alpha_{\ell}}{2} \twonorm{\hat\Delta^t}^2  - \tau_{\ell}(\loss)
\left(2 * 16 d \twonorm{\hat\Delta^t}^2  + 2 (\bare+ \e)^2\right) \\ 
\nonumber
& \ge &
\half\bar\alpha_{\ell}
\twonorm{\hat\Delta^t}^2   -2 \tau_{\ell}(\loss)(\bare + \e)^2
\een
and by Lemma~\ref{lemma::oneextra}, for any iteration $t \ge T$, 
\bens
\onenorm{\hat\Delta^t - \hat\beta} 
&\le &  
4 \sqrt{d} \twonorm{\beta^t - \hat\beta} + 8 \sqrt{d}
\twonorm{\hat\beta - \beta^*} + 2 \cdot \min\left(\frac{2
    \bar\eta}{\lambda}, R\right) \\
&\le &  
4 \sqrt{d} \twonorm{\hat\Delta^t} +(\bare + \e).
\eens
By convexity of function $g$, we have 
\ben
g(\beta^t) - g(\hat\beta)
 - \ip{\grad{g}(\hat\beta), \beta^t - \hat\beta} \ge 0.
\een
Thus
\bens
\lefteqn{
\phi(\beta^t) - \phi(\hat\beta)
 - \ip{\grad{\phi}(\hat\beta), \beta^t - \hat\beta} } \\
& = & 
\loss(\beta^t) - \loss(\hat\beta) -
\ip{\grad\loss(\hat\beta), \beta^t - \hat\beta}  
+ \lambda (g(\beta^t) - g(\hat\beta)
 - \ip{\grad{g}(\hat\beta), \beta^t - \hat\beta} ). 
\eens
Moreover, by the first order optimality condition for $\hat\beta$,  we have  for all feasible $\beta^t \in \Omega$
\bens
\ip{\grad{\phi}(\hat\beta), \beta^t - \hat\beta} \ge 0,
 \eens
and thus
\bens
\phi(\beta^t) - \phi(\hat\beta) & \ge &  
\loss(\beta^t) - \loss(\hat\beta) -
\ip{\grad\loss(\hat\beta), \beta^t - \hat\beta} = \T(\beta^t, \hat\beta),
\eens
where similar to \eqref{eq::TRSC1}, we have 
\bens
\T(\beta^t, \hat\beta)
 & \ge & \alpha_1 \twonorm{\hat\Delta^t}^2  -
\tau_{\ell}(\loss) \onenorm{\hat\Delta^t}^2 \\
& \ge &
 (\alpha_1 - 32 d \tau_{\ell}(\loss))
 \twonorm{\hat\Delta^t}^2  - 2 \tau_{\ell}(\loss)(\bare + \e)^2 \\
& = &
\half \bar\alpha_{\ell} \twonorm{\hat\Delta^t}^2  - 2 \tau_{\ell}(\loss)(\bare + \e)^2,
\eens
and by  Lemma~\ref{lemma::oneextra},
\bens
\onenorm{\hat\Delta^t}^2 
& \le & 32 d \twonorm{\hat\Delta^t}^2  + 
2 \left(8\sqrt{d} \ve_{\static} + 2 \cdot \min\left(\frac{2
      \bar\eta}{\lambda}, R\right)\right)^2 \\
& \le & 32 d \twonorm{\hat\Delta^t}^2  + 2 (\bare + \e)^2.
\eens
\end{proofof}

\section{Some auxiliary results}
\label{sec::aux}
We first need to state the following form of the Hanson-Wright
inequality as recently derived in Rudelson and Vershynin~\cite{RV13},
and an auxiliary result in Lemma~\ref{lemma::oneeventA} which may be
of independent interests.
\begin{theorem}
\label{thm::HW}
Let $X = (X_1, \ldots, X_m) \in \R^m$ be a random vector with independent components $X_i$ which satisfy
$\expct{X_i} = 0$ and $\norm{X_i}_{\psi_2} \leq K$. Let $A$ be an $m \times m$ matrix. Then, for every $t > 0$,
\bens
\prob{\abs{X^T A X - \expct{X^T A X} } > t} 
\leq 
2 \exp \left[- c\min\left(\frac{t^2}{K^4 \fnorm{A}^2}, \frac{t}{K^2 \twonorm{A}} \right)\right].
\eens
\end{theorem}
We note that following the proof of Theorem~\ref{thm::HW}, it
is clear that the following holds:
Let $X = (X_1, \ldots, X_m) \in \R^m$ be a random vector 
as defined in Theorem~\ref{thm::HW}.
Let $Y, Y'$ be  independent copies of $X$. Let $A$ be an $m \times m$ matrix. Then, for every $t > 0$,
\ben
\label{eq::HWdecoupled} 
\prob{\abs{Y^T A Y'} > t} 
\leq 
2 \exp \left[- c\min\left(\frac{t^2}{K^4 \fnorm{A}^2}, \frac{t}{K^2 \twonorm{A}} \right)\right].
\een
We next need to state Lemma~\ref{lemma::oneeventA}, which we prove in Section~\ref{sec::proofoftensorA}.
\begin{lemma}
\label{lemma::oneeventA}
Let $u, w \in S^{n-1}$. 
Let $A \succ 0$ be an $m \times m$ symmetric positive definite matrix.
Let $Z$ be an $n \times m$
random matrix with independent entries $Z_{ij}$ satisfying
$\E Z_{ij} = 0$ and  $\norm{Z_{ij}}_{\psi_2} \leq K$.
Let $Z_1, Z_2$ be independent copies of $Z$.
Then for every $t > 0$,
\bens
\prob{\abs{u^T Z_1 A^{1/2} Z_2^T w} >  t}
& \le & 
2 \exp\left(-c\min\left(\frac{t^2}{K^4\tr(A)}, \frac{t}{K^2 \twonorm{A}^{1/2}}\right)\right), \\
\prob{\abs{u^T Z A Z^T w - \E u^T Z A Z^T w } >  t}
& \le & 
2 \exp\left(-c\min\left(\frac{t^2}{K^4\fnorm{A}^2}, \frac{t}{K^2 \twonorm{A}}\right)\right),
\eens
where $c$ is the same constant as defined in Theorem~\ref{thm::HW}.
\end{lemma}

\subsection{Proof of Lemma~\ref{lemma::trBest}}
\label{sec::proofofTR}
\begin{proofof2}
First we write
\bens
 X X^T - \tr(A) I_{n}
& = & 
\big( Z_1 A^{1/2} + B^{1/2} Z_2) \big(Z_1 A^{1/2} +B^{1/2} Z_2\big)^T - \tr(A) I_{n} \\ 
& = & 
\big( Z_1 A^{1/2} + B^{1/2} Z_2) \big(Z_2^T B^{1/2} +A^{1/2} Z_1^T\big) 
- \tr(A) I_{n} \\
& = & 
 Z_1 A^{1/2} Z_2^T B^{1/2} + B^{1/2} Z_2  Z_2^T B^{1/2} \\
&& + B^{1/2} Z_2  A^{1/2} Z_1^T+
 Z_1 A Z_1^T- \tr(A) I_{n}.
\eens
Thus we have for $\check\tr(B)  := \onem \big(\fnorm{X}^2 -n \tr(A)\big)$,
\bens
\lefteqn{\onen (\check\tr(B) - \tr(B)) := 
\inv{mn}\big(\fnorm{X}^2 - n\tr(A) -m \tr(B) \big) } \\
& = & \inv{mn} (\tr(XX^T) -   n \tr(A) - m\tr(B))\\
& = &  
\frac{2}{mn} 
\tr( Z_1 A^{1/2} Z_2^T B^{1/2} ) + 
\left(\frac{\tr(B^{1/2} Z_2  Z_2^T B^{1/2})}{mn} -\frac{
    \tr(B)}{n}\right) \\
&&
+ \frac{\tr( Z_1 A Z_1^T)}{mn}- \frac{\tr(A)}{m}.
\eens
By constructing a new matrix $A_n = I_n \otimes A$, which is block diagonal with $n$
identical submatrices $A$ along its diagonal, we prove the following
large deviation bound: for $t_1 = C_0 K^2 \fnorm{A} \sqrt{n \log m }$
and $n > \log m$,
\bens
\lefteqn{\prob{\abs{\tr(Z_1 A Z_1^T)- n \tr(A)} \ge t_1} = 
\prob{\abs{\mvec{Z_1}^T (I \otimes A) \mvec{Z_1}- n \tr(A)} \ge t_1}
}\\
& \le &  \exp\left(-c\min\left(\frac{t_1^2}{K^4 \fnorm{A_n}^2},
    \frac{t_1}{K^2 \twonorm{A_n}}\right)\right) \\
& \le & 2\exp\left(-c\min\left(\frac{ (C_0 K^2 \sqrt{n \log m}
      \fnorm{A})^2}{n K^4\fnorm{A}^2},
\frac{C_0 K^2 \sqrt{n \log m} \fnorm{A}}{K^2 \twonorm{A}}\right)\right) \\
& \le &  2 \exp\left(-4 \log m\right),
\eens
where the first inequality holds by Theorem~\ref{thm::HW} and the
second inequality holds given that $\fnorm{A_n}^2 = n \fnorm{A}^2$ and
$\twonorm{A_n} = \twonorm{A}$.

Similarly, by constructing a new matrix $B_m = I_m \otimes B$, which is block diagonal with $m$
identical submatrices $B$ along its diagonal, we prove the following
large deviation bound: for $t_2 =C_0  K^2\fnorm{B} \sqrt{m \log m}$
and $m \ge 2$,
\bens
\lefteqn{\prob{\abs{\tr(Z_2^T B Z_2)- m \tr(B)} \ge t_2} = 
\prob{\abs{\mvec{Z_2}^T (I_m \otimes B) \mvec{Z_2}- m \tr(B)} \ge t_2}}\\
& \le &  \exp\left(-c\min\left(\frac{t_2^2}{K^4 m \fnorm{B}^2}, \frac{t_2}{K^2 \twonorm{B}}\right)\right) \\
& \le & 2\exp\left(-c\min\left(\frac{ (C_0 K^2 \sqrt{m \log m}  \fnorm{B})^2}{K^4 m\fnorm{B}^2},
\frac{C_0 K^2 \sqrt{m \log m} \fnorm{B}}{K^2 \twonorm{B}}\right)\right) \\
& \le &  2 \exp\left(-4 \log m\right).
\eens
Finally, we have by~\eqref{eq::HWdecoupled} for $t_0 = C_0 K^2 \sqrt{\tr(A) \tr(B) \log m}$, 
\bens
\lefteqn{\prob{ \abs{\mvec{Z_1}^T B^{1/2} \otimes
    A^{1/2} \mvec{Z_2}} >t_0}} \\
& \le & 2\exp\left(-c\min\left(\frac{t_0^2}{K^4 
\fnorm{B^{1/2} \otimes A^{1/2}}^2},\frac{t_0}{K^2 \twonorm{B^{1/2}\otimes A^{1/2}}}\right)\right)\\
& = & 
2\exp\left(-c\min\left(\frac{(C_0 \sqrt{\tr(A) \tr(B) \log
        m})^2}{\tr(A) \tr(B)}, 
\frac{C_0\sqrt{\tr(A) \tr(B) \log m} }{\twonorm{B}^{1/2} \twonorm{A}^{1/2}}\right)\right)\\
& \le & 2 \exp(-4 \log m),
\eens
where we use the fact that $r(A) r(B) \ge \log m$,
$\twonorm{B^{1/2} \otimes A^{1/2}} = \twonorm{B}^{1/2} \twonorm{A}^{1/2}$ and
\bens
\fnorm{B^{1/2} \otimes A^{1/2}}^2 & = & 
\tr((B^{1/2} \otimes A^{1/2})(B^{1/2} \otimes A^{1/2}))
=\tr(B \otimes A) = \tr(A)\tr(B).
\eens
Thus we have with probability $1- 6/m^4$,
\bens
\lefteqn{\onen \abs{\check\tr(B) - \tr(B)}
= \inv{mn} \abs{\tr(XX^T) -   f\tr(A) - m\tr(B)}}\\
& \le &
\frac{2}{mn} \abs{\mvec{Z_1}^T (B^{1/2} \otimes A^{1/2})\mvec{Z_2}} \\
&& + \abs{\frac{\tr(Z_2^T B Z_2)}{mn} - \frac{\tr(B)}{n}}
+ \abs{\frac{\tr(Z_1 A Z_1^T)}{mn}- \frac{\tr(A)}{m }}\\
& \le &    \inv{mn} (2t_0 + t_1 + t_2 )  =\frac{\sqrt{\log
    m}}{\sqrt{m n}}  C_0 K^2 \left(\frac{\fnorm{A}  }{\sqrt{m} }+ 2\sqrt{\tau_A \tau_B} +
  \frac{\fnorm{B}}{\sqrt{n}}\right) \\
& \le & 2 C_0 \frac{\sqrt{\log m}}{\sqrt{m n}}  K^2 D_1 =: D_1 r_{m,m},
\eens
where recall  $ r_{m,m} = 2 C_0 K^2 \frac{\sqrt{\log m}}{\sqrt{m n}}$,
$D_1=\frac{\fnorm{A}  }{\sqrt{m} }+ \frac{\fnorm{B}}{\sqrt{n}}$, and
\bens
 2\sqrt{\tau_A \tau_B} \le  \tau_A + \tau_B \le \frac{\fnorm{A}
 }{\sqrt{m} }+  \frac{\fnorm{B}}{\sqrt{n}}.
\eens
To see this, recall
\ben
\label{eq::tracefnorm}
m \tau_A & = & \sum_{i=1}^m\lambda_{i}(A) \le \sqrt{m} (\sum_{i=1}^m
\lambda^2_{i}(A))^{1/2} = \sqrt{m} \fnorm{A} \; \quad \text{ and }\\
\nonumber
n \tau_B & = & \sum_{i=1}^n \lambda_{i}(B) \le \sqrt{n} (\sum_{i=1}^n
\lambda^2_{i}(B))^{1/2} = \sqrt{n} \fnorm{B} 
\een
where $\lambda_{i}(A), i=1, \ldots, m$ and $\lambda_{i}(B), i=1,
\ldots, n$  denote the eigenvalues of positive semidefinite covariance
matrices $A$ and $B$ respectively.

Denote  by $\B_6$ the following event  
$$\left\{\onen \abs{\check\tr(B) - \tr(B)}  \le D_1 r_{m,m}\right\}.$$
Clearly $\hat\tr(B) := (\check\tr(B))_+$ by definition~\eqref{eq::trBest}.
As a consequence, on $\B_6$, 
 $\hat\tr(B) = \check\tr(B) > 0$ when   $\tau_B > D_1 r_{m,m}$;  hence
\bens
\onen \abs{\hat\tr(B) - \tr(B)} = \onen \abs{\check\tr(B) - \tr(B)} \le
D_1 r_{m,m}.
\eens
Otherwise,  it is possible that $\check\tr(B) < 0$.
However, suppose we set  
$$\hat \tau_B : = 
\onen \hat\tr(B) :=\onen (\check\tr(B) \vee 0),$$ 
then we can also guarantee that 
\bens
\abs{\hat\tau_B - \tau_B} = \abs{\tau_B} \le D_1 r_{m,m} \; \; 
\text{ in case } \; \; \tau_B \le D_1 r_{m,m}.
\eens
The lemma is thus proved.
\end{proofof2}

\subsection{Proof of Lemma~\ref{lemma::Tclaim1}}
\label{sec::proofTclaim1}

Following Lemma~\ref{lemma::oneeventA}, we have for all $t >0$, $B
\succ 0$ being an $n \times n$ symmetric positive definite matrix, and 
$v, w \in \R^m$
\ben
\label{eq::crossB}
\prob{\abs{v^T Z_1^T B^{1/2} Z_2 w} >  t}
 \le  
2 \exp\left[-c\min\left(\frac{t^2}{K^4\tr(B)},
 \frac{t}{K^2 \twonorm{B}^{1/2}}\right) \right]
\een
and
\ben
\prob{\abs{v^T Z^T B Z w - \E v^T Z^T B Z w } >  t} \le 
2 \exp\left(-c\min\left(\frac{t^2}{K^4\fnorm{B}^2},
 \frac{t}{K^2 \twonorm{B}}\right)\right).
\een
\begin{proofof}{Lemma~\ref{lemma::Tclaim1}}
Let $e_1, \ldots, e_m \in \R^m$ be the canonical basis spanning
$\R^m$.
Let $x_1, \ldots, x_m, x'_1, \ldots, x'_m \in \R^n$ be the column vectors
$Z_1, Z_2$ respectively. Let $Y \sim e_1^T Z_0^T$. Let $w_i =
\frac{A^{1/2}  e_i}{\twonorm{A^{1/2} e_i}}$ for all $i$.
Clearly the condition on the stable rank of $B$ guarantees that 
$$n \ge r(B) = \frac{\tr(B)}{\twonorm{B}} =
\frac{\tr(B)\twonorm{B}}{\twonorm{B}^2}\ge \fnorm{B}^2 /\twonorm{B}^2 
 \ge \log m.$$
By~\eqref{eq::HWdecoupled}, we obtain for $t' = C_0 M_{\e} K
\sqrt{\tr(B)\log m}$
\bens
\lefteqn{
\prob{\exists j, \abs{\e^T B^{1/2} Z_2 e_j} >  t'}= } \\
&& 
\prob{\exists j, \frac{M_{\e}}{K}\abs{e_1^T Z_0^T B^{1/2} Z_2 e_j} > 
C_0 M_{\e} K \sqrt{\log m}\tr(B)^{\half}} \\
& \le &
\exp(\log m)\prob{\abs{Y^T B^{1/2} x'_j} > C_0 K^2
  \sqrt{\log m}\tr(B)^{\half} } 
\le  2/m^3
\eens
where the last inequality holds by the union bound, 
given that $\frac{\tr(B)}{\twonorm{B}} \ge \log m$;
Similarly, for all $j$ and $t = C_0 K^2 \sqrt{\log m} \tr(B)^{1/2}$,
\bens 
\prob{\abs{Y^T B^{1/2} x'_j} >  t} & \le &
2 \exp\left(-c\min\left(\frac{t^2}{K^4 \tr(B)},
 \frac{t}{K^2 \twonorm{B}^{1/2}}\right)\right), \\
& \le & 2 \exp\left(-c\min\left(C_0^2 \log m,
\frac{C_0 \log^{1/2} m \sqrt{\tr(B)}}{\twonorm{B}^{1/2}} \right)\right) \\
& \le &  2 \exp\left(-c \min(C_0^2, C_0) \log m\right)  \le 
2 \exp\left(-4 \log m\right).
\eens
Let $v, w \in S^{m-1}$. 
Thus we have by Lemma~\ref{lemma::oneeventA}, 
for $t_0 = C_0 M_{\e} K \sqrt{n\log m}$, $\tau= C_0 K^2 \sqrt{n\log m}$, 
$w_j= \frac{A^{1/2}  e_j}{\twonorm{A^{1/2} e_j}}$ and $n \ge \log m$,
\bens
\lefteqn{
\prob{\exists j, \abs{\e^T Z_1 w_j} >  t_0}
\le  \prob{\exists j, \frac{M_{\e}}{K}\abs{Y^T Z_1 w_j} > C_0 M_{\e} K \sqrt{n \log m} } }\\
& \le &
m \prob{\abs{Y^T  Z_1 w_j} > C_0 K^2 \sqrt{n \log m} } \\ 
& = & \exp(\log m)\prob{\abs{e_1^T Z_0^T Z_1 w_j} > \tau} 
 \le  2 \exp\left(-c\min\left(\frac{\tau^2}{n
       K^4},\frac{\tau}{K^2}\right)\right) =: V\\
\eens
where 
\bens
V & \le &
 2 \exp\left(-c\min\left(\frac{(C_0 K^2 \sqrt{n \log m})^2}{n  K^4}
 \frac{  C_0  K^2 \sqrt{n\log m}}{K^2}\right)+ \log m\right) \\
 & \le &
2 m \exp\left(-c\min\left(C_0^2 \log m, C_0 \log^{1/2} m \sqrt{n}
\right)\right) \\
& \le & 
2 m \exp\left(-c \min(C_0^2, C_0) \log m\right)  \le  2 \exp\left(-3\log m\right).
\eens
Therefore we have  with probability at least $1 - 4/m^3$,
\bens
\label{eq::eventB4a}
\norm{ Z_2^T B^{\half} \e}_{\infty}
& := & \max_{j=1, \ldots, m} \ip{\e^T B^{1/2} Z_2, e_j}  \le t' = C_0
M_{\e} K \sqrt{\tr(B)\log m}  \\
\norm{A^{\half} Z_1^T \e}_{\infty}
& := & \max_{j=1, \ldots, m} \ip{A^{1/2} e_j, Z_1^T  \e} \\
&\le &
\max_{j=1, \ldots, m} \twonorm{A^{1/2} e_j}
\max_{j=1, \ldots, m}  \ip{w_j, Z_1^T \e} \\
& \le & 
a_{\max}^{1/2} t_0 = a_{\max}^{1/2}  C_0 M_{\e} K \sqrt{n \log m}.
\eens
The ``moreover'' part follows exactly the same arguments as above.
Denote by $\bar\beta^* := \beta^*/\twonorm{\beta^*} \in E \cap
S^{m-1}$ and $w_i := A^{1/2} e_i/\twonorm{ A^{1/2} e_i}$.
By~\eqref{eq::crossB}
\bens
\lefteqn{\prob{\exists i,  \ip{w_i, Z_1^T B^{1/2} Z_2 \bar \beta^*}
    \ge  
C_0 K^2 \sqrt{\log m } {\tr(B)^{1/2}} }} \\
 & \le &
\sum_{i=1}^m \prob{\ip{w_i, Z_1^T B^{1/2} Z_2 \bar\beta^*} \ge  C_0 K^2
\sqrt{\log m \tr(B)}} \\
& \le & 
2 \exp\left(-c\min\left(C_0^2\log m,C_0 \log m\right)+ \log m\right) \le  2/ m^3.
\eens
Now for $t = C_0 K^2 \sqrt{\log m} \fnorm{B}$ and  $\fnorm{B}/\twonorm{B} \ge \sqrt{\log m}$,
\bens
\lefteqn{\prob{\exists e_i:  \ip{e_i, (Z^T B Z- \tr(B) I_{m}) \bar \beta^*} 
 \ge  C_0 K^2 \sqrt{\log m} \fnorm{B}}} \\
& \le & 
2m\exp \left[- c\min\left(\frac{t^2}{K^4  \fnorm{B}^2}, 
\frac{t}{K^2 \twonorm{B}} \right)\right] 
\le  2/ m^3.
\eens
By the two inequalities immediately above, we have with probability at
least $1 -  4 /m^3$, 
\bens
\lefteqn{\norm{X_0^T W \beta^*}_{\infty} 
= \norm{A^{1/2} Z_1^T B^{1/2} Z_2 \beta^*}_{\infty}  } \\
& \le & \twonorm{\beta^*} \max_{e_i}\twonorm{A^{1/2} e_i}
 \left(\sup_{w_i} \ip{w_i, Z_1^T B^{1/2} Z_2 \bar\beta^*}\right)  \\
& \le &  C_0 K^2 \twonorm{\beta^*}\sqrt{\log m} a_{\max}^{1/2}  \sqrt{\tr(B)}
\eens
and 
\bens
\lefteqn{
\norm{(Z^T B Z- \tr(B) I_{m}) \beta^*}_{\infty}
 =  \norm{(Z^T B Z- \tr(B) I_{m}) \bar \beta^*}_{\infty} \twonorm{\beta^*} }\\
&=  & \twonorm{\beta^*} \left(\sup_{e_i} \ip{e_i, (Z^T B Z- \tr(B)
    I_{m}) \bar \beta^*}\right)  \\
& \le &  C_0 K^2 \twonorm{\beta^*}\sqrt{\log m} \fnorm{B}.
\eens
The last two bounds follow exactly the same arguments as above, except
that we replace $\beta^*$ with $e_j, j=1, \ldots, m$ and apply the
union bounds to $m^2$ 
instead of $m$ events, and thus 
$\prob{\B_{10}} \ge 1 -  4 /m^2$.
\end{proofof}

\section{Proof of Corollary~\ref{coro::low-noise}}
\label{sec::lownoiseproof}
\begin{proofof2}
Now following \eqref{eq::oracle}, we have on event $\B_0$,
\bens
\norm{\hat\gamma - \hat\Gamma \beta^*}_{\infty} 
& \le &
\rho_{n} \left(\left(\frac{3}{4}D_2 + D_2 \inv{\sqrt{m}}\right)
 K \twonorm{\beta^*} + D_0  M_{\e} \right)
\eens
where $2D_1 \le 2 \twonorm{A} +  2 \twonorm{B} = D_2$, and for 
$(D'_0)^2 \le 2 \twonorm{B} + 2 a_{\max}$, 
\bens
D_0 \le D'_0  & \le &  \sqrt{2(\twonorm{B} + a_{\max})}
 \le 2(a_{\max} + \twonorm{B}) = D_2, \\
\text{ and } \; \; 
D'_0 \tau_{B}^{1/2} & \le &  (\twonorm{B}^{1/2} + a_{\max}^{1/2})\tau_{B}^{1/2} 
 \le \tau_B + \half(\twonorm{B} + a_{\max})  \le \frac{3}{4} D_2
\eens
given that under (A1) : $\tau_A = 1$, $\twonorm{A} \ge a_{\max} \ge a_{\max}^{1/2} \ge 1$.
Hence the lemma holds for $m \ge 16$ and $\psi = C_0 D_2 K \left(K \twonorm{\beta^*} + M_{\e} \right)$.
\end{proofof2}

\section{Proof of Corollary~\ref{coro::D2improv}}

\begin{proofof2}
Suppose that event $\B_0$ holds.
Recall $D_0'=\twonorm{B}^{1/2} + a_{\max}^{1/2}$.
Denote by $\rho_{n} := C_0 K  \sqrt{\frac{\log m}{n}}.$
By~\eqref{eq::oracle} and the fact that 
$2 D_1 := 2(\frac{\fnorm{A}  }{\sqrt{m}} +\frac{\fnorm{B}  }{\sqrt{n}}
)\le  2(\twonorm{A}^{1/2} + \twonorm{B}^{1/2}) (\sqrt{\tau_A} +
\sqrt{\tau_B}) \le 
D_{\ora} D_0'$,
\bens
\label{eq::oracleII}
\norm{\hat\gamma - \hat\Gamma \beta^*}_{\infty}
& \le &
D_0' K  \tau_B^{1/2} \twonorm{\beta^*} \rho_{n}
+ 2D_1 K \inv{\sqrt{m}} \norm{\beta^*}_{\infty} \rho_{n} + D_0 M_{\e} \rho_{n} \\
& \le &
D_0' K \twonorm{\beta^*} \rho_{n}
\left( \tau_B^{1/2}  +  \frac{D_{\ora}}{\sqrt{m}} \right) + D_0  M_{\e} \rho_{n} 
\eens
The corollary is thus proved.
\end{proofof2}

\section{Proof of Lemma~\ref{lemma::lowerREI}}
\label{sec::records}
\begin{proofof2}
In view of Remark~\ref{rem::error-bound}, Condition \eqref{eq::trBlem}
implies that  \eqref{eq::trB} in Theorem~\ref{thm::AD} holds for 
$k = s_0$ and $\ve =\inv{2 M_A}$.
Now, by Theorem~\ref{thm::AD}, we have $\forall u, v \in E \cap
S^{m-1}$, under (A1) and (A3), condition \eqref{eq::Deltacond} holds
under event $\A_0$, and so long as  
$m n \ge 4096 C_0^2 D_2^2 K^4 \log m/\lambda_{\min}(A)^2$, 
\bens
\abs{u^T \Delta v} & \le&
  8C \vp(s_0) \ve + 2 C_0 D_2 K^2\sqrt{\frac{\log m}{m n}} =: \delta  \text{ with } \; \\
\delta & \le & 
\frac{\lambda_{\min}(A)}{16} + 
\frac{\lambda_{\min}(A)}{32} = \frac{3}{32} \lambda_{\min}(A)  \le
\inv{8},
\eens
which holds for all
\bens
\ve \le \half \frac{\lambda_{\min}(A) }{64 C \vp(s_0)} := \inv{2M_A}  \le \inv{128 C}
\eens
with $\prob{\A_0} \ge 1- 4 \exp\left(-c_2\ve^2 \frac{\tr(B)}{K^4\twonorm{B}}\right)-2
  \exp\left(-c_2\ve^2 \frac{n}{K^4}\right) - 6 /m^3$. 

Hence, by Corollary~\ref{coro::BC},  $\forall  \theta \in \R^m$,
\bens
\theta^T \hat\Gamma_A \theta \ge \alpha \twonorm{\theta}^2 - \tau
\onenorm{\theta}^2\; \; \text{ and } \; \; 
\theta^T \hat\Gamma_A \theta \le \tilde\alpha \twonorm{\theta}^2 + \tau
\onenorm{\theta}^2,
\eens
where $\alpha = \frac{5}{8} \lambda_{\min}(A)$ and $\tilde\alpha =
\frac{11}{8} \lambda_{\max}(A)$ and $\tau = \frac{3}{8}
\frac{\lambda_{\min}(A)}{s_0}$.

Now for $s_0 \ge 32$ as defined in~\eqref{eq::s0cond}, 
we have 
\ben 
\label{eq::tights0}
s_0 & \le & \frac{n}{\log m}\frac{ \lambda^2_{\min}(A)} {1024 C^2
  \vp(s_0)^2} \\
\label{eq::tights0plus}
\text{ and } \quad s_0 + 1 & \ge &  \frac{n}{\log m} \frac{\lambda^2_{\min}(A)}{1024 C^2    \vp^2(s_0+1)} 
\een
given that $\tau_B + \rho_{\max}(s_0+1, A) = O(\lambda_{\max}(A))$  in
view of \eqref{eq::eigencond} and (A3).
Thus
\bens
\frac{384 C^2 \vp(s_0)^2}{\lambda_{\min}(A)}\frac{\log m}{n} \le 
\tau & = & \frac{3}{8} \frac{\lambda_{\min}(A)}{s_0} \\
& \le & \frac{33}{32(s_0+1)}\frac{3}{8} \frac{\lambda_{\min}(A)}{s_0} \\
& \le &\frac{396 C^2 \vp^2(s_0+1)}{\lambda_{\min}(A)}\frac{\log  m}{n}.
\eens 
The lemma is thus proved in view of Remark~\ref{rem::error-bound}.
\end{proofof2}

\begin{remark}
\label{rem::error-bound}
Clearly the condition on $\tr(B)/\twonorm{B}$ as stated in
Lemma~\ref{lemma::lowerREI} ensures that we have
for $\ve = \inv{2M_A}$ and $s_0 \asymp \frac{4 n}{M_A^2 \log m}$,
\bens
\ve^2 \frac{\tr(B)}{K^4\twonorm{B}} 
& \ge & \frac{\ve^2}{K^4} c' K^4 \frac{s_0}{\ve^2} \log\left(\frac{3e m}{s_0 \ve}\right) \\
& \ge & c' s_0 \log\left(\frac{6e m M_A}{s_0}\right),
\eens
and hence
\bens
\exp\left(-c_2\ve^2 \frac{\tr(B)}{K^4\twonorm{B}}\right)
& \le &
\exp\left(-c' c_2 s_0 \log\left(\frac{6e m M_A}{s_0}\right)\right) \\
& \asymp &
\exp\left(-c_3\frac{4 n}{M_A^2 \log m} \log\left(\frac{3e M_A^3 m \log m}{2n}\right)\right).
\eens
\end{remark}

\subsection{Comparing the two type of $\RE$ conditions in
 Theorems~\ref{thm::lasso} and~\ref{thm::DS}}
\label{sec::basediscuss}
We define $\W(d_0, k_0)$, where $0 < d_0 < m$
 and $k_0$ is a positive number,
as the set of vectors in $\R^m$ which satisfy the following cone constraint:
\bens
\label{eq::cone-init}
\W(d_0, k_0) = \left\{x \in \R^m \;|\; \exists I \in \{1, \ldots, m\}, \size{I} = d_0
\; \mbox{ s.t. } \; \norm{x_{I^c}}_1 \leq k_0 \norm{x_{I}}_1
\right\}. 
\eens
For each vector $x \in \R^m$, let ${T_0}$ denote the locations of the $d_0$
largest coefficients of $x$ in absolute values. The following
elementary estimate~\cite{RZ13} 
will be used in conjunction with the RE condition.
\begin{lemma}
\label{lemma::lower-bound-Az}
For each vector $x \in \W(d_0, k_0)$, let
${T_0}$ denotes the locations of the $d_0$
largest coefficients of $x$ in absolute values.  Then
\ben
\label{eq::cone-top-norm}
\twonorm{x_{T_{0}}} 
\geq \frac{\twonorm{x}}{\sqrt{1 + k_0}}.
\een
\end{lemma}

\begin{lemma}
\label{lemma::translation}
Suppose all conditions in Lemma~\ref{lemma::lowerREI} hold.
Let $k_0 := 1+\lambda$. 
Suppose that $d_0 = o\left(s_0 /64(1+3\lambda/4)^2\right)$.
Now suppose that 
$$\tau (1 + 3k_0)^2 2 d_0 = 2 \tau (4 + 3\lambda)^2 d_0 \le \alpha/2.$$
Then on event $\A_0$, we have $\RE^2(2 d_0, 3 k_0, \hat\Gamma_A)$
condition holds on $\hat\Gamma_A$ in the sense that
\ben 
\label{eq::REGamma}
\min_{x \in \W(2d_0, 3k_0)}\frac{x^T \hat\Gamma_A x}{\twonorm{x_{T_0}}^2}
\ge \frac{\alpha}{2}. 
\een
Under (A2) and (A3), we could set $d_0$ such that for some large enough constant $C_A$, 
\ben
 \label{eq::2d0cond}
d_0 \le \frac{n}{C_A \kappa(A)^2 \log m} = O\left(\frac{\lambda^2_{\min}(A)}{\vp^2(s_0+1)}
\frac{n}{\log m}\right)
\een
where $\kappa(A) := \frac{\lambda_{\max}(A)}{\lambda_{\min}(A)}$,
so that $d_0 =O(s_0)$ and \eqref{eq::REGamma} holds.
\end{lemma}
\begin{proof}
Now following the proof Lemma~\ref{lemma::REcomp}, Part I.
We have on $\A_0$, the Lower-$\RE$ condition holds for $\Gamma_A$.
Thus for $x \in \W(2d_0, 3k_0) \cap S^{m-1}$ and  $\tau (1 +3 k_0)^2 2 d_0 \le \alpha/2$, 
$$\onenorm{x}^2 \le (1+3 k_0)^2 \onenorm{x_{T_0}}^2 \le  
(1 + 3 k_0)^2 2 d_0 \twonorm{x_{T_{0}}}^2.$$
Thus
\bens
x^T \hat\Gamma_A x & \ge &  
\left(\alpha \twonorm{x}^2 - \tau \onenorm{x}^2\right) \\
& \ge &  
\left(\alpha \twonorm{x}^2 - \tau (1 + 3 k_0)^2 2d_0
  \twonorm{x_{T_{0}}}^2\right) \\
& \ge &  
\left(\alpha - \tau (1 + 3 k_0)^2 2 d_0 \right)
\twonorm{x_{T_{0}}}^2 
\ge \frac{\alpha}{2} \twonorm{x_{T_{0}}}^2.
\eens
Thus \eqref{eq::REGamma} holds. 
Now~\eqref{eq::2d0cond} follows from \eqref{eq::tights0}, which holds by
definition of $s_0$ as in \eqref{eq::s0cond}, where $s_0$ is tightly
bounded in the sense that both~\eqref{eq::tights0} and
\eqref{eq::tights0plus} need to hold.
\end{proof}

\begin{remark}
We note that \eqref{eq::REGamma} can be understood to be the
$\RE(2d_0, 3 k_0)$ condition on $\hat\Gamma_A$.
In view of Lemma~\ref{lemma::lowerREI}, it is clear that for $d_0
\asymp \sqrt{n/\log m}$, it holds that
$$4 d_0 (4 + 3 \lambda)^2 = o(s_0)$$
given that $ \tau s_0  = O(\alpha)$ on event $\A_0$; indeed, we have by Lemma~\ref{lemma::lowerREI} 
the Lower-$\RE$ condition holds for $\hat \Gamma_A := A^T A -
\hat{\tr}(B) I_m$, with $\alpha, \tau > 0$ such that 
\bens
\text{curvature} \; \; 
\alpha = \frac{5}{8}\lambda_{\min}(A) \; \text{ and tolerance }\;\;
\tau :=  \frac{3}{8}\frac{\lambda_{\min}(A)}{s_0}, 
\eens
where recall $s_0 \ge 32$ is as defined in  \eqref{eq::s0cond}; moreover,
we replaced the parameter $M_A \asymp \frac{\rho_{\max}(s_0,A) + \tau_B}{\lambda_{\min}(A)}$ with $\kappa(A)$ in view of 
\eqref{eq::eigencond} and (A3).
\end{remark}

\section{Proof of Theorem~\ref{thm::main}}
\label{sec::proofofmain}
Denote by  $\beta = \beta^*$. Let $S := \supp{\beta}$, $d = \size{S}$ and 
$$\upsilon = \hat{\beta} - \beta,$$
where $\hat\beta$ is as defined in \eqref{eq::origin}.

We first show Lemma~\ref{lemma:magic-number}, followed by the 
proof of Theorem~\ref{thm::main}.

\begin{lemma}{\textnormal~\cite{BRT09,LW12}}
\label{lemma:magic-number}
Suppose that~\eqref{eq::psimain} holds.
Suppose that there exists a parameter $\psi$ such that 
\bens
\sqrt{d} \tau \le \frac{\psi}{b_0} \sqrt{\frac{\log m}{n}} \quad \text{ and } \quad
\lambda \geq 4 \psi \sqrt{\frac{\log m}{n}},
\eens
where $b_0, \lambda$ are as defined in \eqref{eq::origin}. Then 
$$\norm{\upsilon_{S^c}}_1 \leq 3 \norm{\upsilon_{S}}_1.$$
\end{lemma}
\begin{proof}
By the optimality of $\hat{\beta}$, we have
\begin{eqnarray*}
\lambda \norm{\beta}_1 - 
\lambda \norm{\hat\beta}_1 
& \geq & 
\inv{2} \hat\beta \hat\Gamma  \hat\beta - \inv{2} \beta \hat\Gamma \beta -
\ip{\hat\gamma, v} \\
& = & 
\inv{2} \up \hat\Gamma \up +\ip{\up, \hat\Gamma \beta} -
\ip{\up, \hat\gamma} \\
& = & 
\inv{2} \up \hat\Gamma \up -\ip{\up, \hat\gamma - \hat\Gamma \beta}.
\end{eqnarray*}
Hence,
we have for $\lambda \geq 4 \psi \sqrt{\frac{\log m}{n}}$,
\begin{eqnarray}
\label{eq::precondition}
\half \up \hat\Gamma \up 
& \leq &
\ip{\up, \hat\gamma - \hat\Gamma \beta} +
 \lambda \left(\norm{\beta}_1-  \norm{\hat\beta}_1 \right)\\
\nonumber
& \leq & \lambda
\left(\norm{\beta}_1-  \norm{\hat\beta}_1\right) +
\norm{\hat\gamma - \hat\Gamma \beta}_{\infty} \norm{\upsilon}_1.
\end{eqnarray}
Hence
\begin{eqnarray}
\label{eq::upperbound}
\up \hat\Gamma \up 
& \leq & 
\lambda \left(2 \norm{\beta}_1- 2 \norm{\hat\beta}_1\right) +
2 \psi \sqrt{\frac{\log m}{n}} \norm{\upsilon}_1 \\
& \leq & 
\nonumber
\lambda \left(2\norm{\beta}_1 - 2\norm{\hat\beta}_1 + \half
  \norm{\upsilon}_1\right) \\
\label{eq::finalupperbound}
& \leq & \lambda \half \left(5 \onenorm{\upsilon_{S}} - 3\onenorm{\upsilon_{S^c}}\right),
\end{eqnarray}
where by the triangle inequality, and $\beta_{\Sc} = 0$, we have
\begin{eqnarray} 
\nonumber
2 \onenorm{\beta} -  2 \onenorm{\hat\beta} + \half \onenorm{\upsilon}
& = &
2 \onenorm{\beta_S} - 2 \onenorm{\hat\beta_{S}} -
2 \onenorm{\upsilon_{\Sc}} + \half\onenorm{\upsilon_S} +\half \onenorm{\upsilon_{S^c}} \\
& \leq &
\nonumber
2 \norm{\up_{S}}_1 -2 \norm{\up_{\Sc}}_1 + \half \norm{\up_S}_1 
+ \half \norm{\up_{S^c}}_1 \\ 
& \leq & 
\label{eq::magic-number-2}
 \half \left(5 \onenorm{\upsilon_{S}} - 3\onenorm{\upsilon_{S^c}}\right).
\end{eqnarray}
We now give a lower bound on the LHS of~\eqref{eq::precondition},
applying the lower-$\RE$ condition as in Definition~\ref{def::lowRE},
\ben
\nonumber
\up^T \hat\Gamma \up
& \ge &
\alpha \twonorm{\up}^2 - \tau \onenorm{\up}^2
\ge  - \tau \onenorm{\up}^2\\
\text{ and hence } \;
- \up^T \hat\Gamma \up
& \le & 
\nonumber
\onenorm{\up}^2 \tau \le \onenorm{\up} 2 b_0 \sqrt{d} \tau\\
& \le &
\nonumber
 \onenorm{\up} 2 b_0 
\frac{\psi}{b_0} \sqrt{\frac{\log m}{n}}
= \onenorm{\up} 2 \psi \sqrt{\frac{\log m}{n}}\\
\label{eq::magic-number-neg}
& \le & \half \lambda (\onenorm{\up_S} + \onenorm{\up_{\Sc}}),
\een
where we use the assumption that 
\bens
\sqrt{d} \tau \le \frac{\psi}{b_0} \sqrt{\frac{\log m}{n}}
  \quad \text{ and } \;
\onenorm{\up} \le \onenorm{\hat{\beta}} + \onenorm{\beta}
\le 2 b_0 \sqrt{d},
\eens
which holds by the triangle inequality and the fact that both 
$\hat{\beta}$ and $\beta$ have $\ell_1$ norm being bounded by $b_0 \sqrt{d}$.
Hence by~\eqref{eq::finalupperbound}
and~\eqref{eq::magic-number-neg}
\ben
\label{eq::magic-number}
0 
& \le &
 - \up \hat\Gamma \up + \frac{5}{2} \lambda \onenorm{\upsilon_{S}} -
\frac{3}{2} 
\lambda \onenorm{\upsilon_{S^c}} \\
\nonumber
& \le &
\half \lambda \onenorm{\upsilon_{S}} + \half \lambda \onenorm{\upsilon_{S^c}}
+ \frac{5}{2}\lambda \onenorm{\upsilon_{S}} - \frac{3}{2} \lambda \onenorm{\upsilon_{S^c}} \\
& \le &
3 \lambda \onenorm{\upsilon_{S}} - \lambda \onenorm{\upsilon_{S^c}}.
\een
Thus we have 
\bens
 \onenorm{\upsilon_{S^c}} \le 3 \onenorm{\upsilon_{S}},
\eens
and the lemma holds.
\end{proof}

\begin{proofof}{Theorem~\ref{thm::main}}
Following the conclusion of Lemma~\ref{lemma:magic-number}, we have
\ben
\label{eq::onenorm}
\onenorm{\upsilon} \le 4 \onenorm{\upsilon_{S}} \le 4 \sqrt{d} \twonorm{\upsilon}.
\een
Moreover, we have 
by the lower-$\RE$ condition as in
Definition~\ref{def::lowRE}
\ben
\label{eq::prelow}
\up^T \hat\Gamma \up
& \ge & 
\alpha \twonorm{\up}^2 - \tau \onenorm{\up}^2 \ge 
(\alpha  - 16 d \tau) \twonorm{\up}^2 \ge \half \alpha \twonorm{\up}^2,
\een
where the  last inequality follows from the assumption that 
$16 d \tau \le \alpha/2$.

Combining the bounds in \eqref{eq::prelow}, \eqref{eq::onenorm} and \eqref{eq::upperbound},
we have
\bens
\half \alpha  \twonorm{\upsilon}^2 
&\le &
\up^T \hat\Gamma \up \le 
\lambda \left(2 \norm{\beta}_1- 2 \norm{\hat\beta}_1\right) +
2 \psi \sqrt{\frac{\log m}{n}} \norm{\upsilon}_1 \\
&\le &
\frac{5}{2} \lambda  \norm{\upsilon_S}_1 
\le 10 \lambda \sqrt{d} \twonorm{\up}.
\eens
And thus we have $\twonorm{\up} \le 20 \lambda \sqrt{d}$.
The theorem is thus proved. 
\end{proofof}

\section{Proofs for the Lasso-type estimator}
\label{sec::lassoall}
Let 
\bens 
M_{+} & = & \frac{32 C \vp(s_0+1)}{\lambda_{\min}(A)}  \text{ and }
\; \; \vp(s_0+1) = \rho_{\max}(s_0+1,A) + \tau_B =: D.
\eens
By definition of $s_0$, we have $s_0 M_A^2 \le \frac{4n}{\log m}$ and
\ben
\nonumber
(s_0+1) \ge \frac{n}{M_+^2 \log m} \\
\text{ given that }\; \;\; 
\label{eq::s0plus1}
\sqrt{s_0+1} \vp(s_0+1) & \ge & \frac{\lambda _{\min}(A)}{32
  C}\sqrt{\frac{n}{\log m}}.
\een
To prove the  first inequality in~\eqref{eq::taumain} and
\eqref{eq::dcond}, we need to show that 
\bens
d \le \frac{\alpha}{32 \tau} = \frac{\alpha}{32}
\frac {s_0}{\lambda_{\min}(A) - \alpha}
= \frac{5 s_0}{96}.
\eens
The first inequality in~\eqref{eq::taumain} holds so long as 
\ben 
\label{eq::dphiII}
d &\le &  
\frac{1}{20}\inv{M_{+}^2}\frac{n}{\log m}  \le \frac{s_0 + 1}{20} \le
\frac{5 (s_0+1)}{100} \le \frac{5 s_0}{96},
\een
where the last inequality holds so long as $s_0 \ge 24$.
To prove the second inequality in~\eqref{eq::dcond}, we need to show that
\bens
d \le \inv{\tau^2}\frac{\log m}{n}
\left(\frac{\psi}{b_0}\right)^2,
\quad \text{ where } \; \; \tau = \frac{3}{5}\frac{\alpha}{s_0} \quad
\text{ for } \; \quad \alpha = \frac{5}{8} \lambda_{\min}(A),
\eens
which in turn ensures that the second inequality
in~\eqref{eq::taumain} holds for $\lambda \ge 4 \psi$, for $\psi$
appropriately chosen.
We use the following inequality in the proof of
Lemma~\ref{lemma::dmain} and  Lemma~\ref{lemma::dmainoracle}:
\ben
\nonumber
\frac{s_0+ 1}{\alpha^2}
& \ge &  
\frac{64}{25 \lambda_{\min}(A)^2}\inv{M_+^2}\frac{n}{\log m}  \ge 
\left(\frac{8}{5} \inv{32 C\vp(s_0+1)}\right)^2
\frac{n}{\log m} \\
\label{eq::s0plus}
& = & \left(\inv{20 C D}\right)^2 \frac{n}{\log m} 
  \ge \left(\frac{1}{10 C D_2} \right)^2 \frac{n}{\log  m},
\een
where we use the fact that $D= \vp(s_0+1) = \rho_{\max}(s_0+1,A) +\tau_B \le
\twonorm{A} + \twonorm{B} :=  D_2/2$.

\subsection{Proof of Lemma~\ref{lemma::dmain}}
\label{sec::dmain}
\begin{proofof2}
Let $C_A = \inv{40 M^2_{+}}$. 
The first inequality in~\eqref{eq::dcond} holds in view of
\eqref{eq::dphiII}.
Recall that $ b_0^2 \ge \twonorm{\beta^*}^2 
\ge \phi b_0^2$  by definition of $ 0 < \phi \le 1$.
Let   $C = C_0/\sqrt{c'}$.
By \eqref{eq::dlassoproof} and \eqref{eq::s0plus},
\bens 
d  & \le  & C_A  c' D_{\phi}\frac{n}{\log m}  \le
\frac{1}{40 M_+^2}\left(\frac{C_0 D_2}{C D_2}\right)^2 D_{\phi}  \frac{n}{\log m} \\
& \le & 
\frac{25}{9}\frac{32}{33}\frac{32}{33} \frac{n}{M_+^2 \log m} 
 \left(\inv{10 C D_2}\right)^2  C_0^2 D_2^2 D_{\phi} \\
& \le  & 
\frac{25}{9}\frac{32}{33}\frac{32 (s_0+1)}{33}
\frac{(s_0 +1)}{\alpha^2} \frac{\log m}{n}
\left(\frac{\psi}{b_0}\right)^2 \\
& \le  & 
\frac{25}{9}\frac{(s_0)^2}{\alpha^2} \frac{\log m}{n}
\left(\frac{\psi}{b_0}\right)^2,
\eens 
where 
\ben
\nonumber
C_0^2 D_2^2 D_{\phi} & = & C_0^2
D_2^2\left(\frac{K^2M^2_{\e}}{b_0^2}+  K^4 \phi \right) \\
\label{eq::dphicondition}
& \le & C_0^2 D_2^2  \frac{K^2}{b_0^2} (M_{\e}+ K \twonorm{\beta^*})^2 =\left(\frac{\psi}{b_0}\right)^2,
\een 
for $\psi = C_0 D_2 K (K \twonorm{\beta^*}+ M_{\e})$ as defined in~\eqref{eq::psijune}.
We have shown that \eqref{eq::dcond} indeed holds, and the lemma is
thus proved.
\end{proofof2}

\subsection{Proof of Lemma~\ref{lemma::dmainoracle}}
\label{sec::dmainoracle}
\begin{proofof2}
Let $C_A = \inv{160 M^2_{+}}$. 
The proof for 
$d \le \frac{\alpha}{32 \tau}  = \frac{5 s_0}{96}$ 
follows from \eqref{eq::dphiII}. 
In order to show the second inequality, we follow the same line of
arguments except that we need to replace one inequality
\eqref{eq::dphicondition} with \eqref{eq::dphicondition2}. 
By definition of $D_0'$, we have 
$ \twonorm{B} + a_{\max} \le (D_0')^2 \le 2(\twonorm{B} + a_{\max}
)$. Let $D= \vp(s_0+1)$.

By \eqref{eq::doraclelocal},~\eqref{eq::s0plus1} and \eqref{eq::s0plus}, we have for $c''
\le \left(\frac{D_0'}{D}\right)^2$, 
\bens
d  & \le  &  
C_A  c' c'' D_{\phi}\frac{n}{\log m}  \le
\inv{160 M_+^2} \frac{n}{\log m} \left(\frac{C_0 D_0'}{C D}\right)^2 D_{\phi} \\
& \le & 
\frac{25}{9}\frac{32^2}{33^2} \left(\inv{20 C D}\right)^2 \left(C_0^2 (D_0')^2
  D_{\phi}\right)  \frac{n}{M_+^2 \log m} \\
& \le  & 
\frac{25}{9}\frac{32^2}{33^2}\frac{(s_0 +1)^2}{\alpha^2} \frac{\log m}{n}
\left(\frac{\psi}{b_0}\right)^2 
\le \frac{25}{9}
\frac{(s_0)^2}{\alpha^2} \frac{\log m}{n}
\left(\frac{\psi}{b_0}\right)^2,
\eens
where assuming that $s_0 \ge 32$, we have the following
 inequality by definition of $s_0$ and $\alpha = \frac{5}{8}\lambda_{\min}(A)$,
\bens
\frac{s_0+ 1}{\alpha^2} \frac{\log  m}{n}
& \ge & \left(\frac{8}{5} \inv{32 C\vp(s_0+1)}\right)^2  \ge \left(\frac{1}{20 C D} \right)^2.
\eens
We now replace~\eqref{eq::dphicondition} with 
\ben
\nonumber
C_0^2 (D_0')^2 D_{\phi} 
& = & C_0^2 (D_0')^2 \frac{K^4}{b_0^2} \left(\frac{M^2_{\e}}{K^2} +
  \tau_B^+ \phi b_0^2\right)\\
\label{eq::dphicondition2}
& \le  &  C_0^2 (D_0')^2  \frac{K^2}{b_0^2} 
\left(M_{\e}+ \tau_B^{+/2}  K\twonorm{\beta^*} \right)^2  \le 
\left(\frac{\psi}{b_0}\right)^2, \\
\nonumber
\text{ where  } \; \; 
D_{\phi} &  := & \frac{K^2M^2_{\e}}{b_0^2}+  \tau_B^+ K^4 \phi 
\le \frac{K^4}{b_0^2}\left(\frac{M^2_{\e}}{K^2} + \tau_B^{+}  \twonorm{\beta^*}^2 \right)
\een
and $\psi = C_0  D_0' K \left(K \tau_B^{+/2}
  \twonorm{\beta^*} + M_{\e} K\right)$ is now as defined in~\eqref{eq::psijune15}.
The lemma is thus proved.
\end{proofof2}

\silent{
Denote by 
\bens
C_{\phi} =  \frac{\twonorm{B} + a_{\max} }{\vp(s_0+1)^2} D_{\phi} =:
c'' D_{\phi}
\eens
where $1 \le D= \rho_{\max}(s_0+1,A) +\tau_B$ and  $C = C_0/\sqrt{c'}$.
\bens
\nonumber
\frac{2 s_0^2}{\alpha^2} & \ge & \left(\frac{s_0+1}{\alpha}\right)^2 
  \ge \frac{\alpha^2}{(16 C D )^4} \left(\frac{n}{\log m}\right)^2  
\eens}

\begin{remark}
Throughout this paper, we assume that $C_0$ is a large enough constant
such that  for $c$ as defined in Theorem~\ref{thm::HW},
\ben\label{eq::defineC0}
c \min\{C_0^2, C_0\} \ge 4.
\een 
By definition of $s_0$, we have for $\vp^2(s_0) \ge 1$,
\bens
s_0  \vp^2(s_0) & \le &  \frac{c'\lambda^2_{\min}(A)}{1024
  C_0^2}\frac{n}{\log m}, \; \; \text{ and hence}\\
s_0 & \le & 
\frac{c'\lambda^2_{\min}(A)}{1024
  C_0^2}\frac{n}{\log m} \le \frac{\lambda^2_{\min}(A)}{1024
  C_0^2}\frac{n}{\log m} =: \check{s}_0.
\eens
\end{remark}

\begin{remark}
The proof shows that one can take $C = C_0/\sqrt{c'}$, and take
\bens
\V = 3 e M_A^3/2 =
  \frac{3 e 64^3 C^3 \vp^3(s_0)}{2\lambda^3_{\min}(A)} \le 
  \frac{3 e 64^3 C_0^3 \vp^3(\check{s}_0)}
{2 (c')^{3/2}\lambda^3_{\min}(A)}.
\eens
Hence a sufficient condition on $r(B)$ is:
\ben
\label{eq::trBLassorem}
r(B) \ge 16c' K^4 \frac{n}{\log m}
\left(3\log\frac{ 64 C_0 \vp(\check{s}_0)}{\sqrt{c'}
\lambda_{\min}(A)}  + \log \frac{3 e m \log m }{2n} \right).
\een
\end{remark}

\section{Proofs for the Conic Programming estimator}
\label{sec::proofofDSlemma}
We next provide proof for Lemmas~\ref{lemma::DS} to~\ref{lemma::grammatrix} in this section.

\subsection{Proof of Lemma~\ref{lemma::DS}}
\begin{proofof2}
Suppose event $\B_0$ holds.
Then by the proof of Corollary~\ref{coro::low-noise},
\bens
\norm{\onen X^T(y - X \beta^*) + \onen \hat\tr(B) \beta^*}_{\infty} 
& = & \norm{\hat\gamma - \hat\Gamma  \beta^*}_{\infty}  \\
& \le & 
2 C_0 D_2 K^2\twonorm{\beta^*}\sqrt{\frac{\log m}{n}} + C_0 D_0 K M_{\e} \sqrt{\frac{\log m}{n}} \\
& =:& 
\mu \twonorm{\beta^*} + \omega.
\eens
The lemma follows immediately for the chosen $\mu, \omega$ as
in~\eqref{eq::paraDS} given that $(\beta^*, \twonorm{\beta^*}) \in \U$.
\end{proofof2}

\subsection{Proof of Lemma~\ref{lemma::DS-cone}}
\begin{proofof2}
By optimality of $(\hat\beta, \hat{t})$, we have
\bens
\onenorm{\hat\beta} 
+ \lambda \twonorm{\hat\beta} \le \norm{\hat\beta}_1 +
\lambda \hat{t} \le \onenorm{\beta^*} + \lambda \twonorm{\beta^*}.
\eens
Thus we have for $S := \supp(\beta^*)$,
\bens
\onenorm{\hat\beta} =
\onenorm{\hat\beta_{\Sc} }+ \onenorm{\hat\beta_{S}}
& \le &  \onenorm{\beta^*}   +   \lambda (\twonorm{\beta^*} -  
 \twonorm{\hat\beta}).
\eens
Now by the triangle inequality, 
\bens
\onenorm{\hat\beta_{\Sc} } =  
\onenorm{v_{\Sc} } 
& \le & 
 \onenorm{\beta^*_{S}} -   \onenorm{\hat\beta_{S}} + 
 \lambda (\twonorm{\beta^*} -  \twonorm{\hat\beta} ) \\
& \le & 
\onenorm{v_{S}} +  \lambda ( \twonorm{\beta^*} - 
 \twonorm{\hat\beta} ) \\
& \le & 
\onenorm{v_{S}} +  \lambda ( \twonorm{\beta^*} -  \twonorm{\hat\beta_S} ) \\
& = & 
\onenorm{v_{S}} 
+  \lambda  \twonorm{v_S}  \le (1+\lambda) \onenorm{v_{S}}.
 \eens
The lemma thus holds given 
\bens
\hat{t} & \le & 
\inv{\lambda }( \onenorm{\beta^*} -  \norm{\hat\beta}_1)+
\twonorm{\beta^*} 
\le \inv{\lambda }\onenorm{v} + \twonorm{\beta^*}.
\eens
\end{proofof2}

\subsection{Proof of Lemma~\ref{lemma::grammatrix}}
\begin{proofof2}
Recall the following shorthand notation:
\bens 
D_0 & = &  (\sqrt{\tau_B} + \sqrt{a_{\max}}) \; 
\;\text{ and }  D_2\; = \; 2 (\twonorm{A} + \twonorm{B}).
\eens 
First we rewrite an upper bound for $v = \hat\beta -\beta^*$, 
$D = \tr(B)$ and $\hat{D} = \hat\tr(B)$,
\bens
\norm{X_0^T X_0 v}_{\infty} & =
& \norm{(X-W)^T X_0 (\hat\beta- \beta^*)}_{\infty} 
\le 
\norm{X^T X_0 (\hat\beta - \beta^*)}_{\infty} +
\norm{ W^T X_0 v}_{\infty}\\
& \le & 
\norm{X^T(X \hat\beta - y)- \hat{D} \hat\beta}_{\infty} +  \norm{X^T \e}_{\infty}
+ \norm{(X^T W -D) \hat\beta}_{\infty} \\
&+& \norm{(\hat{D}-D) \hat\beta}_{\infty} +\norm{W^T X_0 v}_{\infty},
\eens
where 
\bens
\norm{X^T X_0 (\hat\beta - \beta^*)}_{\infty}
& \le & 
\norm{X^T (X_0 \hat\beta -  y + \e )}_{\infty}\\
& = & 
\norm{X^T ((X -W)\hat\beta -  y)}_{\infty} + \norm{X^T  \e }_{\infty} \\
& \le &
\norm{X^T (X \hat\beta -  y) - \hat{D} \hat\beta}_{\infty} + \norm{X^T
  \e }_{\infty} \\
&& +  
\norm{(X^T W - D) \hat \beta}_{\infty} + \norm{(\hat{D}- D) \hat \beta}_{\infty}.
\eens
On event $\B_0$, we have by Lemma~\ref{lemma::DS-cone} and the fact that $\hat\beta \in \U$,
\bens
I :=  \norm{\hat\gamma - \hat\Gamma  \hat\beta}_{\infty} 
& = & 
\norm{\onen X^T(y - X \hat\beta) + \onen \hat{D}
  \hat\beta}_{\infty}\le 
\mu \hat{t}  + \omega \\
& \le & 
 \mu (\inv{\lambda}\onenorm{v} +
\twonorm{\beta^*} )+ \omega  \\
& = &
2 D_2 K \rho_{n} (\inv{\lambda}\onenorm{v} +
\twonorm{\beta^*} )+ D_0 \rho_{n} M_{\e};
\eens
and on event $\B_4$,
\bens
II & := & \onen \norm{X^T \e}_{\infty}
\le  \onen (\norm{X_0^T \e}_{\infty} + \norm{W^T \e}_{\infty}) \\
& \le &  \rho_{n} M_{\e}(a_{\max}^{1/2} + \sqrt{\tau_B})  = D_0 \rho_{n} M_{\e}.
\eens
Thus on event $\B_0$, we have 
\bens
I + II \le 2 D_2 K \rho_{n} (\inv{\lambda}\onenorm{v} +\twonorm{\beta^*} )+ 2 D_0 \rho_{n} M_{\e} =
\mu(\inv{\lambda}\onenorm{v} +\twonorm{\beta^*})+ 2 \omega.
\eens
Now on event $\B_6$, we have for $2 D_1 \le  D_2$
\bens
IV := \norm{(\hat{D}- D) \hat \beta}_{\infty} 
& \le & 
\abs{\hat{D}- D} \norm{\hat\beta}_{\infty} 
\le  2 D_1  K  \inv{\sqrt{m}}  \rho_{n}  (\norm{\beta^*}_{\infty} + \norm{v}_{\infty} ) \\
& \le &  D_2 K \inv{\sqrt{m}} \rho_{n} (\twonorm{\beta^*} + \norm{v}_1).
\eens
On event $\B_5 \cap \B_{10}$, we have
\bens
III := \onen \norm{(X^T W -D) \hat \beta}_{\infty} 
& \le &
\onen \norm{(X^T W -D) \beta^*}_{\infty} + \onen \norm{(X^T W -D) v}_{\infty} \\
& \le &
 \onen \norm{X_0^T W\beta^*}_{\infty} + \onen \norm{(W^T W -D)
   \beta^*}_{\infty} \\
& + & 
\onen \left( \norm{(Z^T B Z- \tr(B) I_{m})}_{\max} + \norm{X_0^T
    W}_{\max}\right)\onenorm{v} \\
& \le &
\rho_{n} K \left( \frac{\fnorm{B}}{\sqrt{n}} +\sqrt{\tau_B}  a^{1/2}_{\max} \right)(\onenorm{v} + \twonorm{\beta^*}), \\
\text{ and } \; 
V = \onen  \norm{W^T X_0 v}_{\infty} & \le & 
 \onen \norm{W^T X_0 }_{\max} \onenorm{v} \le 
\rho_{n} K \sqrt{\tau_B}  a^{1/2}_{\max} \onenorm{v}.
\eens
Thus we have on $\B_0 \cap \B_{10}$, 
\bens 
III + IV + V
& \le &  
\rho_{n} K \left(\twonorm{B} +\tau_B+  a_{\max} +  \frac{2}{\sqrt{m}} 
 (\twonorm{A} +  \twonorm{B} ) \right)(\onenorm{v} + \twonorm{\beta^*}) \\
& \le &  
\rho_{n} K \left(4 \twonorm{B} + 3 \twonorm{A}\right)(\onenorm{v} +
\twonorm{\beta^*}) \\
& \le & 2 D_2 K  \rho_{n}  (\onenorm{v} +\twonorm{\beta^*}) \\
& \le &  \mu   (\onenorm{v}
+\twonorm{\beta^*}),
\eens
where $D_0 \le D_2$ and $\tau_A = 1$, and
\bens
\norm{\onen X_0^T X_0 v}_{\infty}
& \le &I + II + III + IV + V \\
& \le &  
\mu (\inv{\lambda}\onenorm{v} 
+\twonorm{\beta^*}) + 2 D_0 M_{\e} \rho_{n}
+  \mu  (\onenorm{v} +\twonorm{\beta^*})\\
& \le &  
2 \mu \twonorm{\beta^*} + \mu (\inv{\lambda} + 1)\onenorm{v} + 2 \omega.
\eens 
The lemma thus holds.
\end{proofof2}

\section{Proof for Theorem~\ref{thm::DSoracle}}
\label{sec::DSoracleproof}

We prove Lemmas~\ref{lemma::DSimprov} to~\ref{lemma::grammatrixopt} in
this section.

\subsection{Proof of Lemma~\ref{lemma::DSimprov}}
\begin{proofof2}
Suppose event $\B_0$ holds. 
Then by the proof of Corollary~\ref{coro::D2improv},
we have for $D_0' = \twonorm{B}^{1/2} + a_{\max}^{1/2}$,
\bens
\norm{\hat\gamma - \hat\Gamma \beta^*}_{\infty}
& \le & 
D_0' \tau_B^{+/2}  K \rho_{n} \twonorm{\beta^*} + D_0 M_{\e} \rho_{n},
\eens
where $\tau_B^{+/2} =
\sqrt{\tau_B} + \frac{D_{\ora}}{\sqrt{m}}$ and $D_{\ora} = 2(\twonorm{B}^{1/2} + \twonorm{A}^{1/2})$.
The lemma follows immediately for $\mu, \omega$ as chosen in \eqref{eq::paraDSimprov}.
\end{proofof2}

\subsection{Proof of Lemma~\ref{lemma::tauB}}
\begin{proofof2}
Suppose event $\B_6$ holds. 
We first show \eqref{eq::tildetauB} and~\eqref{eq::tildetauBbound}.
Recall $r_{m,m}:=2 C_0 K^2 \sqrt{\frac{\log m}{m n}} \ge  2C_0
K^2\frac{\log^{1/2} m}{m}$.
By Lemma~\ref{lemma::trBest}, we have on event $\B_6$,  
\bens
\label{eq::tauBoracle}
\abs{\hat\tau_{B} - \tau_B}
&  \le  & D_1 r_{m,m}.
\eens
Moreover, we have under (A1), 
\bens 
1 = \tau_A  \le  D_1 :=
\frac{\fnorm{A}}{m^{1/2}} + \frac{\fnorm{B}}{n^{1/2}}\le \twonorm{A} +
\twonorm{B} \le (\frac{D_{\ora}}{2})^2,
\eens
in view of \eqref{eq::tracefnorm}.
Hence
$$\sqrt{D_1} \le \frac{D_{\ora}}{2} = \twonorm{B}^{1/2} + \twonorm{A}^{1/2}.$$
By definition and construction, we have $\tau_B, \hat\tau_B \ge 0$,
\bens
\abs{\hat\tau_B^{1/2} - \tau^{1/2}_B} & \le &  \hat\tau_B^{1/2} +
\tau^{1/2}_B, \\
\text{ and } \; \;  \abs{\hat\tau_B^{1/2} - \tau^{1/2}_B}^2 & \le &  
\abs{(\hat\tau_B^{1/2} + \tau_B^{1/2})(\hat\tau_B^{1/2} - \tau_B^{1/2})} 
= \abs{\hat\tau_{B} - \tau_B}.
\eens
Thus,
\bens
\abs{\hat\tau_B^{1/2} - \tau^{1/2}_B} & \le &  
 \sqrt{\abs{\hat\tau_{B} - \tau_B}} \le
\sqrt{D_1} r^{1/2}_{m,m} \le \frac{D_{\ora}}{2} r^{1/2}_{m,m}
\eens
and for $C_{6} \ge D_{\ora} \ge 2\sqrt{D_1}$ and $D_{\ora}
=2 (\twonorm{A}^{1/2} +\twonorm{B}^{1/2})$,
\ben
\label{eq::right}
\hat\tau_B ^{1/2}  - \frac{D_{\ora}}{2} r_{m,m}^{1/2} \le
\tau_B ^{1/2}  \le  
\hat\tau_B ^{1/2}  + \frac{ D_{\ora}}{2} r_{m,m}^{1/2}.
\een
Thus we have for $\tau_B^{+/2}$ as defined
in~\eqref{eq::defineDtau}, \eqref{eq::right} and the fact that
\bens
r^{1/2}_{m,m} \ge \sqrt{2 C_0 } K  \frac{(\log m)^{1/4}}{\sqrt{m}} \ge
2/\sqrt{m} \; \text{  for $m \ge 16$ and $C_0 \ge 1$},
\eens
the following inequalities hold: for $K \ge 1$,
\ben 
\label{eq::tildeBbounds}
\tau_B^{+/2} &  := &   \tau_B^{1/2} + D_{\ora} m^{-1/2} \\
\nonumber
& \le &  
\hat\tau_B^{1/2} + \frac{D_{\ora}}{2} r_{m,m}^{1/2} +
\frac{D_{\ora}}{2} r^{1/2}_{m,m} \\
\nonumber
& \le & 
\hat\tau_B^{1/2} + D_{\ora} r_{m,m}^{1/2} \le \tilde\tau_B^{1/2},  
\een
where the last inequality holds by the choice of 
$\tilde\tau_B^{1/2} \ge \hat\tau_B^{1/2} + D_{\ora} r_{m,m}^{1/2}$ as in \eqref{eq::muchoice}.
Moreover, by \eqref{eq::right},
\bens
\tilde\tau_B^{1/2}  & := &  \hat\tau_B^{1/2} +   C_{6} r_{m,m}^{1/2}
\le  \tau_B^{1/2} +   \frac{D_{\ora}}{2} r_{m,m}^{1/2} +   C_{6} r_{m,m}^{1/2} \\
& \le &  \tau_B^{1/2} + \frac{3}{2} C_{6} r_{m,m}^{1/2}, \\
\text{ and } \quad \tilde\tau_B  
& := &  
(\PaulBhalf+ C_{6} r_{m,m}^{1/2})^2 \le 
2 \hat \tau_B +2 C_{6}^2  r_{m,m} \\
& \le &  2 \tau_B + 2 D_1 r_{m,m} + 2 C_{6}^2 r_{m,m} \\
& \le &  2 \tau_B + \frac{D_{\ora}^2}{2} r_{m,m} + 2 C_{6}^2 r_{m,m}
\le  2 \tau_B + 3 C_{6}^2 r_{m,m}.
\eens
Thus \eqref{eq::tildetauB} and~\eqref{eq::tildetauBbound} hold
given that $2 D_1 \le D_{\ora}^2/2 \le C_6^2/2$.

Finally, we have  for $\tau_B^-$ as defined in~\eqref{eq::ora-sparsity},
\bens
\tilde\tau_B^{1/2}  \tau_B^- \le 
(\tau_B^{1/2} + \frac{3}{2} C_{6} r_{m,m}^{1/2} )
\tau_B^- \le \frac{\tau_B^{1/2} + \frac{3}{2} C_{6}  r_{m,m}^{1/2}}
{\tau_B^{1/2} + 2C_{6} r_{m,m}^{1/2}} \le 1.
\eens
\end{proofof2}

\begin{remark}
The set $\U$ in our setting is equivalent to the following:
for $\mu, \omega$ as defined in \eqref{eq::muchoice} and $\beta \in \R^m$,
\ben
\label{eq::defineGamma}
\; \; \; \; 
\U=  \left\{(\beta, t) \; : \; \norm{\onen X^T(y - X \beta) + 
\onen \hat\tr(B) \beta}_{\infty} \le \mu t + \omega, \twonorm{\beta} \le t\right\}.
\een
\end{remark}

\subsection{Proof of Lemma~\ref{lemma::grammatrixopt}}
\begin{proofof2}
For the rest of the proof, we will follow the notation in the proof for
Lemma~\ref{lemma::grammatrix}.
Notice that the bounds as stated in Lemma~\ref{lemma::DS-cone} remain
true with $\omega, \mu$ chosen as in~\eqref{eq::paraDSimprov}, so long
as  $(\beta^*, \twonorm{\beta^*}) \in \U$. This indeed holds by Lemma~\ref{lemma::DSimprov}:
for $\omega$ and $\mu$  \eqref{eq::muchoice} as chosen  in Theorem~\ref{thm::DSoracle}, we have by
\eqref{eq::tildeBbounds}, 
\bens
 \mu \asymp D_0' \tilde\tau_B^{1/2}  K \rho_{n}   \ge D_0' K \rho_{n}
 \tau_B^{+/2}, \; \; 
\text{ where } \; \; \tau_B^{+/2} =  ( \sqrt{\tau_B} + \frac{D_{\ora}}{\sqrt{m}}),
\eens
which ensures that  $(\beta^*, \twonorm{\beta^*}) \in \U$ by
Lemma~\ref{lemma::DSimprov}.

On event $\B_0$, we have by Lemma~\ref{lemma::DS-cone} and 
the fact that $\hat\beta \in \U$ as in \eqref{eq::defineGamma}
\bens
I + II & := & 
\norm{\hat\gamma - \hat\Gamma  \hat\beta}_{\infty} + \onen
\norm{X^T \e}_{\infty} \\
& \le & 
\norm{\onen X^T(y - X \hat\beta) + \onen \hat{D}  \hat\beta}_{\infty}
+ \omega \le \mu \hat{t}  + 2 \omega \\
& \le & 
\mu (\inv{\lambda}\onenorm{v} + \twonorm{\beta^*} )+2 \omega,
\eens
for  $\omega, \mu$ as chosen in \eqref{eq::muchoice}. 
Now on event $\B_6$, we have under (A1),
\bens
IV := \norm{(\hat{D}- D) \hat \beta}_{\infty} 
& \le & 
\abs{\hat{D}- D} \norm{\hat\beta}_{\infty} 
\le  2 D_1  K  \inv{\sqrt{m}}  \rho_{n}  (\norm{\beta^*}_{\infty} + \norm{v}_{\infty} ) \\
& \le &  D_0' \frac{D_{\ora}}{\sqrt{m}} K  \rho_{n} (\twonorm{\beta^*} + \norm{v}_1),
\eens
where $2D_1 \le D_{\ora} D_0'$ for $1 \le D_0' :=  \twonorm{B}^{1/2} +
a_{\max}^{1/2}$, for $a_{\max} \ge \tau_A = 1$ and $D_{\ora} =   2 \left(\twonorm{B}^{1/2} +
  \twonorm{A}^{1/2} \right)$.
Hence
\bens
\lefteqn{III + IV + V 
\le \rho_{n} K \sqrt{\tau_B} \left(\twonorm{B}^{1/2} + a^{1/2}_{\max} \right)(\onenorm{v} + \twonorm{\beta^*})} \\
&&
+ 2 D_1 K \inv{\sqrt{m}} \rho_{n} (\twonorm{\beta^*} + \norm{v}_1)
+ \rho_{n} K \sqrt{\tau_B}  a^{1/2}_{\max} \onenorm{v} \\
& \le &  
D_0' K \rho_{n}  (\onenorm{v} + \twonorm{\beta^*})  
( \sqrt{\tau_B} +  \frac{D_{\ora}}{\sqrt{m}}) + \rho_{n} K \sqrt{\tau_B}  a^{1/2}_{\max} \onenorm{v} \\ 
&\le &
D_0' K \rho_{n} \tau_B^{+/2} (\onenorm{v} + \twonorm{\beta^*} ) + D_0'  K \rho_{n} \sqrt{\tau_B} 
\onenorm{v} \\ 
&\le &
C_0 D_0' K^2 \sqrt{\frac{\log m}{n}}
(\tau_B^{1/2} + \frac{D_{\ora}}{\sqrt{m}}) (2\onenorm{v} + \twonorm{\beta^*} ) \\
& \le & \mu (2\onenorm{v} + \twonorm{\beta^*}),
\eens
for $\mu$ as defined in   \eqref{eq::muchoice} in view of
\eqref{eq::tildeBbounds}.

Thus we have 
\bens
I + II + III + IV + V 
& \le &  \mu  (\inv{\lambda}\onenorm{v} + \twonorm{\beta^*} )+  2 \omega
+  \mu (2\onenorm{v} + \twonorm{\beta^*})  \\
& = & 
2\mu ((1 + \inv{2\lambda})\onenorm{v} + \twonorm{\beta^*}) + 2 \omega,
\eens
and the improved bound as stated in the Lemma thus holds.
\end{proofof2}

\section{Some geometric analysis results}
\label{sec::geometry}
Let us define the following set of vectors in $\R^m$:
\bens
\Cone(s) := \{\up: \onenorm{\up} \le \sqrt{s} \twonorm{\up}\}
\eens
For each vector $x \in \R^m$, let ${T_0}$ denote the locations of the $s$
largest coefficients of $x$ in absolute values.
Any vector $x \in S^{m-1}$ satisfies:
\ben
\label{eq::init-norm-inf}
\norm{x_{T_0^c}}_{\infty}  \leq \norm{x_{T_0}}_{1}/s
& \leq & \frac{ \twonorm{x_{T_0}}}{\sqrt{s}}.
\een
We need to state the following result from~\cite{MPT08}.
Let $S^{m-1}$ be the unit sphere in $\R^m$, for $1 \leq s \leq m$, 
\beq
U_s \; := \; \{ x \in \R^{m}: |\supp(x)| \leq s \}.
\eeq
The sets $U_s$ is an union of the $s$-sparse vectors.
The following three lemmas are well-known and mostly standard; 
See~\cite{MPT08} and~\cite{LW12}.
\begin{lemma}
\label{eq::embedding}
For every $1\le s \le m$ and every $I \subset \{1, \ldots, m\}$ with
$\abs{I} \le s$,
\bens
\sqrt{\abs{I}} B_1^m \cap S^{m-1} \subset 2 \conv( U_{s} \cap S^{m-1}) 
=: 2 
\conv\left(\bigcup_{\size{J} \leq s} E_J \cap S^{m-1}\right) 
\eens
and moreover, for $\rho \in (0, 1]$,
\bens
\sqrt{\abs{I}} B_1^m \cap \rho B_2^{m} \subset (1+\rho) \conv( U_{s} \cap B_2^{m}) 
=: (1+\rho) \conv\left(\bigcup_{\size{J} \leq {s}} E_J \cap S^{m-1}\right).
\eens
\end{lemma}

\begin{proof}
Fix  $x \in \R^m$.
Let $x_{T_0}$ denote the subvector of $x$
confined to the locations of its $s$ largest coefficients in absolute values;
moreover, we use it to represent its $0$-extended
version $x' \in \R^m$ such that $x'_{T^c} =0$ and
 $x'_{T_0} =x_{T_0}$.
Throughout this proof, $T_0$ is understood to be the locations of the $s$ largest 
coefficients in absolute values in $x$.

Moreover, let $(x_i^*)_{i=1}^m$ be non-increasing rearrangement of 
$(\abs{x_i})_{i=1}^m$.
Denote by 
\bens
L & = & \sqrt{s} B_1^m \cap \rho B_2^m \quad \text{and} \\ 
R & = & 2 \conv\left(\bigcup_{\size{J} \leq s} E_J \cap B_2^{m}\right) =
 2 \conv\big( E \cap B_2^{m}\big).
\eens
Any vector $x \in \R^{m}$ satisfies:
\ben
\label{eq::init-norm-inf}
\norm{x_{T_0^c}}_{\infty}  
\leq \norm{x_{T_0}}_{1}/s
& \leq & \frac{ \twonorm{x_{T_0}}}{\sqrt{s}}.
\een
It follows that for any $\rho >0$, $s \ge 1$ and 
for all $z \in L$, we have the $i^{th}$ largest coordinate in absolute value in $z$ 
is at most $\sqrt{s}/i$, and
\bens
\sup_{z \in L} \ip{x, z} & \le &
\max_{\twonorm{z} \le \rho} \ip{x_{T_0}, z}
+ \max_{\onenorm{z} \le \sqrt{s}} \ip{x_{T_0^c}, z}   \\
 & \le &  \rho \twonorm{x_{T_0}} + 
\norm{x_{T_0^c}}_{\infty} \sqrt{s} \\
 & \le & 
\twonorm{x_{T_0}} \left( \rho  + 1\right),
\eens
where clearly
$\max_{\twonorm{z} \le \rho} \ip{x_{T_0}, z} = \rho  \sum_{i=1}^{s} (x_i^{*2})^{1/2}$.
And denote by $S^J := S^{m-1} \cap E_J$,
\bens
\sup_{z \in R} \ip{x, z} & = & (1+\rho)
\max_{J: \size{J} \le s} \max_{z \in S^J} \ip{x, z} \\
 & = & (1+ \rho) \twonorm{x_{T_0}},
\eens
given that for a convex function $\ip{x, z}$, the maximum happens at
an extreme point; and in this case, it happens for $z$ such that $z$
is supported on $T_0$,  such that $z_{T_0} =
\frac{x_{T_0}}{\twonorm{x_{T_0}}}$ and $z_{T_0^c} =0$.
\end{proof}

\begin{lemma}
\label{lemma::bigcone}
Let $1/5 > \delta > 0$.
Let $E=\cup_{|J| \leq s} E_J$ for $0 < s < m/2$ and $k_0>0$.
Let $\Delta$ be a $m \times m$ matrix such that
\ben
\label{eq::conecond}
\abs{u^T \Delta v} \le \delta,  \;\; \; \forall u, v \in E \cap S^{m-1}
\een
Then for all $v \in \big(\sqrt{s} B_1^m \cap B_2^m\big)$, 
\ben
\label{eq::origset}
\abs{\up^T \Delta \up} & \le & 4\delta.
\een
\end{lemma}

\begin{proof}
First notice that
\ben
\label{eq::decoupledbig}
\max_{\up \in \big(\sqrt{s} B_1^m \cap B_2^m\big)} \abs{\up^T \Delta \up}
& \le & 
\max_{w, u \in  \big(\sqrt{s} B_1^m \cap B_2^m\big)}  \abs{w^T \Delta u}.
\een
Now that we have decoupled $u$ and $w$ on the RHS of
\eqref{eq::decoupledbig},
 we first fix $u$.

Then for any fixed $u \in S^{m-1}$ and matrix $\Delta \in \R^{m \times m}$,
$f(w) = \abs{w^T \Delta u}$ is a convex function of $w$, and hence 
for $w \in \big(\sqrt{s} B_1^m \cap  B_2^{m}\big) \subset 
2 \conv\left(\bigcup_{\size{J} \leq s} E_J \cap S^{m-1}\right)$,
\bens
\max_{w \in \big(\sqrt{s} B_1^m \cap B_2^m\big)}  \abs{w^T \Delta u} 
& \le & 
2 \max_{w \in \conv( E \cap  S^{m-1})} \abs{w^T \Delta  u} \\
& =& 
2 \max_{w \in E \cap  S^{m-1}} \abs{w^T \Delta u},
\eens
where the maximum occurs at an extreme point of the set 
$ \conv( E \cap  S^{m-1})$ because of the convexity of the function
$f(w)$.

Clearly the RHS of \eqref{eq::decoupledbig} is bounded by
\bens
\max_{u, w \in \big(\sqrt{s} B_1^m \cap B_2^m\big)}  \abs{w^T
  \Delta u}
& = &
\max_{u \in \big(\sqrt{s} B_1^m \cap B_2^m\big)} 
\max_{w \in \big(\sqrt{s} B_1^m \cap B_2^m\big)}  \abs{w^T \Delta u}
\\
& \le &
2 \max_{u \in \big(\sqrt{s} B_1^m \cap B_2^m\big)} 
\max_{w \in \big(E\cap S^{m-1}\big)}  \abs{w^T \Delta u}\\
& = &
2 \max_{u \in \big(\sqrt{s} B_1^m \cap B_2^m\big)} g(u),
\eens
where the function $g$ of $u \in \big(\sqrt{s} B_1^m \cap
B_2^m\big)$ is defined as
\bens
g(u) =\max_{w \in \big(E \cap S^{m-1}\big)}  \abs{w^T \Delta u};
\eens
$g(u)$ is convex since it is the maximum of a function 
$f_w(u) := \abs{w^T \Delta u}$ 
which is convex in $u$ for each $w \in (E \cap S^{m-1})$.

Thus we have 
for $u \in (\sqrt{s} B_1^m \cap  B_2^{m}) \subset 2 \conv\left(\bigcup_{\size{J} \leq s} E_J \cap  S^{m-1}\right) =: 2\conv\left(E \cap S^{m-1}\right)$,
\ben
\nonumber
\max_{u \in  \big(\sqrt{s} B_1^m \cap B_2^m\big)} g(u)
& \le & 
2 \max_{u\in \conv(E \cap  S^{m-1})} g(u) \\
& = & 
\label{eq::extremeg}
2 \max_{u\in E \cap  S^{m-1}} g(u) \\
& = & 
\label{eq::last}
2 \max_{u \in  E \cap  S^{m-1}} \max_{w \in  E \cap  S^{m-1}} \abs{w^T  \Delta u}
\le 4 \delta,
\een
where~\eqref{eq::extremeg} holds given that the maximum occurs at an extreme point of the set 
$ \conv( E \cap  B_2^{m})$, because of the convexity of the function $g(u)$.
\end{proof}

\begin{corollary}
\label{coro::bigcone} 
Suppose all conditions in Lemma~\ref{lemma::bigcone} hold.
Then $\forall \up \in \Cone(s)$,
\ben
\label{eq::cone}
\abs{\up^T \Delta \up} & \le & 4\delta  \twonorm{\up}^2.
\een
\end{corollary}

\begin{proof}
It is sufficient to show that $ \forall \up \in \Cone(s) \cap S^{m-1}$,
\bens
\abs{\up^T \Delta \up} & \le & 4\delta.
\eens
Denote by $\Cone := \cone(s)$.
Clearly this set of vectors satisfy:
\bens
\cone \cap S^{m-1}\subset \big(\sqrt{s} B_1^m \cap B_2^m\big).
\eens
Thus~\eqref{eq::cone} follows from~\eqref{eq::origset}.
\end{proof}

\begin{remark}
Suppose we relax the definition of $\Cone(s)$ to be:
\bens
\Cone(s) := \{\up: \onenorm{\up} \le 2 \sqrt{s} \twonorm{\up}\}.
\eens
Clearly, $\Cone(s, 1) \subset \Cone(s)$.
given that $\forall u \in  \Cone(s, 1)$, we have 
\bens
\onenorm{u} \le 2 \onenorm{u_{T_0}} \le 2 \sqrt{s} \twonorm{u_{T_0}} \le 
2 \sqrt{s} \twonorm{u}.
\eens
\end{remark}

\begin{lemma}
\label{lemma::bigconeII}
Suppose all conditions in Lemma~\ref{lemma::bigcone} hold.
Then for all $\up \in \R^m$,
\ben
\label{eq::bigcone}
\abs{\up^T \Delta \up} \le 4 \delta (\twonorm{\up}^2 + \inv{s} \onenorm{\up}^2).
\een
\end{lemma}

\begin{proof}
The lemma follows given that $\forall \up \in \R^m$, one of the following must hold:
\ben
\label{eq::cone-II}
\text{ if }  \up \in \Cone(s) \;\; \;
\abs{\up^T \Delta \up} & \le & 4 \delta \twonorm{\up}^2; \\
\label{eq::anticone}
\text{ otherwise }  \;\;\;
\abs{\up^T \Delta \up} & \le &  \frac{4\delta} 
{s}\onenorm{\up}^2,
\een
leading to the same conclusion in \eqref{eq::bigcone}.

We have shown~\eqref{eq::cone-II} in Lemma~\ref{lemma::bigcone}.
Let $\Cone(s)^c$ be the  complement set of $\cone(s)$ in $\R^{m}$. 
That is, we focus now on the set of vectors such that 
\bens
\Cone(s)^c := \{\up: \onenorm{\up} \ge \sqrt{s} \twonorm{\up}\}
\eens
and show that for $u = \sqrt{s} \frac{v}{\onenorm{v}}$,
\bens
\frac{\abs{v^T \Delta v} }{\onenorm{v}^2} 
& := & \inv{s} \abs{u^T \Delta u}  \le \inv{s} \delta.
\eens
Now, the last inequality holds by Lemma~\ref{lemma::bigcone} given that 
\bens
u \in (\sqrt{s} B_1^m \cap B_2^m) \subset 
2 \conv\left(\bigcup_{\size{J} \leq s} E_J \cap B_2^{m}\right) 
\eens 
and thus 
\bens
\frac{\abs{v^T \Delta v} }{\onenorm{v}^2} 
& \le & \inv{s} \sup_{u \in \sqrt{s} B_1^m \cap B_2^m}
 \abs{u^T \Delta u}  \le  \inv{s}4 \delta.
\eens
\end{proof}

\section{Proof of Corollary~\ref{coro::BC}}
\label{sec::appendLURE}
\begin{proofof2}
First we show that for all $\up \in \R^m$, \eqref{eq::conebound} holds.
It is sufficient to check that the condition~\eqref{eq::conecond} in
Lemma~\ref{lemma::bigcone} holds. Then, \eqref{eq::conebound} follows
from Lemma~\ref{lemma::bigconeII}: for $\up \in \R^m$, 
\ben
\label{eq::conebound}
\abs{\up^T \Delta \up} \le 4 \delta (\twonorm{\up}^2 + \inv{k} \onenorm{\up}^2) 
\le \frac{3}{8}\lambda_{\min}(A)  (\twonorm{\up}^2 + \inv{k} \onenorm{\up}^2).
\een
The Lower and Upper $\RE$ conditions thus immediately follow.
The Corollary is thus proved.
\end{proofof2}

\section{Proof of Theorem~\ref{thm::AD}}
\label{sec::proofofthmAD}
We first state the following preliminary results in 
Lemmas~\ref{lemma::normA}  and~\ref{lemma::orthogSp}; their proofs 
appear in Section~\ref{sec::proofofnormA}.
Throughout this section, the choice of $C = C_0/\sqrt{c'}$ satisfies the conditions on $C$ in Lemmas~\ref{lemma::normA}
and~\ref{lemma::orthogSp}, where recall $\min\{C_0, C_0^2\} \ge 4/c$ for $c$ as
defined in Theorem~\ref{thm::HW}.
For a set $J \subset \{1, \ldots, m\}$, denote $F_J=A^{1/2} E_J$, where recall
$E_J = \spin\{e_j: j \in J\}$.
Let $Z$ be an $n \times m$ random matrix with independent entries $Z_{ij}$ satisfying
$\E Z_{ij} = 0$, $1 = \E Z_{ij}^2 \le \norm{Z_{ij}}_{\psi_2} \leq K$. 
Let $Z_1, Z_2$ be independent copies of $Z$.
\begin{lemma}
\label{lemma::normA}
Suppose all conditions in Theorem~\ref{thm::AD} hold.
Let $$E  = \bigcup_{\abs{J}=k} E_J \cap S^{m-1}.$$
Suppose that for some $c' > 0$ and $\ve \le \inv{C}$, where $C = C_0/\sqrt{c'}$,
\ben
 \label{eq::ALocalkronsum}
r(B) := \frac{\tr(B)}{\twonorm{B}} & \ge & c' k K^4 \frac{\log(3e m/k \ve)}{\ve^2}.
\een
Then for all vectors $u, v \in E \cap S^{m-1}$, on event $\B_1$, 
where $\prob{\B_1} \ge 1- 2 \exp\left(-c_2\ve^2 \frac{\tr(B)}{K^4\twonorm{B}}\right)$ for $c_2 \ge 2$,
\bens
\abs{u^T Z^T B Z v - \E u^T Z^T B Z v } 
& \le & 4 C \ve \tr(B).
\eens
\end{lemma}

\begin{lemma}
\label{lemma::orthogSp}
Suppose that $\ve \le 1/C$, where $C$ is as defined in Lemma~\ref{lemma::normA}.
Suppose that \eqref{eq::ALocalkronsum} holds.
Let
\ben
\label{eq::spsetEF}
E = \bigcup_{\abs{J} =k} E_J \; \; \text{ and } \; \; F =
\bigcup_{\abs{J} =k} F_J.
\een
Then on event $\B_2$, 
where $\prob{\B_2} \ge 1- 2 \exp\left(-c_2\ve^2
  \frac{\tr(B)}{K^4\twonorm{B}}\right)$ for $c_2 \ge 2$,
we have for all vectors $u \in E \cap  S^{m-1}$ and $w \in F \cap  S^{m-1}$,
\bens
\abs{w^T Z_1^T B^{1/2} Z_2 u} 
& \le & 
\frac{C \ve \tr(B)}{(1-\ve)^2\twonorm{B}^{1/2}} \le
{4C \ve \tr(B)}/{\twonorm{B}^{1/2}}.
\eens
\end{lemma}
In fact, the same conclusion holds for all $y, w \in F\cap S^{m-1}$; 
and in particular, for $B = I$, we have the following.
\begin{corollary}
\label{coro::tartan}
Suppose all conditions in 
Lemma~\ref{lemma::normA} hold.
Suppose that $F = A^{1/2} E$ for $E$ as defined in Lemma~\ref{lemma::normA}.
Let
\ben
\label{eq::BI} 
n & \ge &  c'k K^4 \frac{\log(3e m/k \ve)}{\ve^2}. 
\een
Then on event $\B_3$, 
where $\prob{\B_3} \ge 1- 2 \exp\left(-c_2\ve^2 n \inv{K^4}\right)$,
we have for all vectors $w,y \in F \cap S^{m-1}$ and  $\ve \le 1/C$ for $C$ is as defined in Lemma~\ref{lemma::normA},
\ben
\label{eq::wyFnorm}
\abs{y^T (\onen  Z^T Z - I) w } & \le & 4 C\ve.
\een
\end{corollary}
We prove Lemmas~\ref{lemma::normA} and~\ref{lemma::orthogSp}
  and Corollary~\ref{coro::tartan} in Section~\ref{sec::proofofnormA}.
We are now ready to prove Theorem~\ref{thm::AD}.

\begin{proofof}{Theorem~\ref{thm::AD}}
Let
\bens
\lefteqn{
\Delta := \hat\Gamma_{A} -A := \onen X^TX - \onen  \hat\tr(B) I_{m} -A} \\
& = & (\onen  X_0^T X_0 -A)+  
 \onen \big(W^T X_0 + X_0^T W\big) + \onen \big(W^T W  - \hat\tr(B) I_{m}\big),
\eens
where recall $X_0 = Z_1 A^{1/2}$.
Notice that 
\bens
\lefteqn{
\abs{u^T(\hat\Gamma_A -A) \up} =  
\abs{u^T(X^T X - \hat\tr(B) I_{m} -A) \up} }\\
& \le &  
\abs{u^T  (\onen X_0^T X_0 - A) \up} + 
\abs{u^T \onen (W^T X_0 + X_0^T W) \up} + 
\abs{u^T (\onen W^T W - \frac{\hat\tr(B)}{f} I_{m})\up}\\
& \le &  
\abs{u^T A^{1/2}\onen Z_1^T Z_1 A^{1/2} \up - u^T A \up} +
\abs{u^T \onen (W^T X_0 + X_0^T W) \up} \\
&& +
\abs{u^T (\onen Z_2^T B Z_2 - \tau_B I_{m})\up} + 
\onen \abs{\hat\tr(B) - \tr(B)} \abs{u^T \up} =: I + II + III + IV.
\eens
For $u  \in E \cap S^{m-1}$, define $h(u) := \frac{A^{1/2}
  u}{\twonorm{A^{1/2} u}}$.
The conditions in~\eqref{eq::ALocalkronsum} and~\eqref{eq::BI} hold
for $k$.

We first bound the middle term as follows.
Fix $u, \up \in E \cap S^{m-1}$.
Then on event $\B_2$, for $\Upsilon = Z_1^T B^{1/2} Z_2$,
\bens
\abs{u^T (W^T X_0 + X_0^T W) \up} & = & 
 \abs{u^T Z_2^T B^{1/2} Z_1 A^{1/2} \up + u^T A^{1/2} Z_1^T B^{1/2} Z_2 \up} \\ 
& \le & 
\abs{u^T \Upsilon^T h(v)}\twonorm{A^{1/2} v}
 + \abs{h(u)^T \Upsilon \up} \twonorm{A^{1/2} u}\\
& \le &  2\max_{w \in F \cap S^{m-1}, \up \in E \cap S^{m-1}}
\abs{w^T \Upsilon \up} \rho_{\max}^{1/2}(k, A) \\
& \le & 8 C \ve \tr(B)\left(\frac{\rho_{\max}(k, A)}{\twonorm{B}}\right)^{1/2}.
 \eens
We now use Lemma~\ref{lemma::normA} to bound both $I$ and $III$.
We have for $C$ as defined in Lemma~\ref{lemma::normA}, on event $\B_1 \cap \B_3$,
\bens
 \abs{u^T (Z_2^T B Z_2 - \tr(B)I_{m})\up} 
& \le  4 C \ve \tr(B).
\eens 
Moreover, by Corollary~\ref{coro::tartan}, we have on event $\B_3$, for all 
 $u, v \in E \cap S^{m-1}$,
\bens
\abs{u^T  (\onen X_0^T X_0 - A) \up} 
& = & 
\abs{u^T A^{1/2} Z^T Z A^{1/2} \up - u^T A \up} \\
& = & 
\abs{h(u)^T (\onen Z^T Z -I) h(\up)} \twonorm{A^{1/2} u} \twonorm{A^{1/2} \up} \\
& \le &
\onen \max_{w, y \in F \cap S^{m-1}} \abs{w^T (Z^T Z -I) y} \rho_{\max}(k, A) \\
& \le & 4 C \ve  \rho_{\max}(k, A).
\eens
Thus we have on event $\B_1 \cap \B_2 \cap \B_3$ and for $\tau_B := \tr(B)/n$,
\bens
I + II + III 
& \le & 
 4 C \ve \left(\rho_{\max}(k, A)  +  
2 \tau_B \left(\frac{\rho_{\max}(k, A)}{\twonorm{B}}\right)^{1/2} + \tau_B \right) \\
& \le &  8 C \ve\left(\tau_B +  \rho_{\max}(k, A) \right).
\eens
On event $\B_6$, we have for $D_1$ as defined in Lemma~\ref{lemma::trBest},
\bens
IV \le \abs{\hat\tau_B - \tau_B} \le  2 C_0 D_1 K^2 \sqrt{\frac{\log
    m}{m n}}.
\eens
The theorem thus holds by the union bound.
\end{proofof}

\section{Proof of Lemma~\ref{lemma::oneeventA}} 
\label{sec::proofoftensorA}
Lemma~\ref{lemma::Auv} is a well-known fact.
\begin{lemma}
\label{lemma::Auv}
Let $A_{u w} := (u \otimes w) \otimes A, \; \text{where } \; u, w \in
\Sp^{m-1}$ for $m \ge 2$. Then $\twonorm{A_{uw}}  \le \twonorm{A} \; \text{ and } \;
\fnorm{A_{uw}}  \le \fnorm{A}.$
\end{lemma}

\begin{proofof}{Lemma~\ref{lemma::oneeventA}}
\silent{
Let $x_1, \ldots, x_m, x'_1, \ldots, x'_m \in \R^n$ be the column vectors
$Z_1, Z_2$ respectively.
We can re-write the quadratic forms as follows:
\bens
u^T Z_1^T B^{1/2} Z_2 w
& = & \sum_{i,j=1, m} u_i w_j x_i B^{1/2} x'_j  \\
& = & \mvec{Z_1}^T \big((u \otimes w) \otimes B^{1/2} \big)\mvec{Z_2} \\
& =: & \mvec{Z_1}^T B_{uw}^{1/2} \mvec{Z_2} \\
\text{and } \; \; u^T Z^T B Z w
& = & \mvec{Z}^T \big((u \otimes w) \otimes B \big)\mvec{Z}
=:\mvec{Z}^T B_{uw} \mvec{Z},
\eens
where clearly by independence of $Z_1, Z_2$,
\bens
\E \mvec{Z_1}^T \big((u \otimes w) \otimes B^{1/2}\big) \mvec{Z_2} & = & 0 \\
\E \mvec{Z}^T \big((u \otimes u) \otimes B\big) \mvec{Z} & = & 
\tr\big((u \otimes u) \otimes B\big) = \tr(B).
\eens
Thus we invoke Theorem~\ref{thm::HW},~\eqref{eq::HWdecoupled} and Lemma~\ref{lemma::Auv} to  conclude
\bens
\prob{\abs{u^T Z_1^T B^{1/2} Z_2 w} >  t}
& \le & 
2 \exp\left(-c\min\left(\frac{t^2}{K^4\fnorm{B^{1/2}_{uw}}^2},
 \frac{t}{K^2 \twonorm{B_{uw}^{1/2}}}\right)\right)  \\
& \le & 
2 \exp\left(-c\min\left(\frac{t^2}{K^4\tr(B)},
 \frac{t}{K^2 \twonorm{B^{1/2}}}\right)\right)   \\
\prob{\abs{u^T Z^T B Z w - \E u^T Z^T B Z w } >  t}
& \le & 
2 \exp\left(-c\min\left(\frac{t^2}{K^4\fnorm{B_{uw}}^2},
 \frac{t}{K^2 \twonorm{B_{uw}}}\right)\right)  \\
& \le & 
2 \exp\left(-c\min\left(\frac{t^2}{K^4\fnorm{B}^2},
 \frac{t}{K^2 \twonorm{B}}\right)\right).
\eens}
Let $z_1, \ldots, z_n, z'_1, \ldots, z'_n \in \R^m$ be the row vectors
$Z_1, Z_2$ respectively.
Notice that we can write the quadratic form as follows:
\bens
u^T Z_1 A^{1/2} Z_2^T w
& = &
\sum_{i,j=1, m} u_i w_j z_i A^{1/2} z'_j  \\
& = & \mvec{Z_1^T}^T \big((u \otimes w) \otimes A^{1/2}
\big)\mvec{Z_2^T} \\
& =: & \mvec{Z_1^T}^T A_{uw}^{1/2} \mvec{Z_2^T}, \\
u^T Z A Z^T w
& = & \mvec{Z^T}^T \big((u \otimes w) \otimes A \big)\mvec{Z^T} \\
& =: & \mvec{Z^T}^T A_{uw} \mvec{Z^T}
\eens
where clearly by independence of $Z_1, Z_2$,
\bens
\E \mvec{Z_1^T}^T \big((u \otimes w) \otimes A^{1/2}\big) \mvec{Z_2^T}
& = & 0, \; \; \text{ and }  \\
\E \mvec{Z}^T \big((u \otimes u) \otimes A\big) \mvec{Z} & = & 
\tr\big((u \otimes u) \otimes A\big) = \tr(A).
\eens
Thus we invoke~\eqref{eq::HWdecoupled} and
Lemma~\ref{lemma::Auv} to show the concentration bounds on event
$\{\abs{u^T Z_1 A^{1/2} Z_2^T w} >  t\}$:
\bens
\prob{
 \abs{u^T Z_1 A^{1/2} Z_2^T w} >  t}
& \le & 
2 \exp\left(-\min\left(\frac{t^2}{K^4\fnorm{A^{1/2}_{uw}}^2},
 \frac{t}{K^2 \twonorm{A_{uw}^{1/2}}}\right)\right)  \\
& \le &
2 \exp\left(-\min\left(\frac{t^2}{K^4\tr(A)},
 \frac{t}{K^2 \twonorm{A^{1/2}}}\right)\right).
\eens
Similarly, we have  by Theorem~\ref{thm::HW} and Lemma~\ref{lemma::Auv},
\bens
\lefteqn{
\prob{\abs{u^T Z A Z^T w - \E u^T Z A Z^T w } >  t}}\\
& \le & 
2 \exp\left(-c\min\left(\frac{t^2}{K^4\fnorm{A_{uw}}^2},
 \frac{t}{K^2 \twonorm{A_{uw}}}\right)\right)  \\
& \le & 
2 \exp\left(-c\min\left(\frac{t^2}{K^4\fnorm{A}^2},
 \frac{t}{K^2 \twonorm{A}}\right)\right).
\eens
The Lemma thus holds.
\end{proofof}

\section{Proof of  Lemmas~\ref{lemma::normA} and~\ref{lemma::orthogSp}
  and Corollary~\ref{coro::tartan}}
\label{sec::proofofnormA}
Throughout the following proof, we denote by $r(B) =
\frac{\tr(B)}{\twonorm{B}}$. Let $\ve \le \inv{C}$ where $C$ is large enough so that 
$ c c' C^2 \ge 4$, and hence the choice of $C = C_0/\sqrt{c'}$
satisfies our need.

\begin{proofof}{Lemma~\ref{lemma::normA}}
First we prove concentration bounds for all pairs of $u, v \in \Pi'$, where 
$\Pi' \subset \Sp^{m-1}$ is an $\ve$-net of $E$. 
Let $t  = C K^2 \ve \tr(B)$.
We have by Lemma~\ref{lemma::oneeventA}, and the union bound,
\bens
\lefteqn{
\prob{\exists u, v \in \Pi', \; 
\abs{u^T Z^T B Z v - \E u^T Z^T B Z v} >  t}} \\
& \leq &
2 \abs{\Pi'}^2 \exp \left[- c\min\left(\frac{t^2}{K^4  \fnorm{B}^2}, 
\frac{t}{K^2 \twonorm{B}} \right)\right] \\
& \leq &
2 \abs{\Pi'}^2
\exp \left[- c\min\left(C^2, \frac{C K^2}{\ve} \right)\frac{\ve^2
    r(B)}{K^4}\right] \\
& \le &  2 \exp\left(-c_2\ve^2 r(B) /K^4\right),
\eens
where we use the fact that $\fnorm{B}^2 \le  \twonorm{B}\tr(B)$ and 
\bens
\abs{\Pi'} \le {m \choose k}(3/\ve)^k \le \exp(k\log(3 e m/k\ve)),
\eens
while 
\bens
c \min\left(C^2, \frac{C K^2}{\ve} \right) \ve^2
\frac{r(B)}{K^4} & = & c C^2 \ve^2 \frac{\tr(B)}{\twonorm{B} K^4} \\
&\ge & c C_0^2 k \log\big(\frac{3e m}{k \ve}\big)
\ge 4k\log\big(\frac{3e m}{k \ve}\big).
\eens
Denote by $\B_2$ the event such that for $\Lambda := \inv{\tr(B)} (Z^T B Z - I)$, 
\bens
\sup_{u, v \in \Pi'} \abs{v^T \Lambda u} 
& \le & C  \ve =:r'_{k,n}
\eens
holds. 
A standard approximation argument shows that under $\B_2$ and
for $\ve \le 1/2$, 
\ben
\label{eq::sphereL}
\sup_{x, y\in \Sp^{m-1} \cap E} \abs{y^T \Lambda  x} 
\le \frac{r'_{k,n}}{(1-\ve)^2} \le 4 C  \ve.
\een
The lemma is thus proved.
\end{proofof}

\begin{proofof}{Lemma~\ref{lemma::orthogSp}}
By Lemma~\ref{lemma::oneeventA}, we have 
for $t = C \ve \tr(B)/\twonorm{B}^{1/2}$ for $C = C_0/\sqrt{c'}$,
\bens
\prob{\abs{w^T Z_1^T B^{1/2} Z_2 u} >  t}
& \le &
\exp\left(-c\min\left(\frac{C^2 \frac{\tr(B)^2}{\twonorm{B}} \ve^2}
{K^4\tr(B)}, \frac{C\ve \tr(B)}{K^2\twonorm{B}}
\right)\right) \\
& \le & 
2 \exp\left(-c\min\left(\frac{C^2\ve^2 r_B}{K^4}, \frac{C\ve r_B}{K^2}\right)\right) \\
& \le & 
2 \exp\left(-c \min\left(C^2, \frac{CK^2}{\ve} \right) \ve^2 r_B/K^4\right).
\eens
Choose an $\ve$-net $\Pi' \subset S^{m-1}$ such that
\ben
\label{eq::Enet}
\Pi' = \bigcup_{\abs{J} = k} \Pi'_{J} \; \; \text{ where } \;\; 
\Pi'_{J} \subset E_J \cap S^{m-1} 
\een
is an $\ve$-net for $E_J \cap S^{m-1}$ and
\bens
\abs{\Pi'} \le {m \choose k}(3/\ve)^k \le \exp(k\log(3 e m/k\ve)).
\eens
Similarly, choose $\ve$-net $\Pi$ of $F \cap S^{m-1}$ of size at most 
$\exp(k\log(3 e m/k\ve))$.
By the union bound and Lemma~\ref{lemma::oneeventA}, and for $K^2 \ge 1$,
\bens
\lefteqn{
\prob{\exists w \in \Pi, u \in \Pi'\; s.t. \; 
\abs{w^T Z_1^T B^{1/2} Z_2 u} 
\ge C \ve{\tr(B)}/{\twonorm{B}^{1/2}}}}\\
& \le &
\abs{\Pi'} \abs{\Pi}  2 \exp\left(-c \min\left(C K^2/\ve, C^2\right) \ve^2 r_B/K^4\right)\\
& \le & 
\exp\left(2 k \log(3em/k\ve)\right) 2 \exp\left(-cC^2 \ve^2 r_B/K^4 \right) \\
& \le & 
2 \exp\left(-c_2\ve^2 r_B/K^4\right),
\eens
where $C$ is large enough such that 
$c c' C^2 := C' > 4$ and  for $\ve \le \inv{C}$,
$$c \min\left(C K^2/\ve, C^2\right) \ve^2\frac{\tr(B)}{\twonorm{B}K^4} 
\ge C' k \log(3 e  m/k\ve)
\ge 4 k \log(3 e  m/k\ve).$$ 
Denote by $\U := Z_1^T B^{1/2} Z_2$.
A standard approximation argument shows that if 
\bens
\sup_{w \in \Pi, u \in \Pi'} \; \; \abs{w^T \U u} 
\le C \ve \frac{\tr(B)}{\twonorm{B}^{1/2}} =: r_{k,n},
\eens
an event which we denote by $\B_2$, then for all $u\in E$ and $w \in F$,
\ben
\label{eq::complete}
\abs{w^T Z_1^T B^{1/2} Z_2 u} \le \frac{r_{k,n}}{(1-\ve)^2}.
\een
The lemma thus holds for $c_2 \ge C'/2 \ge 2$.
\end{proofof}

\begin{proofof}{Corollary~\ref{coro::tartan}}
Clearly~\eqref{eq::wyFnorm} implies that \eqref{eq::ALocalkronsum}
holds for $B=I$.
Clearly~\eqref{eq::BI}  holds following the analysis of 
Lemma~\ref{lemma::normA} by setting $B = I$, while replacing event $\B_1$ with
$\B_3$, which denotes an event such that 
\bens
\sup_{u, v \in \Pi} \onen \abs{v^T (Z^T Z - I) u} 
& \le &  C \ve.
\eens
The rest of the proof follows by replacing $E$ with $F$ everywhere.
The corollary thus holds.
\end{proofof}

\bibliography{subgaussian}

\begin{thebibliography}{57}

\bibitem[\protect\citeauthoryear{Agarwal, Negahban and
  Wainwright}{2012}]{ANW12}
\begin{barticle}[author]
\bauthor{\bsnm{Agarwal},~\bfnm{A.}\binits{A.}},
  \bauthor{\bsnm{Negahban},~\bfnm{S.}\binits{S.}} \AND
  \bauthor{\bsnm{Wainwright},~\bfnm{M.}\binits{M.}}
(\byear{2012}).
\btitle{Fast global convergence of gradient methods for high-dimensional
  statistical recovery}.
\bjournal{Annals of Statistics}
\bvolume{40}.
\end{barticle}
\endbibitem

\bibitem[\protect\citeauthoryear{Allen and Tibshirani}{2010}]{AT10}
\begin{barticle}[author]
\bauthor{\bsnm{Allen},~\bfnm{G.}\binits{G.}} \AND
  \bauthor{\bsnm{Tibshirani},~\bfnm{R.}\binits{R.}}
(\byear{2010}).
\btitle{Transposable regularized covariance models with an application to
  missing data imputation}.
\bjournal{Annals of Applied Statistics}
\bvolume{4}
\bpages{764-790}.
\end{barticle}
\endbibitem

\bibitem[\protect\citeauthoryear{Belloni, Rosenbaum and Tsybakov}{2014}]{BRT14}
\begin{bmisc}[author]
\bauthor{\bsnm{Belloni},~\bfnm{A.}\binits{A.}},
  \bauthor{\bsnm{Rosenbaum},~\bfnm{M.}\binits{M.}} \AND
  \bauthor{\bsnm{Tsybakov},~\bfnm{A.}\binits{A.}}
(\byear{2014}).
\btitle{Linear and Conic Programming Estimators in High-Dimensional
  Errors-in-variables Models}.
\bnote{arXiv:1408.0241}.
\end{bmisc}
\endbibitem

\bibitem[\protect\citeauthoryear{Bickel, Ritov and Tsybakov}{2009}]{BRT09}
\begin{barticle}[author]
\bauthor{\bsnm{Bickel},~\bfnm{P.}\binits{P.}},
  \bauthor{\bsnm{Ritov},~\bfnm{Y.}\binits{Y.}} \AND
  \bauthor{\bsnm{Tsybakov},~\bfnm{A.}\binits{A.}}
(\byear{2009}).
\btitle{Simultaneous Analysis of {L}asso and {D}antzig Selector}.
\bjournal{The Annals of Statistics}
\bvolume{37}
\bpages{1705--1732}.
\end{barticle}
\endbibitem

\bibitem[\protect\citeauthoryear{Bonilla, Chai and Williams}{2008}]{BCW08}
\begin{binproceedings}[author]
\bauthor{\bsnm{Bonilla},~\bfnm{E.}\binits{E.}},
  \bauthor{\bsnm{Chai},~\bfnm{K.~M.}\binits{K.~M.}} \AND
  \bauthor{\bsnm{Williams},~\bfnm{C.}\binits{C.}}
(\byear{2008}).
\btitle{Multi-task Gaussian process prediction}.
In \bbooktitle{In Advances in Neural Information Processing Systems 20 (NIPS
  2010)}.
\end{binproceedings}
\endbibitem

\bibitem[\protect\citeauthoryear{Cand\`{e}s and Tao}{2007}]{CT07}
\begin{barticle}[author]
\bauthor{\bsnm{Cand\`{e}s},~\bfnm{E.}\binits{E.}} \AND
  \bauthor{\bsnm{Tao},~\bfnm{T.}\binits{T.}}
(\byear{2007}).
\btitle{The {D}antzig selector: statistical estimation when p is much larger
  than n.}
\bjournal{Annals of Statistics}
\bvolume{35}
\bpages{2313-2351}.
\end{barticle}
\endbibitem

\bibitem[\protect\citeauthoryear{Carroll, Gail and Lubin}{1993}]{CGL93}
\begin{barticle}[author]
\bauthor{\bsnm{Carroll},~\bfnm{R.~J.}\binits{R.~J.}},
  \bauthor{\bsnm{Gail},~\bfnm{M.~H.}\binits{M.~H.}} \AND
  \bauthor{\bsnm{Lubin},~\bfnm{J.~H.}\binits{J.~H.}}
(\byear{1993}).
\btitle{Case-control studies with errors in predictors}.
\bjournal{Journal of American Statistical Association}
\bvolume{88}
\bpages{177 -- 191}.
\end{barticle}
\endbibitem

\bibitem[\protect\citeauthoryear{Carroll, Gallo and Gleser}{1985}]{CGG85}
\begin{barticle}[author]
\bauthor{\bsnm{Carroll},~\bfnm{R.~J.}\binits{R.~J.}},
  \bauthor{\bsnm{Gallo},~\bfnm{P.~P.}\binits{P.~P.}} \AND
  \bauthor{\bsnm{Gleser},~\bfnm{L.~J.}\binits{L.~J.}}
(\byear{1985}).
\btitle{Comparison of least squares and errors-in-variables regression with
  special reference to randomized analysis of covariance}.
\bjournal{Journal of American Statistical Association}
\bvolume{80}
\bpages{929 -- 932}.
\end{barticle}
\endbibitem

\bibitem[\protect\citeauthoryear{Carroll and Wand}{1991}]{CW91}
\begin{barticle}[author]
\bauthor{\bsnm{Carroll},~\bfnm{R.}\binits{R.}} \AND
  \bauthor{\bsnm{Wand},~\bfnm{M.}\binits{M.}}
(\byear{1991}).
\btitle{Semiparametric estimation in logistic measurement error models}.
\bjournal{J. R. Statist. Soc. B}
\bvolume{53}
\bpages{573-585}.
\end{barticle}
\endbibitem

\bibitem[\protect\citeauthoryear{Carroll et~al.}{2006}]{carr:rupp:2006}
\begin{bbook}[author]
\bauthor{\bsnm{Carroll},~\bfnm{R.~J.}\binits{R.~J.}},
  \bauthor{\bsnm{Ruppert},~\bfnm{D.}\binits{D.}},
  \bauthor{\bsnm{Stefanski},~\bfnm{L.~A.}\binits{L.~A.}} \AND
  \bauthor{\bsnm{Crainiceanu},~\bfnm{C.~M.}\binits{C.~M.}}
(\byear{2006}).
\btitle{Measurement Error in Nonlinear Models (Second Edition)}.
\bpublisher{Chapman \& Hall}.
\end{bbook}
\endbibitem

\bibitem[\protect\citeauthoryear{Chen and Caramanis}{2013}]{CC13}
\begin{binproceedings}[author]
\bauthor{\bsnm{Chen},~\bfnm{Y.}\binits{Y.}} \AND
  \bauthor{\bsnm{Caramanis},~\bfnm{C.}\binits{C.}}
(\byear{2013}).
\btitle{Noisy and Missing Data Regression: Distribution-Oblivious Support
  Recovery}.
In \bbooktitle{Proceedings of The 30th International Conference on Machine
  Learning {ICML}-13}.
\end{binproceedings}
\endbibitem

\bibitem[\protect\citeauthoryear{Chen, Donoho and
  Saunders}{1998}]{Chen:Dono:Saun:1998}
\begin{barticle}[author]
\bauthor{\bsnm{Chen},~\bfnm{S.}\binits{S.}},
  \bauthor{\bsnm{Donoho},~\bfnm{D.}\binits{D.}} \AND
  \bauthor{\bsnm{Saunders},~\bfnm{M.}\binits{M.}}
(\byear{1998}).
\btitle{Atomic decomposition by basis pursuit}.
\bjournal{SIAM Journal on Scientific and Statistical Computing}
\bvolume{20}
\bpages{33--61}.
\end{barticle}
\endbibitem

\bibitem[\protect\citeauthoryear{Cohen and Kohn}{2011}]{CK11}
\begin{barticle}[author]
\bauthor{\bsnm{Cohen},~\bfnm{M.~R.}\binits{M.~R.}} \AND
  \bauthor{\bsnm{Kohn},~\bfnm{A.~K.}\binits{A.~K.}}
(\byear{2011}).
\btitle{Measuring and interpreting neuronal correlations}.
\bjournal{Nature Neuroscience}
\bvolume{14}
\bpages{809-811}.
\end{barticle}
\endbibitem

\bibitem[\protect\citeauthoryear{Cook and Stefanski}{1994}]{Cook:Stef:1994}
\begin{barticle}[author]
\bauthor{\bsnm{Cook},~\bfnm{J.~R.}\binits{J.~R.}} \AND
  \bauthor{\bsnm{Stefanski},~\bfnm{L.~A.}\binits{L.~A.}}
(\byear{1994}).
\btitle{Simulation-extrapolation estimation in parametric measurement error
  models}.
\bjournal{Journal of the American Statistical Association}
\bvolume{89}
\bpages{1314--1328}.
\end{barticle}
\endbibitem

\bibitem[\protect\citeauthoryear{Dawid}{1981}]{Dawid81}
\begin{barticle}[author]
\bauthor{\bsnm{Dawid},~\bfnm{A.~P.}\binits{A.~P.}}
(\byear{1981}).
\btitle{Some Matrix-Variate Distribution Theory: Notational Considerations and
  a Bayesian Application}.
\bjournal{Biometrika}
\bvolume{68}
\bpages{265--274}.
\end{barticle}
\endbibitem

\bibitem[\protect\citeauthoryear{Dempster, Laird and Rubin}{1977}]{DLR77}
\begin{barticle}[author]
\bauthor{\bsnm{Dempster},~\bfnm{A.}\binits{A.}},
  \bauthor{\bsnm{Laird},~\bfnm{N.}\binits{N.}} \AND
  \bauthor{\bsnm{Rubin},~\bfnm{D.}\binits{D.}}
(\byear{1977}).
\btitle{Maximum likelihood from incomplete data via the EM algorithm}.
\bjournal{Journal of the Royal Statistical Society, Series B}
\bvolume{39}
\bpages{1-38}.
\end{barticle}
\endbibitem

\bibitem[\protect\citeauthoryear{Dutilleul}{1999}]{Dut99}
\begin{barticle}[author]
\bauthor{\bsnm{Dutilleul},~\bfnm{P.}\binits{P.}}
(\byear{1999}).
\btitle{The MLE Algorithm for the matrix normal distribution}.
\bjournal{Journal of Statistical Computation and Simulation}
\bvolume{64}
\bpages{105--123}.
\end{barticle}
\endbibitem

\bibitem[\protect\citeauthoryear{Efron}{2009}]{Efr09}
\begin{barticle}[author]
\bauthor{\bsnm{Efron},~\bfnm{B.}\binits{B.}}
(\byear{2009}).
\btitle{Are a set of microarrays independent of each other?}
\bjournal{Ann. App. Statist.}
\bvolume{3}
\bpages{922--942}.
\end{barticle}
\endbibitem

\bibitem[\protect\citeauthoryear{Fan and Li}{2001}]{FL01}
\begin{barticle}[author]
\bauthor{\bsnm{Fan},~\bfnm{J.}\binits{J.}} \AND
  \bauthor{\bsnm{Li},~\bfnm{R.}\binits{R.}}
(\byear{2001}).
\btitle{Variable selection via nonconcave penalized likelihood and its oracle
  properties}.
\bjournal{Journal of American Statistical Association}
\bvolume{96}
\bpages{1348--1360}.
\end{barticle}
\endbibitem

\bibitem[\protect\citeauthoryear{Fuller}{1987}]{Full:1987}
\begin{bbook}[author]
\bauthor{\bsnm{Fuller},~\bfnm{W.~A.}\binits{W.~A.}}
(\byear{1987}).
\btitle{Measurement error models}.
\bpublisher{John Wiley and Sons}.
\end{bbook}
\endbibitem

\bibitem[\protect\citeauthoryear{Gautier and Tsybakov}{2011}]{GT11}
\begin{bmisc}[author]
\bauthor{\bsnm{Gautier},~\bfnm{E.}\binits{E.}} \AND
  \bauthor{\bsnm{Tsybakov},~\bfnm{A.}\binits{A.}}
(\byear{2011}).
\btitle{High-dimensional instrumental variables regression and confidence
  sets}.
\bnote{arXiv:1105.2454}.
\end{bmisc}
\endbibitem

\bibitem[\protect\citeauthoryear{Gupta and Varga}{1992}]{GV92}
\begin{barticle}[author]
\bauthor{\bsnm{Gupta},~\bfnm{A.}\binits{A.}} \AND
  \bauthor{\bsnm{Varga},~\bfnm{T.}\binits{T.}}
(\byear{1992}).
\btitle{Characterization of Matrix Variate Normal Distributions}.
\bjournal{Journal of Multivariate Analysis}
\bvolume{41}
\bpages{80-88}.
\end{barticle}
\endbibitem

\bibitem[\protect\citeauthoryear{Hall and Ma}{2007}]{HM07}
\begin{barticle}[author]
\bauthor{\bsnm{Hall},~\bfnm{P.}\binits{P.}} \AND
  \bauthor{\bsnm{Ma},~\bfnm{Y.}\binits{Y.}}
(\byear{2007}).
\btitle{Semiparametric estimators of functional measurement error models with
  unknown error}.
\bjournal{Journal of the Royal Statistical Society B}
\bvolume{69}
\bpages{429-446}.
\end{barticle}
\endbibitem

\bibitem[\protect\citeauthoryear{Hwang}{1986}]{HWang86}
\begin{barticle}[author]
\bauthor{\bsnm{Hwang},~\bfnm{J.~T.}\binits{J.~T.}}
(\byear{1986}).
\btitle{Multiplicative Errors-in-Variables Models with Applications to Recent
  Data Released by the U.S. Department of Energy}.
\bjournal{Journal of American Statistical Association}
\bvolume{81}
\bpages{680--688}.
\end{barticle}
\endbibitem

\bibitem[\protect\citeauthoryear{Iturria, Carroll and Firth}{1999}]{ICF99}
\begin{barticle}[author]
\bauthor{\bsnm{Iturria},~\bfnm{S.~J.}\binits{S.~J.}},
  \bauthor{\bsnm{Carroll},~\bfnm{R.~J.}\binits{R.~J.}} \AND
  \bauthor{\bsnm{Firth},~\bfnm{D.}\binits{D.}}
(\byear{1999}).
\btitle{Polynomial regression and estimating functions in the presence of
  multiplicative measurement error}.
\bjournal{Journal of the Royal Statistical Society, Series B, Methodological}
\bvolume{61}
\bpages{547-561}.
\end{barticle}
\endbibitem

\bibitem[\protect\citeauthoryear{Kalaitzis et~al.}{2013}]{KLLZ13}
\begin{binproceedings}[author]
\bauthor{\bsnm{Kalaitzis},~\bfnm{A.}\binits{A.}},
  \bauthor{\bsnm{Lafferty},~\bfnm{J.}\binits{J.}},
  \bauthor{\bsnm{Lawrence},~\bfnm{N.}\binits{N.}} \AND
  \bauthor{\bsnm{Zhou},~\bfnm{S.}\binits{S.}}
(\byear{2013}).
\btitle{The Bigraphical Lasso}.
In \bbooktitle{Proceedings of The 30th International Conference on Machine
  Learning {ICML}-13}
\bpages{1229-1237}.
\end{binproceedings}
\endbibitem

\bibitem[\protect\citeauthoryear{Kass, Ventura and Brown}{2005}]{KassVB05}
\begin{barticle}[author]
\bauthor{\bsnm{Kass},~\bfnm{Robert}\binits{R.}},
  \bauthor{\bsnm{Ventura},~\bfnm{Valerie}\binits{V.}} \AND
  \bauthor{\bsnm{Brown},~\bfnm{Emery}\binits{E.}}
(\byear{2005}).
\btitle{Statistical Issues in the Analysis of Neuronal Data}.
\bjournal{J Neurophysiol}
\bvolume{94}
\bpages{8--25}.
\end{barticle}
\endbibitem

\bibitem[\protect\citeauthoryear{Liang, H\"{a}rdle and Carroll}{1999}]{LHC99}
\begin{barticle}[author]
\bauthor{\bsnm{Liang},~\bfnm{H.}\binits{H.}},
  \bauthor{\bsnm{H\"{a}rdle},~\bfnm{W.}\binits{W.}} \AND
  \bauthor{\bsnm{Carroll},~\bfnm{R.~J.}\binits{R.~J.}}
(\byear{1999}).
\btitle{Estimation in a semiparametric partially linear errors-in-variables
  model}.
\bjournal{Ann. Statist.}
\bvolume{27}
\bpages{1519-1535}.
\end{barticle}
\endbibitem

\bibitem[\protect\citeauthoryear{Liang and Li}{2009}]{LL09}
\begin{barticle}[author]
\bauthor{\bsnm{Liang},~\bfnm{H.}\binits{H.}} \AND
  \bauthor{\bsnm{Li},~\bfnm{R.}\binits{R.}}
(\byear{2009}).
\btitle{Variable selection for partially linear models with measurement
  Errors}.
\bjournal{Journal of the American Statistical Association}
\bvolume{104}
\bpages{234-248}.
\end{barticle}
\endbibitem

\bibitem[\protect\citeauthoryear{Loh and Wainwright}{2012}]{LW12}
\begin{barticle}[author]
\bauthor{\bsnm{Loh},~\bfnm{P.}\binits{P.}} \AND
  \bauthor{\bsnm{Wainwright},~\bfnm{M.}\binits{M.}}
(\byear{2012}).
\btitle{High-dimensional regression with noisy and missing data: Provable
  guarantees with nonconvexity}.
\bjournal{The Annals of Statistics}
\bvolume{40}
\bpages{1637--1664}.
\end{barticle}
\endbibitem

\bibitem[\protect\citeauthoryear{Loh and Wainwright}{2015}]{LW15}
\begin{barticle}[author]
\bauthor{\bsnm{Loh},~\bfnm{P.}\binits{P.}} \AND
  \bauthor{\bsnm{Wainwright},~\bfnm{M.}\binits{M.}}
(\byear{2015}).
\btitle{Regularized M-estimators with nonconvexity: Statistical and algorithmic
  theory for local optima}.
\bjournal{Journal of Machine Learning Research}
\bvolume{16}
\bpages{559--616}.
\end{barticle}
\endbibitem

\bibitem[\protect\citeauthoryear{Ma and Li}{2010}]{ML10}
\begin{barticle}[author]
\bauthor{\bsnm{Ma},~\bfnm{Y.}\binits{Y.}} \AND
  \bauthor{\bsnm{Li},~\bfnm{R.}\binits{R.}}
(\byear{2010}).
\btitle{Variable selection in measurement error models}.
\bjournal{Bernoulli}
\bvolume{16}
\bpages{274-300}.
\end{barticle}
\endbibitem

\bibitem[\protect\citeauthoryear{Mendelson, Pajor and
  Tomczak-Jaegermann}{2008}]{MPT08}
\begin{barticle}[author]
\bauthor{\bsnm{Mendelson},~\bfnm{S.}\binits{S.}},
  \bauthor{\bsnm{Pajor},~\bfnm{A.}\binits{A.}} \AND
  \bauthor{\bsnm{Tomczak-Jaegermann},~\bfnm{N.}\binits{N.}}
(\byear{2008}).
\btitle{Uniform uncertainty principle for Bernoulli and subgaussian ensembles}.
\bjournal{Constructive Approximation}
\bvolume{28}
\bpages{277--289}.
\end{barticle}
\endbibitem

\bibitem[\protect\citeauthoryear{Nesterov}{2007}]{Nesterov07}
\begin{barticle}[author]
\bauthor{\bsnm{Nesterov},~\bfnm{Y.}\binits{Y.}}
(\byear{2007}).
\btitle{Gradient methods for minimizing composite objective function}.
\bjournal{CORE DISCUSSION PAPER 2007/76, Center for Operations Research and
  Econometrics (CORE), Catholic University of Louvain (UCL)}
\bpages{1-30}.
\end{barticle}
\endbibitem

\bibitem[\protect\citeauthoryear{Rosenbaum and Tsybakov}{2010}]{RT10}
\begin{barticle}[author]
\bauthor{\bsnm{Rosenbaum},~\bfnm{M.}\binits{M.}} \AND
  \bauthor{\bsnm{Tsybakov},~\bfnm{A.}\binits{A.}}
(\byear{2010}).
\btitle{Sparse recovery under matrix uncertainty}.
\bjournal{The Annals of Statistics}
\bvolume{38}
\bpages{2620-2651}.
\end{barticle}
\endbibitem

\bibitem[\protect\citeauthoryear{Rosenbaum and Tsybakov}{2013}]{RT13}
\begin{barticle}[author]
\bauthor{\bsnm{Rosenbaum},~\bfnm{M.}\binits{M.}} \AND
  \bauthor{\bsnm{Tsybakov},~\bfnm{A.}\binits{A.}}
(\byear{2013}).
\btitle{Improved matrix uncertainty selector}.
\bjournal{IMS Collections}
\bvolume{9}
\bpages{276-290}.
\end{barticle}
\endbibitem

\bibitem[\protect\citeauthoryear{Rudelson and Vershynin}{2013}]{RV13}
\begin{barticle}[author]
\bauthor{\bsnm{Rudelson},~\bfnm{M.}\binits{M.}} \AND
  \bauthor{\bsnm{Vershynin},~\bfnm{R.}\binits{R.}}
(\byear{2013}).
\btitle{{H}anson-{W}right inequality and sub-gaussian concentration}.
\bjournal{Electronic Communications in Probability}
\bvolume{18}
\bpages{1--9}.
\end{barticle}
\endbibitem

\bibitem[\protect\citeauthoryear{Rudelson and Zhou}{2013}]{RZ13}
\begin{barticle}[author]
\bauthor{\bsnm{Rudelson},~\bfnm{M.}\binits{M.}} \AND
  \bauthor{\bsnm{Zhou},~\bfnm{S.}\binits{S.}}
(\byear{2013}).
\btitle{Reconstruction from anisotropic random measurements}.
\bjournal{IEEE Transactions on Information Theory}
\bvolume{59}
\bpages{3434-3447}.
\end{barticle}
\endbibitem

\bibitem[\protect\citeauthoryear{Rudelson and Zhou}{2015}]{RZ15}
\begin{bmisc}[author]
\bauthor{\bsnm{Rudelson},~\bfnm{M.}\binits{M.}} \AND
  \bauthor{\bsnm{Zhou},~\bfnm{S.}\binits{S.}}
(\byear{2015}).
\btitle{High dimensional errors-in-variables models with dependent
  measurements}.
\bnote{arXiv:1502.02355}.
\end{bmisc}
\endbibitem

\bibitem[\protect\citeauthoryear{Ruff and Cohen}{2014}]{RC14a}
\begin{barticle}[author]
\bauthor{\bsnm{Ruff},~\bfnm{D.~A.}\binits{D.~A.}} \AND
  \bauthor{\bsnm{Cohen},~\bfnm{M.~R.}\binits{M.~R.}}
(\byear{2014}).
\btitle{Attention can either increase or decrease spike count correlations in
  visual cortex}.
\bjournal{Nature Neuroscience}
\bvolume{17}
\bpages{1591-7}.
\end{barticle}
\endbibitem

\bibitem[\protect\citeauthoryear{S{\o}resen, Frigenssi and
  Thoresen}{2014a}]{SFT14}
\begin{barticle}[author]
\bauthor{\bsnm{S{\o}resen},~\bfnm{{\O}.}\binits{{\O}.}},
  \bauthor{\bsnm{Frigenssi},~\bfnm{A.}\binits{A.}} \AND
  \bauthor{\bsnm{Thoresen},~\bfnm{M.}\binits{M.}}
(\byear{2014}a).
\btitle{Measurement Error in {L}asso: Impact and Likelihood Bias Correction}.
\bjournal{Statistical Sinica Preprint}.
\end{barticle}
\endbibitem

\bibitem[\protect\citeauthoryear{S{\o}resen, Frigenssi and
  Thoresen}{2014b}]{SFT14b}
\begin{bmisc}[author]
\bauthor{\bsnm{S{\o}resen},~\bfnm{{\O}.}\binits{{\O}.}},
  \bauthor{\bsnm{Frigenssi},~\bfnm{A.}\binits{A.}} \AND
  \bauthor{\bsnm{Thoresen},~\bfnm{M.}\binits{M.}}
(\byear{2014}b).
\btitle{Covariate Selection in High-Dimensional Generalized Linear Models with
  Measurement Error}.
\bnote{arXiv:1407.1070}.
\end{bmisc}
\endbibitem

\bibitem[\protect\citeauthoryear{St\"{a}dler, Stekhoven and
  B\"{u}hlmann}{2014}]{SSB14}
\begin{barticle}[author]
\bauthor{\bsnm{St\"{a}dler},~\bfnm{N.}\binits{N.}},
  \bauthor{\bsnm{Stekhoven},~\bfnm{D.~J.}\binits{D.~J.}} \AND
  \bauthor{\bsnm{B\"{u}hlmann},~\bfnm{P.}\binits{P.}}
(\byear{2014}).
\btitle{Pattern Alternating Maximization Algorithm for Missing Data in
  High-Dimensional Problems}.
\bjournal{Journal of Machine Learning Research}
\bvolume{15}
\bpages{1903-1928}.
\end{barticle}
\endbibitem

\bibitem[\protect\citeauthoryear{Stefanski}{1985}]{Stef:1985}
\begin{barticle}[author]
\bauthor{\bsnm{Stefanski},~\bfnm{L.~A.}\binits{L.~A.}}
(\byear{1985}).
\btitle{The effects of measurement error on parameter estimation}.
\bjournal{Biometrika}
\bvolume{72}
\bpages{583--592}.
\end{barticle}
\endbibitem

\bibitem[\protect\citeauthoryear{Stefanski}{1990}]{Stef:1990}
\begin{barticle}[author]
\bauthor{\bsnm{Stefanski},~\bfnm{L.~A.}\binits{L.~A.}}
(\byear{1990}).
\btitle{Rates of convergence of some estimators in a class of deconvolution
  problems}.
\bjournal{Statistics and Probability Letters}
\bvolume{9}
\bpages{229--235}.
\end{barticle}
\endbibitem

\bibitem[\protect\citeauthoryear{Stefanski and Cook}{1995}]{Stef:Cook:1995}
\begin{barticle}[author]
\bauthor{\bsnm{Stefanski},~\bfnm{L.~A.}\binits{L.~A.}} \AND
  \bauthor{\bsnm{Cook},~\bfnm{J.~R.}\binits{J.~R.}}
(\byear{1995}).
\btitle{Simulation-extrapolation: {T}he measurement error jackknife}.
\bjournal{Journal of the American Statistical Association}
\bvolume{90}
\bpages{1247--1256}.
\end{barticle}
\endbibitem

\bibitem[\protect\citeauthoryear{Strimmer}{2003}]{Str03}
\begin{barticle}[author]
\bauthor{\bsnm{Strimmer},~\bfnm{K.}\binits{K.}}
(\byear{2003}).
\btitle{Modeling gene expression measurement error: a quasi-likelihood
  approach}.
\bjournal{BMC Bioinformatics}
\bvolume{4}.
\end{barticle}
\endbibitem

\bibitem[\protect\citeauthoryear{Tibshirani}{1996}]{Tib96}
\begin{barticle}[author]
\bauthor{\bsnm{Tibshirani},~\bfnm{R.}\binits{R.}}
(\byear{1996}).
\btitle{Regression shrinkage and selection via the {L}asso}.
\bjournal{J. Roy. Statist. Soc. Ser. B}
\bvolume{58}
\bpages{267-288}.
\end{barticle}
\endbibitem

\bibitem[\protect\citeauthoryear{Tropp}{2004}]{Tropp:04}
\begin{barticle}[author]
\bauthor{\bsnm{Tropp},~\bfnm{J.~A.}\binits{J.~A.}}
(\byear{2004}).
\btitle{Greed is good: {A}lgorithmic results for sparse approximation}.
\bjournal{IEEE Trans. Inform. Theory}
\bvolume{50}
\bpages{2231--2241}.
\end{barticle}
\endbibitem

\bibitem[\protect\citeauthoryear{Tropp and Gilbert}{2007}]{TG07}
\begin{barticle}[author]
\bauthor{\bsnm{Tropp},~\bfnm{J.}\binits{J.}} \AND
  \bauthor{\bsnm{Gilbert},~\bfnm{A.}\binits{A.}}
(\byear{2007}).
\btitle{Signal recovery from random measurements via orthogonal matching pur-
  suit}.
\bjournal{IEEE Trans. Inform. Theory}
\bvolume{53}
\bpages{4655–-4666}.
\end{barticle}
\endbibitem

\bibitem[\protect\citeauthoryear{Vial}{1982}]{Vial82}
\begin{barticle}[author]
\bauthor{\bsnm{Vial},~\bfnm{J.~P.}\binits{J.~P.}}
(\byear{1982}).
\btitle{Strong convexity of sets and functions}.
\bjournal{Journal of Mathematical Economics}
\bvolume{9}
\bpages{187--205}.
\end{barticle}
\endbibitem

\bibitem[\protect\citeauthoryear{Werner, Jansson and Stoica}{2008}]{WJS08}
\begin{barticle}[author]
\bauthor{\bsnm{Werner},~\bfnm{K.}\binits{K.}},
  \bauthor{\bsnm{Jansson},~\bfnm{M.}\binits{M.}} \AND
  \bauthor{\bsnm{Stoica},~\bfnm{P.}\binits{P.}}
(\byear{2008}).
\btitle{On Estimation of Covariance Matrices With Kronecker Product Structure}.
\bjournal{IEEE Transactions on Signal Processing}
\bvolume{56}
\bpages{478 -- 491}.
\end{barticle}
\endbibitem

\bibitem[\protect\citeauthoryear{Xu and You}{2007}]{XY07}
\begin{barticle}[author]
\bauthor{\bsnm{Xu},~\bfnm{Q.}\binits{Q.}} \AND
  \bauthor{\bsnm{You},~\bfnm{J.}\binits{J.}}
(\byear{2007}).
\btitle{Covariate Selection for Linear Errors-in-Variables Regression Models}.
\bjournal{Communications in Statistics -- Theory and Methods}
\bvolume{36}.
\end{barticle}
\endbibitem

\bibitem[\protect\citeauthoryear{Yu et~al.}{2009}]{Yu09}
\begin{barticle}[author]
\bauthor{\bsnm{Yu},~\bfnm{K.}\binits{K.}},
  \bauthor{\bsnm{Lafferty},~\bfnm{J.}\binits{J.}},
  \bauthor{\bsnm{Zhu},~\bfnm{S.}\binits{S.}} \AND
  \bauthor{\bsnm{Gong},~\bfnm{Y.}\binits{Y.}}
(\byear{2009}).
\btitle{Large-scale collaborative prediction using a nonparametric random
  effects model}.
\bjournal{Proceedings of the 26th International Conference on Machine
  Learning}.
\end{barticle}
\endbibitem

\bibitem[\protect\citeauthoryear{Zhang}{2010}]{Zhang10}
\begin{barticle}[author]
\bauthor{\bsnm{Zhang},~\bfnm{C.~H.}\binits{C.~H.}}
(\byear{2010}).
\btitle{Nearly unbiased variable selection under minimax concave penalty}.
\bjournal{Ann. Statist.}
\bvolume{38}
\bpages{894-942}.
\end{barticle}
\endbibitem

\bibitem[\protect\citeauthoryear{Zhou}{2014}]{Zhou14a}
\begin{barticle}[author]
\bauthor{\bsnm{Zhou},~\bfnm{S.}\binits{S.}}
(\byear{2014}).
\btitle{Gemini: Graph estimation with matrix variate normal instances}.
\bjournal{Annals of Statistics}
\bvolume{42}
\bpages{532--562}.
\end{barticle}
\endbibitem

\bibitem[\protect\citeauthoryear{Zhou, Lafferty and Wasserman}{2010}]{ZLW08}
\begin{barticle}[author]
\bauthor{\bsnm{Zhou},~\bfnm{S.}\binits{S.}},
  \bauthor{\bsnm{Lafferty},~\bfnm{J.}\binits{J.}} \AND
  \bauthor{\bsnm{Wasserman},~\bfnm{L.}\binits{L.}}
(\byear{2010}).
\btitle{Time Varying Undirected Graphs}.
\bjournal{Machine Learning}
\bvolume{80}
\bpages{295-319}.
\end{barticle}
\endbibitem

\end{thebibliography}

\end{document}